\documentclass[11pt]{article}
\usepackage[a4paper, margin=1in]{geometry}
\usepackage[utf8]{inputenc}

\usepackage{amsmath} 
\usepackage{amssymb}
\usepackage{amsthm}
\usepackage{authblk}
\usepackage{bbm}
\usepackage{comment}
\usepackage{enumitem}
\usepackage{mathtools}
\usepackage{subcaption}
\usepackage[backend=biber]{biblatex}

\usepackage{tikz}
\usetikzlibrary{spy}
\usepackage{pgfplots}
\pgfplotsset{compat=1.18}
\usepgfplotslibrary{fillbetween}

\usepackage{algorithm}
\usepackage{algpseudocode}
\makeatletter
\let\OldStatex\Statex
\renewcommand{\Statex}[1][3]{%
	\setlength\@tempdima{\algorithmicindent}%
	\OldStatex\hskip\dimexpr#1\@tempdima\relax}
\makeatother

\usepackage[colorlinks]{hyperref}
\usepackage[capitalise,nameinlink,noabbrev]{cleveref}

\usepackage[acronym]{glossaries}

\newacronym{AA}{AA}{Applied Analysis}
\newacronym{AMI}{AMI}{Adjusted Mutual Information}
\newacronym{BMC}{BMC}{Block Markov Chain}
\newacronym[\glslongpluralkey={Block Markov Decision Processes}]{BMDP}{BMDP}{Block Markov Decision Process}
\newacronym{CASA}{CASA}{Centre for Analysis, Scientific Computing and Applications}
\newacronym{CAIC}{CAIC}{Consistent Akaike Information Criterion}
\newacronym{CPSRL}{CPSRL}{Clustered Posterior Sampling for Reinforcement Learning}
\newacronym{CQT}{CQT}{Coherence \& Quantum Technology}
\newacronym{CWI}{CWI}{National Research Institute for Mathematics and Computer Science}
\newacronym{CUDA}{CUDA}{Compute Unified Device Architecture}
\newacronym{DNA}{DNA}{Deoxyribonucleic acid}
\newacronym{DIAM}{DIAM}{Applied Mathematics}
\newacronym{ERRG}{ERRG}{Erd\"{o}s--R\'{e}nyi Random Graph}
\newacronym{EEMCS}{EEMCS}{Electrical Engineering, Mathematics \&{} Computer Science}
\newacronym{EURANDOM}{EURANDOM}{European Institute for Statistics, Probability, Stochastic Operations Research and their Applications}
\newacronym{GPU}{GPU}{Graphics Processing Unit}
\newacronym{HMM}{HMM}{Hidden Markov Model}
\newacronym{INSY}{INSY}{Intelligent Systems}
\newacronym{LMAB}{LMAB}{Latent Multi-Armed Bandit}
\newacronym{LSBM}{LSBM}{Labeled Stochastic Block Model}
\newacronym{MAB}{MAB}{Multi-Armed Bandit}
\newacronym{MC}{MC}{Markov Chain}
\newacronym{MCS}{M\&{}CS}{Mathematics and Computer Science}
\newacronym[\glslongpluralkey={Markov Decision Processes}]{MDP}{MDP}{Markov Decision Process}
\newacronym{MI}{MI}{Mutual Information}
\newacronym{NAS}{NAS}{Network Architectures and Services}
\newacronym{OFU}{OFU}{Optimism in the Face of Uncertainty}
\newacronym{PAC}{PAC}{Probably Approximately Correct}
\newacronym{PSRL}{PSRL}{Posterior Sampling for Reinforcement Learning}
\newacronym{PROB}{PROB}{Probability}
\newacronym[\glslongpluralkey={Partially Observable Markov Decision Processes}]{POMDP}{POMDP}{Partially Observable Markov Decision Process}
\newacronym{PyPI}{PyPI}{Python Package Index}
\newacronym[\glslongpluralkey={Rich Observation Markov Decision Processes}]{ROMDP}{ROMDP}{Rich Observation Markov Decision Process}
\newacronym{RL}{RL}{Reinforcement Learning}
\newacronym{SBM}{SBM}{Stochastic Block Model}
\newacronym{SPOR}{SPOR}{Statistics, Probability, and Operations Research}
\newacronym{SOR}{SOR}{Stochastic Operations Research}
\newacronym{ST}{ST}{Software Technology}
\newacronym{STO}{STO}{Stochastics}
\newacronym{STAT}{STAT}{Statistics}
\newacronym{SVD}{SVD}{Singular Value Decomposition}
\newacronym{TS}{TS}{Thompson Sampling}
\newacronym{TUDelft}{TU Delft}{Delft University of Technology}
\newacronym{TUe}{TU/e}{Eindhoven University of Technology}
\newacronym{UCB}{UCB}{Upper Confidence Bound}
\newacronym{UCBVI}{UCBVI}{Upper Confidence Bound Value Iteration}
\newacronym{UTQ}{UTQ}{University Teaching Qualification}
\newacronym{QCE}{Q\&{}CE}{Quantum \& Computer Engineering}
\newacronym{QED}{QED}{Quality--and--Efficiency--Driven}
\newacronym{WiCoS}{WiCoS}{Wireless Communication and Sensing}
\newacronym[longplural={Spatialized Random Graphs}]{SRG}{SRG}{Spatialized Random Graph}
\newacronym[longplural={Random Sequential Adsorption}]{RSA}{RSA}{Random Sequential Adsorption}

\newcommand{\defeq}{=\vcentcolon}
\newcommand{\eqdef}{\vcentcolon=}
\newcommand{\mbb}[1]{\mathbbm{#1}}
\newcommand{\mc}[1]{\mathcal{#1}}
\newcommand{\mf}[1]{\mathfrak{#1}}
\newcommand{\mbf}[1]{\mathbf{#1}}
\newcommand{\mcA}{\mathcal{A}}

\newcommand{\mcC}{\mathcal{C}}
\newcommand{\mcD}{\mathcal{D}}
\newcommand{\mcE}{\mathcal{E}}
\newcommand{\mcF}{\mathcal{F}}
\newcommand{\mcG}{\mathcal{G}}
\newcommand{\mcH}{\mathcal{H}}
\newcommand{\mcI}{\mathcal{I}}
\newcommand{\mcJ}{\mathcal{J}}
\newcommand{\mcL}{\mathcal{L}}
\newcommand{\mcM}{\mathcal{M}}

\newcommand{\mcP}{\mathcal{P}}

\newcommand{\mcR}{\mathcal{R}}
\newcommand{\mcS}{\mathcal{S}}
\newcommand{\mcT}{\mathcal{T}}

\newcommand{\mcV}{\mathcal{V}}

\newcommand{\mcX}{\mathcal{X}}
\newcommand{\mcY}{\mathcal{Y}}

\newcommand{\lb}{\left(}
\newcommand{\rb}{\right)}
\newcommand{\mbbP}{\mathbb{P}}
\newcommand{\mbbE}{\mathbb{E}}
\newcommand{\mbbN}{\mathbb{N}}
\newcommand{\mbbR}{\mathbb{R}}

\newcommand{\norm}[1]{\left\lVert#1\,\right\rVert}

\newcommand{\kl}{\textnormal{kl}}
\newcommand{\KL}{\textnormal{KL}}
\newcommand{\clust}{\textnormal{clust}}
\newcommand{\alg}{\textnormal{alg}}
\newcommand{\reg}{\textnormal{Reg}}
\newcommand{\gap}{\textnormal{gap}}

\newcommand{\bkt}[1]{\left[#1\right]}

\DeclareMathOperator*{\argmax}{arg\,max}
\DeclareMathOperator*{\argmin}{arg\,min}

\DeclarePairedDelimiter\ceil{\lceil}{\rceil}
\DeclarePairedDelimiter\floor{\lfloor}{\rfloor}

\newtheorem{theorem}{Theorem}[section]
\newtheorem{corollary}[theorem]{Corollary}
\newtheorem{lemma}[theorem]{Lemma}
\newtheorem{proposition}[theorem]{Proposition}

\newtheorem{definition}[theorem]{Definition}

\theoremstyle{definition}

\theoremstyle{definition}

\makeatletter
\AddToHook{env/lemma/begin}{\crefalias{theorem}{lemma}}
\AddToHook{env/corollary/begin}{\crefalias{theorem}{corollary}}
\AddToHook{env/proposition/begin}{\crefalias{theorem}{proposition}}
\AddToHook{env/conjecture/begin}{\crefalias{theorem}{conjecture}}
\AddToHook{env/definition/begin}{\crefalias{theorem}{definition}}
\AddToHook{env/example/begin}{\crefalias{theorem}{example}}
\makeatother

\Crefname{algocf}{Algorithm}{Algorithms}
\crefname{theorem}{Theorem}{Theorems}
\crefname{lemma}{Lemma}{Lemmas}
\crefname{corollary}{Corollary}{Corollaries}
\crefname{proposition}{Proposition}{Propositions}
\crefname{conjecture}{Conjecture}{Conjectures}
\crefname{definition}{Definition}{Definitions}
\Crefname{example}{Example}{Examples}
\crefname{assumption}{Assumption}{Assumptions}
\crefname{remark}{Remark}{Remarks}

\newlist{assumptionenum}{enumerate}{1} 
\setlist[assumptionenum]{label=(\roman*), ref=\theassumption(\roman*)}
\crefname{assumptionenumi}{assumption}{assumptions}
\newlist{definitionenum}{enumerate}{1} 
\setlist[definitionenum]{label=(\roman*), ref=\thetheorem(\roman*)}
\crefname{definitionenumi}{definition}{definitions}
\newlist{lemmaenum}{enumerate}{1} 
\setlist[lemmaenum]{label=(\roman*), ref=\thetheorem(\roman*)}
\crefname{lemmaenumi}{lemma}{lemmas}
\newlist{propenum}{enumerate}{1} 
\setlist[propenum]{label=(\roman*), ref=\thetheorem(\roman*)}
\crefname{propenumi}{proposition}{propositions}
\newlist{corollaryenum}{enumerate}{1} 
\setlist[corollaryenum]{label=(\roman*), ref=\thetheorem(\roman*)}
\crefname{corollaryenumi}{corollary}{corollaries}

\newlist{todolist}{itemize}{1}
\setlist[todolist]{label=$\square$}

\AtEveryBibitem{%
    \clearfield{issn}%
    \clearfield{doi}%
    \clearfield{url}%
    \clearname{editor}%
    \clearlist{location}%
    \ifentrytype{book}{}{
        \clearlist{publisher}%
        \clearname{editor}%
        \clearlist{address}%
    }%
    \clearfield{urldate}%
    \ifentrytype{online}{}{
        \clearfield{eprint}%
        \clearfield{primaryClass}%
        \clearfield{eprinttype}%
    }%
}

\makeatletter
  \newcommand*{\Description}[1]{}
\makeatother

\newcommand{\papertitle}{Asymptotically optimal reinforcement learning in Block Markov Decision Processes}

\title{\papertitle}
\author[1]{Thomas van Vuren}
\author[1]{Fiona Sloothaak}
\author[2]{Maarten G. Wolf}
\author[1]{Jaron Sanders}
\affil[1]{Eindhoven University of Technology, The Netherlands}
\affil[2]{Koninklijke KPN N.V., The Netherlands}

\hypersetup{
    pdftitle={\papertitle},
    pdfauthor={Thomas van Vuren, Fiona Sloothaak, Maarten G. Wolf, Jaron Sanders},
    pdfsubject={Reinforcement Learning, Block Markov Decision Processes},
    pdfkeywords={reinforcement learning, Markov decision processes, asymptotic optimality}
}

\begin{document}

\maketitle

\begin{abstract}
	The curse of dimensionality renders \gls{RL} impractical in many real-world settings with exponentially large state and action spaces.
Yet, many environments exhibit exploitable structure that can accelerate learning.
To formalize this idea, we study \gls{RL} in \glspl{BMDP}.
\glspl{BMDP} model problems with large observation spaces, but where transition dynamics are fully determined by latent states.
Recent advances in clustering methods have enabled the efficient recovery of this latent structure.
However, a regret analysis that exploits these techniques to determine their impact on learning performance remained open.
We are now addressing this gap by providing a regret analysis that explicitly leverages clustering, demonstrating that accurate latent state estimation can indeed effectively speed up learning.

Concretely, this paper analyzes a two-phase \gls{RL} algorithm for \glspl{BMDP} that first learns the latent structure through random exploration and then switches to an optimism-guided strategy adapted to the uncovered structure.
This algorithm achieves a regret that is $O(\sqrt{T}+n)$ on a large class of \glspl{BMDP} susceptible to clustering.
Here, $T$ denotes the number of time steps, $n$ is the cardinality of the observation space, and the Landau notation $O(\cdot)$ holds up to constants and polylogarithmic factors.
This improves the best prior bound, $O(\sqrt{T}+n^2)$, especially when $n$ is large.
Moreover, we prove that no algorithm can achieve lower regret uniformly on this same class of \glspl{BMDP}.
This establishes that, on this class, the algorithm achieves asymptotic optimality.

\end{abstract}

\glsresetall

\section{Introduction}

\subsection{Motivation}

The field of machine learning, and artificial intelligence more generally, has taken society by storm.
One of its key techniques, \gls{RL} \cite{sutton2020reinforcement}, is applied everywhere.
It is, for example, applied to problems in finance, games, healthcare, industrial control, and robotics.

However, despite its promise, \gls{RL} still faces numerous practical challenges.
One key difficulty is the curse of dimensionality \cite{sutton2020reinforcement}, which can make learning in problems with large state spaces intractable.
This challenge is particularly nasty because the state space size is often exponential in the number of system components.

These issues can be unavoidable if no additional structure is present in a system.
There namely exist \glspl{MDP} with $n$ states on which the expected regret after $T$ time steps is at least $\Omega(\sqrt{nT})$ \cite{osband2016lower}.
When $n$ is exponentially large, such regret can render \gls{RL} impractical.

Fortunately, many real-world environments exhibit exploitable structure.
For instance, if different states share similar transition probabilities, then we could first learn a clustering of the state space.
This reduces the effective size of the state space, mitigating the impact of the curse of dimensionality.

\subsection{State-space reduction in Block MDPs}
\label{sec:this_paper}

In this paper, we conduct a regret analysis for an episodic learning strategy that exploits a state-space reduction technique tailored to \glspl{BMDP}.

\glspl{BMDP} \cite{NIPS2016_2387337b} are \glspl{MDP} in which observations are drawn from a large set of contexts.
The transitions dynamics, however, are determined fully by just a few unobserved latent states.
Specifically, each context $x$ is associated to a latent state $s = f(x)$, and the system moves from latent state $s$ to latent state $s'$ according to a transition kernel $p(s' \mid s,a)$.
The agent then observes only a context $y \in f^{-1}(s')$ drawn from the emission distribution $q( \cdot \mid s')$.
The map $f$ is called the latent state decoding function, and the presumed challenge is that it is hidden from the agent.

The structural assumptions just discussed are what make \glspl{BMDP} interesting and different from generic \glspl{MDP}.
The remaining parameters, including the reward and episode structure, are summarized in \Cref{fig:BMDP_illustration}.
Please see \Cref{sec:models} for details.

\begin{figure}[hbtp]
	\centering
	\includegraphics[width=0.95\linewidth]{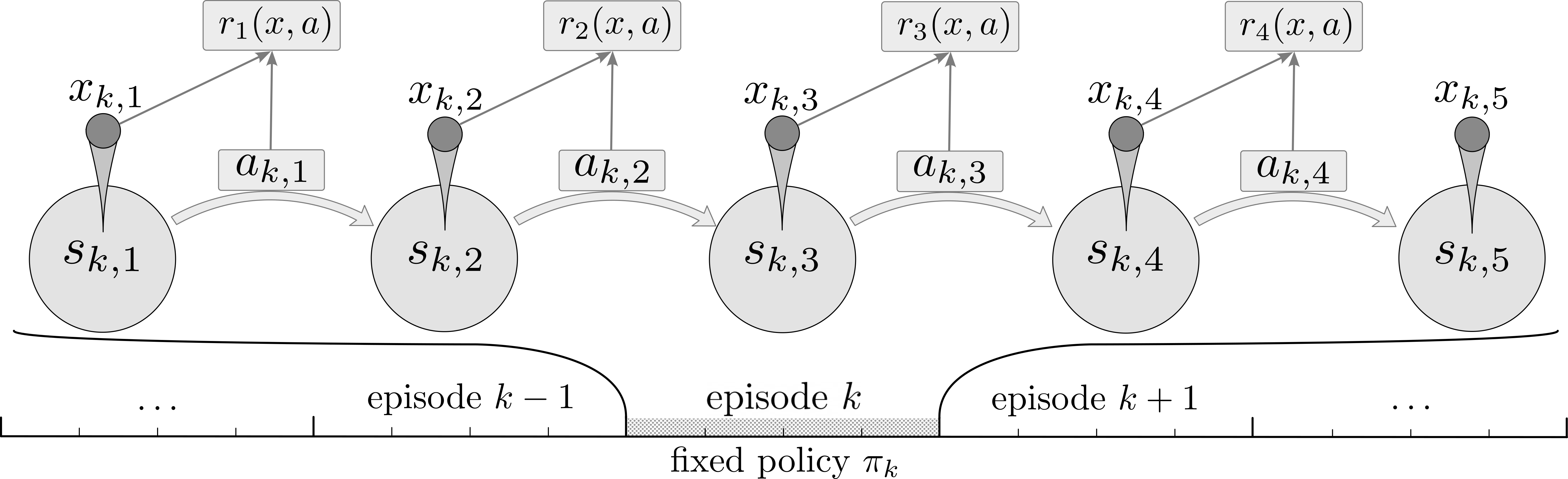}
	\caption{%
		Illustration of an episodic \gls{BMDP}.
		Shown from bottom to top are the hidden latent state sequence $s_{k,h}$, the action sequence $a_{k,h}$, the observed context sequence $x_{k,h}$, and the reward sequence $r_h(x,a)$.
		After $H$ transitions, the environment resets, and the next episode starts from a random context drawn from $\mu$.
		Each episode, the agent gets to choose and fix a policy.
	}
	\Description{
		The figure depicts a sequence of latent states connected by transitions.
		These transitions depend on the chosen action.
		Each time step, a context belonging to that latent state is observed.
		A reward is generated that depends on the context, the chosen action, and the time step.
		After H transitions, the environment resets, and the next episode starts from a random context drawn from a fixed distribution.
		Each episode, the agent gets to choose and fix a policy.
	}
	\label{fig:BMDP_illustration}
\end{figure}

Recently, a clustering technique developed for \glspl{BMC} \cite{Sanders:2020}, inspired by techniques for (labeled) \glspl{SBM} \cite{abbe2018community,yun2016optimal}, was generalized to \glspl{BMDP} \cite{Jedra:2022}.
Specifically, \cite{Jedra:2022} mathematically studies a technique that estimates the latent state decoding function efficiently based on data generated under a fixed policy.
This technique was then used in \cite{Jedra:2022} to study the sample complexity of a reward-free \gls{RL} algorithm.
A regret analysis when exploiting these techniques, however, remained open.
This gap is addressed in this manuscript.

In typical regret analyses, one studies not only specific algorithms that minimize regret but also the least amount of regret that \emph{every} algorithm will acquire \cite{bubeck2012regret}.
The fact that some regret will always be accumulated is, ultimately, because exploration must be balanced with exploitation.

Most methods tackle this so-called exploration--exploitation trade-off using the optimism--in--the--face--of--uncertainty principle \cite{sutton2020reinforcement}.
It is implemented in \glspl{MDP} by computing \glspl{UCB} on the estimated long-term reward following a visit to a state--action pair \cite{UCRLVI,dann2017unifying,pmlr-v97-zanette19a}.
\glspl{UCB} are added to the immediate reward during policy updates as an exploration bonus, encouraging visits to state--action pairs with uncertain value.
This reduces uncertainty, leading to smaller bonuses and less exploration overall.

Now, in \glspl{BMDP}, we can leverage the block structure by aggregating related transitions to reduce the value estimates' variances.
This yields tighter \glspl{UCB} and smaller exploration bonuses, ultimately reducing regret.
But, the quality of those value estimates would then depend on how well the block structure is estimated.
Misclassified contexts induce biases and therefore complicate efficient exploration.

To avoid this source of approximation error, we adopt a two--phase approach by building on \cite{Jedra:2022}.
Under their structural assumptions, clustering techniques efficiently recover the latent structure through random exploration.
This enables a sharp regret analysis, which we present in this paper.

\subsection{Summary of results}
\label{sec:summary_of_main_results}

\subsubsection{Structural properties}

We analyze the following subclasses of \glspl{BMDP}:

\begin{definition}[$\eta$-Reachability]
	\label{def:reachability}

	Let $\eta>1$ be independent of $n$.
	A \gls{BMDP} $\Phi$ is $\eta$-reachable if it has the following properties:
	\begin{definitionenum}
		\item \label{def:reachability:p}
		$
			\max_{a}
			\max_{s_1,s_2,s_3}
			\{ p(s_2\mid s_1,a) / p(s_3\mid s_1,a), p(s_1\mid s_2,a) / p(s_1\mid s_3,a) \}
			\leq
			\eta
		$;
		\item \label{def:reachability:q}
		$
			\max_{s}
			\max_{x,y\in f^{-1}(s)}
			q(x\mid s)/q(y\mid s)
			\leq
			\eta
		$; and
		\item \label{def:reachability:f}
		for all sufficiently large $n$,
		$
			\max_{s_1,s_2}
			|f^{-1}(s_1)|/|f^{-1}(s_2)|
			\leq
			\eta
		$
		.
	\end{definitionenum}
\end{definition}

\begin{definition}[$\mcI$-Identifiability]
	\label{def:identifiability}

	Let $\mcI>0$ be independent of $n$.
	An $\eta$-reachable \gls{BMDP} $\Phi$ is $\mcI$-identifiable if, for all sufficiently large $n$, $\min_{x}I_{\eta}(x;\Phi) \allowbreak \geq \allowbreak \mcI$; see \eqref{eqn:I_eta_def} for $I_\eta$'s definition.
\end{definition}

A brief explanation of \Cref{def:reachability,def:identifiability} is as follows.
The reachability property ensures that random exploration achieves coverage of the observation space.
The identifiability property ensures that it is information-theoretically possible to detect the latent states.
The parameter $I_{\eta}$ can be interpreted as a measure of separation between latent states.
For more detailed discussion on \Cref{def:reachability,def:identifiability}, we refer to \Cref{sec:on-the-structural-properties}.

\subsubsection{Regret lower bound}

Our first result is a regret lower bound for $\eta$-reachable and $\mcI$-identifiable \glspl{BMDP}.
To state it, let $n$ denote the number of contexts, $S$ the number of latent states, $A$ the number of actions, $H$ the episode horizon, $K$ the number of episodes, and $T_K=KH$ the total number of time steps.
We denote by $\reg_K(\mcL;\Phi)$ the expected regret after $K$ episodes of a learning algorithm $\mcL$ on the \gls{BMDP} $\Phi$; see \eqref{eqn:regret_def} for the precise definition.

\begin{theorem}
	\label{thm:regret_lower_bound}

	Assume that the following hold:
	\begin{equation}
		\label{eqn:parameter_assumptions}
		S = o(n)
		,
		\quad
		S/4\in \mathbb{N}_{\geq 2}
		,
		\quad
		H\geq 2
		,
		\quad
		\textnormal{and}
		\quad
		A\geq S/2
		.
	\end{equation}
	There exists a constant $\mf{d}>0$ independent of $n$ such that the following holds:
	for any $\eta>1$ independent of $n$
	and any learning algorithm $\mcL$,
	there exists a \gls{BMDP} $\Phi$ with initial context distribution $\mu=\mathrm{Unif}([n])$
	that is
	$\eta$-reachable
	and $I_\eta := \mf{d}\eta^{-1}(\eta-1)^2/(\eta+1)^3$-identifiable.
	If moreover $T_K=\Omega(n\vee A)$, then
	\begin{equation}
		\label{eqn:regret_lower_bound}
		\reg_K(\mcL;\Phi)
		=
		\Omega(\sqrt{AT_K}+n)
		.
	\end{equation}
\end{theorem}

\Cref{thm:regret_lower_bound} identifies the two asymptotically leading regret contributions.
Firstly, the $\Omega(\sqrt{AT_K})$ term captures the cost of learning the latent state transition probabilities.
Secondly, the $\Omega(n)$ term encapsulates the cost of learning the latent structure.

\subsubsection{Two-phase algorithm and regret upper bound}
\label{sec:main_results:regret_upper_bound}

Our second result establishes that an explicit two--phase algorithm achieves near-optimal worst-case regret in $\eta$-reachable and $\mcI$-identifiable \glspl{BMDP}.
\Cref{alg:BUCBVIouter,alg:BUCBVI,alg:UpdateLatentStates} in \Cref{sec:algorithm} contain the pseudocode.

The first phase consists of exploration happening uniformly at random.
Its duration is set by a threshold $\Theta^{\clust}$.
Based on the observations made, the latent structure is estimated using \cite[Algorithms 1 and 2]{Jedra:2022}.
The second phase uses an optimism--based strategy to balance exploration with exploitation while refining its transition probability estimates.
It adapts \cite[Algorithms 1 and 2]{UCRLVI}, designed for \glspl{MDP} without latent structure, by incorporating improved value estimates and smaller exploration bonuses that take advantage of the learned structure.

\begin{theorem}
	\label{thm:BUCBbd}
	Assume that $T_K= \Omega(H^4S^3A)$, $A = n^{O(1)}$, and $SH=O(n)$.
	Assume, moreover, that the \gls{BMDP} $\Phi$ is $\eta$-reachable and $\mcI$-identifiable, and that $\mu=\mathrm{Unif}([n])$.
	If \cref{alg:BUCBVIouter} is run with threshold $\Theta^{\clust} = \omega(nA(S^3 \vee \ln n))$, then for any $c>0$,
	\begin{equation}
		\label{eqn:regret_bound}
		\textnormal{Reg}_K(\mcL;\Phi)
		=
		\tilde{O}\bigl(
		H\sqrt{SAT_K}
		+
		nSH^3
		+
		T_K/n^c
		+
		\Theta^{\clust}
		\bigr)
		.
	\end{equation}
	Here, the notation $\tilde{O}(\cdot)$ suppresses polylogarithmic factors.
\end{theorem}

To facilitate comparison with \Cref{thm:regret_lower_bound}, consider that \eqref{eqn:regret_bound} implies for $\Theta^{\clust} = nS^3A \ln^2 n$ and $T_K=n^{O(1)}$ that
\begin{equation}
	\label{eqn:simplified_regret_upper_bound}
	\reg_K(\mcL;\Phi)
	=
	\tilde{O}
	\bigl(
	H\sqrt{SAT_K}+n(S^3A+SH^3)
	\bigr)
	.
\end{equation}
This result is optimal in the sense that it matches \Cref{thm:regret_lower_bound}'s dependence on $n$ and $T_K$ up to polylogarithmic factors.

It is also instructive to compare \Cref{thm:regret_lower_bound} to other regret bounds for \glspl{BMDP} when clustering.
The one paper closest in spirit that we found is \cite{azizzadenesheli2016reinforcement}.
There, the authors prove a regret bound for an online algorithm that also uses spectral methods.
When specialized to $\eta$-reachable and $\mcI$-identifiable \glspl{BMDP}, their bound is only $\tilde{O}(DS\sqrt{AT}\allowbreak+\allowbreak n^2SA)$.
Here, $D$ bounds the time required to move between two contexts, replacing $H$ in their nonepisodic setting.
Their result, however, does apply also to other \glspl{BMDP}, not just $\eta$-reachable and $\mcI$-identifiable ones.

Our result improves the regret's dependence on the number of contexts from $\tilde{O}(n^2)$ to $\tilde{O}(n)$, at the cost of a greater dependence on $S$, $A$, and $H$.
This leads to a better bound for sufficiently large observation spaces, specifically when $n=\omega(S^2+H^3/A)$.
We achieve this by performing dedicated exploration, and incorporating the efficient latent state estimation algorithm from \cite{Jedra:2022}.

In addition, \eqref{eqn:simplified_regret_upper_bound} improves the leading-order regret for large $T_K$ by $\sqrt{S}$.
This is significant because \cite{azizzadenesheli2016reinforcement} assumes rewards depend solely on the latent state.
In this case it is unnecessary to learn the emission kernel, simplifying the analysis considerably.
Our improvement stems from leveraging the algorithms and proof techniques in \cite{UCRLVI}, rather than the earlier \cite{JMLR:v11:jaksch10a}, upon which \cite{azizzadenesheli2016reinforcement} builds.

\subsection{Summary of proof techniques}

\subsubsection{Proof of the regret lower bound}

\Cref{thm:regret_lower_bound}'s proof is presented in \Cref{sec:regret_lower_bound}.
We define a class of hard--to--learn \glspl{BMDP}, adapted from \cite{Jedra:2022}, and show that it includes instances inducing $\Omega(\sqrt{AT_K}+n)$ regret.

To establish the $\Omega(n)$ contribution, we prove a new lower bound on the worst-case number of misclassifications of \textit{any} latent state estimator on \glspl{BMDP} in our class (see \Cref{prop:lower_bound_misclassification_rate}).
Existing lower bounds, such as \cite[Theorem 1]{Jedra:2022} and the closely related \cite[Theorem 1]{Sanders:2020}, do bound the number of misclassifications on any given \gls{BMDP}, but only for algorithms satisfying a consistency condition.
Instead, our class is such that any algorithm violating this condition must perform poorly on at least one instance.
The construction is based on a construction used for random graph clustering \cite{zhang2016minimax}.

The proof for the $\Omega(\sqrt{AT_K})$ contribution leverages techniques developed for tabular \glspl{MDP} \cite{domingues2021episodic,NEURIPS2019_10a5ab2d}
These were in turn adapted from the literature on bandit learning \cite{bubeck2012regret}.
A direct application of \cite{domingues2021episodic}'s construction to the induced \gls{MDP} on the latent states would lead to an $\Omega(\sqrt{HSAT_K})$ lower bound for general \glspl{BMDP}.
However, the resulting instances would not be $\eta$-reachable or $\mcI$-identifiable.
We therefore modify those constructions to comply with \Cref{def:reachability,def:identifiability}.
This leads to a slightly weaker but still $\Omega(\sqrt{T_K})$-type regret.

\subsubsection{Proof of the regret upper bound}

\Cref{thm:BUCBbd}'s proof is presented in \Cref{sec:regret_analysis}.
We establish there that our algorithm's first phase successfully recovers the latent structure.
We subsequently bound the regret incurred during the algorithm's second phase.

For the first phase, we refine \cite{Jedra:2022}'s analysis of \cite[Algorithms 1 and 2]{Jedra:2022} in two ways.
First, whereas \cite{Jedra:2022} establishes exact latent state recovery with high probability as $n\rightarrow\infty$, we provide an explicit probability bound.
This is required for bounding the expected regret.
Second, we relax their assumptions, allowing $S$, $A$ and $p$ to depend on $n$.
Moreover, \Cref{def:identifiability}'s condition is slightly weaker than its analogue in \cite{Jedra:2022}; see \Cref{sec:on-the-structural-properties} for a comparison.
These relaxations are necessary, because we could only prove \Cref{thm:regret_lower_bound} under these weaker assumptions.

For the second phase, we leverage that \Cref{alg:BUCBVIouter} has already recovered the latent structure exactly.
In this case, the exploration bonuses need not account for uncertainty about this structure.
This allows us to broadly follow the analysis of \cite{UCRLVI}.
The main new difficulty arises from the structure of the \gls{BMDP} kernel, involving both a latent state transition kernel and emission probabilities.
Their estimators are statistically dependent, and a key part of our analysis is a decomposition that disentangles their contributions to the regret.

\subsection{Related literature}

Our work connects to the following three research directions.

\subsubsection{Function approximation for block and low-rank \texorpdfstring{\glspl{MDP}}{MDPs}}
\label{sec:related-literature:function-approximation}

\gls{RL} in \glspl{BMDP} often uses function approximation.
This assumes that the decoding- or value function belongs to a known, well-behaved function class $\mcF$.
Prior work has used function approximation for \gls{PAC} learning \cite{NIPS2016_2387337b,pmlr-v70-jiang17c,NEURIPS2018_5f0f5e5f,pmlr-v97-du19b,pmlr-v162-zhang22aa}, which bounds the time required to find a near-optimal policy; reward-free learning \cite{pmlr-v162-zhang22aa, pmlr-v119-misra20a}, where the agent explores without observing rewards; and regret minimization \cite{foster2020instancedependent}, which is the focus of this paper.
These methods have also been applied to more general models such as low-rank \glspl{MDP} \cite{agarwal2020flambe,DBLP:journals/corr/abs-2110-04652, modi2022modelfree, mhammedi2023efficient}, which include \glspl{BMDP} as special cases.

The use of function approximation is typically motivated by settings with infinite or continuous context spaces.
Function approximation is then necessary to generalize to unseen contexts.
By contrast, we consider \glspl{BMDP} with finite context spaces and arbitrary decoding functions.
In this case, \cite{Jedra:2022} notes that function approximation achieves the same sample complexity scaling in $n$ as tabular methods.
In other words, these methods do not exploit the inherent block structure, but instead benefit from the restrictions imposed by the class $\mcF$.
We therefore also do not rely on function approximation or oracles for latent state estimation.

\subsubsection{Latent state estimation and clustering in block models}

Another line of research develops explicit algorithms for learning the latent structure from observations.
This includes \cite{azizzadenesheli2016reinforcement}, which estimates the latent structure using a spectral method based on \cite{JMLR:v15:anandkumar14b}.
However, as explained in \Cref{sec:main_results:regret_upper_bound}, this leads to suboptimal regret in the worst case.
Instead, we build on \cite{Jedra:2022}, whose algorithm uses a two-step approach: an initial spectral estimate is refined using a likelihood-based improvement step.
This approach was inspired by similar techniques developed for \glspl{BMC} \cite{Sanders:2020} (but see also \cite{10.5555/3454287.3454690,8918022,zhu2022learning,van2024estimating}), which, in turn, build on techniques for \glspl{SBM} \cite{yun2014community,yun2016optimal}.

\cite{Jedra:2022} also shows that their algorithm is nearly optimal by deriving an information-theoretic lower bound.
This is done using change-of-measure arguments, likewise inspired by \cite{Sanders:2020,yun2014community,yun2016optimal}.
As detailed in \Cref{sec:summary_of_main_results}, our work adds to these developments in latent state estimation by proving a new lower bound on the number of misclassifications, and refining \cite{Jedra:2022}'s analysis of their algorithm.

\subsubsection{Regret minimization in \texorpdfstring{\glspl{MDP}}{MDPs} with known structure}

Recall that, after learning the latent structure, we adapt the optimism-based approach of \cite{UCRLVI} to minimize the regret while refining our transition probability estimates.
We use specialized kernel estimates and bonus terms that exploit the latent structure, leading to improved regret.
This approach has also been successfully applied to other types of structured \glspl{MDP}.
For instance, \cite{NEURIPS2020_e61eaa38} designs exploration bonuses for factored \glspl{MDP}, \cite{gast2022reinforcement} for Markovian bandits, and \cite{pmlr-v125-jin20a} for linear \glspl{MDP}.
Similar two--phase approaches have also been considered in the context of bandit learning \cite{pmlr-v32-gentile14,gopalan2016lowrank,jedra2024low}.

Linear \glspl{MDP} in particular include \glspl{BMDP} (with known latent structure) as a special case.
One can therefore also use algorithms designed for linear \glspl{MDP} after the initial structure-learning phase.
However, because linear \glspl{MDP} are more general than \glspl{BMDP}, these algorithms do not fully exploit the discrete structure of \glspl{BMDP}.
They therefore do not achieve optimal regret on \glspl{BMDP}, with the best regret bounds being $\tilde{O}(SA\sqrt{HT_K})$ \cite{he2023nearly, agarwal2023vo,hu2022nearly}.
We exploit the additional structure of \glspl{BMDP} to improve this by a factor of $\sqrt{SA}$ in \Cref{thm:BUCBbd}.

\section{Model definitions}
\label{sec:models}

\subsection*{Elementary notation}

For $m\in \mathbb{N}_+$, let $[m] \eqdef \{1,2,\ldots, m\}$.
Functions like $v:[m]\rightarrow\mbbR$ are identified with vectors in $\mbbR^m$, with components $v(x)$ for $x\in[m]$ relative to the standard basis.
For $v,w\in\mbbR^m$, let $\langle v,w\rangle$ denote the standard inner product.
Finally, let $\mathrm{Unif}(\mcX)$ refer to the uniform distribution on $\mcX\subset[m]$.

\subsection{Block Markov Decision Processes}
\label{sec:Block-MDPs}

\glspl{BMDP} are described by tuples $\Phi = (n, S, A, H, \allowbreak f, \allowbreak p, \allowbreak q, \mu, r)$.
Here, $n, S, A, H \in \mbbN_+$ denote the number of contexts, latent states, actions, and episode horizon, respectively.
Unique to a Block \gls{MDP} is that there is a \emph{decoding function} $f:[n]\rightarrow [S]$ that assigns contexts to latent states, and that the latent states determine the dynamics.
Specifically, \glspl{BMDP} have a \textit{latent state transition kernel} $p:[S]\times[A]\times[S]\rightarrow [0,1]$ and an \textit{emission kernel} $q:[S]\times[n]\rightarrow[0,1]$.
Here, $p(s'\mid s,a)$ is the probability of transitioning to state $s'$ given that the system is in state $s$ when choosing action $a$, and $q(x\mid s)$ the probability of observing context $x$ given that the system is in latent state $s$.

The triple $(f,p,q)$ specifies the \gls{MDP}'s transition kernel $P:[n]\times[A]\times[n]\rightarrow [0,1]$ completely.
Specifically,
\begin{equation}
	P(y \mid x, a)
	\eqdef
	p( f(y) \mid f(x), a )
	q( y \mid f(y) )
	\label{def:BMDP-transition-kernel}
\end{equation}
is the probability of transitioning to context $y$ given that the current context is $x$ when choosing action $a$.
Note that the probability of moving to context $y$ depends only on the current latent state and not the specific context.

Finally, $r = \{r_h\}_{h\in[H]}$ with $r_h:[n]\times[A]\rightarrow [0,1]$ denotes a bounded, nonstationary and deterministic reward function that may depend on the context that is observed.

\subsection{Policies}

A \textit{memoryless policy} $\pi$ is a collection $\pi=\{\pi_h\}_{h\in [H]}$ of mappings $\pi_h:[n]\times[A]\rightarrow [0,1]$ where $\pi_h(a\mid x)$ denotes the probability of selecting action $a$ when observing context $x$ at round $h$.
Memoryless policies do not depend on the sequence of contexts prior to the current context.
Because we may only consider memoryless policies, we will simply refer to them as policies.

Together, a \gls{BMDP} $\Phi$ and a policy $\pi$ define a probability measure $\mbbP_{\Phi,\pi}$ and a random variable $(x_1,a_1,\ldots,x_H,a_H,x_{H+1})$ representing the sequence of contexts and actions satisfying $\mbbP_{\Phi,\pi}[x_1=x]=\mu(x)$, $\mbbP_{\Phi,\pi}[a_h=a\mid x_h=x]=\pi_h(a\mid x)$ and $\mbbP_{\Phi,\pi}[x_{h+1}=y\mid x_h=x,a_h=a]=P(y\mid x,a)$.

\subsection{Value functions}

For $h\in[H]$, define the Q-function $Q^{\pi}_h$ and value function $V^{\pi}_h$ of a policy $\pi$ element-wise as
\begin{equation}
	\label{eqn:valuefuncdef}
	Q^{\pi}_h(x,a)
	\eqdef
	\sum_{h'=h}^H
	\mbbE_{\Phi,\pi}[
	r_h(x_{h'},a_{h'}) \mid x_h=x, a_h = a
	]
	\quad
	\mathrm{and}
	\quad
	V^{\pi}_h(x)
	\eqdef
	\mbbE_{\Phi,\pi}[Q_h^{\pi}(x,a_h)]
\end{equation}
for $x\in[n]$, $a\in[A]$.
These satisfy the following Bellman recursion:
\begin{equation}
	\label{eqn:bellman_equation}
	Q^{\pi}_h(x,a) = r_h(x,a) + \sum_{y}P(y\mid x,a)V_{h+1}^{\pi}(y)
	,
	\quad
	V_h^{\pi}(x) = \sum_{a}\pi_h(a\mid x)Q_h^{\pi}(x,a)
	,
	\quad
	V_{H+1}^{\pi}(x) \eqdef 0
	.
\end{equation}

When $n$, $S$, $A$ and $H$ are all finite, there always exists a policy $\pi^*=\{\pi_h^*\}_{h\in[H]}$ that satisfies $V_h^{\pi^*}(x)=\sup_{\pi\in\Pi}V^{\pi}_h(x)$ for all $x\in[n]$ and $h\in[H]$.
Such a policy can be explicitly constructed through backward value iteration using \eqref{eqn:bellman_equation}.
Namely, starting from the initial condition $V_{H+1}^{\pi^*}=0$, we can obtain an optimal policy and value function at round $h\in[H]$ by setting
\begin{equation}
	\label{eqn:backward_value_iteration}
	\pi_h^*(\,\cdot\mid x)
	=
	\mathrm{Unif}(\argmax_{a}Q_h^{\pi^*}(x,a))
	\quad
	\textnormal{and}
	\quad
	V_h^{\pi^*}(x)
	=
	\max_{a}Q_{h}^{\pi^*}(x,a)
	.
\end{equation}

\subsection{Episodic learning algorithms and regret}
\label{sec:preliminaries-episodes}

An \textit{episodic learning algorithm} selects a policy for each episode based on its previous observations and the known reward function.
Formally, such an algorithm is a collection $\mcL=\{\mcL_k\}_{k\in\mbbN_+}$ of maps where for any $D_k\in(([n]\times[A])^H\times[n])^{k-1}$ and reward function $r$, $\mcL_k(D_k,r)$ is a policy.

A \gls{BMDP} and learning algorithm give rise to a probability measure $\mbbP_{\Phi,\mcL}$ along with a sequence of random variables $\mcD_k$ for $k\geq 2$ representing the contexts and actions during the first $k-1$ episodes:
\begin{equation}
	\label{eqn:historydef}
	\mcD_{k}
	\eqdef
	\otimes_{k'=1}^{k-1}\mcD_{(k')}
	\quad \mathrm{where}
	\quad
	\mcD_{(k')}
	\eqdef
	(x_{k',1},a_{k',1},\ldots,x_{k',H},a_{k',H},x_{k',H+1})
	.
\end{equation}
Then $\mbbP_{\Phi,\mcL}[x_{k,1}=x]=\mu(x)$, $\mbbP_{\Phi,\mcL}[a_{k,h}=a \mid x_h=x,\mcD_k]=\pi_{k,h}(a\mid x)$, and $\mbbP_{\Phi,\mcL}[x_{k,h+1}=y\mid x_{k,h}=x,a_{k,h}=a]=P(y\mid x,a)$, where $\pi_{k}=\mcL_k(\mcD_k,r)$ denotes the policy selected by $\mcL$ for episode $k$.

We define the \textit{regret} of a learning algorithm $\mcL$ on a \gls{BMDP} $\Phi$ after $K$ episodes by
\begin{equation}
	\label{eqn:regret_def}
	\reg_K(\mcL;\Phi)
	\eqdef
	\sum_{k=1}^K
	\mbbE_{\Phi,\mcL}
	\bigl[V_{1}^{\pi^*}(x_{k,1})-V_{1}^{\pi_k}(x_{k,1})\bigr]
	.
\end{equation}

\subsection{Latent state estimation}
\label{sec:preliminaries_latent_state_estimation}

We also consider algorithms that estimate the latent structure based on a sequence of transitions of the \gls{BMDP}.
Formally, these are defined as a collection of maps $\mcC=\{\mcC_k\}_{k\in\mbbN_+}$ where $\mcC_k(D_k,r)$ is a decoding function for any $D_k\in\mcH^{k-1}$ and reward function $r$.

Given an estimated decoding function $\hat{f}$, we define the set of misclassified contexts in $\mcS\subset[S]$ as
\begin{equation}
	\label{eqn:misclassdef}
	\mcE(\hat{f};f,\mcS)
	\eqdef
	\cup_{s\in\mcS} (\hat{f}^{-1}(\hat{\gamma}^*_f(s))\setminus f^{-1}(s))
	\quad
	\mathrm{with}
	\quad
	\hat{\gamma}^*_f
	\in
	\argmin_{\gamma\in\mathrm{Perm}(\mcS)}
	|\cup_{s\in\mcS} (\hat{f}^{-1}(\gamma(s))\setminus f^{-1}(s))|
	.
\end{equation}
For convenience, we also introduce the shorthand notation $\mcE(\hat{f};f)\eqdef \mcE(\hat{f};f,[S])$.

We next recall an information-theoretic quantity from \cite{Jedra:2022}.
It controls the optimal misclassification count of certain well-behaved estimation algorithms; see \cite[Theorem 1]{Jedra:2022}.
It is analogous to similar quantities controlling the optimal rate of cluster recovery in \glspl{BMC} \cite{Sanders:2020} and \glspl{SBM} \cite{zhang2016minimax, yun2016optimal}.

Let $\Phi$ have decoding function $f$, and define for any $\pi$, $s$, and $a$,
\begin{equation}
	\label{eqn:definition_of_visitation_probabilities}
	\omega_{\Phi,\pi}(s,a)
	\eqdef
	\frac{1}{H}\sum_{h=1}^H\mbbP_{\Phi,\pi}[f(x_h)=s,a_{h}=a]
	.
\end{equation}
Given a \gls{BMDP} $\Phi$, context $x$, latent state $\tilde{s}\neq f(x)$, and $c>0$, let $\Psi$ be the \gls{BMDP} obtained from $\Phi$ by moving context $x$ to latent state $\tilde{s}$, with $c$ parametrizing the emission probability of $x$ in $\tilde{s}$; see \cite[Appendix D]{Jedra:2022} for the detailed construction of $\Psi$.
Using the representation in \cite[Eqn. (35)]{Jedra:2022}, the information-theoretic quantity is now given by
\begin{align}
	I_{\tilde{s}}(x;c,\Phi,\pi)
	\eqdef
	 &
	n
	\KL(
	p^{\mathrm{in}}_{x,\tilde{s},c,\omega_{\Psi,\pi}}
	,
	p^{\mathrm{in}}_{x,f(x),1,\omega_{\Psi,\pi}}
	)
	\nonumber
	\\
	 &
	\phantom{=}
	+
	n
	c
	q(x\mid f(x))
	\sum_{a}
	\omega_{\Psi,\pi}(f(x),a)
	\KL(
	p(\,\cdot\mid \tilde{s},a)
	,
	p(\,\cdot\mid f(x),a)
	)
	,
	\label{eqn:Ijxcdef}
\end{align}
with $p^{\mathrm{in}}_{x,\tilde{s},c,\omega}:[2]\times[S]\times[A]\rightarrow[0,1]$ defined for any $x$, $\tilde{s}$, $c>0$, and $\omega:[S]\times [A]\rightarrow [0,1]$ satisfying $\sum_{s,a}\omega(s,a)=1$, as
\begin{equation}
	\label{eqn:p_in_def}
	p^{\mathrm{in}}_{x,\tilde{s},c,\omega}(i,s,a)
	\eqdef
	\begin{cases}
		\omega(s,a)p(\tilde{s} \mid s,a)cq(x \mid f(x)),
		 &
		\textnormal{if } i = 1,
		\\
		\omega(s,a)(1-p(\tilde{s} \mid s,a)cq(x \mid f(x)))
		 &
		\textnormal{if } i = 2.
		\\
	\end{cases}
\end{equation}

Let $\pi_U$ denote the uniform policy given by $\pi_{U,h}(\,\cdot\mid x)\eqdef \mathrm{Unif}([A])$ for all $x$ and $h$.
The parameter $I_{\eta}$ appearing in \Cref{def:identifiability} is now defined for $\eta>1$ as
\begin{equation}
	\label{eqn:I_eta_def}
	I_{\eta}(x;\Phi)
	\eqdef
	\min_{\tilde{s}\neq f(x)}\inf_{c\in[1/\eta,\eta]}I_{\tilde{s}}(x;c,\Phi,\pi_U)
	.
\end{equation}

\section{Regret lower bound}
\label{sec:regret_lower_bound}

This section proves \Cref{thm:regret_lower_bound}.
\Cref{sec:proof_outline_lower_bound} outlines the proof, and \Cref{sec:proof_of_regret_lower_bound} contains the details.

\subsection{Proof outline}
\label{sec:proof_outline_lower_bound}

To prove \Cref{thm:regret_lower_bound}, we construct a hard--to--learn subclass $\Lambda^*$ consisting of $\eta$-regular and $\mcI_{\eta}$-identifiable \glspl{BMDP}.
We then lower bound the maximum regret over \glspl{BMDP} in $\Lambda^*$.

The proof takes inspiration from \cite[Theorem 5]{Jedra:2022}.
It is, however, modified to
(i) comply with \Cref{def:reachability,def:identifiability} and
(ii) obtain a worst-case regret lower bound.
We distinguish three parts:

\begin{enumerate}
	\item[Part 1.]
	      We define $\Lambda^*$ and show that all \glspl{BMDP} in $\Lambda^*$ are $\eta$-regular and $\mcI_{\eta}$-identifiable.

	      Intuitively, $\Lambda^*$ is constructed so that rewards are received only upon visiting a specific subset of latent states, $\mcS_1 = \{ \floor{S/2}+1,\ldots,S \}$, say.
	      In each latent state $s$, there is one action $a_s^*$ that is $\varepsilon$ more likely to transition to $\mcS_1$ relative to the other actions.
	      Even if a context $x$ is revealed to belong to $s$, the agent must still identify the optimal action $a_s^*$.
	      Conversely, even if $a_s^*$ is known for each $s$, the agent must still infer which latent state $x$ corresponds to.
	      By allowing $\varepsilon$ to depend on whether $s$ belongs to $\mcS_1$ or $\mcS_1^{\mathrm{c}}$, and tuning its value separately for each, we control which of these two learning challenges dominates.
	      This gives rise to the two terms $\Omega(\sqrt{AT_K})$ and $\Omega(n)$ in \eqref{eqn:regret_lower_bound}.

	\item[Part 2.]
	      The $\Omega(\sqrt{AT_K})$ term in \eqref{eqn:regret_lower_bound} arises from having to learn the transition probabilities \emph{even if} the latent structure is revealed.

	      To establish this, we analyze the induced \gls{MDP} on the latent states.
	      Following \cite{domingues2021episodic, doi:10.1287/moor.2017.0928}, we use a change--of--measure argument to first lower bound the number of times an algorithm selects suboptimal actions while visiting $\mcS_1^{\mathrm{c}}$.
	      We then use the transition structure of the \glspl{BMDP} in $\Lambda^*$ to bound how often $\mcS_1^{\mathrm{c}}$ is visited.
	      Specifying $\varepsilon = \Theta(\sqrt{A/T_K})$ yields the $\Omega(\sqrt{AT_K})$ contribution in \eqref{eqn:regret_lower_bound}.

	\item[Part 3.]
	      The $\Omega(n)$ term in \eqref{eqn:regret_lower_bound} arises from having to learn the latent structure \emph{even if} the optimal actions are revealed.

	      To formalize this, we fix a set of optimal actions $\{a_s^*\}_{s}$ and consider the basic estimator that assign contexts $x$ to latent states $s$ whose action $a_s^*$ are chosen most.
	      We then show that in \glspl{BMDP} where $\{a_s^*\}_{s}$ are indeed optimal, the regret is bounded from below by the number of contexts misclassified by this estimator.
	      Applying \Cref{prop:lower_bound_misclassification_rate}, which lower bounds the maximum number of misclassifications over \glspl{BMDP} in $\Lambda^*$ of \emph{any} latent-state estimator, yields the $\Omega(n)$ contribution when specifying $\varepsilon = \Theta(1)$.
\end{enumerate}

\subsection{Proof of \texorpdfstring{\Cref{thm:regret_lower_bound}}{the regret lower bound}}
\label{sec:proof_of_regret_lower_bound}

Presume throughout this section that $S$, $A$, and $H$ satisfy \eqref{eqn:parameter_assumptions}.

Let $\mcS_1 = \{ S/2+1,\ldots,S \}$ and $\mcS_0 = \mcS_1^{\mathrm{c}}$, and define $\iota : [S] \rightarrow \{0,1\}$ as the map that assigns
\begin{equation}
	s
	\mapsto
	\iota(s)
	\eqdef
	\arg
	\{
	i \in \{ 0, 1 \}
	:
	s \in \mathcal{S}_i
	\}
	.
	\label{def:iota}
\end{equation}

\subsubsection*{Part 1: Construction of the hard--to--learn subclass}

The hard--to--learn class will contain a subset of the following \glspl{BMDP} that are described by parameters $\varepsilon = (\varepsilon_0, \varepsilon_1)$, functions $\tilde{p} = (\tilde{p}_0, \tilde{p}_1)$, and actions $a^* = \{a_s^*\}_{s}$.

\begin{definition}
	[Candidate hard--to--learn \glspl{BMDP}]
	\label{def:hard-to-learn-BMDP-instance}

	For parameters $\varepsilon_0,\varepsilon_1\in [0,1/2]$, actions $a_s^*\in [A]$ for $s\in[S]$, a decoding function $f:[n]\rightarrow [S]$,
	and any two functions $\tilde{p}_0(\cdot \mid s)$ and $\tilde{p}_1(\cdot \mid s)$ that are probability distributions on $\mcS_0$ and $\mcS_1$ for every $s \in S$, respectively,\footnote{That is, for $i \in \{ 0, 1 \}$, $s\in[S]$, $\sum_{s'\in\mcS_i}\tilde{p}_i(s' \mid s) = 1$.}
	let
	\begin{equation}
		\Phi( \varepsilon, \tilde{p}, f, a^*)
	\end{equation}
	be a \gls{BMDP} with:
	\begin{definitionenum}
		\item \label{def:hard-BMDP--latent-state-transition-probabilities}
		latent state transition probabilities
		\begin{equation}
			p(s'\mid s,a)
			=
			\begin{cases}
				\frac{1}{2}
				\tilde{p}_{\iota(s')}(s' \mid s)
				 &
				\text{if }
				a\neq a^*_s,
				\\
				\frac{1}{2}
				(1-2\varepsilon_{\iota(s)})\tilde{p}_{0}(s' \mid s)
				 &
				\text{if }
				a=a^*_s
				,
				s'\in\mcS_0
				,
				\\
				\frac{1}{2}
				(1+2\varepsilon_{\iota(s)})\tilde{p}_{1}(s' \mid s)
				 &
				\text{if }
				a=a^*_s
				,
				s'\in\mcS_1
				;
			\end{cases}
			\label{eqn:latent_state_transition_kernel_macro_def}
		\end{equation}

		\item \label{def:hard-BMDP--uniform_emissions}
		uniform emission probabilities:
		for $s\in[S]$ and $y\in f^{-1}(s)$,
		$q(y\mid s)=1/|f^{-1}(s)|$;

		\item \label{def:hard-BMDP--unit-rewards}
		unit rewards only on $\mathcal{S}_1$: for $h\in[H]$ and $a\in[A]$,
		$
			r_h(x,a)
			=
			\mathbbm{1}\{x \in f^{-1}(\mcS_1)\}
		$;

		\item
		decoding function $f$;

		\item \label{def:hard-BMDP--uniform_initial_distribution}
		uniform initial context distribution: for $x \in[n]$, $\mu(x)=1/n$.
	\end{definitionenum}
\end{definition}

\Cref{fig:hard_to_learn_illustration} depicts the transition graph of such \glspl{BMDP} schematically.
Note that for every state $s$, the specific action $a_s^*$ results in the highest chance to transition to $\mcS_1$.
Because the only rewarding states are in $\mcS_1$, this makes $a_s^*$ the unique optimal action for latent state $s$.
Ultimately, it is this structure that allows us to lower bound the regret (see \Cref{lem:regret_decomposition_lower_bound}).

\begin{figure}[hbtp]
	\centering
	\includegraphics[width=0.75\linewidth]{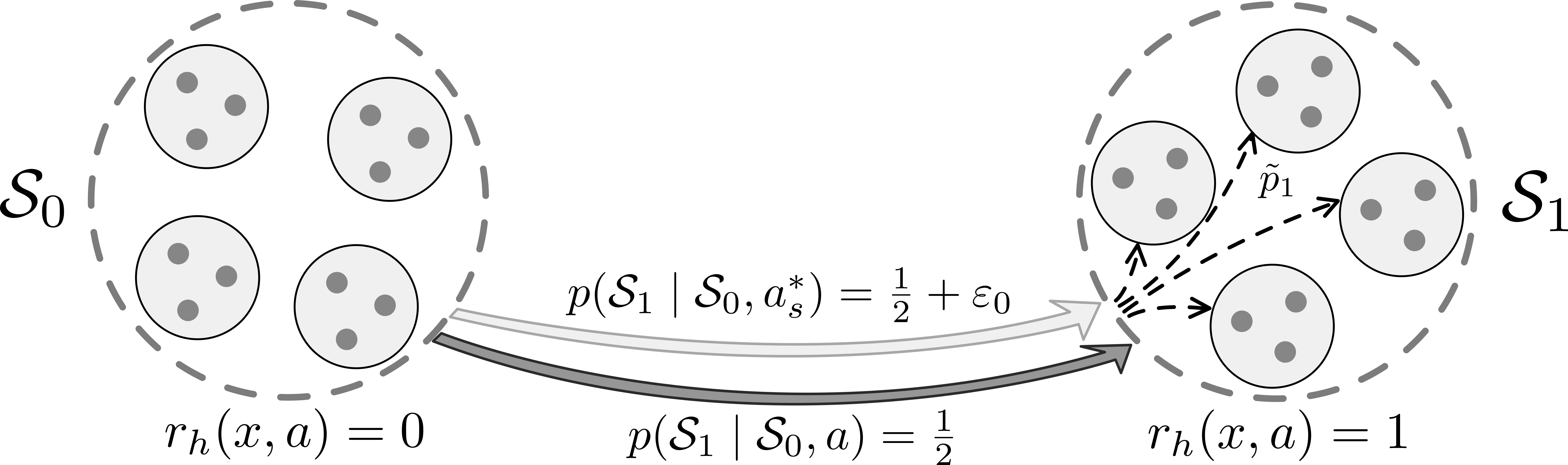}
	\caption{
		A candidate hard--to--learn \gls{BMDP} $\Phi( \varepsilon, \tilde{p}, f, \{a_s^*\}_{s} )$.
		The thick arrows represent the probabilities of transitioning to $\mcS_1$ for different actions.
		Conditional on observing a transition from $s$ to $\mcS_i$, we transition to $s'\in\mcS_i$ with probability $\tilde{p}_i(s'\mid s)$, represented by the thin arrows.
	}
	\Description{
		The figure shows two sets of latent states.
		On the first, the reward is always zero, and on the second it is always one.
		Two arrow representing transitions from the first to the second subset are shown, highlighting the different transition probabilities for the different actions.
		Thin arrows depict the distribution with which the different latent states in the second subset are visited.
	}
	\label{fig:hard_to_learn_illustration}
\end{figure}

To prove \eqref{eqn:regret_lower_bound}, we will consider hard--to--learn \glspl{BMDP} that differ in both their optimal actions and decoding functions.
This variation prevents algorithms from achieving low regret by tailoring its strategy to any single instance.
We next define specific sets of optimal action assignments and decoding functions that capture the necessary structural diversity.

We consider decoding functions for which all latent states contain roughly the same number of contexts.
The following lemma ensures a decomposition with the right properties is possible.
Its proof is based on the paragraphs surrounding \cite[(5.2)]{zhang2016minimax} and is given in \Cref{sec:proof_of_latent_state_subset_sizes}.
\begin{lemma}
    \label{lem:latent_state_subset_sizes}
    Let $n_0 = \floor{n/2}$ and $n_1 = n - n_0$.
    Then, for $i\in\{0,1\}$, there exist nonnegative integers $\mf{s}^{-1}_i,\mf{s}^0_i,\mf{s}^{+1}_i$ satisfying $\mf{s}^{+1}_i \geq S/36$ and $\mf{s}_i^{0}>0$ such that $\sum_{\sigma\in\{-1,0,+1\}}\allowbreak\mf{s}^{\sigma}_i(\floor{n_i/(S/2)}+\sigma)=n_i$.
\end{lemma}

\begin{definition}
	[Candidate hard--to--learn \gls{BMDP} classes]
	\label{def:hard_to_learn_class}

    With $\mf{s}^{\sigma}_i$ as in \Cref{lem:latent_state_subset_sizes}, let $\mcF\subset[S]^{[n]}$ denote the set of decoding functions $f$ that satisfy:
    \begin{definitionenum}
        \item \label{def:hard_to_learn_class:context_split}
            $f^{-1}(\mcS_0)=\{1,\ldots,\floor{n/2}\}$ and $f^{-1}(\mcS_1)=\{\floor{n/2}+1,\ldots,n\}$,

        \item \label{def:hard_to_learn_class:size_distribution_of_latent_states}
          for $i\in\{0,1\}$,
          \begin{align}
              \label{eqn:size_distribution_of_latent_states}
               &
              \mcS_i
              =
              \mcS_i^{-1}(f) \cup \mcS_i^{0}(f) \cup \mcS_i^{+1}(f)
              \\
               &
              \quad
              \mathrm{with}
              \quad
              \mcS_i^{\sigma}(f)
              \eqdef
              \biggl\{
                s\in\mcS_i
                :
                |f^{-1}(s)| = \biggl\lfloor \frac{|f^{-1}(\mcS_i)|}{|\mcS_i|} \biggr\rfloor + \sigma
              \biggr\}
              \quad
              \mathrm{for}
              \quad
              \sigma\in \{-1,0,+1\}
              ,
              \nonumber
          \end{align}
        \item \label{def:hard_to_learn_class:sizes_of_subsets}
              for $i\in\{0,1\}$ and $\sigma \in \{-1,0,+1\}$, $|\mcS_i^{\sigma}(f)|=\mf{s}^{\sigma}_i$.
    \end{definitionenum}
    Furthermore, let $\mcA \subset [A]^S$ be the set of all action assignments $a^*=\{a_s^*\}_{s}\in [A]^S$ that satisfy $a^*_{s_1}\neq a^*_{s_2}$ for any two distinct $s_1, s_2 \in \mcS_1$.

	Then, for $\varepsilon_0, \varepsilon_1 \in [0,1/2]$
	and any two functions $\tilde{p}_0(\cdot \mid s)$ and $\tilde{p}_1(\cdot \mid s)$ that are probability distributions on $\mcS_0$ and $\mcS_1$ for every $s \in [S]$, respectively,
	let
	\begin{equation}
		\Lambda( \varepsilon, \tilde{p} )
		\eqdef
		\bigl\{
		\Phi( \varepsilon, \tilde{p}, f, a^* )
		:
		f\in\mcF,
		a^* \in \mcA
		\bigr\}
		.
	\end{equation}
\end{definition}

Observe that for $S$ and $A$ satisfying \eqref{eqn:parameter_assumptions}, \Cref{lem:latent_state_subset_sizes} ensures that $\Lambda( \varepsilon, \tilde{p} )$ is nonempty.
However, for arbitrary $\varepsilon, \tilde{p}$, it will not necessarily be that every \glspl{BMDP} in $\Lambda( \varepsilon, \tilde{p} )$ satisfies \Cref{def:reachability,def:identifiability}.
The following proposition, proved in \Cref{sec:proof_of_kernel_allows_identifiability} using a packing lemma \cite{jin2020reward}, characterizes the values of $\varepsilon, \tilde{p}$ for which this \emph{is} the case.

\begin{proposition}
	\label{prop:kernel_allows_identifiability}
	For every $\eta>1$ independent of $n$,
	there exists a constant $0 < \epsilon < (1/2)(\eta-1)/(\eta+1)$,
	also independent of $n$,
	such that
	for every $(\varepsilon_0, \varepsilon_1)$ satisfying $\epsilon_{\max} = \varepsilon_0 \vee \varepsilon_1 < \epsilon$
	there is a pair $\tilde{p}^*=(\tilde{p}^*_0,\tilde{p}^*_1)$ satisfying
	\begin{equation}
		\label{eqn:tilde_p_distance_from_uniform}
		\bigl| \tilde{p}^*_0(s' \mid s) - 2/S \bigr|
		\vee
		\bigl| \tilde{p}^*_1(s' \mid s) - 2/S \bigr|
		\leq
		\frac{2\kappa}{S}
		\quad
		\mathrm{with}
		\quad
		\kappa
		\eqdef
		\frac{ \eta(1-2\epsilon_{\max})-(1+2\epsilon_{\max}) }{\eta(1-2\epsilon_{\max})+(1+2\epsilon_{\max})}
	\end{equation}
	such that:
	\begin{propenum}
		\item \label{prop:kernel_allows_identifiability--reachability}
		      every $\Phi \in \Lambda( \varepsilon, \tilde{p}^* )$ is $\eta$-reachable;

		\item \label{prop:kernel_allows_identifiability--identifiability}
		  	  there exists a constant $\mf{d}>0$ independent of $n$ and $\eta$ such that every $\Phi \in \Lambda( \varepsilon, \tilde{p}^* )$ is $\mcI_{\eta}$-reachable with $\mathcal{I}_{\eta} = \mf{d}\eta^{-1}(\eta-1)^2/(\eta+1)^3$.
	\end{propenum}
\end{proposition}

\begin{definition}[Identifiable, hard--to--learn \gls{BMDP} class]
	\label{def:identifiable_hard_to_learn_class}

	For $\eta > 1$ independent of $n$, let $\epsilon$ be as in \Cref{prop:kernel_allows_identifiability}.
	For $0 \leq \varepsilon_0 \wedge \varepsilon_1 \leq \varepsilon_0 \vee \varepsilon_1 < \epsilon$, let $\tilde{p}^*$ be given by \Cref{prop:kernel_allows_identifiability}.
	Then, let
	\begin{equation}
		\label{eqn:hard_to_learn_class_def}
		\Lambda^*( \varepsilon )
		\eqdef
		\Lambda( \varepsilon, \tilde{p}^* )
		.
	\end{equation}
\end{definition}

To prove \Cref{thm:regret_lower_bound}, it suffices to establish \eqref{eqn:regret_lower_bound} for some $\Phi \in \Lambda^*( \varepsilon )$ for appropriate values of $\varepsilon$.
We next provide a regret decomposition for \glspl{BMDP} in $\Lambda^*( \varepsilon )$ that will allow us to do so.

The transition probabilities in \eqref{eqn:latent_state_transition_kernel_macro_def} are chosen such that failing to select action $a_s^*$ when in latent state $s$ leads to an $\varepsilon_{\iota(s)}$-reduced probability of transitioning to $\mcS_1$, where it is possible to obtain a reward.
Exploiting this, \Cref{sec:proof_of_regret_decomposition_lower_bound} proves the following:

\begin{lemma}
	\label{lem:regret_decomposition_lower_bound}

	Write $s_{k,h}=f(x_{k,h})$ and let $\Phi \in \Lambda^*( \varepsilon )$.
	If $\varepsilon_0 \leq \varepsilon_1$, then
	\begin{equation}
		\label{eqn:regret_decomposition_lower_bound}
		\reg_K(\mcL;\Phi)
		\geq
		\sum_{i \in \{0, 1 \}}
		\reg_{K,i}(\mcL;\Phi)
		,
	\end{equation}
	where for $i\in\{0,1\}$,
	\begin{equation}
		\reg_{K,i}(\mcL;\Phi)
		\eqdef
		\varepsilon_{i}
		\sum_{s\in\mcS_i}
		\sum_{k=1}^K
		\sum_{h=1}^{H-1}
		\Bigl(
		\mbbE_{\Phi,\mcL}
		\bigl[
			\mathbbm{1}\{s_{k,h}=s\}
			\bigr]
		-
		\mbbE_{\Phi,\mcL}
		\bigl[
			\mathbbm{1}\{s_{k,h}=s, a_{k,h}= a_s^*\}
			\bigr]
		\Bigr)
		.
	\end{equation}
\end{lemma}

Because $\mcA$ and $\mcF$ are countable and finite, so too is $\Lambda^*(\varepsilon)$.
The arithmetic mean regret over $\Lambda^*(\varepsilon)$ therefore lower bounds the maximum regret.
With \Cref{lem:regret_decomposition_lower_bound}, this implies that for $\varepsilon_0 \leq \varepsilon_1$,
\begin{equation}
	\label{eqn:bound_supremum_by_average}
	\sup_{\Phi \in \Lambda^*(\varepsilon)}
	\reg_K(\mcL;\Phi)
	\geq
	\sum_{i\in\{0,1\}}
	R_i
	\quad
	\mathrm{with}
	\quad
	R_i
	\eqdef
	\frac{1}{|\Lambda^*(\varepsilon)|}
	\sum_{\Phi \in \Lambda^*(\varepsilon)}
	\reg_{K,i}(\mcL;\Phi)
	.
\end{equation}
We conclude \Cref{thm:regret_lower_bound}'s proof by showing that there exist $0<\varepsilon_0 \leq \varepsilon_1<\epsilon$ such that $R_0=\Omega(\sqrt{AT_K})$ and $R_1=\Omega(n)$ whenever $T_K=\Omega(n\vee A)$.
This is done next in Parts 2 and 3, respectively.

\subsubsection*{Part 2: Bounding the regret associated with learning the transition kernel}

To bound $R_0$ from below, we first upper bound the expected number of times the optimal action $a_s^*$ is played in $\mcS_0$.
\Cref{lem:change_of_measure_argument} relates the behavior of any algorithm on \glspl{BMDP} $\Phi \in \Lambda^*(\varepsilon)$ to its behavior on \glspl{BMDP} $\Phi^{(0)}$ for which all actions have the same probability of transitioning to $\mcS_1$.
It is proven in \Cref{sec:proof_of_change_of_measure_argument}, and relies on a change--of--measure argument adapted from \cite{NEURIPS2019_10a5ab2d}.

\begin{lemma}
	\label{lem:change_of_measure_argument}

	Let $\Phi \in \Lambda^*(\varepsilon)$ have optimal actions $\{a_s^*\}_{s}\in\mcA$.
	Let $\Phi^{(0)}$ be the \gls{BMDP} obtained from $\Phi$ by setting $\varepsilon_0$ to zero.
	If $\varepsilon_0\leq 1/8$, then
	\begin{align}
		 &
		\Bigl|
		\sum_{s\in\mcS_0}
		\sum_{k=1}^K
		\sum_{h=1}^{H-1}
		\Bigl(
		\mbbE_{\Phi,\mcL}
		\bigl[
			\mathbbm{1}\{s_{k,h}=s,a_{k,h}=a_s^*\}
			\bigr]
		-
		\mbbE_{\Phi^{(0)},\mcL}
		\bigl[
			\mathbbm{1}\{s_{k,h}=s,a_{k,h}=a_s^*\}
			\bigr]
		\Bigr)
		\Bigr|
		\nonumber
		\\
		 &
		\leq
		2 \varepsilon_0
		(T_K-K)
		\Bigl(
		\sum_{s\in\mcS_0}
		\sum_{k=1}^K
		\sum_{h=1}^H
		\mbbE_{\Phi^{(0)},\mcL}
		\bigl[
			\mathbbm{1}\{s_{k,h}=s,a_{k,h}=a_s^*\}
			\bigr]
		\Bigr)^{ \frac{1}{2} }
		.
		\label{eqn:change_of_measure_argument}
	\end{align}
\end{lemma}

Remark that the $\Phi^{(0)}$ in \Cref{lem:change_of_measure_argument} do not depend on the specific set of optimal actions $\{a^*_s\}_{s\in\mcS_0}$ at $\mcS_0$ of $\Phi$.
This is useful for bounding $R_0$, because it means we can move the summation over $a^*\in\mcA$ implied by the arithmetic mean over $\Phi \in \Lambda^*(\varepsilon)$ inside the expectation $\mbbE_{\Phi^{(0)},\mcL}$.
This idea is used in \Cref{sec:proof_of_sum_over_optimal_actions} to establish the following:

\begin{lemma}
	\label{lem:sum_over_optimal_actions}

	If $\varepsilon_0\leq 1/8$, then
	\begin{align}
		\frac{1}{|\Lambda^*(\varepsilon)|}
		\sum_{\Phi \in \Lambda^*(\varepsilon)}
		\sum_{s\in\mcS_0}
		\sum_{k=1}^K
		\sum_{h=1}^{H-1}
		\mbbE_{\Phi,\mcL}[\mathbbm{1}\{s_{k,h}=s,a_{k,h}=a_s^*\}]
		\leq
		\frac{T_K-K}{2A}
		+
		2\varepsilon_0(T_K-K)\sqrt{\frac{T_K}{2A}}
		.
		\label{eqn:sum_over_optimal_actions}
	\end{align}
\end{lemma}

In \Cref{sec:proof_of_transition_regret_lower_bound}, we combine \Cref{lem:regret_decomposition_lower_bound,lem:sum_over_optimal_actions} with a lower bound on the expected number of visits to $\mcS_0$ to bound $R_0$.
We then show that this implies $R_0=\Omega(\sqrt{AT_K})$ for suitable $\varepsilon_0$ and $\varepsilon_1$:
\begin{proposition}
	\label{prop:transition_regret_lower_bound}

	Let $\varepsilon_0\leq \varepsilon_1 \leq 1/8$.
	Then
	\begin{equation}
		R_0
		\geq
		(T_K-K)\varepsilon_0
		\biggl(
		\frac{1}{2}(1-2\varepsilon_1)
		-
		\frac{1}{2A}
		-
		2\varepsilon_0\sqrt{ \frac{T_K}{2A}}
		\biggr)
		.
		\label{eqn:transition_regret_lower_bound}
	\end{equation}
	In particular, this bound holds for any $0<\varepsilon_1\leq \epsilon\wedge 1/8$, with $\epsilon$ as in \Cref{prop:kernel_allows_identifiability}, when $\varepsilon_0 = \varepsilon_1 \wedge \sqrt{A/T_K}(1-2\varepsilon_1-1/A)/\sqrt{32}$.
	Under these choices, if $T_K=\Omega(A)$ and $\varepsilon_1 \not\rightarrow 0$, then $R_0=\Omega(\sqrt{AT_K})$.
\end{proposition}

\subsubsection*{Part 3: Bounding the regret associated with learning the latent structure}

We first relate $R_1$ to the number of misclassified contexts of an estimator based on $\mcL$'s selected actions.

Fix a set of optimal actions $a^*\in\mcA$, and let $\pi_k=\mcL_k(\mcD_k,r)$ denote the policy selected by a learning algorithm $\mcL$ in episode $k$.
Define an estimate decoding function $\hat{f}_{\pi_k}$ as follows:
\begin{equation}
	\label{eqn:learning_alg_to_clustering_alg}
	\hat{f}_{\pi_k}(x)
	\eqdef
	\begin{cases}
		\min
		\Bigl\{
		\argmax_{s \in \mcS_1}
		\frac{1}{H-1}
		\sum_{h\in[H-1]}
		\pi_{k,h}(a_s^*\mid x)
		\Bigr\}
		 &
		x \in f^{-1}(\mcS_1) = [n] \setminus [\floor{n/2}]
		,
		\\
		1
		 &
		x \in f^{-1}(\mcS_0) = [\floor{n/2}]
		.
	\end{cases}
\end{equation}
For $x \in f^{-1}(\mcS_1)$, \eqref{eqn:learning_alg_to_clustering_alg} assigns $x$ to the latent state whose optimal action is selected most frequently by $\pi_k$.
Because our analysis only depends on $\hat{f}_{\pi_k}|_{x \in f^{-1}(\mcS_1)}$, we arbitrarily set $\hat{f}_{\pi_k}$ to $1$ on $x \in f^{-1}(\mcS_0)$.

Recall from \Cref{def:hard_to_learn_class} that each $s\in\mcS_1$ has a distinct $a_s^*$.
Consequently, if $\hat{f}_{\pi_k}$ misclassifies a context $x$, then $\pi_k$ selects $a_{f(x)}^*$ no more than half of the time by construction.
\Cref{sec:proof_of_relation_between_regret_and_clustering_error} uses this observation to bound $\reg_{K,1}(\mcL;\Phi)$ in \eqref{eqn:regret_decomposition_lower_bound} by the expected number of misclassifications:
\begin{lemma}
	\label{lem:relation_between_regret_and_clustering_error}

	Let $\kappa$ be as in \eqref{eqn:tilde_p_distance_from_uniform} and $\tilde{p}^*$ be given by \Cref{prop:kernel_allows_identifiability}.
	Fix $a^*\in\mcA$ and let $\hat{f}_{\pi_k}$ be as in \eqref{eqn:learning_alg_to_clustering_alg}.
	Then, for $f\in\mcF$ and $\Phi=\Phi(\varepsilon,\tilde{p}^*,f,a^*)$,
	\begin{equation}
		\reg_{K,1}(\mcL;\Phi)
		\geq
		\frac{\varepsilon_1(1-\kappa)(H-1)}{2(n+S+2)}
		\sum_{k=1}^K
		\mbbE_{\Phi,\mcL}[|\mcE(\hat{f}_{\pi_k}; f, \mcS_1)|]
		.
	\end{equation}
\end{lemma}

To bound $R_1$ in \eqref{eqn:bound_supremum_by_average} using \Cref{lem:relation_between_regret_and_clustering_error}, we will bound the expected number of misclassifications any estimate decoding function must make, averaged over the \glspl{BMDP} in the hard--to--learn class.
This is the content of \Cref{prop:lower_bound_misclassification_rate} below, proven in the remainder of this section.

Recall $\mcC$'s definition in \Cref{sec:preliminaries_latent_state_estimation}.
Let $\hat{f}_{\mcC,k}\eqdef \mcC_{k}(\mcD_{k},r)$ denote the output of $\mcC$ after $k-1$ episodes.
Following \cite{Jedra:2022,Sanders:2020}, we start by considering pairs of \glspl{BMDP} that differ only at a single context $x$.
We then show that the probability of misclassifying $x$ cannot be small in both \glspl{BMDP} under the same algorithm $\mcC$.
This is done in \Cref{sec:proof_of_local_clustering_error_lower_bound}, where we use a change--of--measure argument \cite{NEURIPS2019_10a5ab2d}.
Specifically, we relate the misclassification probabilities in the two \glspl{BMDP} to the expected log-likelihood ratio of the observed sequence of contexts.
The latter is then bounded using results from \cite{Jedra:2022}, yielding the following:
\begin{lemma}
	\label{lem:local_clustering_error_lower_bound}

	There exists a constant $C>0$ depending only on $\eta$ such that
	for all sufficiently large $n$ the following holds:
	for any \glspl{BMDP} $\Phi^{(1)},\Phi^{(2)} \in \Lambda^*(\varepsilon)$ whose decoding functions $f^{(1)}$ and $f^{(2)}$ differ at only a single $x \in[n]$,
	any latent state estimation algorithm $\mcC$,
	and any number of episodes $k\in\mbbN_+$,
	\begin{align}
		\mbbP_{\Phi^{(1)},\mcL}[\hat{f}_{\mcC,k}(x)\neq f^{(1)}(x)]
		+
		\mbbP_{\Phi^{(2)},\mcL}[\hat{f}_{\mcC,k}(x) \neq f^{(2)}(x)]
		\geq
		\frac{1}{2}
		\exp
		\Bigl(
		-
		C\frac{T_{k}}{n}
		\Bigr)
		.
		\label{eqn:local_clustering_error_lower_bound}
	\end{align}
\end{lemma}

To prepare for the application of \Cref{lem:local_clustering_error_lower_bound} in bounding the arithmetic mean of the expected misclassification count, we express the latter in terms of individual misclassification probabilities.
Fix $\varepsilon$, $\tilde{p}^*$, $a^*$ and $\mcL$, and adopt the short-hand notations
\begin{equation}
	\label{eqn:short-hand_notation_for_measures}
	\mbbP_{f}
	\eqdef
	\mbbP_{\Phi(\varepsilon,\tilde{p}^*,f,a^*),\mcL}
	\quad
	\mathrm{and}
	\quad
	\mbbE_f
	\eqdef
	\mbbE_{\Phi(\varepsilon,\tilde{p}^*,f,a^*),\mcL}
	.
\end{equation}
Let $\hat{\gamma}^*_f$ be a permutation as in \eqref{eqn:misclassdef}.
Observe then that
\begin{equation}
    \begin{split}
         &
        \frac{1}{|\mcF|}
        \sum_{f\in\mcF}
        \mbbE_{f}[|\mcE(\hat{f}_{\mcC,k};f,\mcS_1)|]
        =
        \frac{1}{|\mcF|}
        \sum_{x \in[n]\setminus[\floor{n/2}]}
        \sum_{f\in\mcF}
        \mbbP_{f}[
            \hat{f}_{\mcC,k}(x) \neq \hat{\gamma}^*_f \circ f(x)
        ]
        .
    \end{split}
    \label{eqn:decomposition_of_ex_nr_of_misclass}
\end{equation}
We now introduce constructions that allow us to reorganize the right-hand side of \eqref{eqn:decomposition_of_ex_nr_of_misclass} in a form suitable for applying \Cref{lem:local_clustering_error_lower_bound} to.

Recall \eqref{eqn:size_distribution_of_latent_states}.
Observe that $\mcS_1^0(f)$ is nonempty by \Cref{lem:latent_state_subset_sizes}.
Let $x \in [n]\setminus [\floor{n/2}]$ be a context in $f^{-1}(\mcS_1)$.
Following the proof of \cite[Lemma 5.1]{zhang2016minimax}, we define $\mcP_x:\mcF\rightarrow [S]^{[n]}$ by
\begin{equation}
	\mcP_x f(y)
	=
	\begin{cases}
		f(y)
		 &
		\textnormal{if } x \neq y,                          \\
		\min
		\{ s\in \mcS_1^0(f): f(x)<s\}
		 &
		\textnormal{if } x = y, f(x) \leq \sup \mcS_1^0(f), \\
		\min
		\mcS_1^0(f)
		 &
		\textnormal{if } x = y, f(x) > \sup \mcS_1^0(f).    \\
	\end{cases}
	\label{eqn:def_of_Px}
\end{equation}
In other words, $\mcP_x f$ is obtained from $f$ by moving the context $x$ to a new latent state in $\mcS_1^0(f)$ according to the rule in \eqref{eqn:def_of_Px}.
Also, let
\begin{equation}
	\label{eqn:def_of_mcFx}
	\mcF_x
	\eqdef
	\{f\in\mcF: f(x)\in \mcS_1^{+1}(f)\}
	.
\end{equation}
\Cref{lem:Px_is_a_permutation} establishes the required properties of $\mcF_x$ and $\mcP_x$.
It is proven in \Cref{sec:proof_of_Px_is_a_permutation}, following \cite{zhang2016minimax}.

\begin{lemma}
	\label{lem:Px_is_a_permutation}

	$|\mcF_x|/|\mcF|\geq 1/18$.
	Moreover, the map $\mcP_x$, restricted to $\mcF_x$, is a permutation of $\mcF_x$ satisfying $\mcP_xf\neq f$ for all $f\in\mcF_x$.
\end{lemma}

Given \Cref{lem:Px_is_a_permutation}, we can lower bound \eqref{eqn:decomposition_of_ex_nr_of_misclass} as follows.
We first use the bound on $|\mcF_x|$ to restrict the summation over $f\in\mcF$ to $f\in\mcF_x$ at the cost of a multiplicative factor $1/18$.
By subsequently using the permutation $f\mapsto\mcP_xf$ to reindex the summation over $f\in\mcF_x$ we obtain that
\begin{align}
	 &
	\frac{1}{|\mcF|}
	\sum_{f\in\mcF}
	\mbbE_{f}[|\mcE(\hat{f}_{\mcC,k};f,\mcS_1)|]
	\nonumber
	\geq
	\sum_{x \in[n]\setminus[\floor{n/2}]}
	\frac{1}{18|\mcF_x|}\sum_{f\in\mcF_x}\mbbP_{f}[\hat{f}_{\mcC,k}(x)\neq \hat{\gamma}^*_f \circ f (x)]
	\nonumber \\
	=
	 &
	\sum_{x \in[n]\setminus[\floor{n/2}]}
	\frac{1}{36|\mcF_x|}\sum_{f\in\mcF_x}
	\Big(
	\mbbP_{f}[\hat{f}_{\mcC,k}(x)\neq \hat{\gamma}^*_f\circ f(x)]
	+
	\mbbP_{\mcP_xf}[\hat{f}_{\mcC,k}(x)\neq \hat{\gamma}^*_{\mcP_xf}\circ \mcP_xf(x)]
	\Big)
	.
	\label{eqn:application_of_permutation}
\end{align}
Since $f$ and $\mcP_x$ differ only at context $x$ by construction, the right-hand side of \eqref{eqn:application_of_permutation} has nearly the right form to apply \Cref{lem:local_clustering_error_lower_bound}.
However, the permutations $\hat{\gamma}^*_f, \hat{\gamma}^*_{\mcP_xf}$ still prevent this.
We address this technical issue in \Cref{sec:proof_of_reduction_to_local_misclassification_probability} and conclude that:

\begin{proposition}
	\label{prop:lower_bound_misclassification_rate}

    Let $C>0$ be the constant from \Cref{lem:local_clustering_error_lower_bound}.
    For all sufficiently large $n$, it holds for any parameters $\varepsilon$ and $\tilde{p}^*$ as in \Cref{def:identifiable_hard_to_learn_class}, any optimal actions $a^*\in\mcA$, any latent state estimation algorithm $\mcC$, and any number of episodes $k\in\mbbN_+$ that
    \begin{align}
        \label{eqn:reduction_to_local_misclassification_probability}
        \frac{1}{|\mcF|}
        \sum_{f\in\mcF}
        \mbbE_{f}[
            |\mcE(\hat{f}_{\mcC,k};f,\mcS_1)|
        ]
        \geq
        \frac{n}{144}
        \exp
        \Bigl(
        -
        C\frac{T_{k}}{n}
        \Bigr)
        .
    \end{align}
\end{proposition}

In \Cref{sec:proof_of_decoding_function_regret_lower_bound}, we then combine \Cref{lem:relation_between_regret_and_clustering_error} with \Cref{prop:lower_bound_misclassification_rate} to bound $R_1$.

\begin{proposition}
	\label{prop:decoding_function_regret_lower_bound}

	Let $\kappa$ be as in \eqref{eqn:tilde_p_distance_from_uniform} and $C>0$ as in \Cref{lem:local_clustering_error_lower_bound}.
	For all sufficiently large $n$,
	\begin{equation}
		R_1
		\geq
        \frac{\varepsilon_1(1-\kappa)}{288C}
        \frac{1-e^{-C(T_K/n)}}{1+\frac{S+2}{n}}
        \biggl(1-\frac{1}{H}\biggr)
        n
		.
		\label{eqn:decoding_function_regret_lower_bound}
	\end{equation}
	If additionally $T_K=\Omega(n)$ and $\varepsilon_1\not\rightarrow 0$, then $R_1=\Omega(n)$.
\end{proposition}

Finally---adopt the assumptions of \Cref{thm:regret_lower_bound} and take $\varepsilon_1 = (\epsilon \wedge 1/8)/2$ and $\varepsilon_0 = \varepsilon_1\wedge \sqrt{A/T_K}(1-2\varepsilon_1-1/A)/\sqrt{32}$.
Since $\epsilon$ is independent of $n$ by \Cref{prop:kernel_allows_identifiability}, these values satisfy the conditions of \Cref{prop:transition_regret_lower_bound,prop:decoding_function_regret_lower_bound}.
By the reduction after \eqref{eqn:bound_supremum_by_average}, these imply \Cref{thm:regret_lower_bound}.

\section{Algorithm description}
\label{sec:algorithm}

In this section we present a two--phase algorithm for the learning problem outlined in the \cref{sec:models}.

Recall \eqref{eqn:historydef}.
For $k\in\mbbN_+$, $a\in[A]$, and $\mcX_1,\mcX_2\subseteq[n]$, define the count
\begin{equation}
	\label{eqn:countsdef}
	\hat{N}_{k}(\mcX_1,a,\mcX_2)\eqdef \sum_{l=1}^{k-1}\sum_{h=1}^H\mbb{1}\{x_{l,h}\in\mcX_1,a_{l,h}=a,x_{l,h+1}\in\mcX_2\}
	,
\end{equation}
of the number of transitions from $\mcX_1$ to $\mcX_2$ while choosing action $a$ prior to episode $k$.
Define also
\begin{equation}
	\label{eqn:countsdef_overload}
	\hat{N}_{k}(\mcX_1,a)\eqdef \hat{N}_{k}(\mcX_1,a,[n])\quad \mathrm{and} \quad
	\hat{N}_{k}^{\mathrm{to}}(\mcX_2)\eqdef \sum_{a}\hat{N}_{k}([n],a,\mcX_2).
\end{equation}
These count the number of times $a$ was chosen while observing $\mcX_1$, and the number of transitions to $\mcX_2$ under any action, respectively.

\cref{alg:BUCBVIouter,alg:BUCBVI,alg:UpdateLatentStates} contain the basic components of the algorithm, which we discuss next.

\subsection{Outer loop}

\cref{alg:BUCBVIouter} describes the outer loop of the algorithm.

Phase 1 always selects the uniform policy.
It lasts until at most $\Theta^{\clust}$ transitions have been observed.
These observations are used to estimate the decoding function with the UpdateLatentStates subroutine in \Cref{alg:UpdateLatentStates}.
Phase 2 selects the optimal policy for the Q-function computed by the ComputeQValues subroutine in \Cref{alg:BUCBVI}.

\begin{algorithm}
	\textbf{Global Input Variables:} Threshold $\Theta^{\clust}$, reward function $r$, parameters $n,S,A,H$
	\begin{algorithmic}[1]
		\State  \textbf{Start of phase 1:}
		\For{$k>0:T_k\leq \Theta^{\clust}$}
		\State $\pi_{k,h}(\, \cdot \mid  x) \gets \text{Unif}([A])$ for $x\in[n]$, $h\in[H]$ \label{ln:uniform_policy}
		\State Let $x_{k,1}\sim \mu$, take $a_{k,h}\sim \pi_{k,h}(\, \cdot \mid  x_{k,h})$ and observe $x_{k,h+1}\sim P(\, \cdot \mid  x_{k,h},a_{k,h})$ for $h\in[H]$
		\EndFor
		\State Update $\hat{f}^{\alg}\gets$ UpdateLatentStates($\mcD_{k+1}$)
		\Comment{Estimate latent states using \Cref{alg:UpdateLatentStates}}
		\State  \textbf{Start of phase 2:}
		\For{$k>0:T_k>\Theta^{\clust}$}
		\State $\bar{Q}^{\pi_k}\gets \text{ComputeQValues}(\mcD_{k},\hat{f}^{\alg})$ \Comment{Compute Q-values using \cref{alg:BUCBVI}}
		\State $\pi_{k,h}(\, \cdot \mid  x)\gets \textnormal{Unif}\bigl( \argmax_{a}\bar{Q}^{\pi_k}_{h}(x,a)\bigr)$ for $x\in[n]$, $h\in[H]$
		\State Let $x_{k,1}\sim \mu$, take $a_{k,h}\sim \pi_{k,h}(\, \cdot \mid  x_{k,h})$ and observe $x_{k,h+1}\sim P(\, \cdot \mid  x_{k,h},a_{k,h})$ for $h\in[H]$
		\EndFor
	\end{algorithmic}
	\caption{Pseudocode for the outer loop.
	}\label{alg:BUCBVIouter}
\end{algorithm}

\subsection{Latent state estimation}

\Cref{alg:UpdateLatentStates} describes the UpdateLatentStates routine that \Cref{alg:BUCBVIouter} uses to estimate the decoding function.
It consists of two steps.
First, an initial estimate is obtained using a spectral method.
This is subsequently refined in an iterative improvement step.
These steps correspond to \cite[Algorithm 1]{Jedra:2022} and \cite[Algorithm 2]{Jedra:2022}, respectively.
\Cref{sec:clusteringalg} reproduces their pseudocode in our notation.

\begin{algorithm}
	\begin{algorithmic}[1]
		\Procedure{UpdateLatentStates}{$\mcD_{k+1}$}
		\State $\hat{f}_0\gets \text{SpectralClustering}(\mcD_{k+1})$ \Comment{Approximate  latent states using \cite[Algorithm 1]{Jedra:2022}}
		\State $\hat{f}\gets \text{ImproveClusters}(\mcD_{k+1},\hat{f}_0)$ \Comment{Improved latent states using \cite[Algorithm 2]{Jedra:2022}}
		\State \Return Estimate decoding function $\hat{f}$
		\EndProcedure
	\end{algorithmic}
	\caption{Pseudocode for estimating the decoding function.}
	\label{alg:UpdateLatentStates}
\end{algorithm}

\subsection{Q-values routine}
\label{sec:Q-values_routine}

\Cref{alg:BUCBVI} computes upper confidence bounds $\bar{Q}^{\pi_k}$ on the optimal Q-function $Q^{\pi^*}$, defined in \eqref{eqn:valuefuncdef}.
It uses an estimate transition kernel to estimate $Q^{\pi^*}$ with the Bellman-type recursion \eqref{eqn:bellman_equation} initialized from $Q^{\pi^*}_H(x,a)=r_H(x,a)$.
An exploration bonus is added to the reward $r_h$ at each step to compensate for uncertainty about the expected next--state value.
This ensures that the resulting Q-value estimate is indeed an upper confidence bound.
The difference with \cite{UCRLVI} is that we use modified transition kernel estimates and exploration bonuses that exploit the structure of \glspl{BMDP}.

\begin{algorithm}
	\begin{algorithmic}[1]

		\Procedure{ComputeQValues}{$\mathcal{D}_k,\hat{f}$}\Comment{Backward value iteration with bonuses}

		\State Calculate estimate transition kernel $\hat{P}_k(y\mid x,a)$ from \eqref{eqn:estimator_P_def}
		\State Calculate exploration bonus $\hat{b}_k(s,a)$ from \eqref{eqn:bonus_def}
		\State Initialize $\bar{Q}^{\pi_k}_{H}(x,a)\gets r_H(x,a)$, $\bar{V}^{\pi_k}_{H}(x)\gets \max_{a}\bar{Q}^{\pi_k}_{H}(x,a)$ for $x\in[n],a\in[A]$
		\For{$h=H-1,H-2,\ldots$,1}
		\State $\bar{Q}^{\pi_k}_{h}(x,a)\gets 1\wedge (r_h(x,a)+\hat{b}_{k}(\hat{f}(x),a)) + \sum_{y}\hat{P}_k(y\mid x,a)\bar{V}_{h+1}^{\pi_k}(y)$ for $x\in[n]$, $a\in[A]$
		\State $\bar{V}^{\pi_k}_{h}(x)\gets \max_{a}\bar{Q}^{\pi_k}_{h}(x,a)$ for $x\in[n]$
		\EndFor
		\State \Return Q-values $\bar{Q}^{\pi_k}$
		\EndProcedure
	\end{algorithmic}
	\caption{Pseudocode for computing the optimistic Q-function.}
	\label{alg:BUCBVI}
\end{algorithm}

Regarding the transition kernel, we first compute maximum likelihood estimators for the latent state transition kernel and emission probabilities on the learned latent states.
These are given by
\begin{equation}
	\label{eqn:estimators_p_q_def}
	\hat{p}_k(s'\mid s,a)
	\eqdef
	\frac{\hat{N}_k((\hat{f}^{\alg})^{-1}(s),a,(\hat{f}^{\alg})^{-1}(s'))}{1\vee \hat{N}_k((\hat{f}^{\alg})^{-1}(s),a)}
	,
	\quad
	\mathrm{and}
	\quad
	\hat{q}_k(y\mid s)
	\eqdef
	\frac{\mathbbm{1}\{y\in (\hat{f}^{\alg})^{-1}(s)\}\hat{N}^{\mathrm{to}}_k(y)}{1\vee \hat{N}^{\mathrm{to}}_k((\hat{f}^{\alg})^{-1}(s))}
	,
\end{equation}
corresponding to the fraction of visits to $(s,a)$ that transition to $s'$ and the fraction of transitions to $s$ for which $y$ is observed, respectively.
The \gls{BMDP} transition kernel is then estimated by
\begin{equation}
	\label{eqn:estimator_P_def}
	\hat{P}_k(y\mid x,a)
	\eqdef
	\hat{p}_k(\hat{f}^{\alg}(y)\mid \hat{f}^{\alg}(x),a)\hat{q}_k(y\mid \hat{f}^{\alg}(y)).
\end{equation}

Regarding the exploration bonuses, we define these as follows:
\begin{equation}
	\label{eqn:bonus_def}
	\hat{b}_k(s,a)
	\eqdef
	\sqrt{\frac{H^2\ln(2HSAT_k^2)}{1\vee \hat{N}_k((\hat{f}^{\alg})^{-1}(s),a)}}
	+
	\sum_{s'}\hat{p}_k(s'\mid s,a)\sqrt{\frac{H^2\ln(2HST_k^2)}{1\vee \hat{N}_k^{\mathrm{to}}((\hat{f}^{\alg})^{-1}(s'))}}
	.
\end{equation}
Compared to \cite{UCRLVI}, the bonuses $\hat{b}_k$ depend on the visitation counts of estimated latent states rather than individual contexts.
They therefore decrease more quickly over time, leading to less exploration and lower regret.
We will show that the estimated value function $\bar{V}^{\pi_k}_{h+1}$ nonetheless is an upper confidence bound for the optimal value function $V^{\pi^*}_{h+1}$.
This is possible because data aggregation within latent states reduces the variance of the estimate transition kernel used to compute $\bar{V}^{\pi_k}_{h+1}$.

Note, however, that this relies on the latent structure being inferred correctly by \Cref{alg:UpdateLatentStates}.
In general, uncertainty about this structure cannot fully be captured by bonuses assigned to latent state--action pairs \cite{DBLP:journals/corr/abs-2110-04652}.
Instead, we focus on the class of $\eta$-reachable and $\mcI$-identifiable \glspl{BMDP} for which the latent structure can be learned efficiently through random exploration.
This ensures the exploration time required for latent state estimation is bounded, enabling our two--phase approach.

\section{Regret analysis of \texorpdfstring{\Cref{alg:BUCBVIouter}}{the algorithm}}
\label{sec:regret_analysis}

This section proves \Cref{thm:BUCBbd}.
\Cref{sec:proof_outline} outlines the proof and \Cref{sec:proof_of_BUCBVI_regret_bound} contains the details.

Throughout this section, we fix $\Phi$ and write $\mbbP=\mbbP_{\Phi,\mcL}$ and $\mbbE=\mbbE_{\Phi,\mcL}$ with $\mcL$ given by \Cref{alg:BUCBVIouter}.

\subsection{Proof outline}
\label{sec:proof_outline}

The regret analysis here follows a strategy similar to those of other optimism-based algorithms \cite{10.1145/1102351.1102459, szita2010model, munos2014bandits, UCRLVI}.
It, however, features modifications to exploit the latent structure of the \gls{BMDP}.

The proof starts by decomposing the regret in \eqref{eqn:regret_def} into different terms.
These terms are then bounded separately.
The regret bound in \eqref{eqn:regret_bound} is an aggregate of these bounds.
Most of the proof is devoted to bounding the regret during this portion, which consists of three parts:

\begin{enumerate}
	\item[Part 1.]
	      We define a sequence of events $\Omega_k$ on which the estimate transition kernel $\hat{P}_k$ from \eqref{eqn:estimator_P_def} is close to the true kernel $P$.
	      Notably, $\Omega_k$ is chosen such that the estimated decoding function $\hat{f}^{\alg}$ equals its true value $f$ on $\Omega_k$.

	      We also bound the probability that an $\Omega_k$ fails.
	      In particular, by refining proofs of \cite{Jedra:2022}, we show that for any $c>0$, $\mbbP[\hat{f}^{\alg}\neq f]=O(1/n^c)$ if \Cref{alg:BUCBVIouter} explores sufficiently.
	      This gives the $\tilde{O}(T_K/n^c)$ term in \eqref{eqn:regret_bound}.
		  
		  Note that Part 1 is the only place we use \Cref{def:reachability,def:identifiability}.
\end{enumerate}

In Parts 2 and 3 next we consider the regret when $\cap_{k\in[K]}\Omega_k$ is true.

\begin{enumerate}
	\item[Part 2.]
	      We show that on $\Omega_k$, $\bar{V}^{\pi_k}_{h}(x)$ in \Cref{alg:BUCBVI} is an upper confidence bound for the optimal value $V^{\pi^*}_h(x)$, i.e., that on $\Omega_k$, $V^{\pi^*}_h(x)\leq \bar{V}^{\pi_k}_{h}(x)$ for all $h\in[H]$ and $x\in[n]$.

	      This implies that the regret per episode is bounded on $\Omega_k$ by $\bar{V}^{\pi_k}_1(x_{k,1})-V^{\pi_k}_1(x_{k,1})$.

	\item[Part 3.]
	      We decompose
		  $\bar{V}^{\pi_k}_1(x_{k,1})-V^{\pi_k}_1(x_{k,1})$ into three terms.
		  Of these, the most challenging is
		  \begin{equation}
			\sum_{h=1}^{H-1}\mbbE\biggl[
				\sum_{y}
				(
					\hat{P}_k(y\mid x_{k,h},a_{k,h})-P(y\mid x_{k,h},a_{k,h})
				)
				(
					\bar{V}^{\pi_k}_{h+1}(y)-V^{\pi^*}_{h+1}(y)
				)
				\mid \mcD_k, x_{k,1}
			\biggr]
			.
			\label{eqn:outline_simulation_lemma_bound}
		  \end{equation}
	      This is due to the dependence between $\hat{P}_k$ and $\bar{V}^{\pi_k}_{h+1}$.
	      Following \cite{UCRLVI}, we relate this term to the exploration bonus $\hat{b}_k$.
		  The resulting corrections contribute the $\tilde{O}(nSH^3)$ term to \eqref{eqn:regret_bound}.

	      Afterwards, the remaining sum of exploration bonuses is bounded using a pigeonhole principle argument.
	      Because the bonuses are independent of $n$ and depend only on the latent state of $x_{k,h}$, this yields the $n$-independent $\tilde{O}(H\sqrt{SAT_K})$ term in \eqref{eqn:regret_bound}.
\end{enumerate}

\subsection{Proof of \texorpdfstring{\Cref{thm:BUCBbd}}{the regret upper bound}}
\label{sec:proof_of_BUCBVI_regret_bound}

\Cref{alg:BUCBVIouter} accumulates regret in two phases.
We bound the regret separately for each phase.

The first phase includes all episodes $k\in[K]$ for which $T_k=kH\leq \Theta^{\clust}$.
Note that, for any given episode, $V^{\pi^*}_1(x_{k,1})-V_1^{\pi_k}(x_{k,1})\leq H$ because the rewards are in $[0,1]$.
It follows that the expected regret during the first phase is bounded by $\Theta^{\clust}$, by construction of \Cref{alg:BUCBVIouter}.

The second phase is treated through decomposition as follows:
\begin{align}
	\text{Reg}_K(\mcL;\Phi)
	\leq
	\,
	&
	\Theta^{\clust}
	+
	H\sum_{\substack{k\in[K]:\\T_k>\Theta^{\clust}}}\mbbP[(\Omega_k)^{\mathrm{c}}]
	+
	\sum_{\substack{k\in[K]:\\T_k>\Theta^{\clust}}}\mbbE\bigl[\bigl( V^{\pi^*}_1(x_{k,1})-\bar{V}^{\pi_k}_1(x_{k,1})\bigr) \mbb{1}\{\Omega_k\}\bigr]
	\nonumber
	\\
	&
	+
	\sum_{\substack{k\in[K]:\\T_k>\Theta^{\clust}}}\mbbE\bigl[\bigl( \bar{V}^{\pi_k}_1(x_{k,1})-V_1^{\pi_k}(x_{k,1})\bigr) \mbb{1}\{\Omega_k\}\bigr]
	\nonumber
	\\
	\defeq
	\,
	&
	\Theta^{\clust}
	+
	R_1
	+
	R_2
	+
	R_3
	.
	\label{eqn:regret_decomp}
\end{align}
For $k\in[K]$, we will consider
\begin{align}
	&
	\Omega_k
	\eqdef
	\Omega^f\cap \Omega_k^1\cap\Omega_k^2\cap\Omega_k^3\cap\Omega_k^4
	\quad
	\mathrm{with}
	\quad
	\Omega^{f}
	\eqdef
	\{ \hat{f}^{\alg}=f \}
	,
	\nonumber
	\\
	&
	\Omega_k^{1}
	\eqdef
	\bigcap_{s,h}\biggl\{
	|\langle \hat{q}_k(\,\cdot\mid s)-q(\,\cdot\mid s), V^{\pi^*}_{h+1}\rangle|
	\leq
	\sqrt{\frac{H^2 \ln(2HST_k^2)}{1\vee \hat{N}^{\mathrm{to}}_k(f^{-1}(s))}}
	\biggr\}
	,
	\nonumber
	\\
	&
	\Omega_k^{2}
	\eqdef
	\bigcap_{s,a,h}
	\biggl\{
	\biggl|\sum_{s'} (\hat{p}_k(s'\mid s,a)-p(s'\mid s,a)) \langle q(\,\cdot\mid s'), V^{\pi^*}_{h+1}\rangle \biggr|
	\leq
	\sqrt{\frac{H^2 \ln(2HSAT_k^2)}{1\vee \hat{N}_k(f^{-1}(s),a)}}
	\biggr\}
	,
	\label{eqn:highprobeventcomponentdef}
	\\
	&
	\Omega_k^{3}
	\eqdef
	\bigcap_{s,y}
	\biggl\{
	|\hat{q}_k(y\mid s)-q(y\mid s)|
	\leq
	\sqrt{\frac{q(y\mid s) \ln(2nST_k^2)}{1\vee\hat{N}^{\mathrm{to}}_k(f^{-1}(s))}}
	+
	\frac{2\ln(2nST_k^2)}{3(1\vee\hat{N}^{\mathrm{to}}_k(f^{-1}(s)))}
	\biggr\}
	,
	\nonumber
	\\
	&
	\Omega_k^{4}
	\eqdef
	\bigcap_{s,s',a}
	\biggl\{
	|\hat{p}_k(s'\mid s,a)-p(s'\mid s,a)|
	\leq
	\sqrt{\frac{p(s'\mid s,a) \ln(2S^2AT_k^2)}{1\vee\hat{N}_k(f^{-1}(s),a)}}
	+
	\frac{2\ln(2S^2AT_k^2)}{3(1\vee\hat{N}_k(f^{-1}(s),a))}
	\biggr\}
	.
	\nonumber
\end{align}
Recall here that $\hat{f}^{\alg}$ denotes the output of \Cref{alg:UpdateLatentStates} using at least $\Theta^{\clust}$ transitions gathered under the uniform policy.
Recall also $\hat{p}_k$ and $\hat{q}_k$ defined in \eqref{eqn:estimators_p_q_def}, and $\hat{N}_k$ defined in \eqref{eqn:countsdef} and \eqref{eqn:countsdef_overload}.

Next, we bound $R_1$, $R_2$, and $R_3$ in \eqref{eqn:regret_decomp}.
This corresponds to Parts 1, 2, and 3 from \Cref{sec:proof_outline}, respectively.

\subsubsection*{Part 1: Bounding the failure probability}

We begin by bounding $\mbbP[\Omega^f]$ by refining \cite{Jedra:2022}'s analysis.

The first step is to address the gap between our notion of $\mcI$-identifiability and the corresponding condition in \cite{Jedra:2022}.
As explained in \Cref{sec:comparison_to_other_identifiability_conditions}, \cite{Jedra:2022} requires an asymptotic bound on $\lim_{\zeta\rightarrow\infty}\min_x I_{\zeta}(x;\Phi)$, rather than $I_{\eta}(x;\Phi)$.
In \Cref{sec:equivalence_of_identifiability_assumptions}, we show that for $\eta$-reachable \glspl{BMDP}, these two conditions are equivalent:
\begin{lemma}
	\label{lem:equivalence-of-assumptions}
	Assume that $\Phi$ is $\eta$-reachable and $\mcI$-identifiable.
	Then there exists a constant $\mcI'>0$ that is independent of $n$ such that for all sufficiently large $n$, $\lim_{\zeta\rightarrow\infty}\min_{x}I_{\zeta}(x;\Phi)\geq \mcI'$.
\end{lemma}

Using \Cref{lem:equivalence-of-assumptions}, we can deduce the following corollary of \cite[Theorem 3]{Jedra:2022}:
\begin{corollary}
	If $S$, $A$ and $p$ are independent of $n$, $\mu=\mathrm{Unif}([n])$, and $\Phi$  is $\eta$-reachable and $\mcI$-identifiable, then there exist constants $C,\mcI'>0$ independent of $n$ such that: if $\Theta^{\clust}=\omega(n)$, then \Cref{alg:UpdateLatentStates} misclassifies at most $O(n\exp(-C (\Theta^{\clust}/n)\mcI'))$ contexts with probability $1-o(1)$.
	
	Consequently, under these assumptions, $\mbbP[\Omega^f]=1-o(1)$ if $\Theta^{\clust}=\omega(n\ln n)$.
\end{corollary}
This result is, however, not enough for our purposes.
We therefore refine it in three ways in \Cref{sec:proof_of_clustering_bound_short}.
First, we make the $1-o(1)$ probability bound explicit by inspecting the proofs in \cite{Jedra:2022} and keeping track of which events imply $\Omega^f$.
Second, we sharpen the resulting probability bound from $1-O(1/n)$ by showing that it can be $1-O(1/n^c)$ for any $c>0$.
Third, we allow $S$, $A$, and $p$ to depend on $n$ and account for this in the asymptotics.
This gives the following result:

\begin{proposition}
	\label{prop:clustering_bound_short}
	Assume $\Phi$ is $\eta$-reachable and $\mcI$-identifiable, and that $\Theta^{\clust}=\omega(nA(S^3\vee \ln n))$, $\mu=\mathrm{Unif}([n])$, $A=n^{O(1)}$, and $H=O(n)$.
	Then, for any $c>0$, $\mbbP[\Omega^f] = 1 - O(1/n^c)$.
\end{proposition}

Given $\Omega^f$, it is straightforward to show that $\Omega_k^1,\ldots,\Omega_k^4$ hold with probability $1-O(1/T_k)$ using Hoeffding and Bernstein's inequalities.
Together with \Cref{prop:clustering_bound_short}, this allows us to prove the following bound on $R_1$ in \cref{sec:proofstep1}:

\begin{proposition}
	\label{prop:highprobevent}
	Adopt the assumptions of \Cref{prop:clustering_bound_short}, and let $k\in[K]$ be such that $T_k=kH>\Theta^{\clust}$.
	Then, for any $c>0$,
	\begin{equation}
		R_1
		=
		H\sum_{k\in[K]:T_k>\Theta^{\clust}}\mbbP[(\Omega_k)^\mathrm{c}]
		=
		O\biggl( \ln K+ \frac{T_K}{n^c} \biggr)
		=
		\tilde{O}\biggl( 1 + \frac{T_K}{n^c} \biggr).
	\end{equation}
\end{proposition}

\subsubsection*{Part 2: Proving optimism}

We next show that $R_2\leq 0$ by proving that on $\Omega_k$, $\bar{V}^{\pi_k}_h(x)-V^{\pi^*}_h(x)\geq 0$ for all $h$ and $x$.

Note from \Cref{alg:BUCBVI} that $\bar{V}^{\pi_k}_h(x)$ equals the optimal value function of a \gls{BMDP} $\bar{\Phi}_k$ with kernel $\hat{P}_k$ and reward function $1\wedge (r_h(x,a)+\mathbbm{1}\{h < H\}\hat{b}_k(\hat{f}^{\alg}(x),a))$ by construction.
Here, $\hat{b}_k$ is given by \eqref{eqn:bonus_def}.
This characterization gives the following decomposition, proven in \Cref{sec:proof_of_optimism_simulation}:
\begin{lemma}
	\label{lem:optimism_simulation}
	For $x\in[n]$, $h\in[H-1]$, and $k\in[K]$,
	\begin{equation}
		\label{eqn:optimism_simulation}
		\begin{split}
			& 
			\bar{V}^{\pi_k}_h(x) - V^{\pi^*}_{h}(x)
			\\
			&
			\geq
			\sum_{h'=h}^{H-1}\mbbE_{\bar{\Phi}_k,\pi^*}\bigl[
			\hat{b}_k(\hat{f}^{\alg}(x_{h'}),a_{h'})+\bigl\langle \hat{P}_k(\,\cdot\mid  x_{h'},a_{h'})-P(\,\cdot\mid  x_{h'},a_{h'}), V^{\pi^*}_{h'+1} \bigr\rangle
			\mid
			x_h = x
			\bigr]
			.
		\end{split}
	\end{equation}
	Here, the expectation is with respect to the sequence $(x_1,a_1,\ldots,a_H,x_{H+1})$.
\end{lemma}

\Cref{lem:optimism_simulation} implies that $\bar{V}^{\pi_k}_1(x)- V^{\pi^*}_1(x)\geq0$ whenever $|\langle \hat{P}_k(x,a)-P(x,a), V^{\pi^*} _{h+1}\rangle| \leq \allowbreak \hat{b}_k(\hat{f}^{\alg}(x), a)$ for all $x$, $a$, and $h$.
The following lemma verifies this condition on $\Omega_k$.
It is proven in \Cref{sec:proof_of_estimation_error_conf_bd}.

\begin{lemma}
	\label{lem:esterrorconfbd}
	For $x\in[n]$, $a\in[A]$, $h\in[H]$, and $k\in[K]$, on $\Omega_k$,
	\begin{equation}
		\label{eqn:estimation_error_confidence_bound}
		|
		\langle
		\hat{P}_k(\,\cdot\mid x,a)-P(\,\cdot\mid x,a),V^{\pi^*}_{h+1}
		\rangle
		|
		\leq
		\hat{b}_k(f(x),a)
		.
	\end{equation}
\end{lemma}

Combining \Cref{lem:optimism_simulation,lem:esterrorconfbd} and using that $\hat{f}^{\alg}=f$ on $\Omega^f$ proves that $\bar{V}^{\pi_k}_h(x)-V^{\pi^*}_h(x)\geq 0$ on $\Omega_k$.
This in turn implies that $R_2\leq 0$ on $\Omega_k$, establishing the following:
\begin{proposition}
	\label{prop:optimism}
	For $k\in[K]$, $\Omega_k\subset \cap_{x,h}\{\bar{V}^{\pi_k}_h(x)\geq V^{\pi^*}_h(x)\}$.
	Therefore,
	\begin{equation}
		R_2
		=
		\sum_{k\in[K]:T_k>\Theta^{\clust}}\mbbE[(V^{\pi^*}_1(x_{k,1})-\bar{V}^{\pi_k}_1(x_{k,1}))\mathbbm{1}\{\Omega_k\}]
		\leq
		0
		.
	\end{equation}
\end{proposition}

\subsubsection*{Part 3: Bounding the estimation error}

To bound $R_3$, we fix an episode $k\in[K]$ such that $T_k=kH>\Theta^{\clust}$.
We will now bound $\mbbE\bigl[\bigl( V^{\pi^*}_1(x_{k,1})-\bar{V}^{\pi_k}_1(x_{k,1})\bigr) \mbb{1}\{\Omega_k\}\bigr]$ for this episode.

Recall \eqref{eqn:historydef}.
Following \cite{UCRLVI}, we decompose
\begin{equation}
	\bar{V}^{\pi_k}_1(x_{k,1})
	-
	V^{\pi_k}_1(x_{k,1})
	=
	\sum_{h=1}^{H-1}\mbbE\bigl[
		\hat{b}_{k}(\hat{f}^{\alg}(x_{k,h}),a_{k,h})
		+
		E_{k,h}(x_{k,h},a_{k,h})
		+
		F_{k,h}(x_{k,h},a_{k,h})
		\mid \mcD_k, x_{k,1}
	\bigr]
	,
	\label{eqn:application_simulation_lemma_true_distribution}
\end{equation}
where
\begin{equation}
	\begin{split}
		&
		E_{k,h}(x,a)
		\eqdef
		\langle \hat{P}_k(\,\cdot\mid x,a)-P(\,\cdot\mid x,a), V^{\pi^*}_{h+1}\rangle
		\quad
		\mathrm{and}
		\\
		&
		F_{k,h}(x,a)
		\eqdef
		\langle \hat{P}_k(\,\cdot\mid x,a)-P(\,\cdot\mid x,a), \bar{V}^{\pi_k}_{h+1}-V^{\pi^*}_{h+1}\rangle
		.
	\end{split}
\end{equation}
A proof of \eqref{eqn:application_simulation_lemma_true_distribution} is included in \Cref{sec:proof_of_application_simulation_lemma_true_distribution}.
We now bound $E_{k,h}$ and $F_{k,h}$.

While $E_{k,h}$ is bounded on $\Omega_k$ by \Cref{lem:esterrorconfbd}, handling $F_{k,h}$ is more challenging because $\hat{P}_k$ and $\bar{V}^{\pi_k}_{h+1}$ are dependent.
\Cref{prop:concentration_bound_nonnegative_vector} below addresses this by relating $F_{k,h}$ to the expected next-context value of $\bar{V}^{\pi_k}_{h+1}-V^{\pi^*}_{h+1}$.
This removes the dependence on $\hat{P}_k$ and enables one to bound $F_{k,h}$ whenever $\bar{V}^{\pi_k}_{h+1}$ is close to $V^{\pi^*}_{h+1}$.

\begin{proposition}
	\label{prop:concentration_bound_nonnegative_vector}
	Let $L\eqdef 1+\ln(nHS^2AT_K^2)$.
	Then, for all $x$, $a$, $h\in[H-1]$, and $k\in[K]$, on $\Omega_k$,
	\begin{equation}
		\label{eqn:concentration_bound_nonnegative_vector_large_P}
		\begin{split}
			F_{k,h}(x,a)
			\leq
			&\,
			(1/H)(2+1/H)
			\langle P(\,\cdot\mid  x,a), \bar{V}^{\pi_k}_{h+1} - V^{\pi^*}_{h+1}\rangle
			\\
			&
			+
			\frac{2H^2SL}{1\vee \hat{N}_k(f^{-1}(f(x)),a)}
			+
			\sum_{s}p(s\mid f(x),a)\frac{4n H^2L}{1\vee \hat{N}^{\mathrm{to}}_k(f^{-1}(s))}
			.
		\end{split}
	\end{equation}
\end{proposition}

\Cref{sec:proof_of_concentration_bound_nonnegative_vector} contains the proof of \Cref{prop:concentration_bound_nonnegative_vector}.
It is similar to that of \cite[Lemma 3]{UCRLVI}, which bounds $F_{k,h}$ for tabular \glspl{MDP}.
We cannot directly apply their argument, however, because $\hat{P}_k$'s elements are products of $\hat{p}_k$ and $\hat{q}_k$.
Instead, we decompose $F_{k,h}$ into separate contributions involving only $\hat{p}_k$ or $\hat{q}_k$.
These can then be bounded using \cite{UCRLVI}'s argument, applied separately to $\hat{p}_k$ and $\hat{q}_k$.

We now return to \eqref{eqn:application_simulation_lemma_true_distribution}.
First, observe that $V^{\pi^*}_{h+1}\geq V^{\pi_k}_{h+1}$ since $\pi^*$ is optimal.
Thus $\bar{V}^{\pi_k}_{h+1} - V^{\pi^*}_{h+1}\leq \bar{V}^{\pi_k}_{h+1} - V^{\pi_k}_{h+1}$.
With \Cref{prop:concentration_bound_nonnegative_vector}, this implies that $F_{k,h}$ is bounded by the expected next--context value of $\bar{V}^{\pi_k}_{h+1} - V^{\pi_k}_{h+1}$.
This allows us to prove the following recursive bound in \Cref{sec:proof_of_estimation_error_bound}:
\begin{proposition}
	\label{prop:estimation_error_bound}
	Write $s_{k,h}=f(x_{k,h})$.
	For $h\in[H-1]$ and $k\in[K]$ such that $T_k=kH>\Theta^{\clust}$,
	\begin{align}
		\mbbE\Bigl[ & \bigl(\bar{V}_h^{\pi_k}(x_{k,h})-V_h^{\pi_k}(x_{k,h}) \bigr) \mathbbm{1}\{\Omega_k\} \Bigr]
		\label{eqn:estimation_error_bound}
		\\
		       &
		\leq
		\sum_{h'=h}^{H-1}\mbbE\biggl[
		\biggl(
			\hat{B}_{k}(s_{k,h'},a_{k,h'},s_{k,h'+1})
			+
			\frac{1}{H}\biggl(2+\frac{1}{H}\biggr)
			\bigl(\bar{V}_{h'+1}^{\pi_k}(x_{k,h'+1})-V_{h'+1}^{\pi_k}(x_{k,h'+1})\bigr)
		\biggr)
		\mathbbm{1}\{\Omega_k\}
		\biggr]
		,
		\nonumber
	\end{align}
	where
	\begin{equation}
		\label{eqn:bounding_term_definition}
		\hat{B}_{k}(s,a,s')
		\eqdef
		\sqrt{\frac{4H^2 L}{1\vee \hat{N}_k(f^{-1}(s),a)}}
		+
		\frac{6H^2SL^2}{1\vee \hat{N}_k(f^{-1}(s),a)}
		+
		\sqrt{\frac{16H^2 L}{1\vee \hat{N}^{\mathrm{to}}_k(f^{-1}(s'))}}
		+
		\frac{4nH^2L}{1\vee \hat{N}^{\mathrm{to}}_k(f^{-1}(s'))}
		,
	\end{equation}
	with $L=1+\ln(nHS^2AT_K^2)$ as in \Cref{prop:concentration_bound_nonnegative_vector}.
\end{proposition}
By using \Cref{prop:estimation_error_bound} iteratively, we can obtain a bound that depends only on the $\hat{B}_{k}( s_{k,h}, \allowbreak a_{k,h}, \allowbreak s_{k,h+1})$.
The following lemma, proven in \Cref{sec:proof_of_recursive_sum}, makes this precise:
\begin{lemma}
	\label{lem:recursive_sum}
	Consider two sequences $\{\Delta_h\}_{h\in[H]}$, $\{B_h\}_{h\in[H]}$ with $\Delta_h,B_h\in\mbbR$ for $h\in[H]$, that satisfy $\Delta_H=0$ and $\Delta_h \leq \sum_{h'=h}^{H-1}(B_{h'} + (1/H)(2+1/H)\Delta_{h'})$ for all $h\in[H-1]$.
	Then $\Delta_1\leq e^2\sum_{h=1}^{H-1}B_{h}$.
\end{lemma}
We will apply \Cref{lem:recursive_sum} with $\Delta_h=\mbbE[ (\bar{V}_h^{\pi_k}(x_{k,h})-V_h^{\pi_k}(x_{k,h}))\mathbbm{1}\{\Omega_k\} ]$.
\Cref{lem:recursive_sum}'s assumptions are met for $B_h = \mbbE[\hat{B}_{k}(s_{k,h},\allowbreak a_{k,h},\allowbreak s_{k,h+1})\mathbbm{1}\{\Omega_k\}]$ by \Cref{prop:estimation_error_bound} and because $\bar{V}^{\pi_k}_H(x)=V^{\pi_k}_H(x)=\max_{a\in[A]}r_H(x,a)$ by definition.
Together with $\hat{B}_k\geq 0$, we conclude that
\begin{equation}
	\label{eqn:application_of_recursive_sum_lemma}
	\mbbE[ (\bar{V}_1^{\pi_k}(x_{k,1})-V_1^{\pi_k}(x_{k,1}))\mathbbm{1}\{\Omega_k\} ]
	\leq
	e^2\sum_{h=1}^{H-1}\mbbE[
		\hat{B}_{k}(s_{k,h},a_{k,h},s_{k,h+1})
	]
	.
\end{equation}
Observe now that $\hat{B}_k(s,a,s')$ depends inversely on the visitation count of the triple $(s,a,s')$.
Hence, once a triple is visited, its corresponding $\hat{B}_k(s,a,s')$ dereases in subsequent episodes.
\Cref{prop:estimation_error_bound} was formulated precisely so that we can exploit this fact.
We do so in \Cref{sec:proof_of_final_bound_R3}, where we prove our final bound on $R_3$.
This result is better than its analogue for tabular \glspl{MDP} because the visitation count of a latent state increases more quickly than that of a single context.
\begin{proposition}
	\label{prop:final_bound_R3}
	If $T_K=KH= \Omega(H^4S^3A)$, then
	\begin{equation}
		\label{eqn:final_bound_R3}
		\begin{split}
			R_3
			=
			\sum_{k\in[K]:T_k>\Theta^{\clust}}\mbbE[(\bar{V}^{\pi_k}_1(x_{k,1})-V^{\pi_k}_1(x_{k_1}))\mbb{1}\{\Omega_k\}]
			=
			\tilde{O}(
			H\sqrt{SAT_K}
			+
			nH^3
			)
			.
		\end{split}
	\end{equation}
\end{proposition}

Finally --- the proof of \Cref{thm:BUCBbd} is done by noting that the assumptions of \Cref{thm:BUCBbd} are the same as those of \Cref{prop:highprobevent,prop:optimism,prop:final_bound_R3}, and using these to bound $R_1$, $R_2$, and $R_3$ in \eqref{eqn:regret_decomp}.

\section{Numerical proof of concept}
\label{sec:numerical}

We now demonstrate the method and briefly compare its performance to the UCBVI algorithms of \cite{UCRLVI}, which do not exploit the latent structure.

We fix $A=3$ and $H=20$.
For $n\in\{100,200\}$ and $S\in\{3,6\}$, we sample $p(\,\cdot\mid s,a)$ independently from $\textnormal{Dir}(1,\ldots,1)$ for each $s$ and $a$.
Here, $\textnormal{Dir}(\alpha_1,\ldots,\alpha_k)$ denotes the $k$-dimensional Dirichlet distribution.
Similarly, we sample $\{q(y\mid s)\}_{y\in f^{-1}(s)}$ independently from $\textnormal{Dir}(\sqrt{n},\ldots,\sqrt{n})$ for each $s$.
We set each latent state to have the same number of contexts.
These \glspl{BMDP} turned out to be $\eta$-reachable and $\mcI$-identifiable with $\eta \approx 25.5$ and $\mcI\approx 0.12$ for $S=3$, and $\eta \approx 33.6$ and $\mcI\approx 0.49$ for $S=6$.
The rewards $r_h(x,a)$ are sampled independently from $\textnormal{Unif}([0,1])$ for each $x$, $a$, and $h$.

\Cref{fig:comparison-in-simple-example} shows the cumulative regret over $K=10^5$ episodes for UCBVI--CH, UCBVI--BF, and \Cref{alg:BUCBVIouter}.
Here, we averaged over 10 runs for a fixed \gls{BMDP}.
\Cref{alg:BUCBVIouter} eventually achieves lower regret for all parameter configurations in \Cref{fig:comparison-in-simple-example}.
The difference increases for larger $n$ and smaller $S$.
Notably, \Cref{alg:BUCBVIouter}'s performance stays the same when $n$ is increased.

\begin{figure}[hbtp]
	\centering

	\begin{subfigure}{0.49\linewidth}
		\begin{tikzpicture}[scale = 0.6]
			\begin{axis}[
					xmin = 0, xmax = 2000000,
					ymin = 0, ymax = 500000,
					xlabel = \empty,
					xticklabel = \empty,
					ylabel = {Cumulative regret},
					label style = {font = \LARGE},
					tick label style = {font = \LARGE},
					scaled x ticks = false,
					width=12cm, height=6cm
				]

				\addplot[mark=none, line width = 1] table[x=Time, y=Regret, col sep=comma]{final_simulations/data/case1/UCBVICH_n100_S3_A3_h20_K100000.csv};

				\addplot[mark=none, line width = 1, color = blue, dashed] table[x=Time, y=Regret, col sep=comma]{final_simulations/data/case1/UCBVIBF_n100_S3_A3_h20_K100000.csv};

				\addplot[mark=none, line width = 1, color = red, loosely dashdotted] table[x=Time, y=Regret, col sep=comma]{final_simulations/data/case1/BUCBVI_n100_S3_A3_h20_K100000.csv};

				\draw (0,0) rectangle (2.1cm,1.05cm);
				\draw (0.7875cm, 1.05cm) -- (2.625cm, 2.535cm);
			\end{axis}

			\begin{axis}[
					xshift = 0.2cm, yshift = 4.3cm,
					xmin = 0, xmax = 400000,
					ymin = 0, ymax = 100000,
					ticks=none,
					anchor={outer north west},
					width=6.6cm, height=3.3cm
				]
				\addplot[mark=none, line width = 1] table[x=Time, y=Regret, col sep=comma]{final_simulations/data/case1/UCBVICH_n100_S3_A3_h20_K100000.csv};
				\addplot[mark=none, line width = 1, color=blue, dashed] table[x=Time, y=Regret, col sep=comma]{final_simulations/data/case1/UCBVIBF_n100_S3_A3_h20_K100000.csv};
				\addplot[mark=none, line width = 1, color=red, loosely dashdotted] table[x=Time, y=Regret, col sep=comma]{final_simulations/data/case1/BUCBVI_n100_S3_A3_h20_K100000.csv};
			\end{axis}
		\end{tikzpicture}
		\centering
		\caption{$S=3,n=100$}
	\end{subfigure}
	\begin{subfigure}{0.49\linewidth}
		\begin{tikzpicture}[scale = 0.6]
			\begin{axis}[
					xmin = 0, xmax = 2000000,
					ymin = 0, ymax = 500000,
					xlabel = \empty,
					xticklabel = \empty,
					ylabel = \empty,
					yticklabel = \empty,
					label style = {font = \LARGE},
					tick label style = {font = \LARGE},
					scaled x ticks = false,
					width=12cm, height=6cm
				]

				\addplot[mark=none, line width = 1] table[x=Time, y=Regret, col sep=comma]{final_simulations/data/case1/UCBVICH_n200_S3_A3_h20_K100000.csv};

				\addplot[mark=none, line width = 1, color = blue, dashed] table[x=Time, y=Regret, col sep=comma]{final_simulations/data/case1/UCBVIBF_n200_S3_A3_h20_K100000.csv};

				\addplot[mark=none, line width = 1, color = red, loosely dashdotted] table[x=Time, y=Regret, col sep=comma]{final_simulations/data/case1/BUCBVI_n200_S3_A3_h20_K100000.csv};

				\draw (0,0) rectangle (2.1cm,1.05cm);
				\draw (0.7875cm, 1.05cm) -- (2.625cm, 2.535cm);
			\end{axis}

			\begin{axis}[
					xshift = 0.2cm, yshift = 4.3cm,
					xmin = 0, xmax = 400000,
					ymin = 0, ymax = 100000,
					ticks=none,
					anchor={outer north west},
					width=6.6cm, height=3.3cm
				]
				\addplot[mark=none, line width = 1] table[x=Time, y=Regret, col sep=comma]{final_simulations/data/case1/UCBVICH_n200_S3_A3_h20_K100000.csv};
				\addplot[mark=none, line width = 1, color=blue, dashed] table[x=Time, y=Regret, col sep=comma]{final_simulations/data/case1/UCBVIBF_n200_S3_A3_h20_K100000.csv};
				\addplot[mark=none, line width = 1, color=red, loosely dashdotted] table[x=Time, y=Regret, col sep=comma]{final_simulations/data/case1/BUCBVI_n200_S3_A3_h20_K100000.csv};
			\end{axis}
		\end{tikzpicture}
		\centering
		\caption{$S=3,n=200$}
	\end{subfigure}

	\begin{subfigure}{0.49\linewidth}
		\begin{tikzpicture}[scale = 0.6]
			\begin{axis}[
					xmin = 0, xmax = 2000000,
					ymin = 0, ymax = 500000,
					xlabel = {Time},
					ylabel = {Cumulative regret},
					label style = {font = \LARGE},
					tick label style = {font = \LARGE},
					width=12cm, height=6cm
				]

				\addplot[mark=none, line width = 1] table[x=Time, y=Regret, col sep=comma]{final_simulations/data/case1/UCBVICH_n100_S6_A3_h20_K100000.csv};

				\addplot[mark=none, line width = 1, color = blue, dashed] table[x=Time, y=Regret, col sep=comma]{final_simulations/data/case1/UCBVIBF_n100_S6_A3_h20_K100000.csv};

				\addplot[mark=none, line width = 1, color = red, loosely dashdotted] table[x=Time, y=Regret, col sep=comma]{final_simulations/data/case1/BUCBVI_n100_S6_A3_h20_K100000.csv};

				\draw (0,0) rectangle (2.1cm,1.05cm);
				\draw (0.7875cm, 1.05cm) -- (2.625cm, 2.535cm);

			\end{axis}

			\begin{axis}[
					xshift = 0.2cm, yshift = 4.3cm,
					xmin = 0, xmax = 400000,
					ymin = 0, ymax = 100000,
					ticks=none,
					anchor={outer north west},
					width=6.6cm, height=3.3cm
				]
				\addplot[mark=none, line width = 1] table[x=Time, y=Regret, col sep=comma]{final_simulations/data/case1/UCBVICH_n100_S6_A3_h20_K100000.csv};
				\addplot[mark=none, line width = 1, color=blue, dashed] table[x=Time, y=Regret, col sep=comma]{final_simulations/data/case1/UCBVIBF_n100_S6_A3_h20_K100000.csv};
				\addplot[mark=none, line width = 1, color=red, loosely dashdotted] table[x=Time, y=Regret, col sep=comma]{final_simulations/data/case1/BUCBVI_n100_S6_A3_h20_K100000.csv};
			\end{axis}
		\end{tikzpicture}
		\centering
		\caption{$S=6,n=100$}
	\end{subfigure}
	\begin{subfigure}{0.49\linewidth}
		\begin{tikzpicture}[scale = 0.6]
			\begin{axis}[
					xmin = 0, xmax = 2000000,
					ymin = 0, ymax = 500000,
					xlabel = {Time},
					ylabel = \empty,
					yticklabel = \empty,
					label style = {font = \LARGE},
					tick label style = {font = \LARGE},
					legend style={at={(0.975, 0.05)}, anchor=south east, font=\LARGE},
					width=12cm, height=6cm
				]

				\addplot[mark=none, line width = 1] table[x=Time, y=Regret, col sep=comma]{final_simulations/data/case1/UCBVICH_n200_S6_A3_h20_K100000.csv};
				\addlegendentry{UCBVI-CH}

				\addplot[mark=none, line width = 1, color = blue, dashed] table[x=Time, y=Regret, col sep=comma]{final_simulations/data/case1/UCBVIBF_n200_S6_A3_h20_K100000.csv};
				\addlegendentry{UCBVI-BF}

				\addplot[mark=none, line width = 1, color = red, loosely dashdotted] table[x=Time, y=Regret, col sep=comma]{final_simulations/data/case1/BUCBVI_n200_S6_A3_h20_K100000.csv};
				\addlegendentry{Algorithm 1}

				\draw (0,0) rectangle (2.1cm,1.05cm);
				\draw (0.7875cm, 1.05cm) -- (2.625cm, 2.535cm);
			\end{axis}

			\begin{axis}[
					xshift = 0.2cm, yshift = 4.3cm,
					xmin = 0, xmax = 200000,
					ymin = 0, ymax = 50000,
					ticks=none,
					anchor={outer north west},
					width=6.6cm, height=3.3cm
				]
				\addplot[mark=none, line width = 1] table[x=Time, y=Regret, col sep=comma]{final_simulations/data/case1/UCBVICH_n200_S6_A3_h20_K100000.csv};
				\addplot[mark=none, line width = 1, color=blue, dashed] table[x=Time, y=Regret, col sep=comma]{final_simulations/data/case1/UCBVIBF_n200_S6_A3_h20_K100000.csv};
				\addplot[mark=none, line width = 1, color=red, loosely dashdotted] table[x=Time, y=Regret, col sep=comma]{final_simulations/data/case1/BUCBVI_n200_S6_A3_h20_K100000.csv};
			\end{axis}
		\end{tikzpicture}
		\centering
		\caption{$S=6,n=200$}
	\end{subfigure}
	\caption{
		The performance of UCBVI--CH, UCBVI--BF, and \Cref{alg:BUCBVIouter} on \Cref{sec:numerical}'s example.
		The plots include 95\% confidence intervals on the cumulative regret, but these are too narrow to see on this scale.
	}
	\label{fig:comparison-in-simple-example}
	\Description{
		Four plots show the cumulative regret over time of three algorithms.
		The regret of the algorithm that exploits the latent structure initially increases linearly.
		Later, a slower increase is observed.
		The same behavior is observed for the algorithms that do not exploit the latent structure, but the period of linear growth lasts considerably longer.
	}
\end{figure}
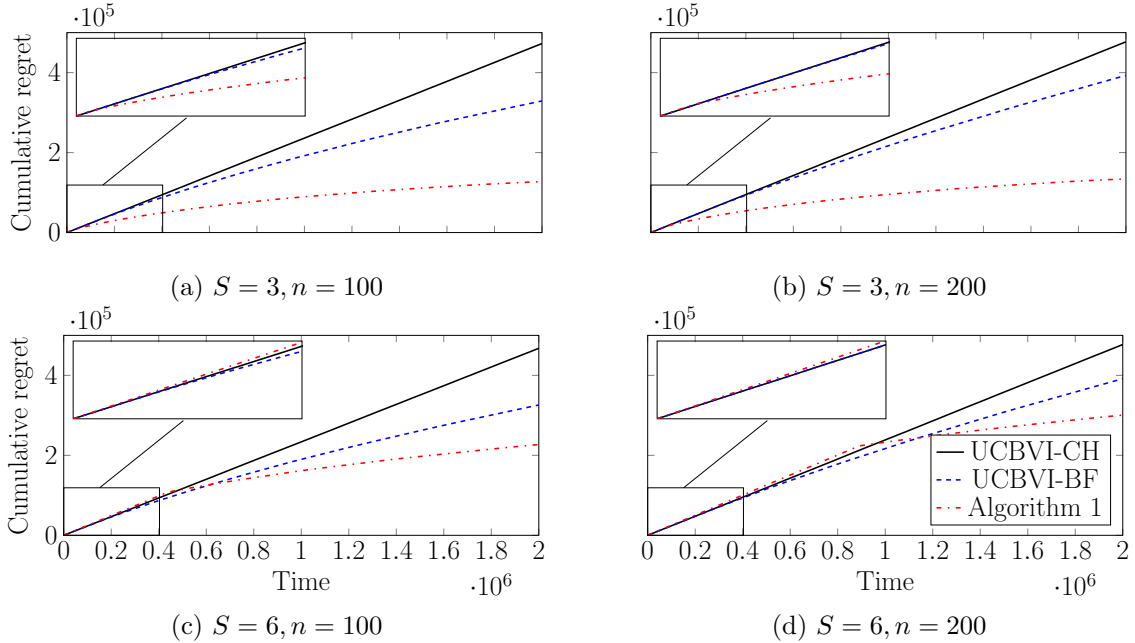

\emergencystretch=1em
\printbibliography

\appendix
\section{Latent state estimation algorithm}
\label{sec:clusteringalg}

We here briefly recall \cite[Algorithm 1]{Jedra:2022} and \cite[Algorithm 2]{Jedra:2022}, which are used in \Cref{alg:UpdateLatentStates} to estimate the decoding function.
We do so to relate their notation to ours, and for ease of reference during the proofs, specifically that of \Cref{prop:clustering_bound_short}.

\subsection{Spectral Clustering Algorithm}
\label{sec:clusteringalgspectral}

\begin{algorithm}
    \captionsetup{labelformat=empty}
    \begin{algorithmic}[1]

        \Procedure{SpectralClustering}{$\mcD_{k+1}$}
        \State $\hat{\mbf{N}}_{a}(x,y)\gets \sum_{k'=1}^k\sum_{h=1}^H\mbb{1}\{x_{k',h}=x,a_{k',h}=a,x_{k',h+1}=y\}$ for $a\in[A]$, $x,y\in[n]$

        \State $\Gamma_a\gets [n]$ for $a\in[A]$
        \For{$k=1,2,\ldots,\floor{n\exp(-T_k/(nA)\ln(T_k/(nA)))}$}
        \State $\Gamma_a\gets \Gamma_a\setminus \{\textnormal{uniformly random } x\in \argmax_{y\in \Gamma_a}\{\sum_{z}\hat{\mbf{N}}_a(y,z)\}\}$ for $a\in[A]$
        \EndFor
        \State $\hat{\mbf{N}}_{\Gamma,a}\gets (\hat{\mbf{N}}_a(x,y)\mathbbm{1}\{(x,y)\in\Gamma_a\times\Gamma_a\})_{x,y\in[n]}$ for $a\in[A]$

        \State $\hat{\mbf{R}}_{a}\gets S\textnormal{-Rank}(\hat{\mbf{N}}_{a})$ for $a\in[A]$ \Comment{Compute rank-$S$ approximation \eqref{eqn:srankdef}}
        \State $\hat{\mbf{R}}\gets \left[\hat{\mbf{R}}_{1}^\intercal \quad\ldots \quad\hat{\mbf{R}}_{A}^\intercal \quad\hat{\mbf{R}}_{1} \quad\ldots \quad\hat{\mbf{R}}_{A}\right]$
        \label{ln:fat_matrix}
        \State $\tilde{\mbf{R}}(x,\,\cdot\,)\gets \mathbbm{1}\{\lVert\hat{\mbf{R}}(x,\,\cdot\,)\rVert_1>0\}\hat{\mbf{R}}(x,\,\cdot\,)/\lVert\hat{\mbf{R}}(x,\,\cdot\,)\rVert_1$ for $x\in[n]$
        \State $\hat{f}_0\gets \textnormal{K-Medians}(\{ \tilde{\mbf{R}}(x,\,\cdot\,) \}_{x\in[n]} )$ \Comment{Solve $K$-Medians problem \cite[Eqn. (61)]{Jedra:2022}}
        \State \Return Initial estimate decoding function $\hat{f}_0$
        \EndProcedure
    \end{algorithmic}
    \caption{\textbf{\cite[Algorithm 1]{Jedra:2022} in our notation}: Pseudocode for the Spectral Clustering Algorithm.}
    \label{alg:spectral}
\end{algorithm}

The first step of \Cref{alg:UpdateLatentStates} consists of the spectral algorithm \cite[Algorithm 1]{Jedra:2022}, whose pseudocode we reproduce here in our notation.

The procedure starts by trimming the transition count matrices matrices $\hat{\mbf{N}}_a$ for each $a$ by setting to zero each row and column corresponding to contexts in $\Gamma_a$.
Here, $\Gamma_a\subset[n]$ are the contexts that remain after removing the $\floor{n\exp(-T_k/(nA)\ln(T_k/(nA)))}$ contexts that are visited most often when choosing action $a$.

Next, a rank-$S$ approximation $\hat{\mbf{R}}_a$ of the resulting matrix $\hat{\mbf{N}}_{\Gamma,a}$ is constructed by calculating the \gls{SVD} $\hat{\mbf{N}}^{\Gamma}_{a}=\mbf{U}_a^\intercal\text{diag}(\sigma_{a,1},\ldots,\sigma_{a,n})\mbf{V}_a$ where $\mbf{U}_a,\mbf{V}_a$ are two $n\times n$ orthonormal matrices, and $\sigma_{a,1}\geq \sigma_{a,2}\geq\ldots\geq\sigma_{a,n}$ are the singular values of $\hat{\mbf{N}}^{\Gamma}_{a}$ in decreasing order, and defining for each $a$
\begin{equation}\label{eqn:srankdef}
    \hat{\mbf{R}}_{a}
    \eqdef
    \mbf{U}_a^\intercal\text{diag}(\sigma_{a,1},\ldots,\sigma_{a,S},0,\ldots,0)\mbf{V}_a.
\end{equation}
The matrices $\hat{\mbf{R}}_{a}$ are subsequently concatenated to obtain the ``fat'' matrix $\hat{\mathbf{R}}$.
After normalizing each row of $\hat{\mathbf{R}}$ by its corresponding $\ell^1$-norm, a weighted $K$-medians procedure is used to estimate $\hat{f}_0$.
We refer to \cite[Eqn. (61)]{Jedra:2022} for the exact optimization problem that is solved to determine $\hat{f}_0$.

\subsection{Iterative Improvement Algorithm}
\label{sec:clusteringalgcia}

\begin{algorithm}
    \captionsetup{labelformat=empty}
    \begin{algorithmic}[1]
        \Procedure{EstimateLatentKernels}{$\mcD_{k+1}, \hat{f}$}
        \State $\hat{p}(s'\mid s,a)\gets \frac{\hat{N}(\hat{f}^{-1}(s),a,\hat{f}^{-1}(s'))}{\hat{N}(\hat{f}^{-1}(s),a)}$ for $s,s'\in[S]$, $a\in[A]$\label{ln:p_def}
        \State $\hat{p}^{\mathrm{bwd}}(s,a\mid s')\gets \frac{\hat{N}(\hat{f}^{-1}(s),a,\hat{f}^{-1}(s'))}{\hat{N}^{\mathrm{to}}(\hat{f}^{-1}(s'))}$ for $s,s'\in[S]$, $a\in[A]$\label{ln:pbwd_def}
        \State \Return Estimate latent state transition kernels $\hat{p}, \hat{p}^{\mathrm{\mathrm{bwd}}}$
        \EndProcedure

        \Procedure{ImproveClusters}{ $\mcD_{k+1}$, $\hat{f}_0$}

        \State $\hat{N}_{k}(x,a,y)\gets \sum_{k'\in[k]}\sum_{h=1}^H\mbb{1}\{x_{k',h}=x,a_{k',h}=a,x_{k',h+1}=y\}$ for $a\in[A]$, $x,y\in[n]$
        \For{$\ell=0,\ldots,\ceil{\ln(n)}-1$}
        \State $\hat{p}_{\ell},\hat{p}_{\ell}^{\mathrm{\mathrm{bwd}}}\gets \text{EstimateLatentKernels}(\mcD_{k+1},\hat{f}_{\ell})$\label{ln:p_pbwd_def}
        \State $\mcL_{\ell}(x,s')\gets \sum_{a,s}\big[\hat{N}_k(x,a,(\hat{f}_{\ell})^{-1}(s))\ln\hat{p}_{\ell}(s\mid s',a)$\label{ln:log_likelihood_def}
        \Statex \hspace{3cm} $+\hat{N}_k((\hat{f}_{\ell})^{-1}(s),a,x)\ln \hat{p}_{\ell}^{\mathrm{bwd}}(s,a\mid s')\big]$ for $x\in[n]$, $s'\in[S]$
        \State $\hat{f}_{\ell+1}(x)\gets \argmax_{s}\mcL_{\ell}(x,s)$ for $x\in[n]$
        \EndFor

        \State \Return Improved estimate decoding function $\hat{f}^{(\ceil{\ln n})}$
        \EndProcedure
    \end{algorithmic}
    \caption{\textbf{\cite[Algorithm 2]{Jedra:2022} in our notation}: Pseudocode for Iterative Improvement Algorithm.}\label{alg:CIA}
\end{algorithm}

The second step of \Cref{alg:UpdateLatentStates} consists of the iterative improvement algorithm \cite[Algorithm 2]{Jedra:2022}, whose pseudocode we reproduce here in our notation.

Starting from some initial estimate $\hat{f}_0$, at each iteration $\ell=0,1,\ldots$ it uses the current estimate decoding function $\hat{f}_{\ell}$ to estimate the latent state transition probabilities from $\hat{N}$.
Note that it also estimates the backward transition probabilities $p^{\mathrm{bwd}}(s,a\mid s')$, i.e., the probability that action $a$ was taken in latent state $s$ given that we are currently in latent state $s'$.
These estimates are then used to estimate for each context $x$, the log-likelihood of the observed data over each potential new assignment $\hat{f}_{\ell+1}(x)=s'$.

The context $x$ is assigned to the latent state $s'$ that maximizes this estimate log-likelihood.
This greedy, local likelihood maximization is repeated $\ceil{\ln n}$ times, each time using the updated decoding to estimate a new log-likelihood function.
In \Cref{alg:UpdateLatentStates}, the improvement algorithm is initialized with the decoding function $\hat{f}_0$ returned by the spectral clustering algorithm.

\section{On the structural properties}
\label{sec:on-the-structural-properties}

This section further explains \Cref{def:identifiability,def:reachability} and relates them to analogous properties used in the literature.

\subsection{Explanation of the structural properties}

\subsubsection{Reachability}

\Cref{def:reachability} imposes a degree of regularity on the latent state transition probabilities, the emission probabilities, and the number of contexts per latent state.
Concretely, since $\sum_{s'}p(s'\mid s,a) = \sum_{y\in f^{-1}(s)}q(y\mid s) = \sum_{s} |f^{-1}(s)|/n = 1$, \Cref{def:reachability} implies that for sufficiently large $n$,
\begin{equation}
	\label{eqn:bounds_on_elements_of_kernels}
	\frac{1}{\eta S}\leq p(s'\mid s,a)\leq \frac{\eta}{S}
	,
	\quad
	\frac{1}{\eta |f^{-1}(s)|}\leq q(y\mid s)\leq \frac{\eta}{|f^{-1}(s)|}
	,
	\quad
	\mathrm{and}
	\quad
	\frac{n}{\eta S}
	\leq
	|f^{-1}(s)|
	\leq
	\frac{\eta n}{S}
\end{equation}
for $s,s'\in[S]$, $y\in f^{-1}(s)$.
In turn, this implies that in expectation, and up to a constant factor, all contexts are observed equally often under any exploration policy.

\subsubsection{Identifiability}

\Cref{def:identifiability} imposes a minimum degree of separation between different latent states.
In particular, it ensures this separation remains bounded away from zero as $n\rightarrow \infty$.

Concretely, the quantity $I_{\tilde{s}}(x;c,\Phi,\pi)$ in \eqref{eqn:Ijxcdef} is the leading-order term in the expansion of the log-likelihood ratio of the context--action sequence between $\Phi$ and the confusing \gls{BMDP} $\Psi$ in \cite[Appendix D]{Jedra:2022} as $n\rightarrow \infty$.
The parameter $c>0$ controls the emission probability of $x$ in $\tilde{s}$.
Thus, by ensuring that $\min_x I_{\eta}(x;\Phi)$ remains bounded away from zero, \Cref{def:identifiability} prevents two distinct latent states from having asymptotically indistinguishable transition and emission probabilities.

\subsection{Relation to literature}

Analogues of \Cref{def:reachability,def:identifiability} are common in the \gls{BMDP} literature.
See, for example, \cite[Definition 2.1 and/or Assumption 3.2]{pmlr-v97-du19b} as well as \cite[Definition 6]{pmlr-v119-misra20a}.

Our formulation is based on \cite{Jedra:2022}, in order to leverage their results on latent state estimation.
In turn, \cite{Jedra:2022} builds on related conditions from the block model literature, including \glspl{BMC} \cite{Sanders:2020} and (labeled) \glspl{SBM} \cite{abbe2018community,yun2016optimal,10.1214/17-AOS1615}.

\subsubsection{Reachability}

When $S$ is independent of $n$, any $\eta$-reachable \gls{BMDP} satisfies \cite[Assumptions 2 and 3]{Jedra:2022}.
Both $\eta$-reachability and \cite[Assumptions 2 and 3]{Jedra:2022} are stronger than the reachability conditions typically considered in the \gls{BMDP} literature.
For example, \cite{pmlr-v97-du19b,pmlr-v119-misra20a} require that every state is reachable under \emph{some} policy, whereas $\eta$-reachability enforces a uniform degree of reachability under \emph{any} policy.

As highlighted in \Cref{sec:related-literature:function-approximation}, this difference reflects distinct objectives.
Works such as \cite{pmlr-v97-du19b,pmlr-v119-misra20a} aim to design efficient exploration policies to learn the latent structure using function approximation.
By contrast, our focus is on analyzing explicit recovery algorithms with provable guarantees in the line of \cite{Sanders:2020,Jedra:2022}.

\subsubsection{Identifiability}
\label{sec:comparison_to_other_identifiability_conditions}

$\mcI$-identifiability is closely related to \cite[Theorem 3]{Jedra:2022}'s assumption $I(\Phi)>0$, where
\begin{equation}
	I(\Phi)
	\eqdef
	-\frac{n}{T_k}\ln\biggl(\frac{1}{2\eta Sn} \sum_x\exp\biggl(-\frac{T_k}{n}I(x;\Phi)\biggr)\biggr)
\end{equation}
with
\begin{equation}
	\label{eqn:def_minimum_information_quantity}
	I(x;\Phi)
	\eqdef
	\lim_{\zeta\rightarrow \infty}I_{\zeta}(x;\Phi)
	=
	\min_{\tilde{s}\neq f(x)}\inf_{c>0}I_{\tilde{s}}(x;c,\Phi,\pi_U)
	.
\end{equation}
When $T_k=\omega(n)$ and $S$, $A$ and $p$ are independent of $n$, one can show that the conditions
\begin{equation}
	I(\Phi)=\Omega(1)\quad\mathrm{and}\quad\min_{x}I(x;\Phi)=\Omega(1)
\end{equation}
are equivalent.
In our framework, however, $S$, $A$, and $p$ do depend on $n$.
In this case, $\min_{x}I(x;\Phi)=\Omega(1)$ still implies $I(\Phi)=\Omega(1)$, but the converse need not hold.
\Cref{def:identifiability} therefore requires an analogue of the stronger condition $\min_{x}I(x;\Phi)=\Omega(1)$.

\Cref{def:identifiability}'s requirement, however, is slightly weaker than $\min_{x}I(x;\Phi)=\Omega(1)$ because $I_{\eta}(x;\Phi)\geq I(x;\Phi)$; (compare the infima in \eqref{eqn:I_eta_def} and \eqref{eqn:def_minimum_information_quantity}).
\Cref{sec:equivalence_of_identifiability_assumptions} contains a proof that that \Cref{def:identifiability} nonetheless implies $\min_{x}I(x;\Phi)=\Omega(1)$, ensuring compatibility with \cite{Jedra:2022}.

For further comparisons with other notions of identifiability in the \gls{BMDP} literature, see \cite[Appendix D.3.5]{Jedra:2022}.

\section{Proofs for \texorpdfstring{\Cref{sec:regret_lower_bound}}{the regret lower bound}}

\subsection{Bounds on visitation rates}

The following lemma bounds visitation rates and is used throughout this section.

\begin{lemma}
	\label{lem:visitation_rate_bound}
	If $\mu=\mathrm{Unif}([n])$, then the following inequalities hold for any $\Phi$, $\pi$, $h$, and $x$:
	\begin{align}
		 &
		\mbbP_{\Phi,\pi}[x_h\in\mcX]
		\geq
		\min
		\Bigl\{
		\frac{|\mcX|}{n}
		,
		\min_{s, a}
		\sum_{y\in\mcX}
		p(f(y) \mid s,a)
		q(y \mid f(y))
		\Bigr\}
		,
		\label{eqn:visitation_rate_bound_lower}
		\\
		 &
		\mbbP_{\Phi,\pi}[x_h\in\mcX]
		\leq
		\max
		\Bigl\{
		\frac{|\mcX|}{n}
		,
		\max_{s, a}
		\sum_{y\in\mcX}
		p(f(y) \mid s,a)
		q(y \mid f(y))
		\Bigr\}
		.
		\label{eqn:visitation_rate_bound_upper}
	\end{align}
\end{lemma}

\begin{proof}
	We will prove \eqref{eqn:visitation_rate_bound_lower}.
	\Cref{eqn:visitation_rate_bound_upper} follows \emph{mutatis mutandis}.

	For $h = 1$, it follows from the assumption on $\mu$ that $\mbbP_{\Phi,\pi}[x_h\in \mcX] = |\mcX|/n$.
	This proves \eqref{eqn:visitation_rate_bound_lower} for $h = 1$.

	For $h > 1$,
	it follows from
	(i) the law of total probability,
	(ii) the definition of a \gls{BMDP} in \Cref{sec:Block-MDPs},
	and
	\eqref{def:BMDP-transition-kernel}
	that
	\begin{align}
		 &
		\mbbP_{\Phi,\pi}[x_{h}\in\mcX]
		\nonumber
		\\
		 &
		\overset{(i)}{=}
		\sum_{x, a}
		\mbbP_{\Phi,\pi}[x_{h} \in \mcX \mid x_{h-1} = x,a_{h-1}=a]
		\mbbP_{\Phi,\pi}[x_{h-1}=x,a_{h-1}=a]
		\nonumber
		\\
		 &
		\geq
		\bigl(
		\min_{x', a'}
		\mbbP_{\Phi,\pi}[x_{h} \in \mcX \mid x_{h-1} = x', a_{h-1}=a']
		\bigr)
		\sum_{x, a}
		\mbbP_{\Phi,\pi}[x_{h-1}=x,a_{h-1}=a]
		\nonumber
		\\
		 &
		=
		\min_{x', a'}
		\mbbP_{\Phi,\pi}[x_{h} \in \mcX \mid x_{h-1} = x', a_{h-1}=a']
		\nonumber
		\\
		 &
		\overset{(ii)}{=}
		\min_{x, a}
		\sum_{y\in\mcX}
		P(y \mid x,a)
		\nonumber
		\\
		 &
		\overset{\eqref{def:BMDP-transition-kernel}}{=}
		\min_{s, a}
		\sum_{y\in\mcX}
		q(y \mid f(y))p(f(y) \mid s,a)
		.
	\end{align}
	This proves \eqref{eqn:visitation_rate_bound_lower} for $h > 1$.
\end{proof}

\subsection{Proof of \texorpdfstring{\Cref{lem:latent_state_subset_sizes}}{latent state subset size bounds}}
\label{sec:proof_of_latent_state_subset_sizes}

The proof of \Cref{lem:latent_state_subset_sizes} adapts the argument given after \cite[Eqn. (5.2)]{zhang2016minimax}, modifying it slightly to avoid a case distinction between $S/2=2$ and $S/2\geq 3$.
In addition, whereas \cite{zhang2016minimax} shows that $\mf{s}_i^{+1}\geq \epsilon S$ for some unspecified $\epsilon>0$, we obtain the quantitative bound $\mf{s}_i^{+1}\geq S/36$.
We do not, however, pursue an upper bound or a bound on $\mf{s}_i^{0}$, both of which are derived in \cite{zhang2016minimax}.

Fix $i \in \{0, 1\}$.
Let $r_i \eqdef n_i - \floor{2 n_i / S}S/ 2$ and note that $0 \leq r_i\leq S/2 - 1$.
We consider separately the cases $r_i \geq S/36$ and $r_i< S/36$.

If $r_i \geq S/36$, let $\mf{s}^{-1}_i=0$, $\mf{s}^0_i = S/2 - r_i$, and $\mf{s}^{+1}_i = r_i$.
Observe then that from the definition of $r_i$ it holds that
$
	\sum_{\sigma\in\{-1,0,+1\}}
	\mf{s}^{\sigma}_i
	(
	\lfloor n_i/(S/2)\rfloor + \sigma
	)
	=
	n_i
	.
$
Moreover, since $r_i \leq S/2-1$ we  have $\mf{s}^0_i \geq 1$, and the subcase implies that $\mf{s}^{+1}_i \geq S/36$.
This proves \Cref{lem:latent_state_subset_sizes} when $r_i\geq S/36$.

If instead $r_i< S/36$, let $\mf{s}^{-1}_i=\floor{S/6}$, $\mf{s}^0_i = S/2 - 2\floor{S/6}-r_i$, and $\mf{s}^{+1}_i = r_i + \floor{S/6}$.
Observe again that by definition of $r_i$ it holds that
$
	\sum_{\sigma\in\{-1,0,+1\}}
	\mf{s}^{\sigma}_i
	(
	\lfloor n_i/(S/2)\rfloor + \sigma
	)
	=
	n_i
	.
$
Moreover, since $\floor{S/6}\leq S/6$, the subcase implies $\mf{s}_i^0 \geq S/2 - S/3 - S/36 \geq S/36$, and because $\mf{s}_i^0$ is integer it must be that $\mf{s}_i^0 \geq 1$.
Finally since $\floor{S/6} \geq S/6 - (1-1/S)$, $r_i \geq 0$ and $S\geq 6$ by \eqref{eqn:parameter_assumptions}, it follows that $\mf{s}^{+1}_i \geq S/6 - (1-1/S) \geq S/36$.
This proves \Cref{lem:latent_state_subset_sizes} when $r_i< S/36$.
\qed

\subsection{Proof of \texorpdfstring{\Cref{prop:kernel_allows_identifiability}}{reachability and identifiability}}
\label{sec:proof_of_kernel_allows_identifiability}

Let $\eta>1$.
Let
$
	\varepsilon = (\varepsilon_0, \varepsilon_1)
$
be such that
$
	\epsilon_{\max}
	=
	\varepsilon_0 \vee \varepsilon_1
	<
	(1/2)(\eta-1)/(\eta+1)
$.

We first define a class of $\tilde{p} = ( \tilde{p}_1, \tilde{p}_2 )$ satisfying \eqref{eqn:tilde_p_distance_from_uniform} such that all $\Phi\in \Lambda^*( \varepsilon, \tilde{p} )$ are $\eta$-reachable.
We then show that for a particular $\tilde{p}^*$ in this class, all $\Phi\in \Lambda^*(\varepsilon, \tilde{p}^*)$ are also $\mcI_{\eta}$-identifiable for $\mcI_{\eta} = \mf{d}\eta^{-1}(\eta-1)^2/(\eta+1)^3$, provided $\epsilon_{\max}$ is sufficiently small.

\subsubsection{Construction of the candidate pairs \texorpdfstring{$( \tilde{p}_0, \tilde{p}_1 )$}{(p0, p1)}}

Let $\kappa$ be as in \eqref{eqn:tilde_p_distance_from_uniform}.

For any $v(\,\cdot\mid  s ): \mcS_0\rightarrow [-1,1]$ satisfying $\sum_{s'\in \mcS_0}v(s' \mid s)=0$ for $s\in[S]$, let
\begin{equation}
	\tilde{p}^{(v)}_0(s' \mid s)
	\eqdef
	\frac{2}{S}
	\bigl(
	1
	+
	\kappa v(s' \mid s)
	\bigr)
	\quad
	\mathrm{and}
	\quad
	\tilde{p}^{(v)}_1(s'' \mid s)
	\eqdef
	\frac{2}{S}
	\bigl(
	1
	+
	\kappa v(s''-S/2 \mid s)
	\bigr)
	\label{eqn:p_tilde_choice_with_packing}
\end{equation}
for $s' \in \mcS_0$ and $s''\in\mcS_1$.
Recall that $\mcS_0 = [S/2]$ and $\mcS_1 = [S] \setminus [S/2]$.
Therefore $s'' - S/2 \in \mcS_0$ whenever $s'' \in \mcS_1$.

Observe now that for any $v$, $s\in[S]$, $s'\in\mcS_0$, and $s''\in\mcS_1$,
\begin{equation}
	| \tilde{p}^{(v)}_0(s' \mid s)-2/S |
	\leq
	\frac{2\kappa}{S}
	\quad
	\textnormal{and}
	\quad
	| \tilde{p}^{(v)}_1(s' \mid s)-2/S |
	\leq
	\frac{2\kappa}{S}
	.
\end{equation}
This implies \eqref{eqn:tilde_p_distance_from_uniform} for $\tilde{p}^{(v)}_0,\tilde{p}^{(v)}_1$.

\subsubsection{Proof of \texorpdfstring{\Cref{prop:kernel_allows_identifiability--reachability}}{reachability}}

To prove \Cref{prop:kernel_allows_identifiability--reachability} for all $\Phi\in \Lambda(\varepsilon, \tilde{p}^{(v)})$, we need to verify \Crefrange{def:reachability:p}{def:reachability:f}.

First, recall the definition of $p( \cdot \mid \cdot, \cdot )$ in terms of $\tilde{p}$ in \eqref{eqn:latent_state_transition_kernel_macro_def}.
Next, conclude that for any $v$:

\begin{itemize}

	\item
	      \Cref{def:reachability:p} is satisfied because for the specific $\bigl(\tilde{p}^{(v)}_1, \tilde{p}^{(v)}_2\bigr)$ in \eqref{eqn:p_tilde_choice_with_packing},
	      \begin{align}
		      \max_{a}
		      \max_{s_1,s_2,s_3}
		      \biggl\{
		      \frac{p(s_2 \mid s_1,a)}{p(s_3 \mid s_1,a)}
		      ,
		      \frac{p(s_1 \mid s_2,a)}{p(s_1 \mid s_3,a)}
		      \biggr\}
		       &
		      \leq
		      \frac{\max_{a_1}\max_{s_1,s_2}p(s_1\mid s_2,a_1)}{\min_{a_2}\max_{s_3,s_4}p(s_3\mid s_4,a_2)}
		      \nonumber \\
		       &
		      \leq
		      \frac{(1+2\epsilon_{\max})(1+\kappa)}{(1-2\epsilon_{\max})(1-\kappa)}
		      ;
	      \end{align}
	      and the definition
	      \begin{equation}
		      \kappa
		      =
		      \frac{\eta(1-2\epsilon_{\max})-(1+2\epsilon_{\max})}{\eta(1-2\epsilon_{\max})+(1+2\epsilon_{\max})}
		      \quad
		      \textnormal{implies}
		      \quad
		      \frac{(1+2\epsilon_{\max})(1+\kappa)}{(1-2\epsilon_{\max})(1-\kappa)}
		      =
		      \eta
		      .
	      \end{equation}

	\item
	      \Cref{def:reachability:q} is satisfied because $q(y \mid s) = 1/|f^{-1}(s)|$ for all $y$ and $s$ by \Cref{def:hard-to-learn-BMDP-instance} and $\eta>1$ by assumption.

	\item
	      Finally, \Cref{def:reachability:f} is satisfied because $|\mcS_0|=|\mcS_1|=S/2$ so that for $f$ satisfying \eqref{eqn:size_distribution_of_latent_states}, if $n >2S+1$,
	      \begin{equation}
		      \max_{s_1,s_2}
		      \frac{|f^{-1}(s_1)|}{|f^{-1}(s_2)|}
		      \leq
		      \frac{\floor{\ceil{n/2}/(S/2)}+1}{\floor{\floor{n/2}/(S/2)}-1}
		      \leq
		      \frac{\floor{(n/2+1/2)/(S/2)}+1}{\floor{(n/2-1/2)/(S/2)}-1}
		      \leq
		      \frac{1+\frac{1+S}{n}}{1-\frac{1+2S}{n}}
		      .
	      \end{equation}
	      This implies \Cref{def:reachability:f} since $n=\omega(S)$ by \eqref{eqn:parameter_assumptions}.
\end{itemize}

\subsubsection{Proof of \texorpdfstring{\Cref{prop:kernel_allows_identifiability--identifiability}}{identifiability}}

The proof of \Cref{prop:kernel_allows_identifiability--identifiability} relies on the following lemma, proven below.
\begin{lemma}
	\label{lem:bound_on_information_quantity}
	There exists a constant $\mf{d}>0$
	independent of $n$ and $\eta$,
	and vectors $v^*(\,\cdot\mid  s)\in [-1,1]^{S/2}$ for $s\in[S]$ possibly dependent on $\mf{d}$,
	for which the pair $(\tilde{p}^*_0,\tilde{p}^*_1) \eqdef \bigl(\tilde{p}_0^{(v^*)}, \tilde{p}_1^{(v^*)}\bigr)$ satisfies the following:
	for any $\Phi\in \Lambda(\varepsilon,\tilde{p}^*)$, $x \in[n]$, $\tilde{s} \neq f(x)$, $c>0$, and policy $\pi$,
	\begin{equation}
		I_{\tilde{s}}(x;c,\Phi,\pi)
		\geq
		\mf{d}c\kappa^2\Bigl(1-\frac{2S}{n}\Bigr)
		\min\Bigl\{1-\frac{2S}{n},(1-2\epsilon_{\max})(1-\kappa)\Bigr\}
		.
		\label{eqn:kernel_allows_identifiability}
	\end{equation}
\end{lemma}
With $\mf{d}$ and $(\tilde{p}^*_0,\tilde{p}^*_1)$ given by \Cref{lem:bound_on_information_quantity} the remaining claim of \Cref{prop:kernel_allows_identifiability} is that for $\epsilon_{\max}$ sufficiently small, any $\Phi\in\Lambda(\varepsilon,\tilde{p}^*)$ is $\mcI_{\eta}$-reachable with $\mathcal{I}_{\eta} = \mf{d}\eta^{-1}(\eta-1)^2/(\eta+1)^3$.

Note that because $\epsilon_{\max}<(1/2)(\eta-1)/(\eta+1)<1/2$, it follows from \eqref{eqn:tilde_p_distance_from_uniform} that $\kappa<1$.
Since $n=\omega(S)$ by \eqref{eqn:parameter_assumptions}, it also follows that for all sufficiently large $n$, $(1-2\epsilon_{\max})(1-\kappa)<1-2S/n$.
By \Cref{lem:bound_on_information_quantity}, and the definition of $\kappa$, we then have for all sufficiently large $n$,
\begin{align}
	\min_{x }
	I_{\eta}(x;\Phi)
	 &
	\overset{\eqref{eqn:I_eta_def}}{=}
	\min_{\tilde{s}\neq f(x)}
	\inf_{c\in[1/\eta,\eta]}
	I_{\tilde{s}}(x;c,\Phi,\pi)
	\nonumber
	\\
	 &
	\geq
	\Bigl(1-\frac{2S}{n}\Bigr)\frac{\mf{d}}{\eta}
	\kappa^2
	(1-2\epsilon_{\max})(1-\kappa)
	\nonumber
	\\
	&
	=
	2\Bigl(1-\frac{2S}{n}\Bigr)\frac{\mf{d}}{\eta}
	(1-4\epsilon_{\max}^2)
	\frac{(\eta(1-2\epsilon_{\max})-(1+2\epsilon_{\max}))^2}{(\eta(1-2\epsilon_{\max})+(1+2\epsilon_{\max}))^3}
	\defeq
	Z(\epsilon_{\max})
	.
\end{align}

Note next that $Z(\epsilon_{\max})$ is concave in $\epsilon_{\max}$ on the interval $\bigl( 0, (1/2)(\eta-1)/(\eta+1) \bigr)$.
Furthermore, observe that it equals $2(1-2S/n)\mf{d}\eta^{-1}(\eta-1)^2/(\eta+1)^3$ at $\epsilon_{\max}=0$ and $0$ at $\epsilon_{\max}=(1/2)(\eta-1)/(\eta+1)$.
If $2(1-2S/n)\geq 3/2$, say, this implies that there exists $\epsilon\in \bigl( 0, (1/2)(\eta-1)/(\eta+1) \bigr)$ depending only on $\eta$ and $\mf{d}$ such that $Z(\epsilon_{\max})\geq \mf{d}\eta^{-1}(\eta-1)^2/(\eta+1)^3$ for all $\epsilon_{\max}<\epsilon$.
Since $n=\omega(S)$, it follows that for $\epsilon_{\max}<\epsilon$, for all sufficiently large $n$,
\begin{align}
	\min_{x }
	I_{\eta}(x;\Phi)
	\geq
	\frac{\mf{d}}{\eta}
	\frac{(\eta-1)^2}{(\eta+1)^3}
	.
\end{align}
That is it.

\paragraph{Proof of \texorpdfstring{\Cref{lem:bound_on_information_quantity}}{identifiability}}

\Cref{eqn:Ijxcdef} together with nonnegativity of the KL-divergence implies that
\begin{equation}
	I_{\tilde{s}}(x;c,\Phi,\pi)
	\geq
	cnq(x\mid f(x))
	\sum_a
	\omega_{\Psi,\pi}(f(x),a)
	\KL\bigl( p(\,\cdot\mid \tilde{s},a),p(\,\cdot\mid f(x),a) \bigr)
	;
	\label{eqn:I_lower_bound_in_terms_of_KL_p}
\end{equation}
and \eqref{eqn:p_tilde_choice_with_packing} ensures that for $s, s' \in [S]$,
\begin{equation}
	\KL(\tilde{p}^{(v)}_0(\,\cdot\mid  s),\tilde{p}^{(v)}_0(\,\cdot\mid  s'))
	=
	\KL(\tilde{p}^{(v)}_1(\,\cdot\mid  s),\tilde{p}^{(v)}_1(\,\cdot\mid  s'))
	=:
	\delta^{(v)}_{s,s'}
\end{equation}
say.
To prove \Cref{lem:bound_on_information_quantity}, it will therefore be sufficient to show that:
\begin{enumerate}[label=(\alph*)]
	\item
	      For any $v(\,\cdot\mid  s)\in [-1,+1]^{S/2}$ and $\Phi\in\Lambda^*(\varepsilon,\tilde{p}^{(v)})$,
	      \begin{equation}
		      I_{\tilde{s}}(x;c,\Phi,\pi)
		      \geq
		      c
			  \Bigl(1-\frac{2S}{n}\Bigr)
			  \min\Bigl\{
				1-\frac{2S}{n},(1-2\epsilon_{\max})(1-\kappa)
			  \Bigr\}
		      \delta^{(v)}_{\tilde{s},f(x)}
		      .
		      \label{eqn:I_lower_bound_in_terms_of_KL_ptilde}
	      \end{equation}

	\item
	      There exists a constant $\mf{d}>0$ independent of $n$ and $\eta$, and $v^*( \,\cdot\mid  s )\in[-1,+1]^{S/2}$ for $s\in[S]$ possibly dependent on $\mf{d}$, such that
	      \begin{equation}
		      \label{eqn:bound_on_KL_div}
		      \delta^{(v^*)}_{s,s'}
		      \geq
		      \mf{d}\kappa^2
		      .
	      \end{equation}
\end{enumerate}

\paragraph{Proof of (a)}

For this proof we will suppress $v$ in the notation and write $\tilde{p}_i = \tilde{p}_i^{(v)}$ for $i\in\{0,1\}$ and $\delta_{s,s'}=\delta^{(v)}_{s,s'}$.

We begin by proving that for all $s, s'$, $a$,
\begin{equation}
	\KL\bigl( p(\,\cdot\mid  s,a), p(\,\cdot\mid  s',a) \bigr)
	\geq
	\delta_{s,s'}
	.
	\label{eqn:KL_ps_psprime_geq_delta}
\end{equation}
We distinguish four cases, using \eqref{eqn:latent_state_transition_kernel_macro_def}:

\begin{itemize}
	\item[$a \neq a^*_s, a^*_{s'}$.]
	      Observe that
	      \begin{equation}
				\delta_{s,s'}
				=
				\frac{1}{2}
				\mathrm{KL}\bigl(
					\tilde{p}_0(\,\cdot\mid  s),\tilde{p}_0(\,\cdot\mid  s')
				\bigr)
				+
				\frac{1}{2}
				\mathrm{KL}\bigl(
					\tilde{p}_1(\,\cdot\mid  s),\tilde{p}_1(\,\cdot\mid  s')
				\bigr)
				.
				\label{eqn:delta_s_s_identity}
	      \end{equation}
	      We therefore have
	      \begin{align}
		       &
		      \KL\bigl( p(\,\cdot\mid  s,a), p(\,\cdot\mid  s',a) \bigr)
		      \\
		       &
		      \overset{\eqref{eqn:latent_state_transition_kernel_macro_def}}{=}
		      \frac{1}{2}
		      \sum_{s''\in\mcS_0}
		      \tilde{p}_0(s'' \mid s)
		      \ln\frac{\tilde{p}_0(s'' \mid s)}{\tilde{p}_0(s'' \mid s')}
		      +
		      \frac{1}{2}
		      \sum_{s''\in\mcS_1}
		      \tilde{p}_1(s'' \mid s)
		      \ln\frac{\tilde{p}_1(s'' \mid s)}{\tilde{p}_1(s'' \mid s')}
		      \overset{\eqref{eqn:delta_s_s_identity}}{=}
		      \delta_{s,s'}
		      .
		      \nonumber
	      \end{align}

	\item[$a=a^*_{s'}\neq a^*_s$.]
	      We have
	      \begin{align}
		       &
		      \KL\bigl(p(\,\cdot\mid  s,a^*_{s'}),p(\,\cdot\mid  s',a^*_{s'})\bigr)
		      \nonumber
		      \\
		       &
		      \overset{\eqref{eqn:latent_state_transition_kernel_macro_def}}{=}
		      \frac{1}{2}
		      \sum_{s''\in\mcS_0}
		      \tilde{p}_0(s'' \mid s)
		      \ln\frac{\tilde{p}_0(s'' \mid s)}{(1-2\varepsilon_{\iota(s')})\tilde{p}_0(s'' \mid s')}
		      \nonumber
		      +
		      \frac{1}{2}
		      \sum_{s''\in\mcS_1}
		      \tilde{p}_1(s'' \mid s)
		      \ln\frac{\tilde{p}_1(s'' \mid s)}{(1+2\varepsilon_{\iota(s')})\tilde{p}_1(s'' \mid s')}
		      \\
		       &
		      =
		      \frac{1}{2}
		      \sum_{s''\in\mcS_0}
		      \tilde{p}_0(s'' \mid s)
		      \ln\frac{1/2}{1/2-\varepsilon_{\iota(s')}}
		      +
		      \frac{1}{2}
		      \sum_{s''\in\mcS_1}
		      \tilde{p}_1(s'' \mid s)
		      \ln\frac{1/2}{1/2+\varepsilon_{\iota(s')}}
		      +
		      \delta_{s,s'}
		      \nonumber
		      \\
		       &
		      =
		      \KL
		      \bigl(
		      \mathrm{Ber}(1/2),\mathrm{Ber}(1/2-\varepsilon_{\iota(s')})
		      \bigr)
		      +
		      \delta_{s,s'}
		      \geq
		      \delta_{s,s'},
	      \end{align}
	      where we have used that $\sum_{s''\in\mcS_i}\tilde{p}_i(s'' \mid s)=1$ for each $s$, $i\in\{0,1\}$.

	\item[$a=a^*_s\neq a^*_{s'}$.]
	      We have
	      \begin{align}
		       &
		      \KL\bigl(p(\,\cdot\mid  s,a^*_{s}),p(\,\cdot\mid  s',a^*_{s})\bigr)
		      \nonumber
		      \\
		       &
		      \overset{\eqref{eqn:latent_state_transition_kernel_macro_def}}{=}
		      \frac{1-2\varepsilon_{\iota(s)}}{2}
		      \sum_{s''\in\mcS_0}
		      \tilde{p}_0(s'' \mid s)
		      \ln\frac{(1-2\varepsilon_{\iota(s)})\tilde{p}_0(s'' \mid s)}{\tilde{p}_0(s'' \mid s')}
		      \nonumber \\
		       &
		      \phantom{=} +
		      \frac{1+2\varepsilon_{\iota(s)}}{2}
		      \sum_{s''\in\mcS_1}
		      \tilde{p}_1(s'' \mid s)
		      \ln\frac{(1+2\varepsilon_{\iota(s)})\tilde{p}_1(s'' \mid s)}{\tilde{p}_1(s'' \mid s')}
		      \nonumber
		      \\
		       &
		      =
		      \mathrm{KL}\bigl(\mathrm{Ber}(1/2-\varepsilon_{\iota(s)}),\mathrm{Ber}(1/2)\bigr)
		      +
		      \delta_{s,s'}
		      \geq
		      \delta_{s,s'}.
	      \end{align}

	\item[$a = a^*_s = a^*_{s'}$.]
	      Finally,
	      \begin{align}
		       &
		      \KL(p(\,\cdot\mid  s,a^*_s),p(\,\cdot\mid  s',a^*_s))
		      \nonumber
		      \\
		       &
		      \overset{\eqref{eqn:latent_state_transition_kernel_macro_def}}{=}
		      \frac{1-2\varepsilon_{\iota(s)}}{2}
		      \sum_{s''\in\mcS_0}
		      \tilde{p}_0(s'' \mid s)
		      \ln\frac{\tilde{p}_0(s'' \mid s)}{\tilde{p}_0(s'' \mid s')}
		      +
		      \frac{1+2\varepsilon_{\iota(s)}}{2}
		      \sum_{s''\in\mcS_1}
		      \tilde{p}_1(s'' \mid s)
		      \ln\frac{\tilde{p}_1(s'' \mid s)}{\tilde{p}_1(s'' \mid s')}
		      \nonumber \\
		       &
		      =
		      \delta_{s,s'}
		      .
	      \end{align}

\end{itemize}

This proves \eqref{eqn:KL_ps_psprime_geq_delta}.
Combined with \eqref{eqn:I_lower_bound_in_terms_of_KL_p}, this implies that
\begin{equation}
	I_{\tilde{s}}(x;c,\Phi,\pi)
	\geq
	cnq(x\mid f(x))\sum_a\omega_{\Psi,\pi}(\tilde{s},a)\delta_{\tilde{s},f(x)}
	.
	\label{eqn:Itildes_intermediate_bound}
\end{equation}

What remains is to show that
\begin{equation}
	cnq(x\mid f(x))\allowbreak\sum_a\omega_{\Psi,\pi}(\tilde{s},a)
	\geq
	c\biggl(1-\frac{2S}{n}\biggr)
	\min\biggl\{
	1-\frac{2S}{n},
	(1-2\epsilon_{\max})(1-\kappa)
	\biggr\}
	.
	\label{eqn:bound_on_visitation_rate}
\end{equation}
Firstly, by \Cref{def:hard-to-learn-BMDP-instance,def:hard_to_learn_class} we have for $\Phi\in\Lambda(\varepsilon,\tilde{p})$ that $q(y \mid s)=1/|f^{-1}(s)|$ and $|f^{-1}(s)|\leq \floor{\ceil{n/2}/(S/2)}+1\leq n/S + 2$.
It follows that
\begin{equation}
	\label{eqn:bound_on_emission_probability}
	nq(x\mid f(x))
	=
	\frac{n}{|f^{-1}(f(x))|}
	\geq
	\frac{n}{\frac{n}{S}+2}
	=
	\frac{S}{1+\frac{2S}{n}}
	\geq
	S\biggl(1-\frac{2S}{n}\biggr)
	.
\end{equation}
Next, recall from the paragraph preceding \eqref{eqn:Ijxcdef} that $\Psi$ is obtained from $\Phi$ by reassigning context $x$ to latent state $\tilde{s}\neq f(x)$ and appropriately renormalizing the emission probabilities.
Let $f_{\Psi}$ and $q_{\Psi}$ denote the decoding function and emission probabilities of $\Psi$, respectively.
It follows from \Cref{lem:visitation_rate_bound} for $\mcX = f^{-1}_{\Psi}(\tilde{s})$ and $\sum_{y\in f^{-1}_{\Psi}(\tilde{s})}q_{\Psi}(y \mid \tilde{s})=1$ that
\begin{equation}
	\label{eqn:visitation_rate_bound_applied_to_omega}
	\begin{split}
		\sum_a
		\omega_{\Psi,\pi}(\tilde{s},a)
		 &
		\overset{\eqref{eqn:definition_of_visitation_probabilities}}{=}
		\frac{1}{H}
		\sum_a
		\sum_{h=1}^H
		\mbbP_{\Psi,\pi}[s_h=\tilde{s},a_h=a]
		=
		\frac{1}{H}
		\sum_{h=1}^H
		\mbbP_{\Psi,\pi}[s_h=\tilde{s}]
		\\
		 &
		=
		\frac{1}{H}
		\sum_{h=1}^H
		\mbbP_{\Psi,\pi}[x_h\in f_{\Psi}^{-1}(\tilde{s})]
		\geq
		\min
		\Bigl\{
		\frac{|f_{\Psi}^{-1}(\tilde{s})|}{n}
		,
		\min_{s,a}
		p(\tilde{s} \mid s,a)
		\Bigr\}
		.
	\end{split}
\end{equation}
Then note that because $x$ was moved to $\tilde{s}$, $f^{-1}_{\Psi}(\tilde{s})= f^{-1}(\tilde{s})+1 \geq \floor{\floor{n/2}/(S/2)} \geq n/S-2$.
Moreover, by combining \eqref{eqn:latent_state_transition_kernel_macro_def} and \eqref{eqn:p_tilde_choice_with_packing} we also find that $\min_{s,a}p(\tilde{s}\mid s,a)\geq (1-2\epsilon_{\max})(1-\kappa)/S$ since $v(s' \mid s)\geq -1$.
Together with \eqref{eqn:visitation_rate_bound_applied_to_omega}, these observations imply that
\begin{equation}
	\label{eqn:bound_on_omega_psi}
	\sum_a
	\omega_{\Psi,\pi}(\tilde{s},a)
	\geq
	\frac{1}{S}
	\min
	\biggl\{1-\frac{2S}{n},(1-2\epsilon_{\max})(1-\kappa)\biggr\}
	.
\end{equation}
\Cref{eqn:bound_on_visitation_rate} now follows by combining \eqref{eqn:bound_on_emission_probability} and \eqref{eqn:bound_on_omega_psi}.
This completes the proof of \eqref{eqn:I_lower_bound_in_terms_of_KL_ptilde}.

\paragraph{Proof of (b)}
Use Pinsker's inequality to write
\begin{equation}
	\delta^{(v)}_{s,s'}
	\geq
	\frac{1}{2}
	\Bigl(
	\sum_{s''\in\mcS_0}
	\bigl|
	\tilde{p}^{(v)}_0(s'' \mid s)-\tilde{p}^{(v)}_0(s'' \mid s')
	\bigr|
	\Bigr)^2
	=
	\frac{2\kappa^2}{S^2}
	\Bigl(
	\sum_{s''\in\mcS_0}
	\bigl|
	v(s'' \mid s)-v(s'' \mid s')
	\bigr|
	\Bigr)^2
	.
	\label{eqn:pinsker_applied}
\end{equation}
We will proceed with case distinctions on the value of $S$.
Recall that $S/4 \in \mbbN_{\geq 2}$ by \eqref{eqn:parameter_assumptions}.

\begin{itemize}
	\item[Case $\frac{S}{4}\geq 12$.]
	      Let $\gamma\eqdef \sqrt{(12/S)\ln S}$.
		  Observe that $\gamma\in (0,1)$ for $S/4\geq 12$, and that $2\ln(S)\leq (S/4)\gamma^2-\ln(S)$.
		  Upon identifying $M \equiv S$, $n\equiv S/4$, $A \equiv 1$, \cite[Lemma~D.6]{jin2020reward} now ensures that there exist vectors $v^*( \,\cdot\mid  s )\in\{-1,+1\}^{S/2}$ such that for all $s,s'$ such that $s\neq s'$,
	      \begin{equation}
		      \label{eqn:uncorrelated_packing}
		      \bigl\langle
		      v^*( \,\cdot\mid  s ),v^*(\,\cdot\mid  s')
		      \bigr\rangle
		      \leq
		      \frac{\gamma S}{2}
			  =
		      \sqrt{3S \ln S}
		      .
	      \end{equation}
	      Then observe that because $v^*( \,\cdot\mid  s )\in\{-1,+1\}^{S/2}$ for each $s$,
	      \begin{equation}
		      \sum_{s''\in\mcS_0}
		      \bigl|v^*(s'' \mid s)-v^*(s'' \mid s')\bigr|
		      =
		      \frac{S}{2}
		      -
		      \langle v^*( \,\cdot\mid  s ), v^*(\,\cdot\mid  s')\rangle
		      \overset{\eqref{eqn:uncorrelated_packing}}{\geq}
		      \frac{S}{2}
		      (1-\gamma).
	      \end{equation}
	      Together with \eqref{eqn:pinsker_applied} we conclude that
	      \begin{equation}
		      \label{eqn:lower_bound_on_KL_divergence}
		      \delta_{s,s'}
		      \geq
		      \frac{1}{2}
		      (1-\gamma)^2\kappa^2
		      \geq
		      \mf{d}_1
		      \kappa^2
		      ,
	      \end{equation}
	      where the final inequality holds for some $\mf{d}_1$ that is independent of $n$ and $\eta$ since $\gamma<1$ for $S/4\geq 12$.

	\item[Case $2\leq\frac{S}{4}< 12$.]
	      We need a different argument for when $S/4<12$ because $\gamma$ then exceeds 1, in which case \cite[Lemma D.6]{jin2020reward} no longer applies.
	      This occurs because \cite[Lemma D.6]{jin2020reward} uses a probabilistic argument to show that a packing satisfying \eqref{eqn:uncorrelated_packing} exists, and the concentration inequalities used therein are not sharp enough to prove this for small $S$.

	      We construct instead an explicit packing of vectors $v^*( \,\cdot\mid  s )\in[-1,+1]^{S/2}$ for $s\in[S]$.
	      This construction fails for large $S$, however, because the resulting lower bound for $\delta_{s,s'}$ is $O(1/S^2)$, whereas that in \eqref{eqn:lower_bound_on_KL_divergence} is $\Omega(1)$.

	      Specifically, for $s\in\mcS_0$, let $v(s'' \mid s)=1$ when $s''=s$ and let $v(s'' \mid s)=-1/(S/2-1)$ when $s''\in \mcS_0\setminus\{s\}$.
	      Meanwhile, for $s\in\mcS_1$ let $v^*(s'' \mid s)=-1$ when $s''=s-S/2$ and let $v^*(s'' \mid s)=1/(S/2-1)$ when $s''\in\mcS_0\setminus\{s-S/2\}$.
	      One can then explicitly verify that for $s,s'\in [S]$ such that $s\neq s'$,
	      \begin{equation}
		      \sum_{s''\in\mcS_0}
		      |v^*(s'' \mid s)-v^*(s'' \mid s')|
		      \geq
		      2\biggl(1-\frac{1}{S/2-1}\biggr)
		      .
	      \end{equation}
	      Together with \eqref{eqn:pinsker_applied} it follows that
	      \begin{equation}
		      \delta_{s,s'}
		      \geq
		      \frac{4\kappa^2}{S^2}
		      \biggl(1-\frac{1}{S/2-1}\biggr)^2
		      \geq
		      \mf{d}_2\kappa^2,
	      \end{equation}
	      where the final inequality holds for some $\mf{d}_2>0$ that is independent of $n$ and $\eta$ since $S/4<68$.
\end{itemize}
Letting $\mf{d}=\mf{d}_1\wedge \mf{d}_2$ proves \eqref{eqn:bound_on_KL_div} for all $S/4\in\mbbN_{\geq 2}$.

\subsection{Proof of \texorpdfstring{\Cref{lem:regret_decomposition_lower_bound}}{the regret decomposition for the lower bound}}
\label{sec:proof_of_regret_decomposition_lower_bound}

Define for $x\in[n]$, $a\in[A]$, $h\in[H]$,
\begin{equation}
	\gap_h(x,a)
	\eqdef
	V^{\pi^*}_{h}(x) - Q^{\pi^*}_{h}(x,a)
	.
\end{equation}

We claim first that
for $\varepsilon \in [0, \tfrac{1}{2}]^2$,
$\Phi \in \Lambda^*(\varepsilon)$,
\begin{equation}
	\reg_K(\mcL;\Phi)
	=
	\sum_{k = 1}^K
	\sum_{h = 1}^H
	\sum_x
	\sum_a
	\mbbE_{\Phi,\mcL}[\mathbbm{1}\{x_{k,h}=x,a_{k,h}=a\}]\gap_h(x,a)
	.
	\label{eqn:standard_mdp_regret_decomposition}
\end{equation}
Indeed: \eqref{eqn:standard_mdp_regret_decomposition} follows from \cite[Proposition 4]{tirinzoni2021fully} when viewing $\Phi$ as an \gls{MDP} with state space $[n]$ and identifying
$\mf{A} \equiv \mcL$,
$\mcM \equiv \Phi$,
$\mcS \equiv [n]$,
$s \equiv x$,
$\mcA \equiv [A]$,
$N_{K,h}(s,a) \equiv \sum_{k=1}^K\mathbbm{1}\{x_{k,h}=x,a_{k,h}=a\}$,
and
$\Delta_{\mcM,h}(s,a)\equiv \gap_h(x,a)$.

We claim second that
if $\varepsilon_1 \geq \varepsilon_0$,
then
for any $\Phi \in \Lambda^*(\varepsilon)$ and $h\in[H]$,
\begin{equation}
	\gap_h(x,a)
	\geq
	\varepsilon_{\iota(f(x))} \mathbbm{1}\{a\neq a^*_{f(x)}, h\leq H-1\}
	.
	\label{eqn:suboptimality_gap_bound}
\end{equation}
Recall that $\iota$ is defined in \eqref{def:iota}.

If \eqref{eqn:suboptimality_gap_bound} is indeed true, then \Cref{lem:regret_decomposition_lower_bound} follows from
(i) \Cref{eqn:standard_mdp_regret_decomposition,eqn:suboptimality_gap_bound};
(ii) the definitions of states, contexts, and $f$;
and
(iii) the fact that $\mathcal{S}_0, \mathcal{S}_1$ partitions the set $[S]$:
\begin{equation}
	\begin{split}
		 &
		\reg_K(\mcL;\Phi)
		\\
		 &
		\overset{(i)}{\geq}
		\sum_k
		\sum_{h=1}^H
		\sum_x
		\sum_a
		\varepsilon_{\iota(f(x))}
		\mbbE_{\Phi,\mcL}
		\bigl[
		\mathbbm{1}\{x_{k,h}=x,a_{k,h}=a\}
		\bigr]
		\mathbbm{1}\{a\neq a^*_{f(x)}, h\leq H-1\}
		\\
		 &
		=
		\sum_k
		\sum_{h=1}^{H-1}
		\sum_x
		\varepsilon_{\iota(f(x))}
		\mbbE_{\Phi,\mcL}
		\bigl[
			\mathbbm{1}\{x_{k,h}=x\}
			(
			1-\mathbbm{1}\{a_{k,h} = a_{f(x)}^*\}
			)
			\bigr]
		\\
		 &
		\overset{(ii)}{=}
		\sum_k
		\sum_{h=1}^{H-1}
		\sum_{s}
		\varepsilon_{\iota(s)}
		\mbbE_{\Phi,\mcL}
		\bigl[
			\mathbbm{1}\{x_{k,h}\in f^{-1}(s)\}(1-\mathbbm{1}\{a_{k,h} = a_{s}^*\})
			\bigr]
		\\
		 &
		\overset{(iii)}{=}
		\sum_k
		\sum_{h=1}^{H-1}
		\sum_{i=0,1}
		\sum_{s\in\mcS_i}
		\varepsilon_i
		\mbbE_{\Phi,\mcL}
		\bigl[
			\mathbbm{1}\{s_{k,h}=s\}(1-\mathbbm{1}\{a_{k,h} = a_s^*\})
			\bigr]
		.
	\end{split}
\end{equation}
We will therefore pursue a proof of \eqref{eqn:suboptimality_gap_bound}.

\subsubsection{Proof of \texorpdfstring{\eqref{eqn:suboptimality_gap_bound}}{suboptimality gap bound}}

Note that $\gap_H(x,a) \geq 0$ by \eqref{eqn:backward_value_iteration}.
This proves \eqref{eqn:suboptimality_gap_bound} for $h = H$, and for $a = a_{f(x)}^*$.
It remains to prove \eqref{eqn:suboptimality_gap_bound} for $h\leq H-1$ and $a\neq a_{f(x)}^*$.

Recall \Cref{def:hard-BMDP--latent-state-transition-probabilities,def:hard-BMDP--unit-rewards} and \eqref{def:BMDP-transition-kernel}.
Notice that since $r_h(x,a)$ and $P(\,\cdot\mid x,a)$ depend only on the latent state $f(x)$, \eqref{eqn:bellman_equation} also implies that $V^{\pi^*}_{h+1}(x)$ depends only on the latent state of $x$.
In fact, we claim that for all $h\in[H-1]$
\begin{equation}
	\label{eqn:induction_claim}
	\exists v_{h+1}\in\mbbR_{\geq0}^2:
	V^{\pi^*}_{h+1}(x)
	=
	v_{h+1}(\iota(f(x)))
	\quad \mathrm{and}\quad
	v_{h+1}(1)-v_{h+1}(0)\geq 1.
\end{equation}
This is clear when $h=H-1$, since $V^{\pi^*}_{H}(x)=\max_{a}r_H(x,a)=\mathbbm{1}\{\iota(f(x)) = 1\}$.
We can then proceed inductively.

Assume that \eqref{eqn:induction_claim} holds for some $h\in[H-1]$.
We then find using \eqref{def:BMDP-transition-kernel}, $\sum_{y\in f^{-1}(s)}q(y \mid s)=1$, and \eqref{eqn:latent_state_transition_kernel_macro_def} that
\begin{equation}
	\label{eqn:expression_next_context_value}
	\begin{split}
		 &
		\langle P(\,\cdot\mid x,a), V^{\pi^*}_{h+1}\rangle
		\\
		 &
		=
		\sum_{i=0,1}
		v_{h+1}(i)
		\sum_{s\in\mcS_i}
		p(s \mid f(x),a)
		\sum_{y\in f^{-1}(s)}
		q(y \mid s)
		=
		\sum_{i=0,1}
		v_{h+1}(i)
		\sum_{s\in\mcS_i}
		p(s \mid f(x),a)
		\\
		 &
		=
		\begin{cases}
			\frac{v_{h+1}(1)+v_{h+1}(0)}{2}
			 &
			\textnormal{if } a\neq a_{f(x)}^*,
			\\
			\frac{v_{h+1}(1)+v_{h+1}(0)}{2}
			+
			\varepsilon_{\iota(f(x))}(v_{h+1}(1)-v_{h+1}(0))
			 &
			\textnormal{if } a=a_{f(x)}^*
			.
		\end{cases}
	\end{split}
\end{equation}
Consequently, using \Cref{def:hard-BMDP--unit-rewards}, \eqref{eqn:expression_next_context_value}, and that $v_{h+1}(1)-v_{h+1}(0)\geq 1$ by the induction hypothesis, we also have that
\begin{align}
	V^{\pi^*}_{h}(x)
	 &
	=
	\max_{a}
	\{
	r_{h}(x,a)
	+
	\langle P(\,\cdot\mid  x,a), V^{\pi^*}_{h+1} \rangle
	\}
	\nonumber
	\\
	 &
	=
	\mathbbm{1}\{\iota(f(x)) = 1\}
	+
	\max_{a}
	\{
	\langle P(\,\cdot\mid  x,a), V^{\pi^*}_{h+1} \rangle
	\}
	\\
	 &
	=
	\mathbbm{1}\{\iota(f(x)) = 1\}
	+
	\frac{v_{h+1}(1)+v_{h+1}(0)}{2}
	+
	\varepsilon_{\iota(f(x))}(v_{h+1}(1)-v_{h+1}(0))
	\nonumber
	\\
	 &
	\defeq
	v_{h}(\iota(f(x)))
	.
	\nonumber
\end{align}
Because $\varepsilon_{1}\geq \varepsilon_{0}$ by assumption it follows that
\begin{equation}
	v_{h}(1)- v_{h}(0)
	=
	1
	+
	(\varepsilon_{1}-\varepsilon_{0})(v_{h+1}(1)-v_{h+1}(0))
	\geq
	1
	.
\end{equation}
We conclude by induction from the base case $h=H-1$ that \eqref{eqn:induction_claim} is true for all $h\in[H-1]$.

Now, by using \eqref{eqn:bellman_equation} and that $r_h(x,a)$ is the same for all $a$, we find that
\begin{equation}
	\label{eqn:bellman_applied_to_gaps}
	\begin{split}
		\gap_h(x,a)
		 &
		=
		V^{\pi^*}_{h}(x)-Q^{\pi^*}_h(x,a)
		\\
		 &
		=
		\max_{a'}
		\{
		r_h(x,a')-r_h(x,a)
		+
		\langle P(\,\cdot\mid  x,a')-P(x,a),V^{\pi^*}_{h+1}\rangle
		\}
		\\
		 &
		=
		\max_{a'}
		\langle P(\,\cdot\mid  x,a')-P(\,\cdot\mid  x,a),V^{\pi^*}_{h+1}\rangle
		.
	\end{split}
\end{equation}
Then using \eqref{eqn:induction_claim} and \eqref{eqn:expression_next_context_value}, we find that if $a\neq a_{f(x)}^*$ and $h\in [H-1]$
\begin{align}
	\gap_h(x,a)
	 &
	\overset{\eqref{eqn:bellman_applied_to_gaps}}{=}
	\max_{a'}
	\langle P(\,\cdot\mid x,a')-P(\,\cdot\mid x,a),V^{\pi^*}_{h+1}\rangle
	\nonumber
	\\
	 &
	=
	\max_{a'}
	(\langle P(\,\cdot\mid x,a'), V^{\pi^*}_{h+1}\rangle - \langle P(\,\cdot\mid x,a),V^{\pi^*}_{h+1}\rangle)
	\nonumber
	\\
	 &
	\overset{\eqref{eqn:expression_next_context_value}}{=}
	\max
	\{ 0, \varepsilon_{\iota(f(x))}(v_{h+1}(1)-v_{h+1}(0))\}
	\nonumber
	\\
	 &
	=
	\varepsilon_{\iota(f(x))}(v_{h+1}(1)-v_{h+1}(0))
	\overset{\eqref{eqn:induction_claim}}{\geq}
	\varepsilon_{\iota(f(x))}
	.
\end{align}
Since we already proved \eqref{eqn:suboptimality_gap_bound} for $a=a_{f(x)}^*$ or $h=H$, we are done.
\qed

\subsection{Proof of \texorpdfstring{\Cref{lem:change_of_measure_argument}}{change of measure argument}}
\label{sec:proof_of_change_of_measure_argument}

Denote by $p$ and $p^{(0)}$ the latent state transition kernels of $\Phi$ and $\Phi^{(0)}$, respectively.
Let $P$ and $P^{(0)}$ be the corresponding transition kernels on the contexts.

We set up to apply \cite[Lemma H.1]{NEURIPS2019_10a5ab2d}.
View $\Phi^{(0)}$ and $\Phi$ as ordinary \glspl{MDP} with state space $[n]$ and identify $\mcM \equiv \Phi^{(0)}$ and $\mcM' \equiv \Phi$.
Moreover, let $\delta_z$ denote the Dirac measure at $z\in[0,1]$ and identify $\nu^{\mcM}(x,a)\equiv P^{(0)}(\,\cdot\mid x,a)\otimes \delta_{r_h(x,a)}$ and $\nu^{\mcM'}(x,a)\equiv P(\,\cdot\mid x,a)\otimes \delta_{r_h(x,a)}$, where $\mu\otimes\nu$ denotes the product measure.
This follows because, conditional on $x_{k,h}=x$ and $a_{k,h}=a$, our rewards are deterministic.
Finally note that although learning algorithms in \cite{NEURIPS2019_10a5ab2d} do not depend on a reward function, this poses no issue for us because $r$ is the same for $\Phi$ and $\Phi^{(0)}$.
We may therefore regard $\tilde{\mcL}=\{\tilde{\mcL}_k\}_{k\in\mbbN_+}$ with $\tilde{\mcL}_k\eqdef \mcL_k(\cdot,r)$ as a valid learning algorithm in the sense of \cite{NEURIPS2019_10a5ab2d}.

Given these identifications, we apply \cite[Lemma H.1]{NEURIPS2019_10a5ab2d} to
\begin{equation}
	Z
	\eqdef
	\frac{1}{T_K-K}\sum_{s\in\mcS_0}\sum_{k=1}^K\sum_{h=1}^{H-1}\mathbbm{1}\{s_{k,h}=s,a_{k,h}=a_s^*\}.
\end{equation}
Note that $Z\in [0,1]$ because $T_K-K = K(H-1)$, and that $Z$ is a measurable function of the observations from the first $K$ episodes.
The random variable $Z$ therefore satisfies the assumptions of \cite[Lemma H.1]{NEURIPS2019_10a5ab2d}.
This implies the following:
\begin{align}
	 &
	\textnormal{kl}(
	\mbbE_{\Phi^{(0)},\mcL}[Z],
	\mbbE_{\Phi,\mcL}[Z]
	)
	\label{eqn:data_processing_inequality}
	\\
	 &
	\leq
	\sum_{x,a}
	\KL\bigl(P^{(0)}(\,\cdot\mid x,a)\otimes \delta_{r_h(x,a)},P(\,\cdot\mid x,a)\otimes \delta_{r_h(x,a)}\bigr)
	\sum_{k=1}^K
	\sum_{h=1}^H
	\mbbE_{\Phi^{(0)},\mcL}[\mathbbm{1}\{x_{k,h}=x,a_{k,h}=a\}]
	\nonumber
	\\
	 &
	=
	\sum_{x,a}
	\KL\bigl(P^{(0)}(\,\cdot\mid x,a),P(\,\cdot\mid x,a)\bigr)
	\sum_{k=1}^K
	\sum_{h=1}^H
	\mbbE_{\Phi^{(0)},\mcL}[\mathbbm{1}\{x_{k,h}=x,a_{k,h}=a\}]
	,
	\nonumber
\end{align}
where $\kl(x,y)=x\ln(x/y)+(1-x)\ln((1-x)/(1-y))$ for $x,y\in [0,1]$.

Recall from \Cref{def:hard_to_learn_class} and \eqref{eqn:hard_to_learn_class_def} that both $p$ and $p^{(0)}$ are given by \eqref{eqn:latent_state_transition_kernel_macro_def} for the same pair $(\tilde{p}^*_0,\tilde{p}^*_1)$ but with different values for $\varepsilon$.
The KL-divergence between $P(\,\cdot\mid x,a)$ and $P^{(0)}(\,\cdot\mid x,a)$ is therefore nonzero only for $x\in f^{-1}(\mcS_0)$ and $a=a^*_{f(x)}$.
For these contexts and actions, we can evaluate the KL-divergence explicitly using \eqref{eqn:latent_state_transition_kernel_macro_def} and $\sum_{y\in f^{-1}(s')}q(y\mid s')=1$ to find that
\begin{align}
		&
	\KL\bigl(
	P^{(0)}(\,\cdot\mid x,a_{f(x)}^*)
	,
	P(\,\cdot\mid x,a_{f(x)}^*)
	\bigr)
	\nonumber
	\\
		&
	=
	\sum_{y}
	P^{(0)}(y\mid x,a_{f(x)}^*)\ln\frac{P^{(0)}(y\mid x,a_{f(x)}^*)}{P(y\mid x,a_{f(x)}^*)}
	=
	\sum_{s'}
	p^{(0)}(s'\mid f(x),a_{f(x)}^*)\ln\frac{p^{(0)}(s'\mid f(x),a_{f(x)}^*)}{p(s'\mid f(x),a_{f(x)}^*)}
	\nonumber
	\\
		&
	=
	\sum_{s'\in\mcS_0}
	\frac{1}{2}
	\tilde{p}^*_0(s' \mid f(x))\ln \frac{\tilde{p}^*_0(s' \mid s)}{(1-2\varepsilon_0)\tilde{p}^*_0(s' \mid s)}
	+
	\sum_{s'\in\mcS_1}
	\frac{1}{2}
	\tilde{p}^*_1(s' \mid s)\ln \frac{\tilde{p}^*_1(s' \mid s)}{(1+2\varepsilon_0)\tilde{p}^*_1(s' \mid s)}
	\nonumber
	\\
		&
	=
	\frac{1}{2}
	\ln \frac{1}{1+2\varepsilon_0}
	+
	\frac{1}{2}
	\ln\frac{1}{1-2\varepsilon_0}
	.
	\label{eqn:KL_divergence_evaluated}
\end{align}
We bound the right-hand side of \eqref{eqn:KL_divergence_evaluated} using the inequality $-\ln (1+x)\leq 1/(1+x)-1$ for $x\in(-1,1)$, followed by the assumption $\varepsilon_0\leq 1/8$:
\begin{equation}
	\label{eqn:KL_divergence_between_kernels_bound}
	\begin{split}
		\KL(P^{(0)}(\,\cdot\mid x,a_s^*),P(\,\cdot\mid x,a_s^*))
		\leq
		\frac{1}{2}
		\biggl(
		\frac{1}{1+2\varepsilon_0}
		+
		\frac{1}{1-2\varepsilon_0}
		\biggr)
		-
		1
		=
		\frac{4\varepsilon_0^2}{1-4\varepsilon_0^2}
		\leq
		8\varepsilon_0^2.
	\end{split}
\end{equation}
Combining \eqref{eqn:data_processing_inequality} with \eqref{eqn:KL_divergence_between_kernels_bound} and recalling that $\KL(P^{(0)}(\,\cdot\mid x,a),P(\,\cdot\mid x,a))=0$ unless $x\in f^{-1}(\mcS_0)$ and $a=a^*_{f(x)}$, we conclude that
\begin{equation}
	\label{eqn:data_processing_inequality_with_bounded_KL_divergence}
	\begin{split}
		8\varepsilon_0^2
		\sum_{s\in\mcS_0}
		\sum_{k=1}^K
		\sum_{h=1}^H
		\mbbE_{\Phi^{(0)},\mcL}[\mathbbm{1}\{s_{k,h}=s,a_{k,h}=a_s^*\}]
		\geq
		\textnormal{kl}(
		\mbbE_{\Phi^{(0)},\mcL}[Z],
		\mbbE_{\Phi,\mcL}[Z]
		)
		.
	\end{split}
\end{equation}
Note that $\kl(u,v)$ for $u,v\in[0,1]$ equals the KL-divergence between two Bernoulli random variables with parameters $u$ and $v$.
Hence, by Pinsker's inequality $\sqrt{\kl(u,v)}\geq \sqrt{2}|u-v|$.
Setting $u=\mbbE_{\Phi^{(0)},\mcL}[Z]$ and $v=\mbbE_{\Phi,\mcL}[Z]$ and taking square roots of both sides of \eqref{eqn:data_processing_inequality_with_bounded_KL_divergence} it follows that
\begin{align}
	\label{eqn:change_of_measure_argument_last_step}
	 &
	\sqrt{8}\varepsilon_0
	\sqrt{\sum_{s\in\mcS_0}\sum_{k=1}^K\sum_{h=1}^H\mbbE_{\Phi^{(0)},\mcL}[\mathbbm{1}\{s_{k,h}=s,a_{k,h}=a_s^*\}]}
	\geq
	\sqrt{2}|\mbbE_{\Phi^{(0)},\mcL}[Z]-\mbbE_{\Phi,\mcL}[Z]|
	\\
	 &
	=
	\frac{\sqrt{2}}{T_K-K}
	\Bigl|
	\sum_{s\in\mcS_0}
	\sum_{k=1}^K
	\sum_{h=1}^{H-1}
	\bigl(
	\mbbE_{\Phi,\mcL}
	[\mathbbm{1}\{s_{k,h}=s,
		a_{k,h}=a_s^*\}]
	-
	\mbbE_{\Phi^{(0)},\mcL}[\mathbbm{1}\{s_{k,h}=s,a_{k,h}=a_s^*\}]
	\bigr)
	\Bigr|
	,
	\nonumber
\end{align}
Dividing both sides of \eqref{eqn:change_of_measure_argument_last_step} by $\sqrt{2}$ gives \eqref{eqn:change_of_measure_argument}.
\qed

\subsection{Proof of \texorpdfstring{\Cref{lem:sum_over_optimal_actions}}{sum over optimal actions bound}}
\label{sec:proof_of_sum_over_optimal_actions}

First note from \Cref{def:hard-to-learn-BMDP-instance,def:hard_to_learn_class} that if $\varepsilon_0,\varepsilon_1>0$, then $\Phi(\varepsilon,\tilde{p},f,a^*)= \Phi(\varepsilon,\tilde{p},f',{a'}^*)\in \Lambda(\varepsilon,\tilde{p})$ if and only if $f=f'$ and $a^*={a'}^*$.
It then follows from \Cref{def:identifiable_hard_to_learn_class} that we may write
\begin{equation}
	\label{eqn:decompose_sum_over_BMDPs}
	\begin{split}
		 &
		\frac{1}{|\Lambda^*(\varepsilon)|}
		\sum_{\Phi\in\Lambda^*(\varepsilon)}
		\sum_{s\in\mcS_0}
		\sum_{k=1}^K
		\sum_{h=1}^{H-1}
		\mbbE_{\Phi,\mcL}[\mathbbm{1}\{s_{k,h}=s,a_{k,h}=a_s^*\}]
		\\
		 &
		=
		\frac{1}{|\mcF|}
		\sum_{f\in\mcF}
		\frac{1}{|\mcA|}
		\sum_{a^*\in\mcA}
		\sum_{s\in\mcS_0}
		\sum_{k=1}^K
		\sum_{h=1}^{H-1}
		\mbbE_{\Phi(\varepsilon,\tilde{p}^*,f,a^*),\mcL}[\mathbbm{1}\{s_{k,h}=s,a_{k,h}=a_s^*\}]
		.
	\end{split}
\end{equation}
We will prove \Cref{lem:sum_over_optimal_actions} by showing that for $f\in\mcF$,
\begin{equation}
	\label{eqn:sum_over_optimal_actions_proof_objective}
	\frac{1}{|\mcA|}
	\sum_{a^*\in\mcA}
	\sum_{s\in\mcS_0}
	\sum_{k=1}^K
	\sum_{h=1}^{H-1}
	\mbbE_{\Phi(\varepsilon,\tilde{p}^*,f,a^*),\mcL}[\mathbbm{1}\{s_{k,h}=s,a_{k,h}=a_s^*\}]
	\leq
	\frac{T_K-K}{2A}
	+
	2\varepsilon_0(T_K-K)\sqrt{\frac{T_K}{2A}}
	.
\end{equation}
Results \eqref{eqn:decompose_sum_over_BMDPs} and \eqref{eqn:sum_over_optimal_actions_proof_objective} would then immediately imply \eqref{eqn:sum_over_optimal_actions}.

Hence, fix $f\in\mcF$.
Recall that $\Phi^{(0)}\in \Lambda^*((0,\varepsilon_1))$ denotes the \gls{BMDP} obtained from $\Phi\in\Lambda^*((\varepsilon_0,\varepsilon_1))$ by setting the parameter $\varepsilon_0$ in \eqref{eqn:latent_state_transition_kernel_macro_def} to zero.
Let $\Psi:\Lambda^*((\varepsilon_0,\varepsilon_1))\rightarrow \Lambda^*((0,\varepsilon_1))$ map $\Psi(\Phi)=\Phi^{(0)}$.
We then use \Cref{lem:change_of_measure_argument} to bound the left-hand side of \eqref{eqn:sum_over_optimal_actions_proof_objective} as follows:

\begin{align}
	 &
	\frac{1}{|\mcA|}
	\sum_{a^*\in\mcA}
	\sum_{s\in\mcS_0}
	\sum_{k=1}^K
	\sum_{h=1}^{H-1}
	\mbbE_{\Phi(\varepsilon,\tilde{p}^*,f,a^*),\mcL}[\mathbbm{1}\{s_{k,h}=s,a_{k,h}=a_s^*\}]
	\nonumber
	\\
	 &
	\leq
	\frac{1}{|\mcA|}
	\sum_{a^*\in\mcA}
	\sum_{s\in\mcS_0}
	\sum_{k=1}^K
	\sum_{h=1}^{H-1}
	\mbbE_{\Psi(\Phi(\varepsilon,\tilde{p}^*,f,a^*)),\mcL}[\mathbbm{1}\{s_{k,h}=s,a_{k,h}=a_s^*\}]
	\label{eqn:optimal_actions_under_true_model_bound}
	\\
	 &
	+
	2\varepsilon_0(T_K-K)
	\frac{1}{|\mcA|}
	\sum_{a^*\in\mcA}
	\sqrt{
	\sum_{s\in\mcS_0}
	\sum_{k=1}^K
	\sum_{h=1}^H
	\mbbE_{\Psi(\Phi(\varepsilon,\tilde{p}^*,f,a^*)),\mcL}[\mathbbm{1}\{s_{k,h}=s,a_{k,h}=a_s^*\}]
	}
	.\nonumber
\end{align}
Our goal is to bound the two terms on the right-hand side of \eqref{eqn:optimal_actions_under_true_model_bound}.

We first perform a preliminary computation.
Let $\Phi\in \Lambda^*(\varepsilon)$ and let $p^{(0)}$ denote the latent state transition kernel of $\Phi^{(0)}=\Psi(\Phi)$.
Observe that because $q(\,\cdot \mid s)$ and $\tilde{p}^*_{\iota(s)}(\,\cdot\mid s)$ are probability distributions on $f^{-1}(s)$ and $\mcS_0$, respectively, \eqref{eqn:latent_state_transition_kernel_macro_def} implies that
\begin{equation}
	\label{eqn:upper_bound_on_transitions_to_S0}
	\max_{s,a}
	\sum_{y\in f^{-1}(\mcS_0)}p^{(0)}(f(y) \mid s,a)q(y\mid f(y))
	=
	\max_{s,a}
	\sum_{s'\in\mcS_0}p^{(0)}(s' \mid s,a)
	=
	\frac{1}{2}\max_{s}\sum_{s'\in\mcS_0}\tilde{p}^*_{0}(s'\mid s)
	=
	\frac{1}{2}.
\end{equation}
Applying \Cref{lem:visitation_rate_bound} for $\mcX=f^{-1}(\mcS_0)=\{1,\ldots,\floor{n/2}\}$ while using \eqref{eqn:upper_bound_on_transitions_to_S0}, and then identifying $K(H-1)=T_K-H$ gives
\begin{equation}
	\label{eqn:maximum_visitation_rate_of_S0}
	\sum_{k=1}^K
	\sum_{h=1}^{H-1}
	\sum_{s\in\mcS_0}
	\mbbE_{\Phi^{(0)},\mcL}[\mathbbm{1}\{s_{k,h}=s\}]
	\leq
	\sum_{k=1}^K
	\sum_{h=1}^{H-1}
	\max
	\biggl\{\frac{\floor{n/2}}{n},\frac{1}{2}\biggr\}
	\leq
	\frac{K(H-1)}{2}
	=
	\frac{T_K-K}{2}
	.
\end{equation}

Next, observe from \eqref{eqn:latent_state_transition_kernel_macro_def} that if $\varepsilon_0 = 0$, then the transition probabilities from $s\in\mcS_0$ are the same for all $a$.
In this case, $p$ and hence $\Phi$ are independent of $a_s^*$ for $s\in\mcS_0$.
We make this explicit by defining
\begin{equation}
	\bar{\Phi}^{(0)}(\varepsilon,\tilde{p}^*,f,\{a_s^*\}_{s\in \mcS_1})
	\eqdef
	\Psi(\Phi(\varepsilon,\tilde{p}^*,f,\{a^*\}_{s\in\mcS_0\cup\mcS_1}))
	,
\end{equation}
where this definition does not depend on the value of $\{a_s^*\}_{s\in \mcS_0}$ on the right-hand side.

Now fix $\{a_s^*\}_{s \in \mcS_1}$.
Since $\mcS_0 = \{1,\ldots,S/2\}$, it follows that
\begin{align}
	 &
	\frac{1}{A^{S/2}}
	\sum_{\{a_s^*\}_{s \in \mcS_0}}
	\sum_{s\in\mcS_0}
	\sum_{k=1}^K
	\sum_{h=1}^{H-1}
	\mbbE_{\Psi(\Phi(\varepsilon,\tilde{p}^*,f,\{a^*\}_{s\in\mcS_0\cup\mcS_1})),\mcL}[\mathbbm{1}\{s_{k,h}=s,a_{k,h}=a_s^*\}]
	\nonumber
	\\
	 &
	=
	\frac{1}{A^{S/2}}
	\sum_{\{a_s^*\}_{s \in \mcS_0}}
	\sum_{s\in\mcS_0}
	\sum_{k=1}^K
	\sum_{h=1}^{H-1}
	\mbbE_{\bar{\Phi}^{(0)}(\varepsilon,\tilde{p}^*,f,\{a^*_s\}_{s \in \mcS_1}),\mcL}[\mathbbm{1}\{s_{k,h}=s,a_{k,h}=a_s^*\}]
	\nonumber
	\\
	 &
	=
	\frac{1}{A^{S/2}}
	\sum_{s\in\mcS_0}
	\sum_{k=1}^K
	\sum_{h=1}^{H-1}
	\mbbE_{\bar{\Phi}^{(0)}(\varepsilon,\tilde{p}^*,f,\{a^*_s\}_{s \in \mcS_1}),\mcL}\Biggl[
	\sum_{\{a_s^*\}_{s \in \mcS_0}}
	\mathbbm{1}\{s_{k,h}=s,a_{k,h}=a_s^*\}
	\Biggr]
	\nonumber
	\\
	 &
	\overset{(i)}{=}
	\frac{A^{S/2-1}}{A^{S/2}}
	\sum_{s\in\mcS_0}
	\sum_{k=1}^K
	\sum_{h=1}^{H-1}
	\mbbE_{\bar{\Phi}^{(0)}(\varepsilon,\tilde{p}^*,f,\{a^*_s\}_{s \in \mcS_1}),\mcL}[\mathbbm{1}\{s_{k,h}=s\}]
	\overset{\eqref{eqn:maximum_visitation_rate_of_S0}}{\leq}
	\frac{T_K-K}{2A}
	,
	\label{eqn:optimal_actions_under_true_model_bound_term_1}
\end{align}
where in $(i)$ we have used that for $s\in\mcS_0$, since $\mcS_0 = \{1,\ldots,S/2\}$,
\begin{equation}
	\sum_{\{a_s^*\}_{s \in \mcS_0}}
	\mathbbm{1}\{a_{k,h}=a_s^*\}
	=
	\sum_{a^*_1,\ldots, a^*_{S/2}}
	\mathbbm{1}\{a_{k,h}=a_s^*\}
	=
	A^{S/2-1}
	.
\end{equation}
The same argument, together with the Cauchy--Schwarz inequality $\sum_{i}u_iv_i \leq \sqrt{\sum_i u_i^2}\sqrt{\sum_i v_i^2}$ for $u_i = A^{-S/2}\sqrt{\sum_{s\in\mcS_0}\sum_{k=1}^K\sum_{h=1}^{H}\mbbE_{\Psi(\Phi(\varepsilon,\tilde{p}^*,f,\{a^*\}_{s\in\mcS_0\cup\mcS_1})),\mcL}[\mathbbm{1}\{s_{k,h}=s,a_{k,h}=a_s^*\}]}$, $v_i = 1$ and $i=\{a_s^*\}_{s\in\mcS_0}$ gives
\begin{align}
	\sum_{i}u_i v_i
	 &
	\leq
	\sqrt{
		\frac{1}{A^{S/2}}
		\sum_{\{a_s^*\}_{s\in\mcS_0}}
		\sum_{s\in\mcS_0}
		\sum_{k=1}^K
		\sum_{h=1}^{H}
		\mbbE_{\Psi(\Phi(\varepsilon,\tilde{p}^*,f,\{a^*\}_{s\in\mcS_0\cup\mcS_1})),\mcL}[
		\mathbbm{1}\{s_{k,h}=s,a_{k,h}=a_s^*\}
		]
	}
	\nonumber
	\\
	 &
	\overset{\eqref{eqn:optimal_actions_under_true_model_bound_term_1}}{\leq}
	\sqrt{\frac{1}{2A^{S/2}}A^{S/2-1}T_K}
	=
	\sqrt{\frac{T_K}{2A}}.
	\label{eqn:optimal_actions_under_true_model_bound_term_2}
\end{align}

Finally, note from \Cref{def:hard_to_learn_class} that $\mcA$ has a product structure since, for each $\{a_s^*\}_{s\in\mcS_1}$, $\mcA$ includes an element $a^*=\{a_s^*\}_{s\in\mcS_0\cup\mcS_1}$ for every $\{a_s^*\}_{s\in\mcS_0}$.
This, together with \eqref{eqn:optimal_actions_under_true_model_bound_term_1} and \eqref{eqn:optimal_actions_under_true_model_bound_term_2} gives the following bound for the right-hand side of \eqref{eqn:optimal_actions_under_true_model_bound}:
\begin{equation}
	\frac{1}{|\mcA|}
	\sum_{a^*\in\mcA}
	\sum_{s\in\mcS_0}
	\sum_{k=1}^K
	\sum_{h=1}^{H-1}
	\mbbE_{\Phi(\varepsilon,\tilde{p}^*,f,a^*),\mcL}[\mathbbm{1}\{s_{k,h}=s,a_{k,h}=a_s^*\}]
	\leq
	\frac{T_K-K}{2A}
	+
	2\varepsilon_0(T_K-K)\sqrt{\frac{T_K}{2A}}
	.
\end{equation}
This proves \eqref{eqn:sum_over_optimal_actions_proof_objective} and therefore \Cref{lem:sum_over_optimal_actions}.
\qed

\subsection{Proof of \texorpdfstring{\Cref{prop:transition_regret_lower_bound}}{transition regret lower bound}}
\label{sec:proof_of_transition_regret_lower_bound}

Using that $q(\,\cdot \mid s)$ is a probability distribution on $f^{-1}(s)$, followed by \eqref{eqn:latent_state_transition_kernel_macro_def} and $\varepsilon_0\leq \varepsilon_1$, and finally that $\tilde{p}^*_{\iota(s)}(\,\cdot\mid s)$ is a probability distribution on $\mcS_0$, it follows that
\begin{align}
	\min_{s,a}
	\sum_{y\in f^{-1}(\mcS_0)}
	p(f(y)\mid s,a)
	q(y\mid f(y))
	&
	=
	\min_{s,a}
	\sum_{s'\in\mcS_0}
	p(s'\mid s,a)
	\nonumber
	\\
	&
	=
	\frac{1-2\varepsilon_1}{2}
	\min_{s\in\mcS_1}
	\sum_{s'\in\mcS_0}
	\tilde{p}_0(s'\mid s)
	=
	\frac{1-2\varepsilon_1}{2}
	.
	\label{eqn:lower_bound_on_transitions_to_S0}
\end{align}
Using \Cref{lem:visitation_rate_bound} for $\mcX=f^{-1}(\mcS_0)=[\floor{n/2}]$ with \eqref{eqn:lower_bound_on_transitions_to_S0} and $K(H-1)=T_K-K$, it follows that
\begin{align}
	\nonumber
	 &
	\reg_{K,0}(\mcL;\Phi)
	=
	\varepsilon_{0}
	\sum_{s\in\mcS_0}
	\sum_{k=1}^K
	\sum_{h=1}^{H-1}
	\biggl(
	\mbbE_{\Phi,\mcL}[\mathbbm{1}\{s_{k,h}=s\}]
	-
	\mbbE_{\Phi,\mcL}[\mathbbm{1}\{s_{k,h}=s,a_{k,h}=a_s^*\}]\biggr)
	\\
	 &
	\geq
	\varepsilon_{0}
	\biggl(
	\frac{1}{2}(1-2\varepsilon_1)(T_K-K)
	-
	\sum_{s\in\mcS_0}
	\sum_{k=1}^K
	\sum_{h=1}^{H-1}
	\mbbE_{\Phi,\mcL}[\mathbbm{1}\{s_{k,h}=s,a_{k,h}=a_s^*\}]
	\biggr)
	.
	\label{eqn:visitation_rate_bound_applied_to_S0}
\end{align} 
Since $\varepsilon_1\leq 1/8$, taking the arithmetic means of both sides of \eqref{eqn:visitation_rate_bound_applied_to_S0} over $\Phi\in \Lambda^*(\varepsilon)$ and applying \Cref{lem:sum_over_optimal_actions} proves \eqref{eqn:transition_regret_lower_bound}.

Now specify $\varepsilon_0 = \varepsilon_1\wedge \sqrt{A/T_K}(1-2\varepsilon_1-1/A)/\sqrt{32}$.
Then $\varepsilon_0\leq \varepsilon_1\leq \epsilon \wedge 1/8$.
Moreover, $\varepsilon_0 > 0$ since $\varepsilon_1\leq 1/8$, ensuring $R_0$ is well-defined and that \eqref{eqn:transition_regret_lower_bound} applies.

It remains to show that $R_0 = \Omega(\sqrt{AT_K})$ whenever $T_K=\Omega(A)$ and $\varepsilon_1=\Omega(1)$.
Let
\begin{equation}
	Z(x)
	\eqdef
	T_K\biggl(1-\frac{1}{H}\biggr)
	x
	\biggl(
	\frac{1}{2}(1-2\varepsilon_1)
	-
	\frac{1}{2A}
	-
	2x\sqrt{ \frac{T_K}{2A}}
	\biggr)
	,
\end{equation}
and observe that because $T_K=KH$ by definition, \eqref{eqn:transition_regret_lower_bound} implies that $R_0\geq Z(\varepsilon_0)$.
Since $\varepsilon_0 = \varepsilon_1 \wedge \sqrt{A/T_K}(1-2\varepsilon_1-1/A)/\sqrt{32}$ we now bound $Z(\varepsilon_0)$ separately for the two $\varepsilon_0$.

If $\varepsilon_0 = \sqrt{A/T_K}(1-2\varepsilon_1-1/A)/\sqrt{32}$, then
\begin{equation}
	\label{eqn:bound_reduction}
	Z\biggl(\frac{1-2\varepsilon_1-1/A}{\sqrt{32}}\sqrt{\frac{A}{T_K}}\biggr)
	=
	\sqrt{AT_K}\biggl(1-\frac{1}{H}\biggr)
	\frac{((1-2\varepsilon_1)-1/A)^2}{16\sqrt{2}}
	\geq
	\frac{\sqrt{AT_K}}{512\sqrt{2}}
	,
\end{equation}
Here, the last step follows because $H\geq 2$ and $A\geq 2$ by \eqref{eqn:parameter_assumptions} and because $\varepsilon_1\leq 1/8$.

Meanwhile, if $\varepsilon_0 = \varepsilon_1$, then $\sqrt{A/T_K}(1-2\varepsilon_1-1/A)/\sqrt{32}\geq \varepsilon_1$.
Therefore
\begin{equation}
	2\varepsilon_1\sqrt{\frac{T_K}{2A}}
	\leq
	\frac{1}{2}\biggl(\frac{1}{2}(1-2\varepsilon_1)-\frac{1}{2A}\biggr)
	,
\end{equation}
and thus also
\begin{equation}
	\label{eqn:case_ii_subcase_a}
	Z(\varepsilon_1)
	\geq
	T_K
	\biggl(1-\frac{1}{H}\biggr)
	\varepsilon_1
	\biggl(
	\frac{1}{4}(1-2\varepsilon_1)
	-
	\frac{1}{4A}
	\biggr)
	\geq
	\varepsilon_1\frac{T_K}{32}
	,
\end{equation}
where in the last inequality we again used $H\geq 2$ and $A\geq 2$ from \eqref{eqn:parameter_assumptions} and $\varepsilon_1\leq 1/8$.

Combining \eqref{eqn:bound_reduction} and \eqref{eqn:case_ii_subcase_a} we conclude that
\begin{equation}
	R_0
	\geq
	\varepsilon_1\frac{T_K}{32}\wedge \frac{\sqrt{AT_K}}{512\sqrt{2}}
	.
\end{equation}
Since $\varepsilon_1 \not\rightarrow 0$ we conclude that $R_0 = \Omega(T_K\wedge \sqrt{AT_K})$.
Moreover, when $T_K = \Omega(A)$ this implies that $R_0 = \Omega(\sqrt{AT_K})$ as claimed.
\qed

\subsection{Proof of \texorpdfstring{\Cref{lem:relation_between_regret_and_clustering_error}}{relation regret and clustering error}}
\label{sec:proof_of_relation_between_regret_and_clustering_error}

Fix $f\in\mcF$ and $a^*\in\mcA$.
Let $\Phi=\Phi(\varepsilon_0,\tilde{p}^*,f,a^*)\in\Lambda^*(\varepsilon)$.
Fix also $x\in f^{-1}(\mcS_1)$.

Observe that $q(y \mid s)=1/|f^{-1}(s)|$ by \Cref{def:hard-BMDP--uniform_emissions} and $\max_{s}|f^{-1}(s)|\leq \floor{\ceil{n/2}/(S/2)}+1$ by \Cref{def:hard_to_learn_class:size_distribution_of_latent_states}.
Moreover, $\min_{s,a}p(s'\mid s,a)=(1/2)\min_s\tilde{p}^*_1(s' \mid s)$ for $s'\in \mcS_1$ by \Cref{def:hard-BMDP--latent-state-transition-probabilities}.
Therefore, since $x\in f^{-1}(\mcS_1)$,
\begin{equation}
	\label{eqn:lower_bound_transition_probability_first_step}
	\min_{s,a}p(f(x)\mid s,a)q(x\mid f(x))
	\geq
	\frac{1}{2(\floor{\ceil{n/2}/(S/2)}+1)}
	\min_s\tilde{p}^*_1(f(x) \mid s)
	.
\end{equation}
It further follows from \Cref{prop:kernel_allows_identifiability} that $\tilde{p}^*_1(s' \mid s)\geq (1-\kappa)/(S/2)$.
With \eqref{eqn:lower_bound_transition_probability_first_step}, this implies
\begin{equation}
	\label{eqn:lower_bound_transition_probability}
	\begin{split}
		\min_{s,a}p(f(x)\mid s,a)q(x\mid f(x))
		\geq
		\frac{1}{\floor{\ceil{n/2}/(S/2)}+1}
		\frac{1-\kappa}{S}
		\geq
		\frac{1-\kappa}{n+S+2}
		.
	\end{split}
\end{equation}
Because $S>0$ and $\kappa\in(0,1)$, we also have that $1/n \geq (1-\kappa)/(n+S+2)$.
It then follows from \eqref{eqn:lower_bound_transition_probability} and \Cref{lem:visitation_rate_bound} for $\mcX=\{x\}$ that for all $h$,
\begin{equation}
	\label{eqn:visitation_rate_bound_context_specialized}
	\mbbP_{\Phi,\mcL}[x_{k,h}=x]
	\geq
	\min
	\biggl\{\frac{1}{n},\frac{1-\kappa}{n+S+2}\biggr\}
	=
	\frac{1-\kappa}{n+S+2}
	.
\end{equation}

Using the definition of conditional probability and \eqref{eqn:visitation_rate_bound_context_specialized} we obtain that
\begin{align}
	 &
	\reg_{K,1}(\mcL;\Phi)
	=
	\varepsilon_1
	\sum_{k=1}^K
	\sum_{h=1}^{H-1}
	\sum_{x\in f^{-1}(\mcS_1)}
	\sum_{a\neq a^*_{f(x)}}
	\mbbE_{\Phi,\mcL}[\mathbbm{1}\{x_{k,h}=x,a_{k,h}=a\}]
	\nonumber
	\\
	 &
	=
	\varepsilon_1
	\sum_{k=1}^K
	\sum_{h=1}^{H-1}
	\sum_{x\in f^{-1}(\mcS_1)}
	\sum_{a\neq a^*_{f(x)}}
	\mbbP_{\Phi,\mcL}[a_{k,h}=a\mid x_{k,h}=x]\mbbP_{\Phi,\mcL}[x_{k,h}=x ]
	\nonumber
	\\
	 &
	\overset{\eqref{eqn:visitation_rate_bound_context_specialized}}{\geq}
	\frac{\varepsilon_1(1-\kappa)}{n+S+2}
	\sum_{k=1}^K
	\sum_{h=1}^{H-1}
	\sum_{x\in f^{-1}(\mcS_1)}
	\sum_{a\neq a^*_{f(x)}}
	\mbbP_{\Phi,\mcL}[a_{k,h}=a\mid x_{k,h}=x]
	.
	\label{eqn:regret_to_misclassification_rate_visitation_rate_bound_applied}
\end{align}
Next, use the law of total expectation and the definition of $\mbbP_{\Phi,\mcL}$ in \Cref{sec:preliminaries-episodes} to write
\begin{equation}
	\mbbP_{\Phi,\mcL}[a_{k,h}=a\mid x_{k,h}=x]
	=
	\mbbE_{\Phi,\mcL}\bigl[\mbbP_{\Phi,\mcL}[a_{k,h}=a\mid x_{k,h}=x,\mcD_k]\bigr]
	=
	\mbbE_{\Phi,\mcL}[\pi_{k,h}(a\mid x)]
	.
\end{equation}
Consequently,
\begin{align}
	 &
	\reg_{K,1}(\mcL;\Phi)
	\overset{\eqref{eqn:regret_to_misclassification_rate_visitation_rate_bound_applied}}{\geq}
	\frac{\varepsilon_1(1-\kappa)}{n+S+2}
	\sum_{k=1}^K
	\sum_{h=1}^{H-1}
	\sum_{x\in f^{-1}(\mcS_1)}
	\mbbE_{\Phi,\mcL}[ 1 - \pi_{k,h}(a_{f(x)}^*\mid x)]
	\nonumber
	\\
	 &
	=
	\frac{\varepsilon_1(1-\kappa)(H-1)}{n+S+2}
	\sum_{k=1}^K
	\sum_{x\in f^{-1}(\mcS_1)}
	\mbbE_{\Phi,\mcL}\biggl[1-\frac{1}{H-1}\sum_{h=1}^{H-1}\pi_{k,h}(a_{f(x)}^*\mid x)\biggr]
	.
\end{align}
Finally, recall from \Cref{def:hard_to_learn_class} that each $s\in\mcS_1$ has a distinct $a_s^*$.
By \eqref{eqn:learning_alg_to_clustering_alg}, it follows that if $(H-1)^{-1}\sum_{h=1}^{H-1} \allowbreak \pi_{k,h}(a_{s}^*\mid x)\geq 1/2$ then $\hat{f}_{\pi_k}(x) = s$.
Conversely, if $\hat{f}_{\pi_k}(x) \neq s$ then $(H-1)^{-1}\sum_{h=1}^{H-1} \allowbreak \pi_{k,h}(a_{s}^*\mid x)< 1/2$ and thus
\begin{equation}
	\label{eqn:bound_on_how_often_wrong_action_is_selected}
	1
	-
	\frac{1}{H-1}\sum_{h=1}^{H-1}\allowbreak \pi_{k,h}(a_{f(x)}^*\mid x)
	\geq
	\frac{1}{2}\mathbbm{1}\{\hat{f}_{\pi_k}(x)\neq f(x)\}
	.
\end{equation}
Finally, because $\sum_{x\in f^{-1}(\mcS_1)}\mathbbm{1}\{\hat{f}_{\pi_k}(x)\neq f(x)\} \geq |\mcE(\hat{f}_{\pi_k};f,\mcS_1)|$ by \eqref{eqn:misclassdef} it follows that
\begin{align}
	\reg_{K,1}(\mcL;\Phi)
	 &
	\overset{\eqref{eqn:bound_on_how_often_wrong_action_is_selected}}{\geq}
	\frac{\varepsilon_1(1-\kappa)(H-1)}{2(n+S+2)}
	\sum_{k=1}^K
	\sum_{x\in f^{-1}(\mcS_1)}
	\mbbE_{\Phi,\mcL}[\mathbbm{1}\{\hat{f}_{\pi_k}(x)\neq f(x)\}]
	\nonumber
	\\
	 &
	\geq
	\frac{\varepsilon_1(1-\kappa)(H-1)}{2(n+S+2)}
	\sum_{k=1}^K
	\mbbE_{\Phi,\mcL}[|\mcE(\hat{f}_{\pi_k};f,\mcS_1)|]
	.
\end{align}
That is it.
\qed

\subsection{Proof of \texorpdfstring{\Cref{lem:local_clustering_error_lower_bound}}{local clustering error lower bound}}
\label{sec:proof_of_local_clustering_error_lower_bound}

Fix $k\in\mbbN_+$ and an estimation algorithm $\mcC$.
Recall that $\hat{f}_{\mcC,k}=\mcC_{k}(\mcD_{k},r)$ denotes the output of $\mcC$ after $k-1$ episodes.
Let $\Phi^{(1)},\Phi^{(2)}$ be as in \Cref{lem:local_clustering_error_lower_bound} and denote by $P^{(1)}$ and $P^{(2)}$ their transition kernels, respectively.
Similarly, let $q^{(1)}$ and $q^{(2)}$ denote their emission probabilities, respectively, and recall from \Cref{def:hard-BMDP--uniform_emissions} that $q^{(j)}(y \mid s)=1/|(f^{(j)})^{-1}(s)|$ for $j\in\{1,2\}$.

Observe that \eqref{eqn:local_clustering_error_lower_bound} is immediate when $\mbbP_{\Phi^{(2)},\mcL}[\hat{f}_{\mcC,k}(x)\neq f^{(2)}(x)]=1$.
Hence, assume that $\mbbP_{\Phi^{(2)},\mcL}[\hat{f}_{\mcC,k}(x)\neq f^{(2)}(x)]\in [0,1)$.

As in \Cref{sec:proof_of_change_of_measure_argument} we set up to apply \cite[Lemma H.1]{NEURIPS2019_10a5ab2d}, this time with $\mcM \equiv \Phi^{(2)}$, $\mcM' \equiv \Phi^{(1)}$, and $Z \equiv \mathbbm{1}\{\hat{f}_{\mcC,k}(x)=f^{(1)}(x)\}$.
One first verifies that $Z$ is indeed $[0,1]$-valued, and that it is a function of the observations from the first $k-1$ episodes because $\hat{f}_{\mcC,k}=\mcC_{k}(\mcD_{k},r)$.
Given the identifications in \Cref{sec:proof_of_change_of_measure_argument}, \cite[Lemma H.1]{NEURIPS2019_10a5ab2d} implies that
\begin{align}
	 &
	\sum_x
	\sum_a
	\KL(P^{(2)}(\,\cdot\mid x,a),P^{(1)}(\,\cdot\mid x,a))
	\sum_{k'=1}^{k-1}
	\sum_{h=1}^H
	\mbbE_{\Phi^{(2)},\mcL}[\mathbbm{1}\{x_{k',h}=x,a_{k',h}=a\}]
	\label{eqn:lemma_H1_applied_for_clustering}
	\\
	 &
	\geq
	\kl\bigl(\mbbP_{\Phi^{(2)},\mcL}[\hat{f}_{\mcC,k}(x)=f^{(1)}(x)],\mbbP_{\Phi^{(1)},\mcL}[\hat{f}_{\mcC,k}(x)=f^{(1)}(x)]\bigr)
	.
	\nonumber
\end{align}

We next bound the right-hand side using \cite[(11)]{doi:10.1287/moor.2017.0928}:
\begin{equation}
	\begin{split}
		 &
		\kl\bigl(
			\mbbP_{\Phi^{(2)},\mcL}
			[\hat{f}_{\mcC,k}(x)=f(x)]
			,
			\mbbP_{\Phi^{(1)},\mcL}
			[\hat{f}_{\mcC,k}((x))=f(x)]
		\bigr)
		\\
		 &
		\geq
		\bigl(
			1-\mbbP_{\Phi^{(2)},\mcL}[\hat{f}_{\mcC,k}(x)=f^{(1)}(x)]
		\bigr)
		\ln \frac{1}{1-\mbbP_{\Phi^{(1)},\mcL}[\hat{f}_{\mcC,k}(x)=f^{(1)}(x)]}-\ln 2
		\\
		 &
		\geq
		\bigl(
			1-\mbbP_{\Phi^{(2)},\mcL}[\hat{f}_{\mcC,k}(x)\neq f^{(2)}(x)]
		\bigr)
		\ln \frac{1}{\mbbP_{\Phi,\mcL}[\hat{f}_{\mcC,k}\neq f^{(1)}(x)]}
		-
		\ln 2
		,
	\end{split}
\end{equation}
where we used in the second inequality that $\{\hat{f}_{\mcC,k}(x)=f^{(1)}(x)\}\subset \{\hat{f}_{\mcC,k}(x)\neq f^{(2)}(x)\}$ because $f^{(1)}(x)\neq f^{(2)}(x)$, along with nonnegativity of $\ln (1/\mbbP_{\Phi,\mcL}[\hat{f}_{\mcC,k}\neq f^{(1)}(x)])$.

We claim that there exists a constant $C>0$ depending only on $\eta$ such that for all sufficiently large $n$,
\begin{align}
	\label{eqn:upper_bound_log_likelihood_ratio}
	 &
	\sum_x
	\sum_a
	\KL(P^{(2)}(\,\cdot\mid x,a),P^{(1)}(\,\cdot\mid x,a))
	\sum_{k'=1}^{k-1}
	\sum_{h=1}^H
	\mbbE_{\Phi^{(2)},\mcL}[\mathbbm{1}\{x_{k',h}=x,a_{k',h}=a\}]
	\leq
	C\frac{T_k}{n}
	.
\end{align}
The result \eqref{eqn:local_clustering_error_lower_bound} would then follow because combining \eqref{eqn:lemma_H1_applied_for_clustering}--\eqref{eqn:upper_bound_log_likelihood_ratio} gives
\begin{equation}
	(1-\mbbP_{\Phi^{(2)},\mcL}[\hat{f}_{\mcC,k}(x)\neq f^{(2)}(x)])\ln \frac{1}{\mbbP_{\Phi^{(1)},\mcL}[\hat{f}_{\mcC,k}\neq f^{(1)}(x)]}
	\leq
	C\frac{T_k}{n}+\ln 2.
\end{equation}
Since $\mbbP_{\Phi^{(2)},\mcL}[\hat{f}_{\mcC,k}(x)\neq f^{(2)}(x)]\in [0,1)$ by assumption and $\exp(-\ln 2/(1-p))\geq 1/2$ for $p\in [0,1)$, this implies that
\begin{equation}
	\label{eqn:relation_between_misclassification_prob_in_two_BMDPs}
	\mbbP_{\Phi^{(1)},\mcL}[\hat{f}_{\mcC,k}\neq f^{(1)}(x)]
	\geq
	\frac{1}{2}
	\exp\biggl(
	-
	\frac{C(T_k/n)}{1-\mbbP_{\Phi^{(2)},\mcL}[\hat{f}_{\mcC,k}(x)\neq f^{(2)}(x)]}
	\biggr)
	.
\end{equation}
The result \eqref{eqn:local_clustering_error_lower_bound} follows from \eqref{eqn:relation_between_misclassification_prob_in_two_BMDPs} by noting that $p+(1/2)\exp(-x/(1-p))\geq (1/2)\exp(-x)$ for $p\in[0,1)$ and $x>0$.

\subsubsection{Proof of \texorpdfstring{\eqref{eqn:upper_bound_log_likelihood_ratio}}{bound on loglikelihood ratio}}

Recall \eqref{eqn:historydef}.
It suffices to show that there exists $C>0$ depending only on $\eta$ such that for sufficiently large $n$,
\begin{equation}
	\label{eqn:upper_bound_log_likelihood_ratio_extended_equivalent}
	\mbbE_{\Phi^{(2)},\mcL}
	\biggl[\ln\frac{\mbbP_{\Phi^{(2)},\mcL}[\mcD_{k+1}]}{\mbbP_{\Phi^{(1)},\mcL}[\mcD_{k+1}]}\biggr]
	\leq
	C\frac{T_k}{n}
	,
\end{equation}
where by a slight abuse of notation, $\mbbP_{\Phi^{(j)},\mcL}[\mcD_{k+1}]$ denotes the probability of observing the particular trajectory $\mcD_{k+1}$ under $\mbbP_{\Phi^{(j)},\mcL}$.

That \eqref{eqn:upper_bound_log_likelihood_ratio_extended_equivalent} implies \eqref{eqn:upper_bound_log_likelihood_ratio} is because, by \Cref{def:hard-BMDP--uniform_initial_distribution} and the transition structure of a \gls{BMDP}, for $j\in\{1,2\}$,
\begin{equation}
	\label{eqn:expansion_of_probability_of_a_trajectory}
	\mbbP_{\Phi^{(j)},\mcL}[\mcD_{k+1}]
	=
	\prod_{k'=1}^{k}
	\mbbP_{\Phi^{(i)},\mcL}[\mcD_{(k')}\mid \mcD_{k'}]
	=
	\prod_{k'=1}^{k}
	\frac{1}{n}
	\prod_{h=1}^H
	\pi_{k',h}(a_{k',h}\mid x_{k',h}) P^{(j)}(x_{k',h+1}\mid x_{k',h},a_{k',h})
	.
\end{equation}
Since the policy $\pi_{k'} = \mcL_{k'}(\mcD_{k'},r)$ selected by $\mcL$ does not depend on $j$, it follows that
\begin{align}
	 &
	\mbbE_{\Phi^{(2)},\mcL}
	\biggl[\ln\frac{\mbbP_{\Phi^{(2)},\mcL}[\mcD_{k+1}]}{\mbbP_{\Phi^{(1)},\mcL}[\mcD_{k+1}]}\biggr]
	=
	\sum_{k'=1}^k
	\sum_{h=1}^H
	\mbbE_{\Phi^{(2)},\mcL}\biggl[\ln\frac{P^{(2)}(x_{k',h+1}\mid x_{k',h},a_{k',h})}{P^{(1)}(x_{k',h+1}\mid x_{k',h},a_{k',h})}\biggr]
	\nonumber
	\\
	 &
	=
	\sum_x
	\sum_a
	\sum_y
	P^{(2)}(y\mid x,a)\ln\frac{P^{(2)}(y\mid x,a)}{P^{(1)}(y\mid x,a)}
	\sum_{k'=1}^k\sum_{h=1}^H\mbbE_{\Phi^{(2)},\mcL}[\mathbbm{1}\{x_{k',h}=x,a_{k',h}=a\}]
	\label{eqn:log_likelihood_expression}
	\\
	&
	=
	\sum_x
	\sum_a
	\KL\bigl(P^{(2)}(\,\cdot\mid x,a),P^{(1)}(\,\cdot\mid x,a)\bigr)\sum_{k'=1}^k\sum_{h=1}^H\mbbE_{\Phi^{(2)},\mcL}[\mathbbm{1}\{x_{k',h}=x,a_{k',h}=a\}]
	.
	\nonumber
\end{align}
By nonnegativity of the KL-divergence, extending the sum over $k'$ in \eqref{eqn:upper_bound_log_likelihood_ratio} to $k$ shows that the left-hand side of \eqref{eqn:expansion_of_probability_of_a_trajectory} bounds that of \eqref{eqn:upper_bound_log_likelihood_ratio} from above.
We therefore pursue a proof of \eqref{eqn:upper_bound_log_likelihood_ratio_extended_equivalent}.

Observe from \eqref{eqn:expansion_of_probability_of_a_trajectory} that $\mbbP_{\Phi^{(j)},\mcL}[\mcD_{(k')}\mid \mcD_{k'}]$ depends on $\mcD_{k'}$ only through $\pi_{k'}$.
Consequently,
\begin{equation}
	\label{eqn:decomposition_of_log_likelihood_ratio_per_episode}
	\mbbE_{\Phi^{(2)},\mcL}\biggl[
		\ln\frac{\mbbP_{\Phi^{(2)},\mcL}[\mcD_{k+1}]}{\mbbP_{\Phi^{(1)},\mcL}[\mcD_{k+1}]}
	\biggr]
	=
	\sum_{k'=1}^k
	\mbbE_{\Phi^{(2)},\mcL}\biggl[
		\ln\frac{\mbbP_{\Phi^{(2)},\mcL}[\mcD_{(k')}\mid \pi_{k' }]}{\mbbP_{\Phi^{(1)},\mcL}[\mcD_{(k')}\mid \pi_{k' }]}
	\biggr]
\end{equation}
We next bound a single term on the right-hand side of \eqref{eqn:decomposition_of_log_likelihood_ratio_per_episode}.

Observe that the \gls{BMDP} $\Phi^{(2)}$ equals the confusing \glspl{BMDP} $\Psi^{(x,j)}$ defined in \cite[(24)--(26)]{Jedra:2022} by identifying
\begin{equation}
	\label{eqn:identifications_with_confusing_BMDP}
	\Phi \equiv \Phi^{(1)},
	\quad
	f \equiv f^{(1)},
	\quad
	q \equiv q^{(1)},
	\quad
	\tilde{f} \equiv f^{(2)},
	\quad
	\tilde{q} \equiv q^{(2)},
	\quad
	j \equiv f^{(2)}(x),
	\quad
	c
	\equiv
	c^*
	\eqdef
	\frac{q^{(2)}(x\mid s)}{q^{(1)}(x\mid s)}
	.
\end{equation}
Given these identifications, \cite[Proposition 8]{Jedra:2022} bounds the expected log-likelihood ratio for a single episode in terms of a quantity $I_j(x;c,\Phi)$.
Using the representation in \cite[Eqn.~(35)]{Jedra:2022} and identifying 
\begin{equation}
	\label{eqn:policy_and_visitation_rate_identification}
	\rho \equiv \pi_U
	,
	\quad
	m^{\Psi}_{\rho}(s,a)
	\equiv
	\omega_{\Phi^{(2)},\pi_U}(s,a)
	,
\end{equation}
with $\omega_{\Phi^{(2)},\pi_U}(s,a)$ as in \eqref{eqn:definition_of_visitation_probabilities}, the quantity $I_j(x;c,\Phi)$ corresponds to the information quantity in \eqref{eqn:Ijxcdef}:
\begin{equation}
	I_j(x;c,\Phi)
	\equiv
	I_{f^{(2)}(x)}(x;c^*,\Phi^{(1)},\pi_U)
	.
\end{equation}
Hence, by \cite[Proposition 8]{Jedra:2022},
\begin{equation}
	\label{eqn:bound_on_log_likelihood_ratio_in_terms_of_information_quantity}
	\mbbE_{\Phi^{(2)},\mcL}\biggl[
	\ln\frac{\mbbP_{\Phi^{(2)},\mcL}[\mcD_{(k')}
	\mid \pi_{k'}]}{\mbbP_{\Phi^{(1)},\mcL}[\mcD_{(k')}
	\mid \pi_{k'}]}
	\,\bigg|\, \pi_{k'} = \pi_U
	\biggr]
	=
	\frac{H}{n}
	\biggl(
		I_{f^{(2)}(x)}(x;c^*,\Phi^{(1)},\pi_{U})+O\biggl(\frac{1}{n}\biggr)
	\biggr)
	.
\end{equation}
Here, we have set their number of episodes $T$ to 1 and identified their episode horizon with $H+1$.
Moreover, \cite{Jedra:2022} actually proves \eqref{eqn:bound_on_log_likelihood_ratio_in_terms_of_information_quantity} for arbitrary policies $\pi_{k'} = \pi \neq \pi_U$, in which case $I_{f^{(2)}(x)}(x;c^*,\Phi^{(1)},\pi_{U})$ is replaced by $I_{f^{(2)}(x)}(x;c^*,\Phi^{(1)},\pi)$.
It follows, using \eqref{eqn:decomposition_of_log_likelihood_ratio_per_episode} and the law of total expectation, that
\begin{align}
	\mbbE_{\Phi^{(2)},\mcL}\biggl[
		\ln\frac{\mbbP_{\Phi^{(2)},\mcL}[\mcD_{k+1}]}{\mbbP_{\Phi^{(1)},\mcL}[\mcD_{k+1}]}
		\biggr]
	 &
	=
	\sum_{k'=1}^k
	\mbbE_{\Phi^{(2)},\mcL}\biggl[
	\mbbE_{\Phi^{(2)},\mcL}\biggl[
	\ln\frac{\mbbP_{\Phi^{(2)},\mcL}[\mcD_{(k')}\mid \pi_{k' }]}{\mbbP_{\Phi^{(1)},\mcL}[\mcD_{(k')}\mid \pi_{k' }]}
	\bigg| \pi_{k'}
	\biggr]
	\biggr]
	\label{eqn:log_likelihood_ratio_partial_bound}
	\\
	 &
	=
	\frac{H}{n}
	\sum_{k'=1}^k
	\biggl(\mbbE_{\Phi^{(2)},\mcL}[I_{f^{(2)}(x)}(x;c^*,\Phi^{(1)},\pi_{k'})]+O\biggl(\frac{1}{n}\biggr)\biggr)
	.
	\nonumber
\end{align}
We finalize the proof by bounding $I_{f^{(2)}(x)}(x;c^*,\Phi^{(1)},\pi_{k'})$.

Recall \eqref{eqn:identifications_with_confusing_BMDP} and \eqref{eqn:policy_and_visitation_rate_identification}.
\cite[Proposition 11]{Jedra:2022} states that for all $x$, $\tilde{s}\neq f^{(1)}(x)$, and $c>0$,
\begin{equation}
	\label{eqn:upper_bound_in_terms_of_I_tilde_for_pi_U}
	I_{\tilde{s}}(x;c,\Phi^{(1)},\pi_U)
	\leq
	\max\{1,c,1/c,\eta\}\tilde{I}_{\tilde{s}}(x;c,\Phi^{(1)})
	,
\end{equation}
where $\tilde{I}_{\tilde{s}}(x;c,\Phi^{(1)})$ is defined in \cite[(37)--(39)]{Jedra:2022}.
\cite[Proposition 12]{Jedra:2022} bounds the latter by
\begin{equation}
	\label{eqn:upper_bound_for_I_tilde_for_pi_U}
	\tilde{I}_{\tilde{s}}(x;c,\Phi^{(1)})
	\leq
	\frac{c\eta^7}{SA}\sum_{s,a}\Bigl(
		\frac{p(f(x)\mid s,a)}{cp(\tilde{s}\mid s,a)}-1
	\Bigr)^2
	+
	\frac{c\eta^7}{SA}\sum_{s,a}\Bigl(
		\frac{p(s\mid f(x),a)}{p(s\mid \tilde{s},a)}-1
	\Bigr)
	.
\end{equation}
We require analogues of \eqref{eqn:upper_bound_in_terms_of_I_tilde_for_pi_U} and \eqref{eqn:upper_bound_for_I_tilde_for_pi_U} for $\pi\neq \pi_U$.

For \eqref{eqn:upper_bound_in_terms_of_I_tilde_for_pi_U}, this is straightforward as the proof of \cite[Proposition 11]{Jedra:2022} goes through \emph{mutatis mutandis} upon replacing $\pi_U$ by $\pi\neq \pi_U$.
The proof of \cite[Proposition 12]{Jedra:2022}, meanwhile, uses the bound $\omega_{\Phi^{(2)},\pi_U}(s,a)\leq \eta^4/(SA)$ from \cite[Proposition 14]{Jedra:2022}.
The only modification required to extend \cite[Proposition 12]{Jedra:2022} to $\pi\neq\pi_U$ is to instead use the inequality $\sum_{s,a}\omega_{\Phi^{(2)},\pi}(s,a)w(s,a)\leq \max_{s,a}w(s,a)$ for any nonnegative $w\in \mbbR^{S A}_{\geq 0}$, which holds because $\omega_{\Phi^{(2)},\pi}(s,a)$ is a probability distribution on $[S]\times[A]$ by \eqref{eqn:definition_of_visitation_probabilities}.
After implementing these changes, we obtain that for all $x$, $\tilde{s}\neq f^{(1)}(x)$, $c>0$ and $\pi$,
\begin{equation}
	\label{eqn:modified_bound_for_I_tilde}
	I_{\tilde{s}}(x;c,\Phi^{(1)},\pi)
	\leq
	\max\{1,c,1/c,\eta\}
	\Bigl(
		c\eta^3\max_{s,a}\Bigl(
			\frac{p(f(x)\mid s,a)}{cp(\tilde{s}\mid s,a)}-1
		\Bigr)^2
		+
		c\eta^3\max_{s,a}\Bigl(
			\frac{p(s\mid f(x),a)}{p(s\mid \tilde{s},a)}-1
		\Bigr)
	\Bigr)
	.
\end{equation}
Finally, note from \Cref{def:identifiable_hard_to_learn_class} that $f^{(1)}$ and $f^{(2)}$ both satisfy \Cref{def:reachability:f} by \Cref{prop:kernel_allows_identifiability--reachability}.
It follows from \eqref{eqn:bounds_on_elements_of_kernels} that for all sufficiently large $n$,
\begin{equation}
	\label{eqn:c_star_bound}
	c^*
	=
	\frac{q^{(2)}(x\mid s)}{q^{(1)}(x\mid s)}
	=
	\frac{|(f^{(1)})^{-1}(s)|}{|(f^{(2)})^{-1}(s)|}
	\in
	[\eta^{-2},\eta^2]
	.
\end{equation}
Together with \Cref{def:reachability:p}, it follows from \eqref{eqn:modified_bound_for_I_tilde} and \eqref{eqn:c_star_bound} that there exists a constant $C>0$ depending only on $\eta$ such that for all sufficiently large $n$, $\max_x I_{f^{(2)}(x)}(x;c^*,\Phi^{(1)},\pi_{k'})\leq C/2$.
Taking $n$ large enough that the $O(1/n)$ term in \eqref{eqn:log_likelihood_ratio_partial_bound} is bounded by $C/2$ proves \eqref{eqn:upper_bound_log_likelihood_ratio_extended_equivalent}.
\qed

\subsection{Proof of \texorpdfstring{\Cref{lem:Px_is_a_permutation}}{properties of the permutation}}
\label{sec:proof_of_Px_is_a_permutation}

\emph{Proof that $|\mcF_x|/|\mcF|\geq 1/18$.}
This follows from the argument at the start of \cite[Section 6.2]{zhang2016minimax}.

Specifically, recall \Cref{def:hard_to_learn_class} and identify $K\equiv |\mcS_1|=S/2$ and $n \equiv |f^{-1}(\mcS_1)|$.
Moreover, identify $K_1 \equiv \mf{s}_1^{0}$, $K_2 \equiv \mf{s}_1^{+1}$ and $K_3 \equiv \mf{s}_1^{-1}$, where the constants $K_1, K_2$ and $K_3$ are defined in \cite[Eqn. (5.1)]{zhang2016minimax}.
Then, we can also identify $\sigma \equiv f$, $1 \equiv x$, $\Theta^L \equiv \mcF$, and $\Theta_1^L\equiv \mcF_x$, with $\Theta^L$ defined below \cite[Eqn. (5.1)]{zhang2016minimax}.

Applying the inequality for $|\Theta_1^L|/|\Theta^L|$ from \cite[Section 6.2]{zhang2016minimax} under these identifications gives\footnote{\cite{zhang2016minimax} only states this inequalty for $K\geq 3$, as a bound on $K_2$ is only established in that case.
Their reasoning, however, is also valid for $K=2$ as well.
Since \Cref{lem:latent_state_subset_sizes} provides the required bound on $|\mcS_1^{+1}|$ in this case, we apply the inequality for all $K$, thereby avoiding a case distinction.}
\begin{equation}
	\frac{|\mcF_x|}{|\mcF|}
	\geq
	\frac{x_2}{|f^{-1}(\mcS_1)|}
	\quad
	\mathrm{with}
	\quad
	x_2
	=
	(\floor{|f^{-1}(\mcS_1)|/|\mcS_1|}+1)\mf{s}_1^{+1}
	.
\end{equation}
By \Cref{lem:latent_state_subset_sizes}, $\mf{s}_1^{+1} \geq S/36$.
Since $|\mcS_1| = S/2$, it follows that
\begin{equation}
	\frac{|\mcF_x|}{|\mcF|}
	\geq
	\frac{(\floor{|f^{-1}(\mcS_1)|/(S/2)}+1)S/36}{|f^{-1}(\mcS_1)|}
	\geq
	\frac{(|f^{-1}(\mcS_1)|/(S/2))S/36}{|f^{-1}(\mcS_1)|}
	=
	\frac{1}{18}
	.
\end{equation}
This proves the first claim.

\emph{Proof that $\mcP_x|_{\mcF_x}$ is a permutation without fixed points.}
We next prove that $\mcP_x|_{\mcF_x}$ is a permutation of $\mcF_x$ for which $\mcP_xf\neq f$ for all $f\in \mcF_x$.
Given the identifications above, the map $f|_{x\in[n]\setminus[\floor{n/2}]}\mapsto \mcP_xf|_{x\in[n]\setminus[\floor{n/2}]}$ can be identified with the map $\sigma_0\mapsto \sigma[\sigma_0]$ from \cite[Section 6.2]{zhang2016minimax}.
There, it is asserted but not explicitly proven that this map is a permutation without fixed points.
For completeness, we include a proof of this fact here.

Given some nonempty subset $\mcV\subset \mbbN_+$ and an element $v\in\mcV$, let $\mcV_v^+\eqdef \{v'\in\mcV:v'>v\}$ denote the set of elements larger than $v$.
Then define the map
\begin{equation}
	\mcT_{\mcV}:\mcV\rightarrow \mcV,
	\quad
	\mcT_{\mcV}(v)
	\eqdef
	\begin{cases}
		\inf
		\{v'\in\mcV:v'>v\}
		 &
		\textnormal{if }\{v'\in\mcV:v'>v\}\neq \emptyset, \\
		\inf
		\mcV
		 &
		\textnormal{otherwise}.
	\end{cases}
\end{equation}
In other words, $\mcT_{\mcV}(v)$ selects the next largest element of $\mcV$ if it exists, and otherwise returns the smallest element of $\mcV$.
It is clear that $\mcT_{\mcV}(v)\neq v$ unless $\mcV$ is a singleton, and that the map $\mcT_{\mcV}$ is injective.

The map $\mcP_x$ in \eqref{eqn:def_of_Px} can now be expressed as:
\begin{equation}
	\label{eqn:premutation_map_reexpressed}
	\mcP_xf(y)
	=
	\begin{cases}
		f(y)
		 &
		\textnormal{if }y\neq x, \\
		\mcT_{\mcS_1^0(f)\cup \{f(x)\}}(f(x))
		 &
		\textnormal{if }y = x.
	\end{cases}
\end{equation}
Next fix $x\in[n]\setminus[\floor{n/2}]$ and $f\in\mcF_x$.
We first verify that $\mcP_xf\in \mcF_x$ and $\mcP_xf\neq f$.
\begin{enumerate}[label=(\roman*)]
	\item
	$\mcP_xf\neq f$ holds because $\mcS_1^{0}(f)$ is nonempty by \Cref{lem:latent_state_subset_sizes}, and therefore \eqref{eqn:premutation_map_reexpressed} and the properties of $\mcT$ imply that $\mcP_xf(x)\neq f(x)$, and hence $\mcP_xf\neq f$.
	\item
	\Cref{def:hard_to_learn_class:context_split} holds because by the previous point $x$ gets moved to a latent state in $\mcS_1^0(f)$ and is therefore still in $\mcS_1$ under $\mcP_xf$; hence, $(\mcP_xf)^{-1}(\mcS_1)=\{\floor{n/2}+1,\ldots,n\}$ as required.
	\item
	\Cref{def:hard_to_learn_class:size_distribution_of_latent_states,def:hard_to_learn_class:sizes_of_subsets} hold because after moving $x$ from $f(x)$ to $\mcP_xf(x)$, the former latent state contains one fewer context while the latter contains one more.
	Since $f(x)\in \mcS_1^{+1}(f)$ and $\mcP_xf(x)\in \mcS_1^{0}(f)$, and because $f(y)=\mcP_xf(y)$ for $y\neq x$, the resulting sets $\mcS_i^{\sigma}(\mcP_xf)$ have the same cardinalities as $\mcS_i^{\sigma}(f)$ for $i\in\{0,1\}$ and $\sigma\in\{-1,0,+1\}$.
	\item
	Finally, $x\in \mcS_1^{+1}(\mcP_xf)$ follows immediately from point (iii).
\end{enumerate}
Together, points (i)--(iv) prove that if $f\in\mcF_x$, then $\mcP_xf\in \mcF_x$ and $\mcP_xf\neq f$.

We conclude the proof by showing that $\mcP_x$ restricted to $\mcF_x$ is a permutation.
Since $\mcF_x$ is a finite set, it suffices to show that $\mcP_x$ is injective on $\mcF_x$.
Hence, let $f,g\in\mcF_x$ and assume that $f\neq g$.
We show that $\mcP_xf\neq \mcP_xg$.
Since $\mcP_xf(y)=f(y)$ for $y\neq x$ it immediately follows that $\mcP_xf\neq \mcP_xg$ whenever there exists $y\neq x$ such that $f(y)\neq g(y)$.

Assume therefore that no such $y$ exists, in which case $f\neq g$ implies $f(x)\neq g(x)$.
It suffices to show that $\mcP_xf(x)\neq \mcP_xg(x)$.
Since $f(x)\in \mcS_1^{+1}(f)$ and $g(x)\in \mcS_1^{+1}(g)$ by \eqref{eqn:def_of_mcFx}, $f(x)\neq g(x)$, and $f(y)=g(y)$ for all $y\neq x$, it follows that $\mcS^{0}_1(f)\cup \{f(x)\}=\mcS^{0}_1(g)\cup \{g(x)\}\defeq \bar{\mcS}$.
The result now follows from \eqref{eqn:premutation_map_reexpressed} and injectivity of the map $\mcT_{\bar{\mcS}}$.
\qed

\subsection{Proof of \texorpdfstring{\Cref{prop:lower_bound_misclassification_rate}}{reduction to local misclassification probabilities}}
\label{sec:proof_of_reduction_to_local_misclassification_probability}

Take $n$ large enough for the implication of \Cref{lem:local_clustering_error_lower_bound} to hold, and for $\min_{s}|f^{-1}(s)| \geq 1$ for all $f\in\mcF$ to hold.

Fix a set of optimal actions $a^*\in\mcA$, a learning algorithm $\mcL$, and parameters $\varepsilon$ and $\tilde{p}^*$ as in \Cref{def:identifiable_hard_to_learn_class}, and adopt the notation in \eqref{eqn:short-hand_notation_for_measures}.
Also fix a number of episodes $k\in\mbbN_+$, and recall that $\hat{f}_{\mcC,k}=\mcC_{k}(\mcD_{k},r)$ denotes the output of the estimation algorithm $\mcC$ after $k-1$ episodes.

The proof follows the outline described above \Cref{prop:lower_bound_misclassification_rate}, but involves additional steps to account for the permutation $\gamma^*_f$ appearing in \eqref{eqn:decomposition_of_ex_nr_of_misclass} and \eqref{eqn:application_of_permutation}.
We first define an auxiliary quantity which we modify the decomposition in \eqref{eqn:decomposition_of_ex_nr_of_misclass} with.

Recall \eqref{eqn:misclassdef}.
For any two decoding functions $f$ and $\hat{f}$ and any $\mcS\subset[S]$, we define the set
\begin{equation}
	\Gamma^*(\hat{f};f,\mcS)
	\eqdef
	\argmin_{\gamma\in \mathrm{Perm}([S])}
	\bigl|\cup_{s\in \mcS}\hat{f}^{-1}(\gamma(s))\setminus f^{-1}(s)\bigr|
	.
\end{equation}
Observe that the number of misclassifications of $\hat{f}$ at $\mcS$ can be expressed as
\begin{equation}
	|\mcE(\hat{f};f,\mcS)|
	=
	\sum_{\gamma\in\Gamma^*(\hat{f},f;\mcS)}
	\frac{\sum_{x\in f^{-1}(\mcS)}\mathbbm{1}\{\gamma^{-1}\circ \hat{f}(x)\neq f(x)\}}{|\Gamma^*(\hat{f};f,\mcS)|}
	.
\end{equation}
We also define the local misclassification rate at $x\in f^{-1}(\mcS)$ as
\begin{equation}
	\label{eqn:local_misclassification_rate_definition}
	\varepsilon^x(\hat{f};f,\mcS)
	\eqdef
	\sum_{\gamma \in\Gamma^*(\hat{f},f;\mcS)}
	\frac{\mathbbm{1}\{\gamma^{-1}\circ \hat{f}(x)\neq f(x)\}}{|\Gamma^*(\hat{f};f,\mcS)|}
	.
\end{equation}

The expected number of misclassifications can now be decomposed in terms of the $\varepsilon^x$ as follows.
Note that because $f^{-1}(\mcS_1)=\{\floor{n/2}+1,\ldots,n\}$ for all $f\in\mcF$, we have that
\begin{equation}
	\label{eqn:global_to_sum_of_local_misclassification_rates}
	\begin{split}
		\frac{1}{|\mcF|}
		\sum_{f\in\mcF}
		\mbbE_{f}[
			|\mcE_k(\hat{f}_{\mcC,k};f,\mcS_1)|
		]
		=
		\sum_{x\in[n]\setminus[\floor{n/2}]}
		\frac{1}{|\mcF|}
		\sum_{f\in\mcF}
		\mbbE_{f}[
			\varepsilon^x(\hat{f}_{\mcC,k};f,\mcS_1)
		]
		.
	\end{split}
\end{equation}
Then apply \Cref{lem:Px_is_a_permutation} to the right-hand side of \eqref{eqn:global_to_sum_of_local_misclassification_rates} in a manner analogous to \eqref{eqn:application_of_permutation} to obtain
\begin{equation}
	\label{eqn:permutation_applied}
	\begin{split}
		 &
		\frac{1}{|\mcF|}
		\sum_{f\in\mcF}
		\mbbE_{f}[
			|\mcE_k(\hat{f}_{\mcC,k};f,\mcS_1)|
		]
		\\
		 &
		\geq
		\sum_{x\in[n]\setminus[\floor{n/2}]}
		\frac{1}{36|\mcF_x|}
		\sum_{f\in\mcF_x}
		\bigl(
		\mbbE_{f}[
				\varepsilon^x(\hat{f}_{\mcC,k};f,\mcS_1)
			]
		+
		\mbbE_{\mcP_xf}[
				\varepsilon^x(\hat{f}_{\mcC,k};\mcP_x f,\mcS_1)
			]
		\bigr)
		.
	\end{split}
\end{equation}
To bound \eqref{eqn:permutation_applied} for all estimation algorithms $\mc{C}$, we take the infimum of both sides over all such $\mc{C}$ and use that the infimum of a sum is bounded by the sum of the infima:
\begin{align}
	\label{eqn:combine_first_three_steps_to_apply_permutation}
	 &
	\inf_{\mcC}
	\frac{1}{|\mcF|}
	\sum_{f\in\mcF}
	\mbbE_{f}[
		|\mcE_k(\hat{f}_{\mcC,k};f,\mcS_1)|
	]
	\\
	 &
	\geq
	\sum_{x\in[n]\setminus[\floor{n/2}]}
	\frac{1}{36|\mcF_x|}
	\sum_{f\in\mcF_x}
	\inf_{\mcC}
	\bigl(
	\mbbE_{f}[
			\varepsilon^x(\hat{f}_{\mcC,k};f,\mcS_1)
		]
	+
	\mbbE_{\mcP_xf}[
			\varepsilon^x(\hat{f}_{\mcC,k};\mcP_x f,\mcS_1)
		]
	\bigr)
	.
	\nonumber
\end{align}
We can now analyze the infimum over $\mcC$ on the right-hand side of \eqref{eqn:combine_first_three_steps_to_apply_permutation} separately for each $x\in[n]\setminus[\floor{n/2}]$ and $f\in\mcF_x$.

Fix $x\in[n]\setminus[\floor{n/2}]$ and $f\in\mcF_x$.
Given an estimation algorithm $\mcC$, define $\mcC^x$ as follows:
\begin{equation}
	\label{eqn:localized_algorithm_def}
	\hat{f}_{\mcC^x,k}(x)
	\eqdef
	\begin{cases}
		\hat{f}_{\mcC,k}(x) & \textnormal{if }y=x,     \\
		f(y)                & \textnormal{if }y\neq x. \\
	\end{cases}
\end{equation}
Now observe that because $f$ and $\mcP_xf$ differ only at $x\in [n]\setminus[\floor{n/2}]$, it must be that
\begin{align}
	\nonumber
	 &
	\inf_{\mcC}
	\bigl(
	\mbbE_{f}[
			\varepsilon^x(\hat{f}_{\mcC,k};f,\mcS_1)
		]
	+
	\mbbE_{\mcP_xf}[
			\varepsilon^x(\hat{f}_{\mcC,k};\mcP_x f,\mcS_1)
		]
	\bigr)
	\\
	 &
	\geq
	\inf_{\mcC}
	\bigl(
	\mbbE_{f}[
			\varepsilon^x(\hat{f}_{\mcC^x,k};f,\mcS_1)
		]
	+
	\mbbE_{\mcP_xf}[
			\varepsilon^x(\hat{f}_{\mcC^x,k};\mcP_x f,\mcS_1)
		]
	\bigr)
	.
	\label{eqn:revealing_everything_but_x}
\end{align}

Note that for all sufficiently large $n$, the set $\Gamma^*(\hat{f}_{\mcC^x,k};f,\mcS_1)$ only contains the identity permutation.
Indeed, for this permutation, the set of misclassified contexts contains at most one context, whereas any other permutation misclassifies at least $\min_{s}|f^{-1}(s)|$ contexts, which is greater than $1$ for all sufficiently large $n$ by \eqref{eqn:parameter_assumptions} and \eqref{eqn:size_distribution_of_latent_states}.
It therefore follows from \eqref{eqn:local_misclassification_rate_definition} that for all sufficiently large $n$,
\begin{equation}
	\mbbE_{f}[
		\varepsilon^x(\hat{f}_{\mcC^x,k};f,\mcS_1)
	]
	=
	\mbbP_{f}[
		\hat{f}_{\mcC^x,k}(x)\neq f(x)
	]
	.
\end{equation}
The same is also true upon replacing $f$ with $\mcP_xf$.
Together with \eqref{eqn:revealing_everything_but_x}, we conclude that
\begin{align}
	\nonumber
	 &
	\inf_{\mcC}
	\bigl(
	\mbbE_{f}[
			\varepsilon^x(\hat{f}_{\mcC,k};f,\mcS_1)
		]
	+
	\mbbE_{\mcP_xf}[
			\varepsilon^x(\hat{f}_{\mcC,k};\mcP_x f,\mcS_1)
		]
	\bigr)
	\\
	 &
	\geq
	\inf_{\mcC}
	\bigl(
	\mbbP_{f}[
			\hat{f}_{\mcC^x,k}(x)\neq f(x)
		]
	+
	\mbbP_{\mcP_xf}[
			\hat{f}_{\mcC^x,k}(x)\neq \mcP_xf(x)
		]
	\bigr)
	\nonumber
	\\
	 &
	\geq
	\inf_{\mcC}
	\bigl(
	\mbbP_{f}[
			\hat{f}_{\mcC,k}(x)\neq f(x)
		]
	+
	\mbbP_{\mcP_xf}[
			\hat{f}_{\mcC,k}(x)\neq \mcP_xf(x)
		]
	\bigr)
	,
	\label{eqn:from_local_error_to_probabilities}
\end{align}
where in passing to the last line we used that the infimum over $\mcC$ includes all algorithms of the form in \eqref{eqn:localized_algorithm_def}.

The result \eqref{eqn:from_local_error_to_probabilities}, together with \eqref{eqn:combine_first_three_steps_to_apply_permutation} now imply that for all sufficiently large $n$,
\begin{align}
	 &
	\inf_{\mcC}
	\frac{1}{|\mcF|}
	\sum_{f\in\mcF}
	\mbbE_{f}[
		|\mcE_k(\hat{f}_{\mcC,k};f,\mcS_1)|
	]
	\\
	 &
	\geq
	\sum_{x\in[n]\setminus[\floor{n/2}]}
	\frac{1}{36|\mcF_x|}
	\sum_{f\in\mcF_x}
	\inf_{\mcC}
	\bigl(
	\mbbP_{f}[
			\hat{f}_{\mcC,k}(x)\neq f(x)
		]
	+
	\mbbP_{\mcP_xf}[
			\hat{f}_{\mcC,k}(x) \neq \mcP_xf(x)
		]
	\bigr)
	.
	\nonumber
\end{align}
Finally note that $\Phi(\varepsilon,\tilde{p}^*,f,a^*)$ and $\Phi(\varepsilon,\tilde{p}^*,\mcP_xf,a^*)$ satisfy the assumptions of \Cref{lem:local_clustering_error_lower_bound}, such that for all sufficiently large $n$,
\begin{equation}
	\inf_{\mcC}
	\frac{1}{|\mcF|}
	\sum_{f\in\mcF}
	\mbbE_{f}[
		|\mcE_k(\hat{f}_{\mcC,k};f,\mcS_1)|
	]
	\geq
	\frac{n}{144}
	\exp
	\Bigl(
	-
	C\frac{T_{k}}{n}
	\Bigr)
	.
\end{equation}
That is it.
\qed

\subsection{Proof of \texorpdfstring{\Cref{prop:decoding_function_regret_lower_bound}}{decoding function regret lower bound}}
\label{sec:proof_of_decoding_function_regret_lower_bound}

We first prove \eqref{eqn:decoding_function_regret_lower_bound}.

Recall \eqref{eqn:bound_supremum_by_average}.
Analogous to \eqref{eqn:decompose_sum_over_BMDPs}, we have for $\varepsilon_0\wedge\varepsilon_1>0$ that
\begin{equation}
	\label{eqn:decomposition_of_sum_over_BMDPs}
	R_1
	=
	\frac{1}{|\mcA|}\sum_{a^*\in\mcA}
	\frac{1}{|\mcF|}\sum_{f\in\mcF}
	\textnormal{Reg}_{K,1}(\mcL;\Phi(\varepsilon_0,\tilde{p}^*,f,a^*))
	.
\end{equation}
Applying \Cref{lem:relation_between_regret_and_clustering_error} to each term on the right-hand side of \eqref{eqn:decomposition_of_sum_over_BMDPs} gives
\begin{equation}
	R_1
	\geq
	\frac{1}{|\mcA|}\sum_{a^*\in\mcA}
	\frac{1}{|\mcF|}\sum_{f\in\mcF}
	\frac{\varepsilon_1(1-\kappa)(H-1)}{2(n+S+2)}
	\sum_{k=1}^K\mbbE_{\Phi(\varepsilon_0,\tilde{p}^*,f,a^*),\mcL}[|\mcE(\hat{f}_{\pi_k}; f, \mcS_1)|]
	.
	\label{eqn:application_of_regret_clustering_error_relation}
\end{equation}
Then observe that we may also apply \Cref{prop:lower_bound_misclassification_rate} separately to each term on the right-hand side of \eqref{eqn:application_of_regret_clustering_error_relation}.
This follows because for each fixed $a^*$, it is possible to construct an estimation algorithm $\mcC$ such that $\hat{f}_{\pi_k}=\hat{f}_{\mcC,k}$.
Consequently, for all sufficiently large $n$,
\begin{equation}
	\label{eqn:combination_of_clustering_error_with_lower_bound}
	R_1
	\geq
	\frac{\varepsilon_1(1-\kappa)(H-1)}{288}\frac{1}{1+\frac{S+2}{n}}
	\sum_{k=1}^K
	\exp\biggl(-C\frac{T_k}{n}\biggr)
	.
\end{equation}
Recalling that $T_k=kH$, we evaluate the sum of exponentials, which is a geometric series, and use the elementary identity $(1-e^{-x})^{-1}\geq x^{-1}$ to write
\begin{equation}
	\label{eqn:bound_on_sum_of_exponentials}
	\begin{split}
		\sum_{k=1}^K
		\exp\biggl(
		-C\frac{T_k}{n}
		\biggr)
		=
		\frac{1-e^{-C(T_K/n)}}{1-e^{-C(H/n)}}
		\geq
		\bigl(1-e^{-C(T_K/n)}\bigr)\frac{n}{CH}
		.
	\end{split}
\end{equation}
Combining \eqref{eqn:combination_of_clustering_error_with_lower_bound} and \eqref{eqn:bound_on_sum_of_exponentials} yields \eqref{eqn:decoding_function_regret_lower_bound} for all sufficiently large $n$.

It remains to prove that \eqref{eqn:decoding_function_regret_lower_bound} implies $R_1=\Omega(n)$ whenever $T_K=\Omega(n)$ and $\varepsilon_1=\Omega(1)$.

Let $\epsilon<(1/2)(\eta-1)/(\eta+1)$ be the constant from \Cref{prop:kernel_allows_identifiability} and recall from \Cref{def:identifiable_hard_to_learn_class} that $0< \epsilon_{\max}=\varepsilon_0\vee\varepsilon_1< \epsilon$.
It then follows from \eqref{eqn:tilde_p_distance_from_uniform} and $\eta>1$ that $1-\kappa = 2(1+2\epsilon_{\max})/(\eta(1-2\epsilon_{\max})+(1+2\epsilon_{\max})) \geq 2/(1+\eta)$.
Together with \eqref{eqn:parameter_assumptions} and \eqref{eqn:decoding_function_regret_lower_bound} it follows that for all sufficiently large $n$,
\begin{equation}
	R_1
	\geq
	\frac{\varepsilon_1}{288(1+\eta) C}\frac{1-e^{-C(T_K/n)}}{1+\frac{S+2}{n}}
	n
	.
\end{equation}

Note finally that $\eta>1$ is is independent of $n$ and thus \Cref{lem:local_clustering_error_lower_bound} ensures $C$ is likewise independent of $n$.
Together with the assumptions $n=\omega(S)$, $T_K=\Omega(n)$ and $\varepsilon_1\not\rightarrow 0$ this implies that $(\varepsilon_1/(288(1+\eta) C))(1-e^{-C(T_K/n)})/(1+(S+2)/n)=\Omega(1)$, from which it follows that $R_1 = \Omega(n)$.
\qed

\section{Proofs for \texorpdfstring{\cref{sec:regret_analysis}}{the regret analysis}}

\subsection{Proof of \texorpdfstring{\Cref{lem:equivalence-of-assumptions}}{equivalence of assumptions}}
\label{sec:equivalence_of_identifiability_assumptions}

For $c>0$ and $s_1,s_2$, define
\begin{align}
	\psi_1(c;s_1,s_2)
	&
	\eqdef
	\psi_1^{\mathrm{in}}(c;s_1,s_2)
	+
	\psi_1^{\mathrm{out}}(c;s_1,s_2)
	\label{eqn:definition_lower_bound_function}
	\\
	&
	\eqdef
	\frac{1}{2\eta^7}\frac{1}{c\vee \eta}\frac{1}{SA}\sum_{s,a}\biggl(\frac{p(s_1\mid s,a)}{p(s_2\mid s,a)}-c\biggr)^2
	+
	\frac{c}{2\eta^8}\frac{1}{SA}\sum_{s,a}\biggl(
	\frac{p(s\mid s_1,a)}{p(s\mid s_2,a)}-1
	\biggr)^2
	\nonumber
\end{align}
and
\begin{align}
	\label{eqn:definition_upper_bound_function}
	\psi_2(c;s_1,s_2)
	&
	\eqdef
	\psi_2^{\mathrm{in}}(c;s_1,s_2)
	+
	\psi_2^{\mathrm{out}}(c;s_1,s_2)
	\\
	&
	\eqdef
	\frac{(1\vee \frac{1}{c^2})\eta^8}{SA}\sum_{s,a}\biggl(\frac{p(s_1\mid s,a)}{p(s_2\mid s,a)}-c\biggr)^2
	+
	\frac{(1\vee c^2)\eta^8}{SA}\sum_{s,a}\biggl(
	\frac{p(s\mid s_1,a)}{p(s\mid s_2,a)}-1
	\biggr)^2
	.
	\nonumber
\end{align}
Observe that $\psi_i^{\mathrm{in}},\psi_i^{\mathrm{out}}\geq 0$ for $i\in\{1,2\}$.

Recall \eqref{eqn:Ijxcdef}.
The proof of \Cref{lem:equivalence-of-assumptions} relies on the following lemma, proven in \Cref{sec:proof_of_information_quantity_bounds}.
\begin{lemma}
	\label{lem:information_quantity_bounds}
	Assume $\Phi$ is $\eta$-reachable and that $\mu=\mathrm{Unif}([n])$.
	Then there exists a constant $\mf{d}_1>0$ that is independent of $n$ such that for all sufficiently large $n$, and any $c>0$, $x$ and $f(x)\neq \tilde{s}\in[S]$,
	\begin{equation}
		\label{eqn:information_quantity_bounds}
		\mf{d}_1\psi_1(c;f(x),\tilde{s})\leq I_{\tilde{s}}(x;c,\Phi,\pi_U)\leq \psi_2(c;f(x),\tilde{s}).
	\end{equation}
\end{lemma}

In what follows, take $n$ large enough for the implication of \Cref{lem:information_quantity_bounds} to hold.
We first introduce some further notation that will be used during the proof.

Recall \eqref{eqn:Ijxcdef}, \eqref{eqn:I_eta_def} and \eqref{eqn:def_minimum_information_quantity}.
Let $c>0$ and define
\begin{equation}
	g_n(c)
	\eqdef
	\min_{x}\min_{\tilde{s}\neq f(x)}I_{\tilde{s}}(x;c,\Phi,\pi_U)
	.
\end{equation}
It follows by interchanging infima and minima that
\begin{equation}
	\min_{x}
	I(x;\Phi)
	=
	\inf_{c>0}
	g_n(c)
	\quad
	\mathrm{and}
	\quad
	\min_{x}
	I_{\eta}(x;\Phi)
	=
	\inf_{c\in[1/\eta,\eta]}
	g_n(c)
	.
\end{equation}
Observe also that $g_n(c)\geq 0$ since $I_{\tilde{s}}(x;c,\Phi,\pi_U)$ is a weighted sum of KL-divergences.

In addition, let
\begin{equation}
	\label{eqn:bar_psi_definition}
	\bar{\psi}_1(c)
	\eqdef
	\min_{s_1\neq s_2}
	\psi_1(c;s_1,s_2)
	\geq
	0
	\quad
	\mathrm{and}
	\quad
	\bar{\psi}_2(c)
	\eqdef
	\min_{s_1\neq s_2}
	\psi_2(c;s_1,s_2)
	\geq
	0
	.
\end{equation}
Observe that by \Cref{def:reachability}, for all sufficiently large $n$, $\cup_{x}\{f(x)\}=[S]$.
Take $n$ large enough that this holds.
Then \Cref{lem:information_quantity_bounds} implies that for any $c>0$
\begin{equation}
	\label{eqn:implication_of_information_quantity_bounds}
	0
	\leq
	\mf{d}_1\bar{\psi}_1(c)
	\leq
	g_n(c)
	\leq
	\bar{\psi}_2(c)
	.
\end{equation}

Finally, observe that \Cref{def:identifiability} implies that $\liminf_{n\rightarrow\infty} \inf_{c\in[1/\eta,\eta]}g_n(c)>0$, and that \Cref{lem:equivalence-of-assumptions} then claims $\liminf_{n\rightarrow\infty} \inf_{c>0}g_n(c)>0$.
To prove \Cref{lem:equivalence-of-assumptions} it therefore suffices to show that
\begin{equation}
	\liminf_{n\rightarrow\infty}\allowbreak \inf_{c>0}g_n(c)=0
	\quad
	\mathrm{implies}
	\quad
	\liminf_{n\rightarrow\infty}\allowbreak\inf_{c\in[1/\eta,\eta]}\allowbreak g_n(c)=0
	.
\end{equation}
Thus assume that $\liminf_{n\rightarrow\infty}\allowbreak \inf_{c>0}\allowbreak g_n(c)=0$.
The proof then proceed in three steps:
\begin{enumerate}
	\item[$(i)$] We first show that $\liminf_{n\rightarrow\infty} \inf_{c\in (\alpha/\eta,\eta/\alpha)}g_n(c)=0$ for any fixed $\alpha \in (0,1)$.
	\item[$(ii)$] Next, we show that $(i)$ implies that $\liminf_{n\rightarrow \infty}\allowbreak\inf_{c\in (\alpha/\eta,\eta/\alpha)}\bar{\psi}_2(c)=0$ for any fixed $\alpha\in(0,1)$.
	\item[$(iii)$] We finally analyze the function $\bar{\psi}_2(c)$ and show that $(ii)$ actually implies that $\liminf_{n\rightarrow\infty}\allowbreak\inf_{c\in [1/\eta,\eta]}\allowbreak\bar{\psi}_2(c)=0$.
\end{enumerate}
The statement of \Cref{lem:equivalence-of-assumptions} would then follow because $g_n(c)\geq 0$ such that
\begin{equation}
	0
	\overset{(iii)}{=}
	\liminf_{n\rightarrow \infty}\inf_{c\in [1/\eta,\eta]}\bar{\psi}_2(c)
	\overset{\eqref{eqn:implication_of_information_quantity_bounds}}{\geq}
	\liminf_{n\rightarrow\infty}\inf_{c\in [1/\eta,\eta]}g_n(c)
	\geq
	0
	.
\end{equation}
We proceed by proving claims $(i)$ to $(iii)$ in order.

\paragraph*{Proof of $(i)$.}
Fix $\alpha\in (0,1)$.
To prove $(i)$, we show that there exists $\mf{c}_1>0$ that is independent of $n$ such that $g_n(c)\geq \mf{c}_1$ for all $c\in(0,\alpha/\eta]$ and $c\in[\eta/\alpha,\infty)$.
Observe, namely, that this would imply
\begin{equation}
	\inf_{c>0}g_n(c)
	=
	\min\biggl\{\inf_{c\in (0,\alpha/\eta]}g_n(c), \inf_{c\in(\alpha/\eta,\eta/\alpha)}g_n(c),  \inf_{c\in [\eta/\alpha,\infty)}g_n(c)\biggr\}
	\geq
	\min\biggl\{\mf{c}_1,  \inf_{c\in(\alpha/\eta,\eta/\alpha)}g_n(c)\biggr\}
	.
\end{equation}
Since $\liminf_{n\rightarrow \infty}\inf_{c>0}g_n(c)=0$ by assumption, and $g_n(c)\geq0$ by \eqref{eqn:implication_of_information_quantity_bounds}, this would then imply $(i)$.

To bound $g_n(c)$, we use \eqref{eqn:implication_of_information_quantity_bounds} and bound $\psi_1(c;s_1,s_2)$ instead.
First let $c\in(0,\alpha/\eta]$ so that $(c\vee \eta)=\eta$ (recall that $\eta>1$ by \Cref{def:reachability}).
It follows from \eqref{eqn:definition_lower_bound_function} and nonnegativity of $\psi^{\mathrm{out}}_1$ that
\begin{equation}
	\label{eqn:bound_psi_for_small_c}
	\psi_1(c;s_1,s_2)
	\geq
	\psi_1^{\mathrm{in}}(c;s_1,s_2)
	=
	\frac{1}{2\eta^8}\frac{1}{SA}\sum_{s,a}\biggl(\frac{p(s_1\mid s,a)}{p(s_2\mid s,a)}-c\biggr)^2
	\geq
	\frac{1}{2\eta^8}\biggl(\frac{1}{\eta}-\frac{\alpha}{\eta}\biggr)^2
	=
	\frac{(1-\alpha)^2}{2\eta^{10}}
	.
\end{equation}
Here, we have used that $p(s_1\mid s,a)/p(s_2\mid s,a)\geq 1/\eta$ by \Cref{def:reachability:p} and the elementary inequality $(t-c)^2\geq (1/\eta-\alpha/\eta)$ for all $t\geq 1/\eta$ and $c\leq\alpha/\eta$.

For $c\geq \eta/\alpha$ we similarly have that
\begin{equation}
	\label{eqn:bound_psi_for_large_c}
	\psi_1(c;s_1,s_2)
	\geq
	\psi_1^{\mathrm{in}}(c;s_1,s_2)
	\geq
	\frac{1}{2c\eta^7}\frac{1}{SA}\sum_{s,a}\biggl(\frac{p(s_1\mid s,a)}{p(s_2\mid s,a)}-c\biggr)^2
	\geq
	\frac{(1-\alpha)^2}{2\alpha\eta^6}
	.
\end{equation}
Here, we have instead used that $p(s_1\mid s,a)/p(s_2\mid s,a)\leq \eta$ also by \Cref{def:reachability:p} and the elementary inequality $(1/c)(t-c)^2\geq (1/c)(\eta-c)^2\geq \eta(1-\alpha)^2/\alpha$ for all $t\leq \eta$ and $c\geq\eta/\alpha$.

Combine \eqref{eqn:implication_of_information_quantity_bounds} with \eqref{eqn:bound_psi_for_small_c} and \eqref{eqn:bound_psi_for_large_c} to conclude that for all $c\in(0,\alpha/\eta]$ and $c\in[\eta/\alpha,\infty)$
\begin{equation}
	g_n(c)
	\geq
	\mf{d}_1\min_{s_1\neq s_2}\psi_1(c;s_1,s_2)
	\geq
	\mf{d}_1\frac{(1-\alpha)^2}{2(\eta^{10}\vee\alpha\eta^6)}
	\defeq
	\mf{c}_1
	>0
	,
\end{equation}
after which claim $(i)$ follows.

\paragraph*{Proof of $(ii)$.}
Fix $\alpha\in(0,1)$.
Note from \eqref{eqn:definition_lower_bound_function} and \eqref{eqn:definition_upper_bound_function} that there exists a constant $\mf{c}_2$ that is independent of $n$ such that for $c\in (\alpha/\eta,\eta/\alpha)$,
\begin{equation}
	\frac{\psi_2^{\mathrm{in}}(c;s_1,s_2)}{\psi_1^{\mathrm{in}}(c;s_1,s_2)}
	\vee
	\frac{\psi_2^{\mathrm{out}}(c;s_1,s_2)}{\psi_1^{\mathrm{out}}(c;s_1,s_2)}
	\leq
	\mf{c}_2
	.
\end{equation}
This follows from \Cref{def:reachability:p} and because for fixed $\alpha\in(0,1)$ and $\eta> 1$, $c\in (\alpha/\eta,\eta/\alpha)$ is bounded away from zero.
Consequently, $\bar{\psi}_1(c)\geq (1/\mf{c}_2)\bar{\psi}_2(c)$ for $c\in(\alpha/\eta,\eta/\alpha)$.
Moreover, claim $(i)$ and \eqref{eqn:implication_of_information_quantity_bounds} imply that $\liminf_{n\rightarrow\infty}\inf_{c\in (\alpha/\eta,\eta/\alpha)}\bar{\psi}_1(c)=0$ so that
\begin{equation}
	0
	=
	\liminf_{n\rightarrow\infty}\inf_{c\in (\alpha/\eta,\eta/\alpha)}\bar{\psi}_1(c)
	\geq
	(1/\mf{c}_2)\liminf_{n\rightarrow \infty}\inf_{c\in (\alpha/\eta,\eta/\alpha)} \bar{\psi}_2(c)
	\geq
	0
	.
\end{equation}
This implies claim $(ii)$.

\paragraph*{Proof of $(iii)$.}
Fix $\alpha\in(0,1)$.
Note firstly from \eqref{eqn:definition_upper_bound_function} that $\min_{s_1\neq s_2}\psi_2^{\mathrm{out}}(c;s_1,s_2)\propto (1\vee c^2)$.
Consequently,
\begin{equation}
	\label{eqn:scaling_property_applied}
	\inf_{c\in(\alpha/\eta,\eta/\alpha)}\min_{s_1\neq s_2}\psi_2^{\mathrm{out}}(c;s_1,s_2)
	\geq
	\alpha^2\inf_{c\in[1/\eta,\eta]}\min_{s_1\neq s_2}\psi_2^{\mathrm{out}}(c;s_1,s_2)
	.
\end{equation}
Claim $(ii)$ then implies that for every fixed $\alpha\in (0,1)$,
\begin{equation}
	\begin{split}
		0
		\overset{(ii)}{=}
		\liminf_{n\rightarrow \infty}\inf_{c\in (\alpha/\eta,\eta/\alpha)} \bar{\psi}_2(c)
		 &
		\geq
		\liminf_{n\rightarrow\infty}
		\inf_{c\in(\alpha/\eta,\eta/\alpha)}\min_{s_1\neq s_2}\psi_2^{\mathrm{out}}(c;s_1,s_2)
		\\
		 &
		\overset{\eqref{eqn:scaling_property_applied}}{\geq}
		\liminf_{n\rightarrow\infty}
		\alpha^2\inf_{c\in[1/\eta,\eta]}\min_{s_1\neq s_2}\psi_2^{\mathrm{out}}(c;s_1,s_2)
		\geq
		0
		.
	\end{split}
\end{equation}
Thus, $\liminf_{n\rightarrow\infty}\inf_{c\in[1/\eta,\eta]}\min_{s_1\neq s_2}\psi_2^{\mathrm{out}}(c;s_1,s_2) = 0$.
Consequently, to prove claim $(iii)$ it suffices to show that $\inf_{c\in[1/\eta,\eta]}\min_{s_1\neq s_2}\psi_2^{\mathrm{in}}(c;s_1,s_2)=0$ by \eqref{eqn:definition_upper_bound_function} and \eqref{eqn:bar_psi_definition}.

To this aim, let $\rho_{s_1,s_2}=\{\rho_{s_1,s_2}(s,a)\}_{s\in[S],a\in[A]}$ for each $s_1\neq s_2$ be defined as $\rho_{s_1,s_2}(s,a)\eqdef p(s_1\mid s,a)/p(s_2\mid s,a)$.
Then
\begin{equation}
	\label{eqn:relation_to_distance_minimization}
	\begin{split}
		\inf_{c\in(\alpha/\eta,\eta/\alpha)}\psi_2^{\mathrm{in}}(c;s_1,s_2)
		 &
		=
		\inf_{c\in(\alpha/\eta,\eta/\alpha)}(1\vee 1/c^2)\frac{\eta^7}{SA}\sum_{s,a}(\rho_{s_1,s_2}(s,a)-c)^2
		\\
		 &
		\geq
		\frac{\eta^7}{SA}\inf_{c\in(\alpha/\eta,\eta/\alpha)}\sum_{s,a}(\rho_{s_1,s_2}(s,a)-c)^2.
	\end{split}
\end{equation}
The infimum on the right-hand side is attained at $c^* = (1/(SA))\sum_{s,a}\rho_{s_1,s_2}(s,a)$, which because $\rho_{s_1,s_2}(s,a)\in [1/\eta,\eta]$ by \Cref{def:reachability:p} satisfies $c^*\in [1/\eta,\eta]$.
It follows that
\begin{equation}
	\label{eqn:distance_minimization}
	\inf_{c\in(\alpha/\eta,\eta/\alpha)}\sum_{s,a}(\rho_{s_1,s_2}(s,a)-c)^2
	=
	\inf_{c\in[1/\eta,\eta]}\sum_{s,a}(\rho_{s_1,s_2}(s,a)-c)^2
\end{equation}
Together, \eqref{eqn:relation_to_distance_minimization} and \eqref{eqn:distance_minimization} imply that
\begin{equation}
	\begin{split}
		\inf_{c\in(\alpha/\eta,\eta/\alpha)}\psi_2^{\mathrm{in}}(c;s_1,s_2)
		 &
		\geq
		\frac{\eta^7}{SA}\inf_{c\in[1/\eta,\eta]}\sum_{s,a}(\rho_{s_1,s_2}(s,a)-c)^2
		\\
		 &
		\geq
		\frac{1}{\eta^2}\inf_{c\in[1/\eta,\eta]}(1\vee 1/c^2)\frac{\eta^7}{SA}\sum_{s,a}(\rho_{s_1,s_2}(s,a)-c)^2
		\\
		 &
		=
		\frac{1}{\eta^2}\inf_{c\in[1/\eta,\eta]}\psi_2^{\mathrm{in}}(c;s_1,s_2)
		,
	\end{split}
\end{equation}
where we have also used that $1\geq (1/\eta^2)(1\vee 1/c^2)$ for $c\in[1/\eta,\eta]$.
Because $\eta$ is independent of $n$ by \Cref{def:reachability}, it finally follows that
\begin{equation}
	0
	=
	\liminf_{n\rightarrow\infty}\inf_{c\in(\alpha/\eta,\eta/\alpha)}\bar{\psi}_2^{\mathrm{in}}(c)
	\geq
	\liminf_{n\rightarrow\infty}\frac{1}{\eta^2}\inf_{c\in[1/\eta,\eta]}\bar{\psi}_2^{\mathrm{in}}(c)
	\geq
	0
	.
\end{equation}
This proves claim $(iii)$, and concludes the proof.
\qed

\subsubsection{Proof of \texorpdfstring{\Cref{lem:information_quantity_bounds}}{lemma}}
\label{sec:proof_of_information_quantity_bounds}

\Cref{lem:information_quantity_bounds} follows as a corollary from the proof of \cite[Proposition 12]{Jedra:2022}.
Namely, the quantity $I_j(x;c,\Phi)$ appearing in \cite[Eqn.~(35)]{Jedra:2022} corresponds in our notation to $I_{j}(x;c,\Phi,\pi_U)$ for $j\in[S]$, after identifying $\rho \equiv \pi_U$ and $m^{\Psi}_{\rho}(s,a) \equiv \omega_{\Psi,\pi_U}(s,a)$ with $\omega_{\Psi,\pi_U}(s,a)$ as in \eqref{eqn:definition_of_visitation_probabilities}.
It then follows from \cite[Proposition 11]{Jedra:2022} that
\begin{equation}
	\label{eqn:upper_bound_in_terms_of_I_tilde}
	I_{\tilde{s}}(x;c,\Phi,\pi_U)
	\leq
	\max\{1,c,1/c,\eta\}\tilde{I}_j(x;c,\Phi)
	,
\end{equation}
where the quantity $\tilde{I}_j(x;c,\Phi)$ is defined in \cite[Eqn. (37)-(39)]{Jedra:2022}.

Now, consider the proof of \cite[Proposition 12]{Jedra:2022}.
There, the following upper bound for $\tilde{I}_j(x;c,\Phi)$ is proven:
\begin{equation}
	\label{eqn:upper_bound_on_I_tilde}
	\tilde{I}_j(x;c,\Phi)
	\leq
	\eta^7\frac{1}{SA}\sum_{s,a}\biggl(\frac{p(f(x)\mid s,a)}{cp(j\mid s,a)}-1\biggr)^2
	+
	c\eta^7\frac{1}{SA}\sum_{s,a}\biggl(\frac{p(s\mid f(x),a)}{p(s\mid j,a)}-1\biggr)^2,
\end{equation}
where we have truncated the proof after inequality (b) and have used their definitions $v_{s,a} = \allowbreak p(f(x)\mid s,a) / \allowbreak p(j\mid s,a)$ and $u_{s,a} = \allowbreak p(s\mid f(x),a) / \allowbreak p(s\mid j,a)$.
Together, \eqref{eqn:upper_bound_in_terms_of_I_tilde} and \eqref{eqn:upper_bound_on_I_tilde} then imply that
\begin{equation}
	\begin{split}
		 &
		I_{\tilde{s}}(x;c,\Phi,\pi_U)
		\\
		 &
		\leq
		(1\vee 1/c^2)\eta^8\frac{1}{SA}\sum_{s,a}\biggl(\frac{p(f(x)\mid s,a)}{p(\tilde{s}\mid s,a)}-c\biggr)^2
		+
		(1\vee c^2)\eta^8\frac{1}{SA}\sum_{s,a}\biggl(\frac{p(s\mid f(x),a)}{p(s\mid \tilde{s},a)}-1\biggr)^2
		\\
		 &
		=
		\psi_2(c;f(x),\tilde{s}).
	\end{split}
\end{equation}
Here, we have used that $\max\{1,c,1/c,\eta\}\leq (c\vee (1/c))\eta$ for $c>0$ because $\eta>1$ by \Cref{def:reachability}.
This concludes the proof of the upper bound in \eqref{eqn:information_quantity_bounds}.

The lower bound on $I_{\tilde{s}}(x;c,\Phi,\pi_U)$ also follows from \cite[Proposition 12]{Jedra:2022}, but requires an extra step to achieve a sufficiently tight bound.
To make this transparent, we first introduce some additional notation, which we will only need during this proof.
Specifically, define the following quantity:
\begin{equation}
	\label{eqn:Ijxcdef_with_different_omega}
	\begin{split}
		I^*_{\tilde{s}}(x;c,\Phi,\pi_U)
		\eqdef
		 &
		n\KL(p^{*,\mathrm{in}}_{x,\tilde{s},c},p^{*,\mathrm{in}}_{x,f(x),1})
		\\
		 &
		+
		ncq(x\mid f(x))\sum_{a}\omega_{\Phi,\pi_U}(f(x),a)\KL(p(\,\cdot \mid \tilde{s},a), p(\,\cdot \mid f(x),a)),
	\end{split}
\end{equation}
with
\begin{align}
	 &
	p^{*,\mathrm{in}}_{x,\tilde{s},c}(i,s,a)
	\eqdef
	\begin{cases}
		\omega_{\Phi,\pi}(s,a)p(\tilde{s}\mid s,a)cq(x\mid f(x))     & \textnormal{if } i = 1,         \\
		\omega_{\Phi,\pi}(s,a)(1-p(\tilde{s}\mid s,a)cq(x\mid f(x))) & \textnormal{otherwise; } i = 2. \\
	\end{cases}
\end{align}
Compared to \eqref{eqn:Ijxcdef} and \eqref{eqn:p_in_def}, the only difference is that we have replaced $\omega_{\Psi,\pi_U}$ with $\omega_{\Phi,\pi_U}$.

The quantity $I^*_{j}(x;c,\Phi,\pi_U)$ for $j\in[S]$ is also closely related to the quantity $\tilde{I}_j(x;c,\Phi)$ from \cite[Eqn. (37)-(39)]{Jedra:2022}.
Specifically, the two are related by a reversal of the arguments of the KL-divergences in \eqref{eqn:Ijxcdef_with_different_omega}.
Note here that their $m_{\rho}(s,a)$ corresponds to our $\omega_{\Phi,\pi}(s,a)$ in \eqref{eqn:definition_of_visitation_probabilities}.

We can now obtain a lower bound for $I^*_{\tilde{s}}(x;c,\Phi,\pi_U)$ from the proof of the lower bound for $\tilde{I}_j(x;c,\Phi)$ in \cite[Proposition 12]{Jedra:2022} as follows.
Notice firstly that the lower bound in \cite[Lemma 6]{Jedra:2022} is symmetric with respect to the arguments of the KL-divergence appearing therein.
This result is used in the first inequality, labeled (a), of the proof of the lower bound in \cite[Proposition 12]{Jedra:2022}.
This inequality is therefore also true when we reverse the arguments of the KL-divergence, in which case we find that the exact same lower bound also applies to $I^*_{\tilde{s}}(x;c,\Phi,\pi_U)$.

The remainder of the proof then goes through unchanged, and we find upon identifying their $\psi_1(p,\eta,c)$ with our $\psi_1(s_1;s_2;c)$ that
\begin{equation}
	I^*_{\tilde{s}}(x;c,\Phi,\pi_U)
	\geq
	\psi_1(c;f(x),\tilde{s})
	.
\end{equation}

It remains to relate $I^*_{\tilde{s}}(x;c,\Phi,\pi_U)$ and $I_{\tilde{s}}(x;c,\Phi,\pi_U)$.
Recall that the quantities $m_{\rho}(s,a)$ and $m_{\rho}^{\Psi}(s,a)$ appearing in \cite{Jedra:2022} correspond in our notation to $\omega_{\Phi,\pi_U}(s,a)$ and $\omega_{\Psi,\pi_U}(s,a)$, respectively.
As noted at the start of \cite[Appendix D.3.3]{Jedra:2022}, $\omega_{\Phi,\pi_U}(s,a)\sim \omega_{\Psi,\pi_U}(s,a)$.
It follows that $(i)$ there exists some constant $\mf{d}_1>0$ that is independent of $n$ such that for all sufficiently large $n$, $\omega_{\Phi,\pi_U}(s,a)/\allowbreak\omega_{\Psi,\pi_U}(s,a)\leq (1/\mf{d}_1)$ for all $s$ and $a$.

Moreover, upon expanding the KL-divergence appearing in \eqref{eqn:Ijxcdef} one finds that $I_{\tilde{s}}(x;c,\Phi,\pi_U)$ is linear in $\omega_{\Psi,\pi_U}$.
Likewise, $I^*_{\tilde{s}}(x;c,\Phi,\pi_U)$ is linear in $\omega_{\Phi,\pi_U}$.
Together with point $(i)$ this implies that $I_{\tilde{s}}(x;c,\Phi,\pi_U) \geq \mf{d}_1I^*_{\tilde{s}}(x;c,\Phi,\pi_U)$ for all sufficiently large $n$, from which the result now follows.
\qed

\subsection{Proof of \texorpdfstring{\Cref{prop:clustering_bound_short}}{clustering performance guarantee}}
\label{sec:proof_of_clustering_bound_short}

Recall that $T_k = kH$.
Let
\begin{equation}
	\label{eqn:k-star_def}
	k^*
	\eqdef
	\max\{k'>0:T_{k'}\leq \Theta^{\clust}\}
\end{equation}
be the number of episodes belonging to the first phase of \Cref{alg:BUCBVIouter}.
Because we assumed $\Theta^{\clust}=\omega(nA(S^3+\ln n))$ and $n=\Omega(H)$, it follows that $T_{k^*} \geq \Theta^{\clust}-H = \omega(nA(S^3+\ln n))$.
Recall also that $\hat{f}_0$ denotes the output of \cite[Algorithm 1]{Jedra:2022}, and that $\hat{f}_{\ell}$ for $\ell\geq 0$ are the intermediate iterates of \cite[Algorithm 2]{Jedra:2022}.
The final estimator is $\hat{f}^{\alg}=\hat{f}_{\ceil{\ln n}}$.

Following \cite{Jedra:2022}, as well as \cite{yun2014community,Sanders:2020}, the proof consists of two steps.
First, we bound the number of misclassifications made by $\hat{f}_0$.
Then, we show that the successive iterates $\hat{f}_{\ell}$ eliminate these errors so that $\hat{f}_{\ceil{\ln n}}$ achieves exact recovery.
These steps correspond to \cite[Theorems 2 and 3]{Jedra:2022}, respectively.

Two preliminary reductions will allow us to adapt and streamline the analysis for our setting:
\begin{enumerate}
	\item
	      Throughout \cite{Jedra:2022}, it is assumed that $I(\Phi)=\Omega(1)$.
	      We do not directly assume this, but it remains true under the assumptions of \Cref{lem:equivalence-of-assumptions}; see \Cref{sec:on-the-structural-properties}.
	\item
	      The trimming step in \cite[Algorithm 1]{Jedra:2022} is unnecessary in our regime.
		  Because $T_{k^*}=\omega(nA(S^3+\ln n))$, for sufficiently large $n$, $\Gamma_a=[n]$ for $a\in[A]$.
	      We assume throughout the proof that $n$ is large enough for this to be the case.
\end{enumerate}

To avoid unnecessary repetition, we refer to \cite{Jedra:2022} wherever arguments are unchanged.
We thus highlight only the adaptations necessary for our setting.
We therefore ask you to keep \cite{Jedra:2022}, and in particular \cite[Appendix F and G]{Jedra:2022}, on the side while reading onwards.

\subsubsection*{Step 1: Bounding $|\mcE(\hat{f}_{0},f)|$}

Recall \eqref{eqn:misclassdef}.
The number of misclassifications $|\mcE(\hat{f}_{0},f)|$ of $\hat{f}_{0}$ is bounded in \cite[Theorem 2]{Jedra:2022}, which is a consequence of the refined statement in \cite[Theorem 11]{Jedra:2022}.
We inspect the proof of \cite[Theorem 11]{Jedra:2022} to identify the event on which the implication of \cite[Theorem 2]{Jedra:2022} is true.

Recall \eqref{eqn:countsdef} and \eqref{eqn:countsdef_overload}, and recall from \Cref{sec:clusteringalgspectral} that $\hat{\mathbf{N}}_a=(\hat{N}_{k^*+1}(x,a,y))_{x,y\in[n]}\in\mathbb{R}^{n\times n}$ for $a\in[A]$.
Let $\mathbf{N}_a \eqdef (\sum_{k'\in[k]}\sum_{h=1}^H\allowbreak\tilde{N}_{a,k',h}(x,y))_{x,y\in[n]}$, where the $\tilde{N}_{a,k',h}(x,y)$ are defined at the start of \cite[Appendix F.1]{Jedra:2022} as $\tilde{N}_{a,k',h}(x,y)\eqdef  \mbbP[x_{k',h}=x,a_{k',h}=a,x_{k',h+1}=y\mid x_{k',h-2}]$ if $h>2$ and $\tilde{N}_{a,k',h}(x,y) \eqdef \mbbP[x_{k',h}=x,a_{k',h}=a,x_{k',h+1}=y]$ otherwise.
Finally, let $\lVert \,\cdot  \rVert$ denote the spectral norm for matrices.

Observe that because $\Gamma_a=[n]$ by reduction (2) above, the matrix $\hat{\mbf{N}}_{\Gamma,a}$, obtained from $\hat{\mbf{N}}$ by setting to zero all rows and columns corresponding to contexts not in $\Gamma_a$, equals $\hat{\mbf{N}}_a$.

Identifying
\begin{equation}
	\label{eqn:identifications_with_jedra}
	\hat{N}_a(x,y)\equiv \hat{N}_{k^*+1}(x,a,y)
	,
	\quad
	TH \equiv T_{k^*}
	,
\end{equation}
it now follows from the final step in the proof of \cite[Theorem 11]{Jedra:2022} that there exists a constant $\mf{c}_2>0$ such that for any $\mf{c}_1>0$
\begin{equation}
	\label{eqn:refined_spectral_norm_bound}
	\Omega_1
	\eqdef \biggl\{
	\max_{a}
	\lVert \hat{\mbf{N}}_a-\mbf{N}_a\rVert \leq \mf{c}_1\sqrt{\frac{T_{k^*}}{nA}}
	\biggr\}
	\subset
	\biggl\{
	\frac{|\mcE(\hat{f}_{0},f)|}{n}
	\leq
	\mf{c}_1\mf{c}_2\sqrt{\frac{nSA}{T_{k^*}}}
	\biggr\}
	.
\end{equation}
We note here that the right-hand side is the event whose probability is bounded by \cite[Theorem 2]{Jedra:2022}, and that $\Omega_1$ is the event whose probability is bounded by \cite[Proposition 19]{Jedra:2022}.
The latter states that there exists $\mf{c}_1>0$ such that \begin{equation}
	\label{eqn:original_omega_1_bound}
	\mbbP[\Omega_1]
	\geq
	1-\frac{2}{n}-2e^{-n}-2e^{-\frac{T_{k^*}}{nA}}
	=
	1-O\biggl(\frac{1}{n}\biggr),
\end{equation}
where equality follows from $T_{k^*}=\omega(nA(S^3\vee \ln n))$ as is assumed in \Cref{prop:clustering_bound_short}.
We will now explain how \eqref{eqn:original_omega_1_bound} can be refined to $1-O(1/n^c)$ for arbitrary $c>0$ by choosing constant $\mf{c}_1$ sufficiently large.

The proof of \cite[Proposition 19]{Jedra:2022} bounds $\max_{a}\lVert \hat{\mbf{N}}_a-\mbf{N}_a\rVert$ by three terms labeled $T_1$, $T_2$, and $T_3$.
These are subsequently bounded in \cite[(75)--(77)]{Jedra:2022}.
To show that $\mbbP[\Omega_1]=1-O(1/n^c)$ it suffices to improve the probability bound in \cite[(77)]{Jedra:2022}, which is a consequence of \cite[Lemma 18]{Jedra:2022}.
\cite[Lemma 18]{Jedra:2022} is proven by showing that a particular random matrix $Q$, defined therein, satisfies a discrepancy property\footnote{
	A matrix $Q$ is said to be discrepant if there exist constants $\xi_1,\xi_2>0$ such that for every $\mathcal{I},\mathcal{J}\subset[n]$, $e(\mcI,\mcJ)\eqdef \sum_{(x,y)\in \mathcal{I}\times\mathcal{J}} Q(x,y)$ satisfies either $e(\mcI,\mcJ)n^2A/(|\mcI||\mcJ|T_{k^*})\leq \xi_1$ or $e(\mcI,\mcJ)\ln(e(\mcI,\mcJ)n^2A/(|\mcI||\mcJ|T_{k^*}))\leq \xi_2(|\mcI|\vee|\mcJ|)\ln(n/(|\mcI|\vee|\mcJ|))$.
}
with high probability as $n\rightarrow \infty$.

To prove that $Q$ is discrepant, \cite{Jedra:2022} separately considers the cases $|\mathcal{J}|> n/5$ and $|\mathcal{J}|\leq n/5$.
The case $|\mathcal{J}|> n/5$ does not require modification.
When $|\mathcal{J}|\leq n/5$, it is shown that $Q$ is discrepant whenever a certain event $\mcE$, whose definition we do not reproduce here, holds.
The probability bound for $\mcE^{\mathrm{c}}$ contributes the term $2/n$ in \eqref{eqn:original_omega_1_bound}, which we now show can be improved to $2/n^c$ for any $c>0$ by allowing $Q$ to be discrepant for a larger constant $\xi_2$.
The latter ultimately leads to a larger constant $\mf{c}_1$ in \eqref{eqn:refined_spectral_norm_bound}.
This tradeoff was also exploited in, e.g., \cite{hoffman2021spectral,ariu2023instance} in the context of random graphs and in \cite{van2024estimating} for \glspl{BMC}.

Specifically, it is shown in the proof of \cite[Lemma 18]{Jedra:2022} that

\begin{equation}
	\label{eqn:bound_event_E}
	\mbbP[\mcE]
	\geq
	1-\sum_{1\leq i\leq j\leq n/5}2\exp((4-\xi_2)j\ln (n/j)).
\end{equation}
The constant $\xi_2$ is then chosen as $\xi_2=7+\ln A/\ln n$, which leads to $\mbbP[\mcE]\geq 1-2/(nA)$.

To show that $\mbbP[\Omega_1]=1-O(1/n^c)$ instead, it suffices to choose $\xi_2$ large enough that $\mbbP[\mcE]=1-O(1/(n^cA))$.
This can be done for any fixed $c>0$ by choosing $\xi_2 = 6 + c + \ln A/\ln n$.
Note that \Cref{prop:clustering_bound_short}'s hypothesis that $A = n^{O(1)}$ ensures that $\xi_2$ is indeed upper bounded by a constant for sufficiently large $n$.
Substitution of this choice into \eqref{eqn:bound_event_E} yields
\begin{equation}
	\mbbP[\mcE]
	\geq
	1-\sum_{1\leq i\leq j\leq n/5}2\exp(-(2+c+\ln A/\ln n)j\ln (n/j))
	\overset{(i)}{\geq}
	1-2n^2\biggl(\frac{1}{n}\biggr)^{2+c+\ln A/\ln n}
	=
	1-\frac{2}{n^cA},
\end{equation}
where in $(i)$ we used that $j\ln (n/j)\geq \ln n$ for $n>5$ and $1\leq j\leq n/5$.
The proof of \cite[Lemma 18]{Jedra:2022} then goes through otherwise unchanged, and we conclude that for any $c>0$, and for sufficiently large $n$ and $\mf{c}_1>0$,
\begin{equation}
	\mbbP[\Omega_1]
	\geq
	1-\frac{2}{n^c}-2e^{-n}-2e^{-\frac{T_{k^*}}{nA}}
	.
\end{equation}
This concludes the analysis of $\hat{f}_{0}$.

\subsubsection*{Step 2: Bounding $|\mcE(\hat{f}_{\ell},f)|$}

Recall \eqref{eqn:misclassdef} and let $\mcE^{(\ell)}\eqdef \mcE(\hat{f}_{\ell},f)$ for $\ell\in\mbbN$ (recall \eqref{eqn:misclassdef}).
We proceed to inspect the proof of \cite[Theorem 3]{Jedra:2022} in \cite[Appendix G]{Jedra:2022} to derive a bound on the probability that $|\mcE(\hat{f}_{\ell},f)|=0$.
To this aim, we start by recalling definitions that are used in the proof of \cite[Theorem 3]{Jedra:2022}.

Following \cite{yun2016optimal,Sanders:2020}, \cite{Jedra:2022} defines a set of well-behaved contexts $\mcH$ in \cite[Definition 8]{Jedra:2022}.

Recall now reduction (2) above, such that also $\Gamma =\cup_{a}\Gamma_a=[n]$.
In this case, and upon making the identifications in \eqref{eqn:identifications_with_jedra}, the set $\mcH$ then reduces to largest subset $\mcH\subset [n]$ satisfying for $x\in\mcH$,
\begin{align}
	 &
	\hat{I}_{\tilde{s}}(x;\Phi)\geq \frac{1}{4\eta^2}\frac{T_{k^*}}{n}I(x;\Phi),
	\qquad
	\forall \tilde{s}\neq f(x),
	\qquad
	\mathrm{and},
	\label{eqn:H1_def}
	\\
	 &
	\sum_{a}
	(\hat{N}_{k^*+1}(x,a,[n]\setminus\mcH)+\hat{N}_{k^*+1}([n]\setminus\mcH,a,x))\leq 2\biggl(\ln\frac{T_{k^*}}{n}\biggr)^2
	.
	\label{eqn:H2_def}
\end{align}
Here, $I(x;\Phi)$ is defined in \eqref{eqn:def_minimum_information_quantity}, and $\hat{I}_{\tilde{s}}(x;\Phi)$ is defined in \cite[(78)]{Jedra:2022}.

Properties \eqref{eqn:H1_def} and \eqref{eqn:H2_def} correspond to properties (H1) and (H2) in \cite[(79)--(80)]{Jedra:2022}, respectively.
Define also
\begin{equation}
	\mcE^{(\ell)}_{\mcH}\eqdef \mcE^{(\ell)}\cap\mcH
	\quad
	\mathrm{and}
	\quad
	\mcE^{(\ell)}_{\mcH^{\mathrm{c}}}\eqdef \mcE^{(\ell)}\setminus \mcE^{{(\ell)}}\cap\mcH^{\mathrm{c}}=\mcE^{(\ell)}_{\mcH}.
\end{equation}

\cite[Proposition 21]{Jedra:2022} implies that the well-behaved contexts in $\mcH$ are eventually classified correctly by \Cref{alg:CIA}.
Specifically, $|\mcE^{(\ceil{\ln nA})}_{\mcH}|=0$ with high probability as $n\rightarrow \infty$.
Meanwhile, \cite[Proposition 20]{Jedra:2022} bounds the number of contexts that are \textit{not} in $\mcH$, therefore bounding $|\mcE^{(\ell)}_{\mcH^{\mathrm{c}}}|$.
We separately inspect these two results to obtain a bound on the probability with which their implications hold.

\subsubsection*{Step 2(a): Bounding $|\mcE^{(\ell)}_{\mcH^{\mathrm{c}}}|$}

We first inspect the proof of \cite[Proposition 20]{Jedra:2022} to derive a bound for $|\mcE^{(\ell)}_{\mcH^{\mathrm{c}}}|$.
In the process, we explain how the results of \cite{Jedra:2022} can be refined to hold for $\eta$-reachable and $\mcI$-identifiable \glspl{BMDP}.

Note that because $|\mcE^{(\ell)}_{\mcH^{\mathrm{c}}}|\leq |\mcH^{\mathrm{c}}|$, it suffices to bound $|\mcH^{\mathrm{c}}|$.
In \cite{Jedra:2022}, the authors separately bound $|\mcH^{\mathrm{c}}\cap\Gamma^{\mathrm{c}}|$ and $|\mcH^{\mathrm{c}}\cap\Gamma|$.
Since $\Gamma=[n]$ by reduction (2), it suffices to consider the bound for $\mcH^{\mathrm{c}}\cap\Gamma=\mcH^{\mathrm{c}}$.

The second part of the proof of \cite[Proposition 20]{Jedra:2022} (specifically during the ``Final construction'') bounds the cardinality of $\mcH^{\mathrm{c}}\cap\Gamma=\mcH^{\mathrm{c}}$ by
\begin{equation}
	\label{eqn:bound_on_Hc_Gamma}
	|\mcH^{\mathrm{c}}|
	\leq
	|\mcH_1^{\mathrm{c}}|+t^*
	,
\end{equation}
where $\mcH_1 \eqdef \{x\in[n]:\textnormal{\eqref{eqn:H1_def} holds}\}$ and the term $t^*$ is defined iteratively.

We proceed to bound the two terms on the right-hand side of \eqref{eqn:bound_on_Hc_Gamma}.

The term $|\mcH_1^{\mathrm{c}}|$ is bounded immediately following \cite[Proposition 22]{Jedra:2022}.
There, it is shown that there exists a constant $C>0$ such that if $T_{k^*}=\omega(n)$, then for all sufficiently large $n$,
\begin{equation}
	\label{eqn:H_1_c_bound_prop_22}
	\mbbP\biggl[|\mcH_1^{\mathrm{c}}|\geq \sum_x\exp\biggl(-C\frac{T_{k^*}}{n}I(x;\Phi)\biggr)\biggr]
	\leq
	(S-1)\sum_x\exp\biggl(-C\frac{T_{k^*}}{n}I(x;\Phi)\biggr).
\end{equation}
Now, let
\begin{equation}
	\label{eqn:constant_s_def}
	s\eqdef \lfloor 2\sum_x\exp(-C(T_{k^*}/n) I(x;\Phi))\rfloor
\end{equation}
and define
\begin{equation}
	\label{eqn:bound_on_mcH1_event}
	\Omega_2\eqdef \{|\mcH_1^{\mathrm{c}}|\leq s/2\}.
\end{equation}
Since $|\mcH_1^{\mathrm{c}}|\in\mbbN$, \eqref{eqn:H_1_c_bound_prop_22} implies that
\begin{equation}
	\label{eqn:bound_on_H1_c}
	\mbbP[\Omega_2]\geq 1-(S-1)\sum_x\exp\biggl(-C\frac{T_{k^*}}{n}I(x;\Phi)\biggr).
\end{equation}

We next bound the term $t^*$ in \eqref{eqn:bound_on_Hc_Gamma}.
First define
\begin{equation}
	\Omega_3\eqdef \biggl\{\forall \mcX\subset[n]: |\mcX|=s, \hat{N}_{k^*+1}(\mcX,\mcX)\geq s\biggl(\ln\frac{T_{k^*}}{n}\biggr)^2\biggr\}.
\end{equation}
Note from \eqref{eqn:identifications_with_jedra} that the probability of $\Omega_3$ can be bounded from below using \cite[Lemma 20]{Jedra:2022}; that is, for all sufficiently large $n$,
\begin{equation}
	\label{eqn:bound_from_lemma_20}
	\mbbP[\Omega_3]
	\geq
	1-2\exp\biggl(-\frac{1}{8}\frac{T_{k^*}}{n}\ln\frac{T_{k^*}}{n}\biggr)
	.
\end{equation}
The proof of \cite[Proposition 20]{Jedra:2022} demonstrates that on $\Omega_3$, $|\mcH_1^{\mathrm{c}}|\leq s/2$ implies that $t^*\leq s/2$.
Recalling \eqref{eqn:bound_on_mcH1_event}, it finally follows that
\begin{equation}
	\label{eqn:implication_of_prop_20}
	\Omega_2\cap\Omega_3
	\subset
	\{|\mcH_1^{\mathrm{c}}|\leq s/2, t^*\leq s/2\}
	.
\end{equation}
Combining \eqref{eqn:bound_on_Hc_Gamma} and \eqref{eqn:implication_of_prop_20}, and recalling \eqref{eqn:constant_s_def}, we conclude that
\begin{equation}
	\label{eqn:Hc_bound_event_inclusion}
	\Omega_2\cap\Omega_3
	\subset
	\biggl\{|\mcH^{\mathrm{c}}|\leq 2\sum_x\exp\biggl(-C\frac{T_{k^*}}{n}I(x;\Phi)\biggr)\biggr\}
	.
\end{equation}

We conclude by showing that \eqref{eqn:Hc_bound_event_inclusion} implies $\mcE_{\mcH^{\mathrm{c}}}^{(\ell)}=0$ with high probability.

Since $\Phi$ is $\eta$-reachable and $\mcI$-identifiable, \Cref{lem:equivalence-of-assumptions} implies there exists a constant $\mc{I}'>0$ such that $I(x;\Phi)\geq \mcI'$.
Together with \eqref{eqn:Hc_bound_event_inclusion} this implies that
\begin{equation}
	\label{eqn:intermediate_inclusion_ill_behaved_contexts}
	\Omega_2\cap\Omega_3
	\subset
	\biggl\{
	|\mcH^{\mathrm{c}}|
	\leq
	2n\exp\biggl(-C\frac{T_{k^*}}{n}\mcI'\biggr)
	\biggr\}
	\subset
	\bigcap_{\ell\geq 0}
	\biggl\{
	\frac{|\mcE^{(\ell)}_{\mcH^{\mathrm{c}}}|}{n}
	\leq
	2\exp\biggl(-C\frac{T_{k^*}}{n}\mcI'\biggr)
	\biggr\}
	,
\end{equation}
where we have also used that $|\mcE_{\mcH^{\mathrm{c}}}^{(\ell)}|\leq|\mcH^{\mathrm{c}}|$ for $\ell\geq 0$.
Since $|\mcE^{(\ell)}_{\mcH^{\mathrm{c}}}|\in\mbbN$ and $T_{k^*}=\omega(nA(S^3\vee \ln n))$, it follows from \eqref{eqn:intermediate_inclusion_ill_behaved_contexts} that for all sufficiently large $n$,
\begin{equation}
	\label{eqn:final_Hc_bound_event_inclusion}
	\Omega_2\cap\Omega_3
	\subset
	\Omega_4
	,
	\quad
	\mathrm{with}
	\quad
	\Omega_4
	\eqdef
	\bigcap_{\ell\geq 0}\{|\mcE^{(\ell)}_{\mcH^{\mathrm{c}}}| = 0\}
	.
\end{equation}
In addition, \eqref{eqn:bound_on_H1_c}, \eqref{eqn:bound_from_lemma_20}, and \Cref{lem:equivalence-of-assumptions} imply that for all sufficiently large $n$,
\begin{equation}
	\label{eqn:bound_illbehaved_contexts}
	\begin{split}
		\mbbP\bigl[
			\Omega_4
			\bigr]
		\overset{\eqref{eqn:final_Hc_bound_event_inclusion}}{\geq}
		\mbbP\bigl[
			\Omega_2\cap\Omega_3
			\bigr]
		\geq
		1
		-
		n(S-1)\exp\biggl(-C\frac{T_{k^*}}{n}\mcI'\biggr)
		-
		2\exp\biggl(-\frac{1}{8}\frac{T_{k^*}}{n}\ln \frac{T_{k^*}}{n}\biggr)
		.
	\end{split}
\end{equation}

This concludes the bound for $|\mcE_{\mcH^{\mathrm{c}}}^{(\ell)}|$.

\subsubsection*{Step 2(b): Bounding $|\mcE_{\mcH}^{(\ell)}|$}

We next obtain a bound for $|\mcE_{\mcH}^{(\ell)}|$ by inspecting the proof of \cite[Proposition 21]{Jedra:2022}.
Recall \eqref{eqn:identifications_with_jedra}, and that $\mcE_{\mcH}^{(\ell)}=\mcE^{(\ell)}\cap\mcH$ where $\mcE^{(\ell)}=\mcE(\hat{f}_{\ell},f)$ and $\mcH$ is as in \eqref{eqn:H1_def} and \eqref{eqn:H2_def}.
Define
\begin{align}
	 &
	\Omega_5^{(\ell)}\eqdef \biggl\{\frac{|\mcE^{(\ell)}|}{n}\leq \mf{c}_3\sqrt{\frac{nSA}{T_{k^*}}}\biggr\}
	,
	\nonumber
	\\
	 &
	\Omega_6\eqdef\bigcap_{\mcX,\mcY\subset[n]}\biggl\{|\hat{N}_{k^*+1}(\mcX,a,\mcY)-\mbbE[\hat{N}_{k^*+1}(\mcX,a,\mcY)]|\leq \mf{c}_4\sqrt{\frac{nT_{k^*}}{A}}\biggr\}
	,
	\quad
	\textnormal{and}
	\label{eqn:definition_of_events_E_H}
	\\
	 &
	\Omega_7
	\eqdef
	\bigcap_{x}
	\biggl\{
	\max
	\biggl\{\sum_{a}\hat{N}_{k^*+1}(x,a,[n]),\sum_{a}\hat{N}_{k^*+1}([n],a,x)\biggr\}\leq \mf{c}_5\frac{T_{k^*}}{n}
	\biggr\}
	,
	\nonumber
\end{align}
where the constants $\mf{c}_3,\mf{c}_4,\mf{c}_5>0$ are chosen as follows.

Firstly, take $\mf{c}_3$ large enough that $\Omega_1\subset\Omega_5^{(0)}$, where we recall that $\mcE^{(0)}=\mcE(\hat{f}_{0},f)$.

Secondly, note from \eqref{eqn:identifications_with_jedra} that \cite[Proposition 24]{Jedra:2022} bounds $\mbbP[\Omega_6^{\mathrm{c}}]$ for sufficiently large $\mf{c}_4$.
Take $\mf{c}_4$ large enough that \cite[Proposition 24]{Jedra:2022} applies.

Lastly, we claim that $\mbbP[\Omega^7]$ can be bounded for sufficiently large $\mf{c}_5$ using \cite[Lemma 17]{Jedra:2022}.
Given this claim, take $\mf{c}_5$ large enough so that \cite[Lemma 17]{Jedra:2022} applies.

To justify the claim, observe from \eqref{eqn:countsdef} that $\sum_{a}\hat{N}_{k^*+1}([n],a,x)$ and $\sum_{a}\hat{N}_{k^*+1}(x,a,[n])$ count the number of transitions to and from a context $x$ in $k^*$ episodes, respectively.
Since these differ by at most one visit per episode (when $x$ is one of the end points of an episode) it follows that
\begin{equation}
	\sum_{a}\hat{N}_{k^*+1}([n],a,x)\leq \sum_{a}\hat{N}_{k^*+1}(x,a,[n])+k^*
	\quad
	\textnormal{almost surely}.
\end{equation}
Finally, recall that $\Gamma_a=[n]$ for $a\in[A]$ by reduction (2).
Together with \eqref{eqn:identifications_with_jedra}, these observations prove that \cite[Lemma 17]{Jedra:2022} bounds $\mbbP[\Omega_7^{\mathrm{c}}]$ for sufficiently large $\mf{c}_5$.

The following lemma summarizes the intermediate result that we require from the proof of \cite[Proposition 21]{Jedra:2022} to bound $|\mcE_{\mcH}^{(\ell)}|$.
It is proven in \Cref{sec:proof_of_bound_on_one_step_improvement_error_rate}.
\begin{lemma}
	\label{lem:bound_on_one_step_improvement_error_rate}
	Assume $\Phi$ is $\eta$-reachable and $\mcI$-identifiable, and that $T_{k^*}=\omega(nS^3A)$.
	Then there exist constants $\mf{n}_1,\mf{c}_5>0$ such that for $n>\mf{n}_1$, $\ell\geq 0$, on $\Omega_1\cap\Omega_5^{(\ell)}\cap\Omega_6\cap\Omega_7$,
	\begin{equation}
		\label{eqn:bound_on_one_step_improvement_error_rate}
		\begin{split}
			\frac{T_{k^*}}{n}
			|\mcE_{\mcH}^{(\ell+1)}|
			\leq
			\mf{c}_5
			\biggl(
			\frac{T_{k^*}}{n}
			\frac{|\mcE_{\mcH}^{(\ell)}|}{n}
			|\mcE_{\mcH}^{(\ell+1)}|
			+
			 &
			\sqrt{|\mcE_{\mcH}^{(\ell+1)}||\mcE_{\mcH}^{(\ell)}|\frac{T_{k^*}A}{n}}
			+
			|\mcE_{\mcH}^{(\ell+1)}|\bigl(\ln\frac{T_{k^*}}{n}\bigr)^2
			\\
			 &
			+
			|\mcE_{\mcH}^{(\ell+1)}|SA\Bigl(
			\frac{|\mcE^{(\ell)}|}{n}
			\frac{T_{k^*}}{nA}
			+
			S\sqrt{\frac{T_{k^*}}{nA}}
			\Bigr)
			\biggr)
			.
		\end{split}
	\end{equation}

\end{lemma}

Recall that $\Omega_1$ and $\Omega_4$ were defined in \eqref{eqn:refined_spectral_norm_bound} and \eqref{eqn:final_Hc_bound_event_inclusion}, respectively.
In the remainder of this section we will use \Cref{lem:bound_on_one_step_improvement_error_rate} to prove that if $T_{k^*}=\omega(nA(S^3\vee \ln n))$, then for sufficiently large $n$,
\begin{equation}
	\label{eqn:inclusion_vanishing_E_H}
	\Omega_1\cap\Omega_4\cap\Omega_6\cap\Omega_7
	\subset
	\{|\mcE_{\mcH}^{(\ceil{\ln n})}|=0\}
	.
\end{equation}
\cite[Appendix G.2]{Jedra:2022} and \cite{yun2016optimal,Sanders:2020} perform similar analyses, proving that $|\mcE_{\mcH}^{(\ceil{\ln n})}|=0$ with high probability as $n\rightarrow\infty$ in their settings.
However, we must carefully keep track of the events and conditions for which $|\mcE_{\mcH}^{(\floor{\ln n})}|=0$, so that we can ultimately obtain an explicit bound on the probability of this event.
Our analysis therefore deviates from that of \cite{Jedra:2022}, such that the remainder of this section can be read independently of \cite{Jedra:2022}.

Assume that $T_{k^*}=\omega(nA(S^3\vee \ln n))$.
Our first step is to show that there exists $\mf{n}_2>0$ such that for $n>\mf{n}_2$ and $\ell\geq 0$, $\Omega_1\cap\Omega_4\cap\Omega_5^{(\ell)}\cap\Omega_6\cap\Omega_7\subset \Omega_1\cap\Omega_4\cap\Omega_5^{(\ell+1)}\cap\Omega_6\cap\Omega_7$.

Since $T_{k^*} = \omega(nA(S^3\vee \ln n))$ there exist $\mf{n}_3,\mf{c}_6>0$ such that for $n>\mf{n}_3$ and $\ell\geq 0$, on $\Omega_4$,
\begin{equation}
	|\mcE^{(\ell)}_{\mcH^{\mathrm{c}}}|/n\leq \mf{c}_6\sqrt{nSA/T_{k^*}}
	.
\end{equation}
It therefore suffices to prove that there exists $\mf{n}_2\geq \mf{n}_3$ such that for $n>\mf{n}_2$ and $\ell\geq 0$, on $\Omega_1\cap\Omega_4\cap\Omega_5^{(\ell)}\cap\Omega_6\cap\Omega_7$,
\begin{equation}
	|\mcE^{(\ell+1)}_{\mcH}|\leq |\mcE^{(\ell)}_{\mcH}|
	.
\end{equation}

Let $\ell\geq 0$.
Assume $\Omega_1\cap\Omega_4\cap\Omega_5^{(\ell)}\cap\Omega_6\cap\Omega_7$ holds and that $n>\mf{n}_1$, where $\mf{n}_1$ is as in \Cref{lem:bound_on_one_step_improvement_error_rate}.
If $\mcE_{\mcH}^{(\ell+1)}=0$ there is nothing left to prove.
Hence, also assume that $\mcE_{\mcH}^{(\ell+1)}>0$.

\Cref{lem:bound_on_one_step_improvement_error_rate} now implies \eqref{eqn:bound_on_one_step_improvement_error_rate}.
Dividing both sides of \eqref{eqn:bound_on_one_step_improvement_error_rate} by $(T_{k^*}/n)\allowbreak |\mcE_{\mcH}^{(\ell+1)}|$, we conclude that
\begin{equation}
	\label{eqn:misclassification_ratio_bound_after_division_before_bounding}
	1 \leq
	\mf{c}_5
	\biggl(
	\frac{|\mcE_{\mcH}^{(\ell)}|}{n}
	+
	\sqrt{\frac{|\mcE_{\mcH}^{(\ell)}|}{|\mcE_{\mcH}^{(\ell+1)}|}\frac{nA}{T_{k^*}}}
	+
	\frac{n}{T_{k^*}}
	\biggl(\ln\frac{T_{k^*}}{n}\biggr)^2
	+
	\frac{nSA}{T_{k^*}}
	\biggl(
	\frac{|\mcE^{(\ell)}|}{n}
	\frac{T_{k^*}}{nA}
	+
	S\sqrt{\frac{T_{k^*}}{nA}}
	\biggr)
	\biggr)
	.
\end{equation}
Moreover, on $\Omega_5^{(\ell)}$, \eqref{eqn:misclassification_ratio_bound_after_division_before_bounding} implies that
\begin{equation}
	\label{eqn:misclassification_ratio_bound_after_division}
	\begin{split}
		1
		\leq
		\mf{c}_5
		\biggl(
		\mf{c}_3\sqrt{\frac{nSA}{T_{k^*}}}
		+
		 &
		\sqrt{\frac{|\mcE_{\mcH}^{(\ell)}|}{|\mcE_{\mcH}^{(\ell+1)}|}\frac{nA}{T_{k^*}}}
		+
		\frac{n}{T_{k^*}}
		\biggl(\ln\frac{T_{k^*}}{n}\biggr)^2
		+
		\frac{nSA}{T_{k^*}}
		\biggl(
		\mf{c}_3\sqrt{\frac{nSA}{T_{k^*}}}\frac{T_{k^*}}{nA}
		+
		S\sqrt{\frac{T_{k^*}}{nA}}
		\biggr)
		\biggr)
		.
	\end{split}
\end{equation}
Given that $T_{k^*}=\omega(nA(S^3\vee \ln n))$, all terms on the right-hand side of \eqref{eqn:misclassification_ratio_bound_after_division} are $o(1)$ except possibly $\sqrt{\smash[b]{(nA/T_{k^*})|\mcE_{\mcH}^{(\ell)}|/|\mcE_{\mcH}^{(\ell+1)}|}}$.
We therefore conclude that there exist constants $\mf{c}_7>0$ and $\mf{n}_4>\mf{n}_3$, both independent of $\ell$, such that for $n>\mf{n}_4$, $\sqrt{\smash[b]{(nA/T_{k^*})|\mcE_{\mcH}^{(\ell)}|/|\mcE_{\mcH}^{(\ell+1)}|}}\geq \mf{c}_7$.
Equivalently,
\begin{equation}
	\label{eqn:misclassification_ratio_bound}
	\frac{|\mcE^{(\ell+1)}_{\mcH}|}{|\mcE^{(\ell)}_{\mcH}|}
	\leq
	\mf{c}_7^2\frac{nA}{T_{k^*}}
	.
\end{equation}
Since $T_{k^*}=\omega(nA(S^3\vee \ln n))$ it follows that there exists a constant $\mf{n}_5>\mf{n}_4$ independent of $\ell$ such that $|\mcE^{(\ell+1)}_{\mcH}|\leq |\mcE^{(\ell)}_{\mcH}|$ almost surely.
We have therefore established that there exists $\mf{n}_2=\mf{n}_3\vee \mf{n}_5$ such that for $n>\mf{n}_2$ and $\ell\geq 0$,
\begin{equation}
	\label{eqn:intermediate_event_inclusion}
	\Omega_1\cap\Omega_4\cap\Omega_5^{(\ell)}\cap\Omega_6\cap\Omega_7
	\subset
	\Omega_1\cap\Omega_4\cap\Omega_5^{(\ell+1)}\cap\Omega_6\cap\Omega_7
	.
\end{equation}

We can now iterate \eqref{eqn:intermediate_event_inclusion} to conclude that for $n>\mf{n}_2$,
\begin{equation}
	\Omega_1\cap\Omega_4\cap \Omega_5^{(0)}\cap\Omega_6\cap\Omega_7
	\subset
	\Omega_1\cap\Omega_4\cap(\cap_{\ell\geq 0}\Omega_5^{(\ell)})\cap\Omega_6\cap\Omega_7
	.
\end{equation}
Recall also that we assumed $\mf{c}_3$ in \eqref{eqn:refined_spectral_norm_bound} is large enough that $\Omega_1\subset \Omega_5^{(0)}$.
Together, these observations imply that
\begin{equation}
	\label{eqn:inclusion_into_intersection}
	\Omega_1\cap\Omega_4\cap\Omega_6\cap\Omega_7
	\subset
	\Omega_1\cap\Omega_4\cap \Omega_5^{(0)}\cap\Omega_6\cap\Omega_7
	\subset
	\Omega_1\cap\Omega_4\cap(\cap_{\ell\geq 0}\Omega_5^{(\ell)})\cap\Omega_6\cap\Omega_7
	.
\end{equation}
Finally, recall that for $n>\mf{n}_2$ and $\ell\geq 0$, \eqref{eqn:misclassification_ratio_bound} holds on $\Omega_1\cap\Omega_4\cap\Omega_5^{(\ell)}\cap\Omega_6\cap\Omega_7$.
Since $|\mcE^{(\ell)}_{\mcH}|\leq n$ for $\ell\geq 0$, it follows from \eqref{eqn:inclusion_into_intersection} that for $n>\mf{n}_2$, on $\Omega_1\cap\Omega_4\cap\Omega_6\cap\Omega_7$,
\begin{equation}
	\label{eqn:iterated_bound_well_behaved_contexts}
	|\mcE^{(\ceil{\ln n})}_{\mcH}|
	\leq
	\biggl(\mf{c}_7^2\frac{nA}{T_{k^*}}\biggr)^{\ceil{\ln n}}|\mcE^{(0)}_{\mcH}|
	\leq
	n\biggl(\mf{c}_7^2\frac{nA}{T_{k^*}}\biggr)^{\ln n}
	=
	n\biggl(\frac{1}{n}\biggr)^{\ln\bigl(\frac{T_{k^*}}{\mf{c}_7^2nA}\bigr)}
	.
\end{equation}
Because $T_{k^*}=\omega(nA(S^3\vee \ln n))$ there exists $\mf{n}_6>\mf{n}_2$ such that for $n>\mf{n}_6$ the right-hand side of \eqref{eqn:iterated_bound_well_behaved_contexts} is strictly less than 1.
\eqref{eqn:inclusion_vanishing_E_H} finally follows for $n>\mf{n}_6$ because $\mcE^{(\ceil{\ln n})}_{\mcH}\in\mbbN$.

\subsubsection*{Combining the bounds}

Using $|\mcE^{(\ell)}|=|\mcE^{(\ell)}_{\mcH}|+|\mcE^{(\ell)}_{\mcH^{\mathrm{c}}}|$ and $\hat{f}^{\alg}=\hat{f}_{\ceil{\ln n}}$, together with \eqref{eqn:final_Hc_bound_event_inclusion} and \eqref{eqn:inclusion_vanishing_E_H}, we find that for all sufficiently large $n$,
\begin{equation}
	\label{eqn:inclusion_perfect_clustering}
	\Omega_1\cap\Omega_4\cap\Omega_6\cap\Omega_7
	\subset
	\{
	|\mcE^{(\ceil{\ln n})}|
	=
	0
	\}
	=
	\{
	\hat{f}^{\alg}
	=
	f
	\}
	.
\end{equation}
Recall also from the discussion following \eqref{eqn:definition_of_events_E_H} that $\mf{c}_4$ in \eqref{eqn:definition_of_events_E_H} is assumed large enough that \cite[Proposition 24]{Jedra:2022} bounds $\mbbP[\Omega_6^{\mathrm{c}}]$ and that $\mf{c}_5$ is assumed large enough that \cite[Lemma 17]{Jedra:2022} bounds $\mbbP[\Omega_7^{\mathrm{c}}]$.
Together with \eqref{eqn:refined_spectral_norm_bound} and \eqref{eqn:bound_illbehaved_contexts}, we find that for all sufficiently large $n$,
\begin{equation}
	\begin{split}
		 &
		\mbbP[\hat{f}^{\alg}=f]
		\overset{\eqref{eqn:inclusion_perfect_clustering}}{\geq}
		\mbbP[\Omega_1\cap\Omega_3\cap\Omega_6\cap\Omega_7]
		\\
		 &
		\geq
		1
		-
		\frac{2}{n^c}
		-
		2e^{-n}-2e^{-\frac{T_{k^*}}{nA}}
		-
		n(S-1)e^{-C\frac{T_{k^*}}{n}\mcI'}
		-
		2e^{-\frac{1}{8}\frac{T_{k^*}}{n}\ln \frac{T_{k^*}}{n}}
		-
		4e^{-2n(1-\ln 2)}
		-
		2e^{-\frac{T_{k^*}}{nA}}
		.
	\end{split}
\end{equation}
The right-hand side is $1-O(1/n^c)$ when $T_{k^*} = \omega(nA(S^3\vee\ln n))$.
\Cref{prop:clustering_bound_short} follows from this.
\qed

\subsubsection{Proof of \texorpdfstring{\Cref{lem:bound_on_one_step_improvement_error_rate}}{lemma}}
\label{sec:proof_of_bound_on_one_step_improvement_error_rate}

The following proof refers to results in \cite[Appendices G.2 and G.4]{Jedra:2022}.
Note that the quantity $e^{(\ell)}$ appearing therein equals $|\mcE^{(\ell)}_{\mcH}|$.
Recall also that $\mcE_{\mcH}^{(\ell)}=\mcE^{(\ell)}\cap\mcH$ where $\mcE^{(\ell)}=\mcE(\hat{f}_{\ell},f)$ and $\mcH$ is as in \eqref{eqn:H1_def} and \eqref{eqn:H2_def}.
Finally, recall the quantity $\mcL_{\ell}(x,s)$ calculated by the improvement algorithm; see \Cref{sec:clusteringalgcia}.

With the identifications in \eqref{eqn:identifications_with_jedra}, recall for $\ell\geq 0$ the following quantity from \cite[(86)]{Jedra:2022}:
\begin{equation}
	\label{eqn:E_term_correction}
	E
	\equiv
	E^{(\ell)}
	\eqdef
	\sum_{x\in \mcE_{\mcH}^{(\ell+1)}}
	[\mcL_{\ell}(x,\hat{f}_{\ell+1}(x))-\mcL_{\ell}(x,f(x))]
	.
\end{equation}
Here, we have corrected a typographical error appearing in \cite[(86)]{Jedra:2022} by replacing $\mcE_{\mcH}^{(\ell)}$ with $\mcE_{\mcH}^{(\ell+1)}$.
This rectification ensures the claims (i) $E^{(\ell)}\geq 0$ and (ii) $E^{(\ell)}=E_1^{(\ell)}+E_2^{(\ell)}+U^{(\ell)}$ from \cite{Jedra:2022} are true, with $E^{(\ell)}_1$ and $E^{(\ell)}_2$ defined in \cite[(87)--(89)]{Jedra:2022} and $U^{(\ell)}\eqdef E^{(\ell)}-E^{(\ell)}_1-E^{(\ell)}_2$; see also the explaining paragraph preceding \cite[(35)]{Sanders:2020}.
Together, claims (i) and (ii) imply that $-E^{(\ell)}_1\leq E^{(\ell)}_2 + |U^{(\ell)}|$.

To prove \Cref{lem:bound_on_one_step_improvement_error_rate} it now suffices to show that if $T_{k^*}=\omega(nS^3A)$, then there exist constants $\mf{n}_{1}, \mf{c}_8, \mf{c}_9, \mf{c}_{10}>0$ such that for $n>\mf{n}_{1}$ and $\ell\geq 0$, the following holds on $\Omega_1\cap\Omega_5^{(\ell)}\cap\Omega_6$:
\begin{align}
	 &
	-
	E^{(\ell)}_1
	\geq
	\mf{c}_8\frac{T_{k^*}}{n}|\mcE_{\mcH}^{(\ell+1)}|
	,
	\label{eqn:E1_bound}
	\\
	 &
	E^{(\ell)}_2
	\leq
	\mf{c}_9
	\biggl(
	\frac{T_{k^*}}{n}
	\frac{|\mcE_{\mcH}^{(\ell)}|}{n}
	|\mcE_{\mcH}^{(\ell+1)}|
	+
	\sqrt{|\mcE_{\mcH}^{(\ell+1)}|\mcE_{\mcH}^{(\ell)}|\frac{T_{k^*}A}{n}}
	+
	|\mcE^{(\ell+1)}_{\mcH}|\biggl(\ln \frac{T_{k^*}}{n}\biggr)^2
	\biggr)
	,
	\label{eqn:E2_bound}
	\\
	 &
	|U^{(\ell)}|
	\leq
	\mf{c}_{10}
	|\mcE_{\mcH}^{(\ell+1)}|
	SA
	\biggl(
	\frac{|\mcE^{(\ell)}|}{n}
	\frac{T_{k^*}}{nA}
	+
	S\sqrt{\frac{T_{k^*}}{nA}}
	\biggr)
	.
	\label{eqn:U_bound}
\end{align}
This would then imply \eqref{eqn:bound_on_one_step_improvement_error_rate} for sufficiently large $\mf{c}_5>0$.

We first consider the term $E_1^{(\ell)}$.
Let $\ell\geq 0$.
It follows from the proof of \cite[Lemma 21]{Jedra:2022} that
\begin{equation}
	-
	E^{(\ell)}_1
	\geq
	\frac{1}{4\eta^2}
	\frac{T_{k^*}}{n}
	\sum_{x\in\mcE_{\mcH}^{(\ell+1)}}
	I(x;\Phi)
	\geq
	\frac{\min_{x}I(x;\Phi)}{4\eta^2}
	\frac{T_{k^*}}{n}
	|\mcE_{\mcH}^{(\ell+1)}|
	.
\end{equation}
It then follows from \Cref{lem:equivalence-of-assumptions} that for $\eta$-reachable and $\mcI$-identifiable \glspl{BMDP}, $\min_{x}I(x;\Phi)\allowbreak /4\eta^2\geq \mf{c}_8$ for some constant $\mf{c}_8>0$ that is independent of $\ell$.
This gives \eqref{eqn:E1_bound}.

Next, the proof of \cite[Lemma 21]{Jedra:2022} also demonstrates that there exists a constant $\mf{c}_9>0$ such that on $\Omega_1$,
\begin{equation}
	E_2^{(\ell)}
	\leq
	\mf{c}_9
	\biggl(
	\frac{T_{k^*}}{n}
	\frac{|\mcE_{\mcH}^{(\ell)}|}{n}
	|\mcE_{\mcH}^{(\ell+1)}|
	+
	\sqrt{|\mcE_{\mcH}^{(\ell+1)}|\mcE_{\mcH}^{(\ell)}|\frac{T_{k^*}A}{n}}
	+
	|\mcE^{(\ell+1)}_{\mcH}|\biggl(\ln \frac{T_{k^*}}{n}\biggr)^2
	\biggr)
	.
\end{equation}
The event $\Omega_1$ is used in the proof of \cite[Lemma 21]{Jedra:2022} when \cite[Proposition 19]{Jedra:2022} is invoked, which implies that $\Omega_1$ holds with high probability as $n\rightarrow \infty$.
This observation also demonstrates that $\mf{c}_9$ does not depend on $\ell$, therefore implying \eqref{eqn:E2_bound}.

We conclude by proving \eqref{eqn:U_bound}.
The corresponding bound in \cite[Lemma 21]{Jedra:2022} and its proof contain two typographical errors.
Though these ultimately do not affect the conclusions of \cite{Jedra:2022}, correcting them is important for us to identify the event on which \eqref{eqn:U_bound} holds.
We therefore carry out the proof of \eqref{eqn:U_bound} in more detail, while nonetheless leveraging intermediate results from the proof of \cite[Lemma 21]{Jedra:2022} when convenient.

Recall that $\hat{p}_{\ell}$ and $\hat{p}_{\ell}^{\mathrm{bwd}}$ were defined in \Cref{sec:clusteringalgcia}.
Let $p_{\ell}$ and $p_{\ell}^{\mathrm{bwd}}$ be defined by replacing $\hat{N}_k$ with $\mbbE[\hat{N}_k]$ in the definitions for $\hat{p}_{\ell}$ and $\hat{p}_{\ell}^{\mathrm{bwd}}$, respectively.
The following expression for $U^{(\ell)}$ can be deduced from \eqref{eqn:identifications_with_jedra}, \eqref{eqn:E_term_correction} and \cite[(87)--(89)]{Jedra:2022}, and corrects a typographical error appearing in the proof of \cite[Lemma 21]{Jedra:2022}:
\begin{align}
	 &
	U^{(\ell)}= U^{(\ell)}_{\mathrm{in}} + U^{(\ell)}_{\mathrm{out}}, \quad \mathrm{where},
	\nonumber
	\\
	 &
	U^{(\ell)}_{\mathrm{in}}
	\eqdef
	\sum_{x\in\mcE_{\mcH}^{(\ell+1)}}
	\sum_{s,a}
	\bigg[
	\hat{N}_{k^*+1}(x,a,\hat{f}_{\ell}^{-1}(s))
	\lb
	\ln\frac{\hat{p}_{\ell}(s\mid \hat{f}_{\ell+1}(x),a)}{p(s\mid \hat{f}_{\ell+1}(x),a)}
	-
	\ln\frac{\hat{p}_{\ell}(s\mid f(x),a)}{p(s\mid f(x),a)}
	\rb
	\bigg],
	\label{eqn:U_bound_2}
	\\
	 &
	U^{(\ell)}_{\mathrm{out}}
	\eqdef
	\sum_{x\in\mcE_{\mcH}^{(\ell+1)}}
	\sum_{s,a}
	\bigg[
	\hat{N}_{k^*+1}(\hat{f}_{\ell}^{-1}(s),a,x)
	\lb
	\ln \frac{\hat{p}_{\ell}^{\mathrm{bwd}}(s,a\mid \hat{f}_{\ell+1}(x))}{p^{\mathrm{bwd}}(s,a\mid \hat{f}_{\ell+1}(x))}
	-
	\ln \frac{\hat{p}_{\ell}^{\mathrm{bwd}}(s,a\mid f(x))}{p^{\mathrm{bwd}}(s,a\mid f(x))}
	\rb
	\bigg].
	\nonumber
\end{align}
Note that the expression for $U^{(\ell)}_{\mathrm{out}}$ in \eqref{eqn:U_bound_2} contains $\hat{N}_{k^*+1}(\hat{f}_{\ell}^{-1}(s),a,x)$ correctly, whereas the corresponding expression in \cite{Jedra:2022} contains $\hat{N}_{k^*+1}\allowbreak (x,a,\hat{f}_{\ell}^{-1}(s))$ accidentally.

The following replaces \cite[Lemma 25]{Jedra:2022}:
\begin{lemma}
	\label{lem:logtransprobratiosbd}
	Assume $\Phi$ is $\eta$-reachable and $\mcI$-identifiable, and that $T_{k^*}=\omega(nS^3A)$.
	Then there exist constants $\mf{n}_7,\mf{c}_{11},\mf{c}_{12}>0$ such that for $n>\mf{n}_7$ and $\ell\geq 0$, the following holds on $\Omega_5^{(\ell)}\cap\Omega_6$:
	\begin{equation}
		\label{eqn:log_p_pbwd_bd}
		\max
		\biggl\{ \left|\ln \frac{\hat{p}_{\ell}(s'\mid s,a)}{p(s'\mid s,a)}\right|, \left|\ln \frac{\hat{p}^{\mathrm{bwd}}_{\ell}(s,a\mid s')}{p^{\mathrm{bwd}}(s,a\mid s')}\right|\biggr\}
		\leq
		\mf{c}_{11}S\frac{|\mcE^{(\ell)}|}{n}
		+
		\mf{c}_{12}S\sqrt{\frac{nA}{T_{k^*}}}
		.
	\end{equation}
\end{lemma}
Note that besides identifying the event on which \eqref{eqn:log_p_pbwd_bd} holds, \Cref{lem:logtransprobratiosbd} also differs from \cite[Lemma 25]{Jedra:2022} because \eqref{eqn:log_p_pbwd_bd} contains $|\mcE^{(\ell)}|$ instead of $e^{(\ell)}=|\mcE^{(\ell)}_{\mcH}|$.
It moreover establishes that the constants $\mf{n}_7,\mf{c}_{11}$ and $\mf{c}_{12}$ are independent of $\ell$.

Using \eqref{eqn:log_p_pbwd_bd} along with the fact that $\cup_{s}\hat{f}_{\ell}^{-1}(s)=[n]$ as well as the bound from $\Omega_7$ in \eqref{eqn:definition_of_events_E_H}, we have that for $n>\mf{n}_7$ and $\ell\geq 0$, the following holds on $\Omega_5^{(\ell)}\cap\Omega_6\cap \Omega_7$:
\begin{align}
	 &
	|U^{(\ell)}|\leq |U^{(\ell)}_{\mathrm{in}}|+|U^{(\ell)}_{\mathrm{out}}|
	\nonumber
	\\
	 &
	\overset{\eqref{eqn:log_p_pbwd_bd}}{\leq}
	\sum_{x\in\mcE_{\mcH}^{(\ell+1)}}
	\sum_{s,a}
	\lb \hat{N}_{k^*+1}(x,a,\hat{f}_{\ell}^{-1}(s))+\hat{N}_{k^*+1}(\hat{f}_{\ell}^{-1}(s),a,x)\rb
	\lb
	\mf{c}_{11}S\frac{|\mcE^{(\ell)}|}{n}
	+
	\mf{c}_{12}S\sqrt{\frac{nA}{T_{k^*}}}
	\rb
	\nonumber
	\\
	 &
	\leq
	2 |\mcE_{\mcH}^{(\ell+1)}|\max_{x\in\mcE_{\mcH}^{(\ell+1)}}
	\left\{
	\sum_{a}\hat{N}_{k^*+1}(x,a,[n]),\sum_{a}\hat{N}_{k^*+1}([n],a,x)
	\right\}
	\lb
	\mf{c}_{11}S\frac{|\mcE^{(\ell)}|}{n}
	+
	\mf{c}_{12}S\sqrt{\frac{nA}{T_{k^*}}}
	\rb
	\nonumber
	\\
	 &
	\leq
	\mf{c}_{10} |\mcE_{\mcH}^{(\ell+1)}|SA
	\lb
	\frac{|\mcE^{(\ell)}|}{n}
	\frac{T_{k^*}}{nA}
	+
	\sqrt{\frac{T_{k^*}}{nA}}
	\rb
	.
\end{align}
Here, $\mf{c}_{10}=2c_3(\mf{c}_{11}\vee \mf{c}_{12})$.
\Cref{lem:bound_on_one_step_improvement_error_rate} finally follows for $\mf{n}_1 \geq \mf{n}_7$.
\qed

\begin{proof}
	[Proof of \Cref{lem:logtransprobratiosbd}]

	Note that \Cref{def:reachability:f} implies that there exists $\mf{n}_8>0$ such that for $n>\mf{n}_8$, $|f^{-1}(s)|/\allowbreak |f^{-1}(s')|\in[1/\eta,\eta]$ for $s,s'\in[S]$.
	Since $\sum_{s}|f^{-1}(s)|=n$, it follows that for $n>\mf{n}_8$ and $s\in[S]$, $S\eta |f^{-1}(s)| \geq \sum_{s'}|f^{-1}(s')|=n$ and $|f^{-1}(s)|/n = |f^{-1}(s)|/\sum_{s'}|f^{-1}(s')|\leq \eta/S$.
	Similar arguments leveraging \Cref{def:reachability:p,def:reachability:q} along with $\sum_{s'}p(s'\mid s,a)=1$ and $\sum_{y\in f^{-1}(s)}q(y\mid s)=1$ can be used to show that there exist constants $0<\mf{d}_1\leq \mf{d}_2$ such that for $n> \mf{n}_8$ and $s,s'\in[S]$, $a\in[A]$ and $x\in f^{-1}(s)$,
	\begin{equation}
		\label{eqn:bounds_implied_by_reachability}
		\mf{d}_1\leq \frac{S|f^{-1}(s)|}{n}\leq\mf{d}_2
		,
		\quad
		\mf{d}_1\leq Sp(s'\mid s,a)\leq\mf{d}_2
		,
		\quad
		\mathrm{and}
		\quad
		\mf{d}_1\leq \frac{nq(x\mid s)}{S}\leq\mf{d}_2
		.
	\end{equation}

	Next, from \eqref{eqn:countsdef}, the definition of a \gls{BMDP} in \Cref{sec:Block-MDPs},
	and \eqref{def:BMDP-transition-kernel}, we have that
	\begin{align}
		 &
		N_{k^*+1}(x,a,y)
		\eqdef
		\mbbE[\hat{N}_{k^*+1}(x,a,y)]
		=
		\sum_{l=1}^{k^*}
		\sum_{h=1}^H
		\mbbP[x_{l,h}=x,a_{l,h}=a,x_{l,h+1}=y]
		\nonumber
		\\
		 &
		=
		\sum_{l=1}^{k^*}
		\sum_{h=1}^H
		\mbbP[f(x_{l,h}) = f(x), a_{l,h}=a]
		\mbbP[x_{l,h} = x \mid f(x_{l,h}) = f(x)]
		\mbbP[x_{l,h+1}=y \mid x_{l,h}=x, a_{l,h}=a]
		\nonumber
		\\
		 &
		=
		\sum_{l=1}^{k^*}
		\sum_{h=1}^H
		\mbbP[f(x_{l,h}) = f(x), a_{l,h}=a]
		q(x\mid f(x))
		p(f(y)\mid f(x),a)
		q(y\mid f(y))
		.
		\label{eqn:expansion_of_expected_visitation_count}
	\end{align}
	Note also that $\pi_k = \pi_U$ for $k\leq k^*$ by \eqref{eqn:k-star_def} and \Cref{alg:BUCBVIouter}, where $\pi_{U,h}(a\mid x)=1/A$ for all $x$, $a$ and $h$.
	Together with the episode structure described in \Cref{sec:preliminaries-episodes}, this implies that $\mbbP[f(x_{l,h}) = f(x), a_{l,h}=a]$ is independent of $l\in[k^*]$.
	Consequently, since $T_{k^*}=k^*H$,
	\begin{equation}
		\label{eqn:expansion_of_expected_visitation_count_simplified}
		N_{k^*+1}(x,a,y)
		=
		T_{k^*}
		q(x\mid f(x))
		p(f(y)\mid f(x),a)
		q(y\mid f(y))
		\omega_{\Phi,\pi_U}(f(x),a)
		,
	\end{equation}
	where $\omega_{\Phi,\pi_U}(s,a)$ was defined in \eqref{eqn:definition_of_visitation_probabilities}.
	Recall from \Cref{sec:proof_of_information_quantity_bounds} that the latter can be identified with $m_{\rho}(s,a)$ defined in the opening paragraph of \cite[Appendix D.3.3]{Jedra:2022}.
	This, according to \cite[Proposition 14]{Jedra:2022}, satisfies $m_{\rho}(s,a)\in [1/(SA\eta^4),\eta^4/(SA)]$.
	Together with \eqref{eqn:bounds_implied_by_reachability} and \eqref{eqn:expansion_of_expected_visitation_count} we conclude that there exist constants $0<\mf{d}_3\leq \mf{d}_4$ such that for $n> \mf{n}_8$, $s\in[S]$, $a\in[A]$, and $\mcV_1,\mcV_2\subset[n]$,
	\begin{equation}
		\label{eqn:cluster_size_bound}
		\mf{d}_3\leq \frac{S|f^{-1}(s)|}{n}\leq \mf{d}_4
		\quad\textnormal{and}\quad
		\mf{d}_3
		\leq
		\frac{n^2AN_{k^*+1}(\mcV_1,a,\mcV_2)}{T_{k^*}|\mcV_1||\mcV_2|}
		\leq
		\mf{d}_4
		.
	\end{equation}
	Note that we may assume without loss of generality that $\mf{d}_3<1<\mf{d}_4$.

	Given $\mf{d}_3,\mf{d}_4$, let
	\begin{equation}
		\label{eqn:c11_c12_def}
		\mf{c}_{11} \eqdef 32(\mf{d}_4/\mf{d}_3^3)
		\quad\mathrm{and}\quad
		\mf{c}_{12} \eqdef 32(\mf{c}_4/\mf{d}_3^2).
	\end{equation}
	It follows from \eqref{eqn:definition_of_events_E_H} and the fact that $T_{k^*}=\omega(nS^3A)$ that there exists a constant $\mf{n}_9>0$ such that for $n> \mf{n}_9$ and $\ell\geq 0$, on $\Omega_5^{(\ell)}$,
	\begin{equation}
		\label{eqn:mcL_mcR_bds}
		\mcL^{(\ell)}
		\eqdef
		\mf{c}_{11}S\frac{|\mcE^{(\ell)}|}{n}
		\leq
		8
		\quad \mathrm{and}\quad
		\mcR
		\eqdef
		\mf{c}_{12}S\sqrt{\frac{nA}{T_{k^*}}}
		\leq
		8
		.
	\end{equation}
	It now suffices to show that \eqref{eqn:cluster_size_bound} and \eqref{eqn:mcL_mcR_bds} implies \eqref{eqn:log_p_pbwd_bd} on $\Omega_6$.
	This would then imply \Cref{lem:logtransprobratiosbd} for $\mf{n}_7=\mf{n}_8\vee \mf{n}_9$.

	We first use \eqref{eqn:mcL_mcR_bds} to derive a bound for $|\hat{f}_{\ell}^{-1}(s)|$.
	Since $\hat{f}^{({\ell})}$ misclassifies exactly $|\mcE^{(\ell)}|$ contexts it follows that for $s\in[S]$ and $\ell\geq 0$,
	\begin{equation}
		\bigl||\hat{f}_{\ell}^{-1}(s)|-|f^{-1}(s)|\bigr|\leq |\mcE^{(\ell)}|
		.
	\end{equation}
	Since $\mf{d}_3<1<\mf{d}_3$, \eqref{eqn:c11_c12_def} and \eqref{eqn:mcL_mcR_bds} imply that $S|\mcE^{(\ell)}|/n\leq 8/\mf{c}_{11}\leq (\mf{d}_3/2) \wedge \mf{d}_4$, where the last inequality is loose.
	Together with \eqref{eqn:cluster_size_bound} it follows that
	\begin{equation}
		\label{eqn:bound_est_cluster_size}
		(1/2)\mf{d}_3n/S \leq |f^{-1}(s)|-|\mcE^{(\ell)}|\leq |\hat{f}_{\ell}^{-1}(s)|\leq |f^{-1}(s)|+|\mcE^{(\ell)}| \leq 2\mf{d}_4 n/S.
	\end{equation}

	Next, fix $s,s'\in[S]$, $a\in[A]$ and $\ell \geq 0$.
	To prove \eqref{eqn:log_p_pbwd_bd} we first bound $|\ln(\hat{p}_{\ell}(s'\mid s,a)/\allowbreak p(s'\mid s,a))|$.
	The proof of \cite[Lemma 25]{Jedra:2022} implies that
	\begin{align}
		\biggl|\ln \frac{\hat{p}_{\ell}(s'\mid s,a)}{p(s'\mid s,a)}\biggr|
			&
		\leq
		\biggl|
		\underbrace{
		\frac{N_{k^*+1}(\hat{f}_{\ell}^{-1}(s),a,\hat{f}_{\ell}^{-1}(s'))}{N_{k^*+1}(f^{-1}(s),a,f^{-1}(s'))}
		}_{\defeq L_1}
		\underbrace{
			\frac{N_{k^*+1}(f^{-1}(s),a,[n])}{N_{k^*+1}(\hat{f}_{\ell}^{-1}(s),a,[n])}
		}_{\defeq L_2}
		\nonumber
		\\
			&
		\phantom{\leq}
		\times
		\underbrace{
		\frac{\hat{N}_{k^*+1}(\hat{f}_{\ell}^{-1}(s),a,\hat{f}_{\ell}^{-1}(s'))}{N_{k^*+1}(\hat{f}_{\ell}^{-1}(s),a,\hat{f}_{\ell}^{-1}(s'))}
		}_{\defeq R_1}
		\underbrace{
		\frac{N_{k^*+1}(\hat{f}_{\ell}^{-1}(s),a,[n])}{\hat{N}_{k^*+1}(\hat{f}_{\ell}^{-1}(s),a,[n])}
		}_{\defeq R_2}
		-
		1
		\biggr|
		.
		\label{eqn:log_ratio_bound}
	\end{align}
	We next separately bound the deviations of $L_1$, $L_2$, $R_1$ and $R_2$ from 1.

	\paragraph{Bounding $L_1$.}

	The following intermediate result appears in the proof of \cite[Lemma 25]{Jedra:2022} upon making the identifications in \eqref{eqn:identifications_with_jedra} as well as the identifications $V\equiv f^{-1}(s)$, $\hat{V}\equiv \hat{f}_{\ell}^{-1}(s)$, $W\equiv f^{-1}(s')$ and $\hat{W}\equiv \hat{f}_{\ell}^{-1}(s')$:
	\begin{equation}
		\label{eqn:Ndiffestclust}
		\begin{split}
			 &
			\left|N_{k^*+1}(\hat{f}_{\ell}^{-1}(s),a,\hat{f}_{\ell}^{-1}(s'))-N_{k^*+1}(f^{-1}(s),a,f^{-1}(s'))\right| \\
			 &
			\qquad\leq N_{k^*+1}(f^{-1}(s),a,\mcE^{(\ell)})+N_{k^*+1}(\mcE^{(\ell)},a,f^{-1}(s'))+N_{k^*+1}(\mcE^{(\ell)},a,\hat{f}_{\ell}^{-1}(s')).
		\end{split}
	\end{equation}
	Using \eqref{eqn:cluster_size_bound}, followed by \eqref{eqn:bound_est_cluster_size} and \eqref{eqn:cluster_size_bound} conclude that
	\begin{equation}
		\label{eqn:Ndiffestclust_simplebd}
		\begin{split}
			\textnormal{RHS of }\eqref{eqn:Ndiffestclust}
			 &
			\leq
			\mf{d}_4\frac{T_{k^*}}{n^2A}|\mcE^{(\ell)}|
			\lb |f^{-1}(s)|+|f^{-1}(s')|+|\hat{f}^{-1}_{\ell}(s')| \rb
			\leq
			4\mf{d}_4\frac{T_{k^*}}{SA}\frac{|\mcE^{(\ell)}|}{n}.
		\end{split}
	\end{equation}
	It also follows from \eqref{eqn:cluster_size_bound} that
	\begin{equation}
		\label{eqn:Nfsfslowerbd}
		N_{k^*+1}(f^{-1}(s),a,f^{-1}(s'))
		\geq
		\mf{d}_3
		\frac{T_{k^*}}{n^2A}
		|f^{-1}(s)||f^{-1}(s')|
		\geq
		\mf{d}_3^3\frac{T_{k^*}}{S^2A}.
	\end{equation}
	We conclude from \eqref{eqn:Ndiffestclust_simplebd} and \eqref{eqn:Nfsfslowerbd} that
	\begin{equation}
		\begin{split}
			\left|L_1-1\right|
			=
			\frac{|N_{k^*+1}(\hat{f}_{\ell}^{-1}(s),a,\hat{f}_{\ell}^{-1}(s'))-N_{k^*+1}(f^{-1}(s),a,f^{-1}(s'))|}{N_{k^*+1}(f^{-1}(s),a,f^{-1}(s'))}
			\leq
			4(\mf{d}_4/\mf{d}_3^3) S\frac{|\mcE^{(\ell)}|}{n}
			=
			\frac{1}{8}
			\mcL^{(\ell)}
			.
		\end{split}
	\end{equation}
	\paragraph{Bounding $L_2$.}
	Next, let $\Delta$ denote the symmetric difference operator.
	It follows from \eqref{eqn:cluster_size_bound} that all $s$ and $a$,
	\begin{equation}
		\label{eqn:Ndiffestclusttrueclust_simplebd}
		\begin{split}
			|N_{k^*+1}(\hat{f}_{\ell}^{-1}(s),a,[n])-N_{k^*+1}(f^{-1}(s),a,[n])|
			 &
			=
			N_{k^*+1}(\hat{f}_{\ell}^{-1}(s)\Delta f^{-1}(s),a,[n])
			\\
			 &
			\leq
			N_{k^*+1}(\mcE^{(\ell)},a,[n])
			\leq
			\mf{d}_4\frac{T_{k^*}}{A}\frac{|\mcE^{(\ell)}|}{n}.
		\end{split}
	\end{equation}
	Moreover, it follows from \eqref{eqn:cluster_size_bound} and \eqref{eqn:bound_est_cluster_size} that
	\begin{equation}
		N_{k^*+1}(\hat{f}_{\ell}^{-1}(s),a,[n])
		\geq
		\mf{d}_3|\hat{f}_{\ell}^{-1}(s)|\frac{T_{k^*}}{nA}
		\geq
		\mf{d}_3^2 \frac{T_{k^*}}{2SA}
		.
	\end{equation}
	Together with \eqref{eqn:Ndiffestclusttrueclust_simplebd} and $\mf{d}_3<1$ we conclude that
	\begin{equation}
		\begin{split}
			\left|L_2-1\right|
			 &
			=
			\frac{|N_{k^*+1}(f^{-1}(s),a,[n])-N_{k^*+1}(\hat{f}_{\ell}^{-1}(s),a,[n])|}{N_{k^*+1}(\hat{f}_{\ell}^{-1}(s),a,[n])}
			\leq
			2(\mf{d}_4/\mf{d}_3^2)S\frac{|\mcE^{(\ell)}|}{n}
			\leq
			\frac{1}{8}
			\mcL^{(\ell)}
			.
		\end{split}
	\end{equation}
	\paragraph{Bounding $R_1$.}
	From \eqref{eqn:definition_of_events_E_H}, on $\Omega_6$,
	\begin{equation}
		\label{eqn:jumpcountconcentration_estimated_clusters}
		|
		\hat{N}_{k^*+1}(\hat{f}_{\ell}^{-1}(s),a,\hat{f}_{\ell}^{-1}(s'))
		-
		N_{k^*+1}(\hat{f}_{\ell}^{-1}(s),a,\hat{f}_{\ell}^{-1}(s'))
		|
		\leq
		\mf{c}_4
		\sqrt{\frac{nT_{k^*}}{A}}.
	\end{equation}
	Together with \eqref{eqn:cluster_size_bound} and \eqref{eqn:bound_est_cluster_size}, it follows from \eqref{eqn:jumpcountconcentration_estimated_clusters} that for sufficiently large $\mf{c}_{12}$,
	\begin{equation}
		\begin{split}
			\left|R_1-1\right|
			 &
			=
			\frac{|\hat{N}_{k^*+1}(\hat{f}_{\ell}^{-1}(s),a,\hat{f}_{\ell}^{-1}(s'))- N_{k^*+1}(\hat{f}_{\ell}^{-1}(s),a,\hat{f}_{\ell}^{-1}(s'))|}{N_{k^*+1}(\hat{f}_{\ell}^{-1}(s),a,\hat{f}_{\ell}^{-1}(s'))}
			\\
			 &
			\leq
			\mf{c}_4\sqrt{\frac{nT_{k^*}}{A}}
			\frac{1}{N_{k^*+1}(\hat{f}_{\ell}^{-1}(s),a,\hat{f}_{\ell}^{-1}(s'))}
			\\
			 &
			\leq
			(\mf{c}_4/\mf{d}_3)\sqrt{\frac{nA}{T_{k^*}}}\frac{n^2}{|\hat{f}_{\ell}^{-1}(s)||\hat{f}_{\ell}^{-1}(s')|}
			\leq
			4(\mf{c}_4/\mf{d}_3^3)S^2\sqrt{\frac{nA}{T_{k^*}}}
			=
			\frac{1}{8}
			\mcR
			.
		\end{split}
	\end{equation}
	\paragraph{Bounding $R_2$.} Observe that
	\begin{equation}
		\begin{split}
			|R_2-1|
			\leq
			\frac{|N_{k^*+1}(\hat{f}_{\ell}^{-1}(s),a,[n])-\hat{N}_{k^*+1}(\hat{f}_{\ell}^{-1}(s),a,[n])|}{N_{k^*+1}(\hat{f}_{\ell}^{-1}(s),a,[n])-|N_{k^*+1}(\hat{f}_{\ell}^{-1}(s),a,[n])-\hat{N}_{k^*+1}(\hat{f}_{\ell}^{-1}(s),a,[n])|}
			.
		\end{split}
	\end{equation}
	Recall from the bound for $L_2$ that $N_{k^*+1}(\hat{f}_{\ell}^{-1}(s),a,[n])\geq \mf{d}_3^2 T_{k^*}/(2SA)$.
	Together with \eqref{eqn:definition_of_events_E_H} and \eqref{eqn:mcL_mcR_bds} it follows that on $\Omega_6$,
	\begin{equation}
		\begin{split}
			|R_2-1|
			 &
			\leq
			\frac{\mf{c}_4\sqrt{\frac{nT_{k^*}}{A}}}{N_{k^*+1}(\hat{f}_{\ell}^{-1}(s),a,[n])-\mf{c}_4\sqrt{\frac{nT_{k^*}}{A}}}
			\leq
			\frac{\mf{c}_4\sqrt{\frac{nT_{k^*}}{A}}}{\mf{d}_3^2\frac{T_{k^*}}{2SA}-\mf{c}_4\sqrt{\frac{nT_{k^*}}{A}}}
			=
			\frac{\mcR/16}{1-\mcR/16}
			\leq
			\frac{1}{8}
			\mcR
			.
		\end{split}
	\end{equation}

	\paragraph{Combining the bounds.}
	We have shown that $L_1,L_2\in[1-\mcL^{(\ell)}/8,1+\mcL^{(\ell)}/8]$ and $R_1,R_2\in[1-\mcR/8,1+\mcR/8]$.
	Along with \eqref{eqn:mcL_mcR_bds}, it follows that
	\begin{align}
		|L_1L_2R_1R_2-1|
			&
		\leq
		\max
		\biggl\{\Bigl|\Bigl(1-\frac{\mcL^{(\ell)}}{8}\Bigr)^2\Bigl(1-\frac{\mcR}{8}\Bigr)^2-1\Bigr|,\Bigl|\Bigl(1+\frac{\mcL^{(\ell)}}{8}\Bigr)^2\Bigl(1+\frac{\mcR}{8}\Bigr)^2-1\Bigr|\biggr\}
		\nonumber
		\\
			&
		\leq
		\mcL^{(\ell)}+\mcR
		.
	\end{align}
	By \eqref{eqn:log_ratio_bound}, this implies
	\begin{equation}
		\biggl|\ln\frac{\hat{p}_{\ell}(s'\mid s,a)}{p(s'\mid s,a)}\biggr|
		\leq
		\mcL^{(\ell)}+\mcR
		.
	\end{equation}

	It remains to show that $|\ln(\hat{p}_{\ell}^{\mathrm{bwd}}(s,a\mid s')/p^{\mathrm{bwd}}(s,a\mid s'))|\leq \mcL^{(\ell)}+\mcR$.

	Let $N_{k^*+1}^{\mathrm{to}}(\mcV)\eqdef \mbbE[\hat{N}_{k^*+1}^{\mathrm{to}}(\mcV)]$ for $\mcV\subset[n]$, where $\hat{N}_{k^*+1}^{\mathrm{to}}$ is defined in \eqref{eqn:countsdef_overload}.
	From the proof of \cite[Lemma 25]{Jedra:2022}, it follows that
	\begin{align}
		 &
		\left|\ln\frac{\hat{p}_{\ell}^{\mathrm{bwd}}(s,a\mid s')}{p^{\mathrm{bwd}}(s,a\mid s')}\right|
		\\
		 &
		\leq
		\biggl|
		\underbrace{\frac{N_{k^*+1}(\hat{f}_{\ell}^{-1}(s),a,\hat{f}_{\ell}^{-1}(s'))}{N_{k^*+1}(f^{-1}(s),a,f^{-1}(s'))}}_{=L_1}
		\underbrace{\frac{N_{k^*+1}^{\mathrm{to}}(f^{-1}(s'))}{N_{k^*+1}^{\mathrm{to}}(\hat{f}_{\ell}^{-1}(s'))}}_{\defeq L_3}
		\underbrace{\frac{\hat{N}_{k^*+1}(\hat{f}_{\ell}^{-1}(s),a,\hat{f}_{\ell}^{-1}(s'))}{N_{k^*+1}(\hat{f}_{\ell}^{-1}(s),a,\hat{f}_{\ell}^{-1}(s'))}}_{=R_1}
		\underbrace{\frac{N_{k^*+1}^{\mathrm{to}}(\hat{f}_{\ell}^{-1}(s))}{\hat{N}_{k^*+1}^{\mathrm{to}}(\hat{f}_{\ell}^{-1}(s))}}_{\defeq R_3}-1
		\biggr|.
		\nonumber
	\end{align}
	One can use arguments analogous to those used for $L_2$ and $R_2$ to bound the ratios $L_3$ and $R_3$, respectively.
	The resulting bounds for $|L_3-1|$ and $|R_3-1|$ then admit the loose upper bounds $\mcL^{(\ell)}/8$ and $\mcR/8$, respectively.
	Consequently, the same bound holds for $|\ln(\hat{p}_{\ell}^{\mathrm{bwd}}(s,a\mid s')/\allowbreak p^{\mathrm{bwd}}(s,a\mid s'))|$ as for $|\ln(\hat{p}_{\ell}(s'\mid s,a)\allowbreak /\allowbreak p(s'\mid s,a))|$.
	The conclusion of \cref{lem:logtransprobratiosbd} follows from this.
\end{proof}

\subsection{Proof of \texorpdfstring{\cref{prop:highprobevent}}{the probability bound for the good event}}
\label{sec:proofstep1}

In \Cref{sec:proof_of_transition_probability_event_bd}, we prove the following proposition:
\begin{proposition}
	\label{prop:transition_probability_event_bd}
	For $i\in\{1,2,3,4\}$ and $k\in\mbbN_+$, $\mbbP[(\Omega_k^i)^{\mathrm{c}}\cap \Omega^f] \leq \frac{1}{T_k}$.
\end{proposition}
Given \Cref{prop:transition_probability_event_bd}, the proof of \Cref{prop:highprobevent} is almost immediate.
In particular, using $(i)$ De Morgan's law, $(ii)$ a union bound, and $(iii)$ \Cref{prop:transition_probability_event_bd}, it follows that
\begin{equation}
	\label{eqn:probdecomp}
	\begin{split}
		 &
		\mbbP\bkt{\Omega_k}
		=
		\mbbP[
			( \Omega_k^{1}\cap\Omega_k^{2}\cap\Omega_k^3\cap\Omega_k^4)
			\cap\Omega^f
		]
		\\
		&
		=
		\mbbP[\Omega^f]
		-
		\mbbP[
			(\Omega_k^{1}
			\cap
			\Omega_k^{2}
			\cap
			\Omega_k^3
			\cap
			\Omega_k^4)^{\mathrm{c}}
			\cap\Omega^f
		]
		\\
		 &
		\overset{(i)}{=}
		\mbbP[\Omega^f]
		-
		\mbbP[
		(
		(\Omega_k^{1})^{\mathrm{c}}
		\cup
		(\Omega_k^{2})^{\mathrm{c}}
		\cup
		(\Omega_k^3)^{\mathrm{c}}
		\cup
		(\Omega_k^4)^{\mathrm{c}}
		)
		\cap
		\Omega^f
		]
		\\
		 &
		\overset{(ii)}{\geq}
		\mbbP[\Omega^f]
		-
		\mbbP[
			( \Omega_k^{1})^{\mathrm{c}} \cap \Omega^f
		]
		-
		\mbbP[(\Omega_k^{2})^{\mathrm{c}}\cap \Omega^f]
		-
		\mbbP[( \Omega_k^{3})^{\mathrm{c}}\cap \Omega^f]
		-
		\mbbP[( \Omega_k^{4})^{\mathrm{c}}\cap \Omega^f]
		\\
		 &
		\overset{(iii)}{\geq}
		\mbbP[\Omega^f]
		-
		\frac{4}{T_k}
		=
		1
		-
		\mbbP[(\Omega^f)^{\mathrm{c}}]
		-
		\frac{4}{T_k}
		.
	\end{split}
\end{equation}
Finally, using that $(iv)$ $T_k = kH$ and $(v)$ $\sum_{k=1}^Kk^{-1}\leq \ln K +1$, as well as \Cref{prop:clustering_bound_short}, we find after summing over all episodes in phase two of \Cref{alg:BUCBVIouter} that for $c>0$,
\begin{equation}
	\begin{split}
		H\sum_{k\in[K]:T_k>\Theta^{\clust}}\mbbP\bkt{\Omega_k^c}
		 &
		\overset{\eqref{eqn:probdecomp}}{\leq}
		H\sum_{k\in[K]:T_k>\Theta^{\clust}}\biggl(
		\mbbP[(\Omega^f)^{\mathrm{c}}]
		+
		\frac{4}{T_k}
		\biggr)
		\\
		 &
		\overset{(iv)}{\leq}
		T_K\mbbP[(\Omega^f)^{\mathrm{c}}]
		+
		\sum_{k=1}^K
		\frac{4}{k}
		\overset{(v)}{=}
		O\biggl(\frac{T_K}{n^c} + \ln K \biggr)
		.
	\end{split}
\end{equation}
This concludes the proof of \Cref{prop:highprobevent}.
\qed

\subsubsection{Proof of \texorpdfstring{\Cref{prop:transition_probability_event_bd}}{proposition}}
\label{sec:proof_of_transition_probability_event_bd}

In the following, let $\hat{p}_k^f:[S]\times [A]\times[S]\rightarrow [0,1]$ and $\hat{q}_k^f:[S]\times[n]\rightarrow [0,1]$ denote the maximum likelihood estimators for the latent state transition kernel and emission probabilities when the true decoding function is known, i.e., for $s,s'\in[S]$, $a\in[A]$ and $x\in[n]$
\begin{equation}
	\label{eqn:pfqfdef}
	\hat{p}_k^f(s'\mid s,a)
	\eqdef
	\frac{\hat{N}_k(f^{-1}(s),a,f^{-1}(s'))}{1\vee \hat{N}_k(f^{-1}(s),a)}
	,
	\!\quad \mathrm{and},\! \quad
	\hat{q}_k^f(x\mid s)
	\eqdef
	\frac{\mbb{1}\{x\in f^{-1}(s)\}\hat{N}^{\mathrm{to}}_k(x)}{1\vee \hat{N}^{\mathrm{to}}_k(f^{-1}(s))}
	.
\end{equation}
Note from \eqref{eqn:estimators_p_q_def} that on $\Omega^f=\{\hat{f}^{\alg}=f\}$, $\hat{p}_k^f=\hat{p}_k$ and $\hat{q}_k^f=\hat{q}_k$.
Below, we will prove that $\hat{p}_k^f$ and $\hat{q}_k^f$ satisfy the following concentration inequalities:
\begin{proposition}
	\label{prop:concentration_inequalities}
	Let $U\in\mbbR^n_{\geq0}$ and $u\in\mbbR^S_{\geq0}$ be two vectors with nonnegative entries, and let $\delta\in(0,1/2]$.
	Then the following inequalities hold for $s,s'\in[S]$, $a\in[A]$, $y\in[n]$:
	\begin{align}
		 &
		\mbbP\biggl[
		|\langle \hat{q}_k^f(\,\cdot \mid s)-q(\,\cdot \mid s), U\rangle|
		>
		\sqrt{\frac{\lVert U\rVert_{\infty}^2 \ln(2 T_k/\delta)}{1\vee \hat{N}^{\mathrm{to}}_k(f^{-1}(s))}}
		\biggr]
		\leq
		\delta
		\label{eqn:hoeffding_q}
		\\
		 &
		\mbbP\biggl[
			|\langle \hat{p}_k^f(\,\cdot \mid s,a)-p(\,\cdot \mid s,a), u\rangle|
			>
			\sqrt{\frac{\lVert u\rVert_{\infty}^2\ln(2 T_k/\delta)}{1\vee \hat{N}_k(f^{-1}(s),a)}}
			\biggr]
		\leq
		\delta
		,
		\label{eqn:hoeffding_p}
		\\
		 &
		\mbbP\biggl[
		|\hat{q}_k^f(y\mid s)-q(y\mid s)|
		>
		\sqrt{\frac{q(y\mid s) \ln(2T_k/\delta)}{1\vee \hat{N}^{\mathrm{to}}_k(f^{-1}(s))}}
		+
		\frac{2\ln(2T_k/\delta)}{3(1\vee\hat{N}^{\mathrm{to}}_k(f^{-1}(s)))}
		\biggr]
		\leq
		\delta
		,
		\label{eqn:bernstein_q}
		\\
		 &
		\mbbP\biggl[
			| \hat{p}_k^f(s'\mid s,a)-p(s'\mid s,a)|
			>
			\sqrt{\frac{p(s'\mid s,a)\ln(2T_k/\delta)}{1\vee \hat{N}_k(f^{-1}(s),a)}}
			+
			\frac{2\ln(2T_k/\delta)}{3(1\vee\hat{N}_k(f^{-1}(s),a))}
			\biggr]
		\leq
		\delta
		.
		\label{eqn:bernstein_p}
	\end{align}
\end{proposition}
We now show how \Cref{prop:transition_probability_event_bd} follows from \Cref{prop:concentration_inequalities}, starting with the case $i=1$.
If $H=1$ the event $\Omega_k^1$ is trivial because $V_{H+1}^{\pi^*}=0$ by convention; recall \eqref{eqn:valuefuncdef}.
Assume therefore that $H>1$.
Note that $V_{h+1}^{\pi^*}\in\mbbR^n$ has nonnegative entries for all $h$ and that $\delta=1/(HST_k)$ satisfies $\delta\leq 1/2$ since $H>1$.
We may therefore apply \Cref{prop:concentration_inequalities} and in particular \eqref{eqn:hoeffding_q} with $U=V_{h+1}^{\pi^*}$ and $\delta=1/(HST_k)$.
Then, using $(i)$ $\hat{f}^{\alg} = f$ and $\hat{q}_k=\hat{q}_k^f$ on $\Omega^f$, $(ii)$ $\lVert V_{h+1}^{\pi^*}\rVert_{\infty}\leq H$ for any $h\in[H]$, and $(iii)$ applying \eqref{eqn:hoeffding_q} with $U=V_{h+1}^{\pi^*}$ and $\delta = 1/(HST_k)$ it follows that
\begin{align}
	 &
	\mathbb{P}[(\Omega_k^1)^{\mathrm{c}}\cap\Omega^f]
	\overset{(i)}{=}
	\mathbb{P}\Bigl[
	\Bigl\{
	\exists s,h: |\langle \hat{q}^f_k(\,\cdot \mid s)-q(\,\cdot \mid s), V^{\pi^*}_{h+1}\rangle|>\sqrt{\frac{H^2\ln(2SHT_k^2)}{1\vee\hat{N}^{\mathrm{to}}_k(f^{-1}(s))}}
	\Bigr\}
	\cap \Omega^f
	\Bigr]
	\nonumber
	\\
	 &
	\overset{(ii)}{\leq}
	\mathbb{P}\Bigl[
	\exists s,h: |\langle \hat{q}^f_k(\,\cdot \mid s)-q(\,\cdot \mid s), V^{\pi^*}_{h+1}\rangle|>\sqrt{\frac{\norm{V^{\pi^*}_{h+1}}_{\infty}^2\ln(2SHT_k^2)}{1\vee\hat{N}^{\mathrm{to}}_k(f^{-1}(s))}}
	\Bigr]
	\\
	 &
	\leq
	\sum_{s}
	\sum_{h=1}^H
	\mathbb{P}\Bigl[
	|\langle \hat{q}^f_k(\,\cdot \mid s)-q(\,\cdot \mid s), V^{\pi^*}_{h+1}\rangle|>\sqrt{\frac{\norm{V^{\pi^*}_{h+1}}_{\infty}^2\ln(2SHT_k^2)}{1\vee\hat{N}^{\mathrm{to}}_k(f^{-1}(s))}}
	\Bigr]
	\nonumber
	\\
	 &
	\overset{(iii)}{\leq}
	\sum_{s}
	\sum_{h=1}^H
	\frac{1}{SHT_k}
	=
	\frac{1}{T_k}
	.
	\nonumber
\end{align}
The case $i=2$ follows analogously from \eqref{eqn:hoeffding_p} with $u(s)=\langle q(\,\cdot \mid s), V_{h+1}^{\pi^*}\rangle$ and $\delta=1/(SAT_k)$.
Lastly, the cases $i=3,4$ follow from \eqref{eqn:bernstein_q} and \eqref{eqn:bernstein_p} with $\delta=1/(nS)$ and $\delta=1/(S^2A)$, respectively.
It therefore only remains to prove \Cref{prop:concentration_inequalities}.

\begin{proof}[Proof of \Cref{prop:concentration_inequalities}]
	We first prove \eqref{eqn:hoeffding_q}.
	Fix $s$ and $U\in\mbbR^n_{\geq0}$, and let $\delta\in(0,1/2]$.

	Let $\mcT'_k(s) \eqdef \{(l,h)\in [k-1]\times [H]: x_{l,h+1}\in f^{-1}(s)\}$ denote the (random) set of times for which $f(x_{l,h+1})=s$.
	Note that $(i)$ $|\mcT'_k(s)|=\hat{N}_k^{\mathrm{to}}(f^{-1}(s))$.
	Conditional on $\{\hat{N}_k^{\mathrm{to}}(f^{-1}(s))=N\}$ for $N>0$ it follows from \eqref{eqn:pfqfdef} and \eqref{eqn:countsdef_overload} that
	\begin{align}
		 &
		N\langle \hat{q}^f_k(\,\cdot \mid s)-q(\,\cdot \mid s), U\rangle
		=
		\sum_{y\in [n]}
		\biggl(
		\mathbbm{1}\{y\in f^{-1}(s)\}\sum_{l\in[k-1]}\sum_{h=1}^H\mathbbm{1}\{x_{l,h+1}=y\}-Nq(y\mid s)
		\biggr)
		U(y)
		\nonumber
		\\
		 &
		=
		\sum_{y\in [n]}
		\biggl(\sum_{l\in[k-1]}\sum_{h=1}^H
		\mathbbm{1}\{x_{l,h+1}\in f^{-1}(s)\}\mathbbm{1}\{x_{l,h+1}=y\}-Nq(y\mid s)
		\biggr)
		U(y)
		\nonumber
		\\
		 &
		\overset{(i)}{=}
		\sum_{l\in[k-1]}
		\sum_{h=1}^H
		\mathbbm{1}\{x_{l,h+1} \in f^{-1}(s)\}\sum_{y\in [n]}(
		\mathbbm{1}\{x_{l,h+1}=y\}-q(y\mid s)
		)
		U(y)
		\nonumber
		\\
		 &
		=
		\sum_{(l,h)\in \mcT'_k(s)}
		\bigl(U(x_{l,h+1})-\sum_{y\in f^{-1}(s)}q(y\mid s) U(y)\bigr),
		\label{eqn:q_esterror_distributed_as}
	\end{align}
	where in passing to the final line we used that $q(y\mid s)=0$ unless $y\in f^{-1}(s)$.

	By definition of $\mcT'_k(s)$ and the Markov property, each $x_{l,h+1}$ for $(l,h)\in \mcT'_k(s)$ is independent and identically distributed according to $q(\,\cdot \mid s)$.
	It then follows from \eqref{eqn:q_esterror_distributed_as} that conditional on $\{\hat{N}^{\mathrm{to}}_k(f^{-1}(s))=N\}$ for $N>0$, $N\langle \hat{q}^f_k(\,\cdot \mid s)-q(\,\cdot \mid s), U\rangle$ is the sum of $N$ i.i.d. random variables $X_t$ for $t=1,\ldots,N$ with distribution $\mbbP[X_t = U(x)-\sum_{y\in f^{-1}(s)}q(y\mid s) U(y)]=q(x\mid s)$ for $x\in[n]$, from which one also checks that $\mbbE[X_t] = 0$.
	Because $U$ has nonnegative entries, the $X_t$ are almost surely bounded in $[0,\lVert U\rVert_{\infty}]$.

	It then follows from Hoeffding's inequality that for any $N>0$ and $\epsilon>0$
	\begin{equation}
		\label{eqn:basic_hoeffding}
		\mbbP\bkt{
		|\langle \hat{q}^f_k(\,\cdot \mid s)-q(\,\cdot \mid s), U\rangle| > \epsilon
		\mid
		\hat{N}^{\mathrm{to}}_k(f^{-1}(s))=N
		}
		\leq
		2\exp\biggl(
		-
		\frac{2N\epsilon^2}{\lVert U\rVert_{\infty}^2}
		\biggr)
		.
	\end{equation}
	Moreover, when $\hat{N}^{\mathrm{to}}_k(f^{-1}(s))=0$, the definition of $\hat{q}_k^f$ in \eqref{eqn:pfqfdef} implies that
	\begin{equation}
		\label{eqn:case_N_equals_zero_hoeffding}
		|\langle \hat{q}_k^f(\,\cdot \mid s) - q(\,\cdot \mid s), U\rangle |
		=
		|\langle q(\,\cdot \mid s), U\rangle|
		\leq
		\lVert U\rVert_{\infty}
		\leq
		\sqrt{\lVert U\rVert_{\infty}^2\ln(2T_k/\delta)}
	\end{equation}
	almost surely for $\delta<1/2$, where we have used that $q(y\mid s)\geq 0$ and $\sum_{y}q(y\mid s)=1$.
	We then rewrite the left-hand side of \eqref{eqn:hoeffding_q} by conditioning on all possible values of $\hat{N}^{\mathrm{to}}_k(f^{-1}(s))$, which is bounded almost surely by $T_k$, and using \eqref{eqn:case_N_equals_zero_hoeffding} and $\mbbP[\hat{N}^{\mathrm{to}}_k(f^{-1}(s))=N]\leq 1$:
	\begin{equation}
		\label{eqn:sum_over_all_values_of_N}
		\begin{split}
			 &
			\mbbP\biggl[
			|\langle \hat{q}^f_k(\,\cdot \mid s)-q(\,\cdot \mid s), U\rangle |>\sqrt{\frac{\norm{U}_{\infty}^2\ln(2T_k/\delta)}{1\vee\hat{N}^{\mathrm{to}}_k(f^{-1}(s))}}
			\biggr]
			\\
			 &
			\leq
			\sum_{N=1}^{T_k}
			\mbbP\biggl[
			|\langle \hat{q}^f_k(\,\cdot \mid s)-q(\,\cdot \mid s), U\rangle|>\sqrt{\frac{\norm{U}_{\infty}^2\ln(2T_k/\delta)}{N}}
			\,\bigg|\,
			\hat{N}^{\mathrm{to}}_k(f^{-1}(s)) = N
			\biggr]
			.
		\end{split}
	\end{equation}
	Finally, applying \eqref{eqn:basic_hoeffding} with $\epsilon = \sqrt{\smash[b]{\norm{U}_{\infty}^2\ln(2T_k/\delta)/N}}$, we conclude that
	\begin{equation}
		\mbbP\biggl[
		|\langle \hat{q}^f_k(\,\cdot \mid s)-q(\,\cdot \mid s), U\rangle |>\sqrt{\frac{\norm{U}_{\infty}^2\ln(2T_k/\delta)}{1\vee\hat{N}^{\mathrm{to}}_k(f^{-1}(s))}}
		\biggr]
		\leq
		2\sum_{N=1}^{T_k}
		\exp(
		-
		\ln(2T_k/\delta)
		)
		=
		2\sum_{N=1}^{T_k}\frac{\delta}{2T_k}
		=
		\delta.
	\end{equation}
	This concludes the proof of \eqref{eqn:hoeffding_q}.

	We next prove \eqref{eqn:bernstein_q}.
	Fix $s$ and $\delta \in(0,1/2]$.
	Note that \eqref{eqn:bernstein_q} holds trivially for $y\notin f^{-1}(s)$ since then $\hat{q}_k^f(y\mid s) = q(y\mid s) = 0$ by definition.
	Hence, fix $y\in f^{-1}(s)$.

	Observe that $\hat{q}_k^f(y\mid s)-q(y\mid s)=\langle\hat{q}_k^f(\,\cdot \mid s)-q(\,\cdot \mid s),e_y\rangle$ with $e_y\in\mbbR^n$ the $y$'th unit vector whose entries are given by $e_y(x)=\mathbbm{1}\{x=y\}$ for $x\in[n]$.
	It then follows from the argument leading to \eqref{eqn:basic_hoeffding} that conditional on $\{\hat{N}^{\mathrm{to}}_k(f^{-1}(s))=N\}$ for $N>0$, $N(\hat{q}^f_k(y\mid s)-q(y\mid s))$ is the sum of $N$ i.i.d. random variables $Y_t$ for $t=1,\ldots,N$ taking values in $\{-q(y\mid s),1-q(y\mid s)\}$ and with distribution $\mbbP[Y_t=1-q(y\mid s)]=q(y\mid s)$ and $\mbbP[Y_t = -q(y\mid s)]=1-q(y\mid s)$.
	It follows that $|Y_t|\leq 1$ almost surely, $\mbbE[Y_t]=0$, and that $Y_t$ has variance bounded by $(1-q(y\mid s))q(y\mid s)\leq q(y\mid s)$.

	We then conclude using Bernstein's inequality that for any $N>0$ and $\epsilon>0$,
	\begin{equation}
		\label{eqn:basic_bernstein}
		\mbbP\bkt{
		| \hat{q}^f_k(y\mid s)-q(y\mid s)| > \epsilon
		\mid
		\hat{N}^{\mathrm{to}}_k(f^{-1}(s))=N
		}
		\leq
		2\exp\biggl(
		-
		\frac{\frac{1}{2}N\epsilon^2}{q(y\mid s)+\frac{1}{3}\epsilon}
		\biggr)
		.
	\end{equation}
	Moreover, let $\epsilon=\epsilon^*(N,L)$ for $N,L>0$ denote the positive solution to $(1/2)N\epsilon^2=L(q(y\mid s)+(1/3)\epsilon)$, and observe that from the elementary inequality $\sqrt{a+b}\leq \sqrt{a}+\sqrt{b}$ for $a,b>0$,
	\begin{equation}
		\label{eqn:optimal_bernstein_epsilon}
		\epsilon^*(N,L)
		=
		\frac{1}{3N}
		\bigl(L + L\sqrt{L+18Nq(y\mid s)}\bigr)
		\leq
		\sqrt{\frac{2q(y\mid s)L}{N}} + \frac{2L}{3N}
		.
	\end{equation}
	Finally, observe also that when $\hat{N}^{\mathrm{to}}_k(f^{-1}(s))=0$, the definition of $\hat{q}_k^f$ in \eqref{eqn:pfqfdef} implies that for $\delta\leq 1/2$ and $T_k\geq 1$,
	\begin{equation}
		\label{eqn:bernstein_case_N_0}
		|\hat{q}_k^f(y\mid s)-q(y\mid s)|
		=
		q(y\mid s)
		\leq
		\sqrt{q(y\mid s)}
		\leq
		\sqrt{q(y\mid s)\ln(2T_k/\delta)}+(2/3)\ln(2T_k/\delta)
		,
	\end{equation}
	where we have used that $q(y\mid s)\in[0,1]$.
	We then rewrite the left-hand side of \eqref{eqn:bernstein_q} in a manner analogous to that in \eqref{eqn:sum_over_all_values_of_N}, but using \eqref{eqn:bernstein_case_N_0} instead of \eqref{eqn:case_N_equals_zero_hoeffding}, and subsequently using \eqref{eqn:optimal_bernstein_epsilon}:
	\begin{equation}
		\begin{split}
			 &
			\mbbP\biggl[
			| \hat{q}^f_k(y\mid s)-q(y\mid s) | >
			\sqrt{\frac{q(y\mid s) \ln(2T_k/\delta)}{1\vee\hat{N}^{\mathrm{to}}_k(f^{-1}(s))}}
			+
			\frac{2\ln(2T_k/\delta)}{3(1\vee\hat{N}^{\mathrm{to}}_k(f^{-1}(s)))}
			\biggr]
			\\
			 &
			\leq
			\sum_{N=1}^{T_k}
			\mbbP\biggl[
			|\hat{q}^f_k(y\mid s)-q(y\mid s)|>
			\sqrt{\frac{q(y\mid s) \ln(2T_k/\delta)}{N}}
			+
			\frac{2\ln(2T_k/\delta)}{3N}
			\,\bigg|\,
			\hat{N}^{\mathrm{to}}_k(f^{-1}(s)) = N
			\biggr]
			\\
			 &
			\leq
			\sum_{N=1}^{T_k}
			\mbbP\bigl[
			|\hat{q}^f_k(y\mid s)-q(y\mid s)|>
			\epsilon^*(N,\ln(2T_k/\delta))
			\mid
			\hat{N}^{\mathrm{to}}_k(f^{-1}(s)) = N
			\bigr]
			.
		\end{split}
	\end{equation}
	It finally follows by applying \eqref{eqn:basic_bernstein} with $\epsilon=\epsilon^*(N,\ln(2T_k/\delta))$ that
	\begin{equation}
		\begin{split}
			\mbbP\biggl[
			| \hat{q}^f_k(y\mid s)-q(y\mid s) |
			>
			 &
			\sqrt{\frac{q(y\mid s) \ln(2T_k/\delta)}{1\vee\hat{N}^{\mathrm{to}}_k(f^{-1}(s))}}
			+
			\frac{2\ln(2T_k/\delta)}{3(1\vee\hat{N}^{\mathrm{to}}_k(f^{-1}(s)))}
			\biggr]
			\\
			 &
			\leq
			2\sum_{N=1}^{T_k}
			\exp(
			-
			\ln(2T_k/\delta)
			)
			=
			2\sum_{N=1}^{T_k}\frac{\delta}{2T_k}
			=
			\delta.
		\end{split}
	\end{equation}
	This concludes the proof of \eqref{eqn:bernstein_q}.

	The analogs \eqref{eqn:hoeffding_p} and \eqref{eqn:bernstein_p} for $\hat{p}_k^f$ of \eqref{eqn:hoeffding_q} and \eqref{eqn:bernstein_q} follow from exactly the same argument, where we replace $\mcT'_k(s)$ by $\mcT'_k(s,a)=\mcT'_k(s,a)\eqdef \{(l,h)\in[k-1]\times [H]:x_{l,h}\in f^{-1}(s),a_{l,h}=a\}$ and condition on $\{\hat{N}_k(f^{-1}(s),a)=N\}$ instead.
\end{proof}

\subsection{Proof of \texorpdfstring{\Cref{lem:optimism_simulation}}{decomposition for optimism}}
\label{sec:proof_of_optimism_simulation}

Because it will be used again later on, we state the following intermediate result as a lemma.
\begin{lemma}
	\label{lem:simulation_lemma}

	For $i\in\{0,1\}$, let $\Phi^{(i)}=(n,S,A,H,\mu, f^{(i)},p^{(i)},q^{(i)},r^{(i)})$ be two \glspl{BMDP} with transition kernels $P^{(i)}$.
	Let $\pi$ be a policy for these \glspl{BMDP}.
	Let $V^{\pi,(i)}$ and $\mbbE_{\Phi^{(i)},\pi}$ denote the value function and expectation of \gls{BMDP} $\Phi^{(i)}$ under policy $\pi$.
	Then, for $i\in \{0,1\}$, $h\in[H]$, $x\in[n]$,
	\begin{align}
		V_h^{\pi,(0)}(x)-V_h^{\pi,(1)}(x)
		=
		\sum_{h'=h}^{H}
		&
		\mbbE_{\Phi^{(i)},\pi}
		\bigl[
			r_{h'}^{(0)}(x_{h'},a_{h'})-r_{h'}^{(1)}(x_{h'},a_{h'})
			\\
			 &
			+
			\bigl\langle
			P^{(0)}(\,\cdot  \mid x_{h'},a_{h'})
			-
			P^{(1)}(\,\cdot  \mid x_{h'},a_{h'}), V_{h'+1}^{\pi,(1-i)}
			\bigr\rangle
			\mid
			x_h=x
		\bigr]
		.
		\nonumber
	\end{align}
\end{lemma}
\begin{proof}
	The case $i = 1$ follows directly from \cite[Lemma E.15]{dann2017unifying} by identifying $M'\equiv \Phi^{(0)}$ and $M'' \equiv \Phi^{(1)}$.
	The case $i = 0$ follows \emph{mutatis mutandis} after replacing $V'_{i+1}(s')$ and $P''(s'\mid s_i,a_i,i)$ with $V''_{i+1}(s')$ and $P'(s'\mid s_i,a_i,i)$, respectively, after the second equality in their proof.
\end{proof}
Recall from the paragraph above \Cref{lem:optimism_simulation} that $\bar{V}^{\pi_k}_h(x)$ equals the optimal value function of a \gls{BMDP} $\bar{\Phi}_k$ with kernel $\hat{P}_k$ and reward function $1\wedge (r_h(x,a)+\mathbbm{1}\{h < H\}\hat{b}_k(\hat{f}^{\alg}(x),a))$.
Therefore, $\bar{V}^{\pi_k}_h(x)\geq \bar{V}^{\pi^*}_h(x)$ for all $x$, $h$, where $\bar{V}^{\pi^*}_h$ denotes the value function of $\bar{\Phi}_k$ under $\pi^*$, and hence
\begin{equation}
	\bar{V}^{\pi_k}_h(x)-V^{\pi^*}_h(x)
	\geq
	\bar{V}^{\pi^*}_h(x)-V^{\pi^*}_h(x)
	.
\end{equation}
Applying \Cref{lem:simulation_lemma} with $\Phi^{(0)}=\bar{\Phi}_k$, $\Phi^{(1)}=\Phi$, $\pi=\pi^*$ and $i=0$, and using that $V^{\pi^*}_{H+1}(x) = 0$ by convention gives \eqref{eqn:optimism_simulation}.
\qed

\subsection{Proof of \texorpdfstring{\cref{lem:esterrorconfbd}}{confidence bound on next context value}}
\label{sec:proof_of_estimation_error_conf_bd}

Because it will be of independent use later on in this appendix, we first state the following intermediate result as a lemma.
\begin{lemma}
	\label{lem:estimation_error_decomposition}
	Let $V\in\mbbR^n$ be any vector.
	Then for $x\in[n]$, $a\in[A]$, $k\in[K]$, on $\Omega^f$,
	\begin{align}
		\label{eqn:decomposition_with_estimate_q}
		 &
		\langle \hat{P}_k
		(\,\cdot \mid x,a)-P(\,\cdot \mid x,a), V\rangle
		\\
		 &
		=
		\sum_{s}
		(\hat{p}_k(s\mid f(x),a)-p(s\mid f(x),a))\langle \hat{q}_k(\,\cdot \mid s),V\rangle
		+
		\sum_{s}
		p(s\mid f(x),a)\langle \hat{q}_k(\,\cdot \mid s)-q(\,\cdot \mid s),V\rangle
		,
		\nonumber
	\end{align}
	and
	\begin{align}
		\label{eqn:decomposition_with_estimate_p}
		 &
		\langle \hat{P}_k
		(\,\cdot \mid x,a)-P(\,\cdot \mid x,a), V\rangle
		\\
		 &
		=
		\sum_{s}
		(\hat{p}_k(s\mid f(x),a)-p(s\mid f(x),a))\langle q(\,\cdot \mid s),V\rangle
		+
		\sum_{s}
		\hat{p}_k(s\mid f(x),a)\langle \hat{q}_k(\,\cdot \mid s)-q(\,\cdot \mid s),V\rangle
		.
		\nonumber
	\end{align}
\end{lemma}

Recall \eqref{eqn:highprobeventcomponentdef}.
Since $\Omega_k\subset\Omega^f$, \eqref{eqn:estimation_error_confidence_bound} now follows from \Cref{lem:estimation_error_decomposition} on $\Omega_k$.
Namely, using \eqref{eqn:decomposition_with_estimate_p} with $V=V_{h+1}^{\pi^*}$ and the triangle inequality, and bounding the result using $\Omega_k^1$ and $\Omega_k^2$, we obtain
\begin{align}
	 &
	|\langle \hat{P}_k(\,\cdot \mid x,a)-P(\,\cdot \mid x,a), V_{h+1}^{\pi^*}\rangle|
	\nonumber
	\\
	 &
	\leq
	\biggl|\sum_{s} (\hat{p}_k(s\mid f(x),a)-p(s\mid f(x),a)) \langle q(\,\cdot \mid s), V^{\pi^*}_{h+1}\rangle \biggr|
	\nonumber
	\\
	 &
	\phantom{\leq}
	+
	\sum_{s}
	\hat{p}_k(s\mid f(x),a)| \langle \hat{q}_k(\,\cdot \mid s)-q(\,\cdot \mid s),V^{\pi^*}_{h+1}\rangle|
	\nonumber
	\\
	 &
	\overset{\Omega_k}{\leq}
	\sqrt{\frac{H^2 \ln(2HSAT_k^2)}{1\vee \hat{N}_k(f^{-1}(f(x)),a)}}
	+
	\sum_{s}
	\hat{p}_k(s\mid f(x),a)
	\sqrt{\frac{H^2 \ln(2HST_k^2)}{1\vee \hat{N}^{\mathrm{to}}_k(f^{-1}(s))}}
	.
	\label{eqn:bound_on_estimation_error_with_optimal_value_function}
\end{align}
Upon comparing the right-hand side of \eqref{eqn:bound_on_estimation_error_with_optimal_value_function} with the expression for $\hat{b}_k(s,a)$ in \eqref{eqn:bonus_def}, we find that the two coincide on $\Omega_k$, where $\hat{f}^{\alg}=f$, therefore establishing the claim \eqref{eqn:estimation_error_confidence_bound}.
\qed

It remains to prove \Cref{lem:estimation_error_decomposition}, which we do next.
\begin{proof}[Proof of \Cref{lem:estimation_error_decomposition}]

	The proof of \Cref{lem:estimation_error_decomposition} proceeds similarly as that of \cite[Lemma 36]{Jedra:2022}.
	However, we avoid premature bounds, and exploit the further property that $\hat{f}^{\alg}=f$ on $\Omega^f$.
	We first prove \eqref{eqn:decomposition_with_estimate_q}.
	To this aim, note that from \eqref{eqn:estimator_P_def}, and on $\Omega^f=\{\hat{f}^{\alg}=f\}$, the estimate kernel $\hat{P}_k$ satisfies
	\begin{equation}
		\label{eqn:estimate_model_on_omega_f}
		\hat{P}_k(y\mid x,a)
		=
		\hat{p}_k(f(y)\mid f(x),a)
		\hat{q}_k(y\mid f(y))
		.
	\end{equation}
	We then introduce an auxiliary \gls{BMDP} transition kernel $\hat{P}'_k$, defined as
	\begin{equation}
		\label{eqn:first_intermediate_model_on_omega_f}
		\hat{P}'_k(y\mid x,a)\eqdef p(f(y)\mid f(x),a)\hat{q}_k(y\mid f(y)).
	\end{equation}
	The kernel $\hat{P}'_k$ is used to expand the left-hand side of \eqref{eqn:decomposition_with_estimate_q} as follows:
	\begin{equation}
		\label{eqn:introducing_first_intermediate_model}
		\langle \hat{P}_k(\,\cdot \mid x,a)-P(\,\cdot \mid x,a),V\rangle
		=
		\langle \hat{P}_k(\,\cdot \mid x,a)-\hat{P}'_k(\,\cdot \mid x,a),V\rangle
		+
		\langle \hat{P}'_k(\,\cdot \mid x,a)-P(\,\cdot \mid x,a),V\rangle
		.
	\end{equation}
	We now use \eqref{eqn:estimate_model_on_omega_f} and \eqref{eqn:first_intermediate_model_on_omega_f} to further expand each of the two terms on the right-hand side of \eqref{eqn:introducing_first_intermediate_model}.
	In particular, using that $(i)$ on $\Omega^f$, $\hat{q}_k(y\mid s)>0$ implies $y\in f^{-1}(s)$ by \eqref{eqn:estimators_p_q_def}, it follows that on $\Omega^f$,
	\begin{align}
		\langle \hat{P}_k(\,\cdot \mid x,a)-\hat{P}'_k(\,\cdot \mid x,a),V\rangle
			&
		=
		\sum_{y}
		(\hat{p}_k(f(y)\mid f(x),a)-p(f(y)\mid f(x),a))\hat{q}_k(y\mid f(y))V(y)
		\nonumber
		\\
			&
		=
		\sum_{s}
		(\hat{p}_k(s\mid f(x),a)-p(s\mid f(x),a))\sum_{y\in f^{-1}(s)}\hat{q}_k(y\mid s)V(y)
		\nonumber
		\\
			&
		\overset{(i)}{=}
		\sum_{s}
		(\hat{p}_k(s\mid f(x),a)-p(s\mid f(x),a))\sum_{y}\hat{q}_k(y\mid s)V(y)
		\nonumber
		\\
			&
		=
		\sum_{s}
		(\hat{p}_k(s\mid f(x),a)-p(s\mid f(x),a))\langle \hat{q}_k(\,\cdot \mid s),V\rangle
		.
	\end{align}
	Similarly, because $q(\,\cdot \mid s)>0$ implies $y\in f^{-1}(s)$, it follows that
	\begin{equation}
		\label{eqn:introducing_first_intermediate_model_term_2}
		\begin{split}
			\langle \hat{P}'_k(\,\cdot \mid x,a)-P(\,\cdot \mid x,a),V\rangle
			 &
			=
			\sum_{y}
			p(f(y)\mid f(x),a)(\hat{q}_k(y\mid f(y))-q(y\mid f(y)))V(y)
			\\
			 &
			=
			\sum_{s}
			p(s\mid f(x),a)\langle \hat{q}_k(\,\cdot \mid s)-q(\,\cdot \mid s),V\rangle.
		\end{split}
	\end{equation}
	Combining \eqref{eqn:introducing_first_intermediate_model}-\eqref{eqn:introducing_first_intermediate_model_term_2} yields \eqref{eqn:decomposition_with_estimate_q}.

	The result \eqref{eqn:decomposition_with_estimate_p} follows using the same steps, but we instead use the auxiliary kernel
	\begin{equation}
		\hat{P}''_k(y\mid x,a)\eqdef \hat{p}_k(f(y)\mid f(x),a)q(y\mid f(y)).
	\end{equation}
	Then, the same argument we used to prove \eqref{eqn:decomposition_with_estimate_q} now leads to
	\begin{equation}
		\begin{split}
			 &
			\langle \hat{P}_k(\,\cdot \mid x,a)-P(\,\cdot \mid x,a),V\rangle
			=
			\langle \hat{P}_k(\,\cdot \mid x,a)-\hat{P}''_k(\,\cdot \mid x,a),V\rangle
			+
			\langle \hat{P}''_k(\,\cdot \mid x,a)-P(\,\cdot \mid x,a),V\rangle
			\\
			 &
			=
			\sum_{s}
			\hat{p}_k(s\mid f(x),a)\langle \hat{q}_k(\,\cdot \mid s)-q(\,\cdot \mid s),V\rangle
			+
			\sum_{s}
			(\hat{p}_k(s\mid f(x),a)-p(s\mid f(x),a))\langle q(\,\cdot \mid s),V\rangle
			.
		\end{split}
	\end{equation}
	This proves \eqref{eqn:decomposition_with_estimate_p} and therefore \Cref{lem:estimation_error_decomposition}.
\end{proof}

\subsection{Proof of \texorpdfstring{\eqref{eqn:application_simulation_lemma_true_distribution}}{the regret decomposition for the upper bound}}
\label{sec:proof_of_application_simulation_lemma_true_distribution}

It suffices to show that for all $x$,
\begin{align}
	\bar{V}^{\pi_k}_1(x)
	-
	V^{\pi_k}_1(x)
	 &
	=
	\sum_{h=1}^{H-1}\mbbE\biggl[\hat{b}_{k}(\hat{f}^{\alg}(x_{k,h}),a_{k,h})
	+
	\langle \hat{P}_k(\,\cdot  \mid x_{k,h},a_{k,h})-P(\,\cdot \mid x_{k,h},a_{k,h}), V^{\pi^*}_{h+1}\rangle
	\nonumber
	\\
	 &
	+
	\langle \hat{P}_k(\,\cdot  \mid x_{k,h},a_{k,h})-P(\,\cdot  \mid x_{k,h},a_{k,h}), \bar{V}^{\pi_k}_{h+1}-V^{\pi^*}_{h+1}\rangle
	\mid \mcD_k, x_{k,1} = x \biggr]
	.
	\label{eqn:equivalent_to_application_simulation_lemma_true_distribution}
\end{align}
Recall from the paragraph directly above \Cref{lem:optimism_simulation} that $\bar{V}^{\pi_k}_h(x)$ is the optimal value function of the \gls{BMDP} $\bar{\Phi}_k$ with kernel $\hat{P}_k$ and reward function $1\wedge (r_h(x,a)+\mathbbm{1}\{h < H\}\hat{b}_k(\hat{f}^{\alg}(x),a))$.
Application of \Cref{lem:simulation_lemma} with $\Phi^{(0)}=\bar{\Phi}_k$, $\Phi^{(1)}=\Phi$, $\pi = \pi_k$ and $i = 1$ then gives
\begin{align}
	 &
	\bar{V}^{\pi_k}_h(x)-V_h^{\pi_k}(x)
	\nonumber
	\\
	&
	=
	\sum_{h'=h}^{H-1}\mbbE_{\Phi,\pi_k}\bigl[\,\hat{b}_k(\hat{f}^{\alg}(x_{h'}),a_{h'})+\bigl\langle \hat{P}_k(\,\cdot \mid x_{h'},a_{h'})-P(\,\cdot \mid x_{h'},a_{h'}), \bar{V}^{\pi_k}_{h'+1} \bigr\rangle\mid x_h=x\bigr]
	\label{eqn:application_simulation_lemma_2}
	\\
	 &
	=
	\sum_{h'=h}^{H-1}\mbbE\bigl[
	\,\hat{b}_k(\hat{f}^{\alg}(x_{k,h'}),a_{k,h'})
	+
	\bigl\langle \hat{P}_k(\,\cdot \mid x_{k,h'},a_{k,h'})-P(\,\cdot \mid x_{k,h'},a_{k,h'}), \bar{V}^{\pi_k}_{h'+1} \bigr\rangle
	\mid \mcD_k,x_{k,h}=x\bigr].
	\nonumber
\end{align}
The last line follows because $\pi_k$, $\hat{b}_k$, $\hat{f}^{\alg}$, $\hat{P}_k$, and $\bar{V}^{\pi_k}_{h+1}$ are all deterministic given $\mcD_k$ for $T_k>\Theta^{\clust}$.
The latter constraint is required because $\hat{f}^{\alg}$ depends on all observations from \Cref{alg:BUCBVIouter}'s first phase.
Thus, the conditional expectation is taken only with respect to the context--action sequence $(x_{k,1},a_{k,1},\ldots,a_{k,H},x_{k,H+1})$ for episode $k$.
This, given $\mcD_k$, has the same distribution as $(x_1,a_1,\ldots,a_H,x_{H+1})$ under $\mbbP_{\Phi,\pi_k}$.

Finally, decompose the right-most term of \eqref{eqn:application_simulation_lemma_2} as
\begin{align}
	\nonumber
	\bigl\langle 
		\hat{P}_k(\,\cdot \mid x_{k,h'},a_{k,h'})&\,-P(\,\cdot \mid x_{k,h'},a_{k,h'}),
		\bar{V}^{\pi_k}_{h'+1}
	\bigr\rangle
	=
	\bigl\langle \hat{P}_k(\,\cdot  \mid x_{k,h},a_{k,h})-P(\,\cdot \mid x_{k,h},a_{k,h}), V^{\pi^*}_{h+1}\bigr\rangle
	\\
	&
	+
	\bigl\langle \hat{P}_k(\,\cdot  \mid x_{k,h},a_{k,h})-P(\,\cdot  \mid x_{k,h},a_{k,h}), \bar{V}^{\pi_k}_{h+1}-V^{\pi^*}_{h+1}\bigr\rangle
	.
	\label{eqn:estimation_error_decomposition}
\end{align}
Together, \eqref{eqn:application_simulation_lemma_2} and \eqref{eqn:estimation_error_decomposition} prove \eqref{eqn:equivalent_to_application_simulation_lemma_true_distribution}.
\qed

\subsection{Proof of \texorpdfstring{\Cref{prop:concentration_bound_nonnegative_vector}}{the bound on the cross term}}
\label{sec:proof_of_concentration_bound_nonnegative_vector}

The proof will rely on the following intermediate result.
Its proof is a straightforward adaptation of the argument in \cite[Lemma 3]{UCRLVI}, and is presented in \Cref{sec:proof_of_concentration_bound_nonnegative_vector} below.
\begin{lemma}
	\label{lem:concentration_bound_nonnegative_vector}
	Let $U\in\mbbR^n$ and $u\in\mbbR^S$.
	Then, for $s\in[S]$, $a\in[A]$, and $k\in[K]$, on $\Omega_k$,
	\begin{align}
		 &
		\langle |\hat{q}_k(\,\cdot \mid s) - q(\,\cdot \mid s)|, U \rangle
		\leq
		\frac{1}{H}\langle q(\,\cdot \mid s), U\rangle
		+
		\frac{2n \lVert U\rVert_{\infty}HL}{1\vee \hat{N}^{\mathrm{to}}_k(f^{-1}(s))}
		\label{eqn:concentration_bound_nonnegative_vector_q}
		\\
		 &
		\mathrm{and}\quad\sum_{s'}|\hat{p}_k(s' \mid s,a) - p(s' \mid s,a)| u(s')
		\leq
		\frac{1}{H}\Bigl(\sum_{s'} p(s'\mid s,a) u(s')\Bigr)
		+
		\frac{2\lVert u\rVert_{\infty}HSL}{1\vee \hat{N}_k(f^{-1}(s),a)}.
		\label{eqn:concentration_bound_nonnegative_vector_p}
	\end{align}
\end{lemma}

Fix $k\in[K]$ and recall that $L=1+\ln(nHS^2AT_K^2)$.
Define for each $h$ the vectors $\Delta_{k,h}\in\mbbR^n$ whose elements are given by $\Delta_{k,h}(x)\eqdef \bar{V}^{\pi_k}_h(x)-V^{\pi^*}_h(x)$ for $x\in[n]$.
Observe that $\Delta_{k,h}\geq 0$ on $\Omega_k$ by \Cref{prop:optimism}.
Observe, moreover, that $\lVert\Delta_{k,h}\rVert_{\infty}\leq H$ because the rewards are bounded in $[0,1]$.
By \eqref{eqn:decomposition_with_estimate_q} of \Cref{lem:estimation_error_decomposition} and the triangle inequality, it follows that $\Omega_k$,
\begin{align}
	|\langle \hat{P}_k(\,\cdot \mid x,a)-P(\,\cdot \mid x,a), \Delta_{k,h+1}\rangle|
	 &
	\leq
	\sum_{s}
	|\hat{p}_k(s\mid f(x),a)-p(s\mid f(x),a)|\langle \hat{q}_k(\,\cdot \mid s),\Delta_{k,h+1}\rangle
	\nonumber
	\\
	 &
	+
	\sum_{s}
	p(s\mid f(x),a)\langle |\hat{q}_k(\,\cdot \mid s)-q(\,\cdot \mid s)| ,\Delta_{k,h+1}\rangle
	,
	\label{eqn:expansion_estimation_error_nonnegative_vector}
\end{align}
where we have also used the nonnegativity of $\Delta_{k,h+1}$, as well as that of $p(s\mid f(x),a)$ and $\hat{q}_k(\,\cdot \mid s)$.
Applying \Cref{lem:concentration_bound_nonnegative_vector} with $U=\Delta_{k,h+1}$ and $u(s)=\langle \hat{q}_k(\,\cdot \mid s),\Delta_{k,h+1}\rangle$, we find that on $\Omega_k$,
\begin{align}
	|\langle \hat{P}_k(\,\cdot \mid x,a)-
	 &
	P(\,\cdot \mid x,a), \Delta_{k,h+1}\rangle|
	\nonumber
	\\
	 &
	\leq
	\frac{1}{H}
	\sum_{s}
	p(s\mid f(x),a)\langle \hat{q}_k(\,\cdot \mid s),\Delta_{k,h+1}\rangle
	+
	\frac{2H^2SL}{1\vee \hat{N}_k(f^{-1}(f(x)),a)}
	\label{eqn:first_bound_of_expansion_estimation_error_nonnegative_vector}
	\\
	 &
	+
	\frac{1}{H}
	\sum_{s}
	p(s\mid f(x),a)\langle q(\,\cdot \mid s),\Delta_{k,h+1}\rangle
	+
	\sum_{s}
	p(s\mid f(x),a)
	\frac{2n H^2L}{1\vee \hat{N}^{\mathrm{to}}_k(f^{-1}(s))}
	.
	\nonumber
\end{align}
Here, we have used our earlier observation that $\lVert \Delta_{k,h+1}\rVert_{\infty}\leq H$, which furthermore implies that $\lVert u\rVert_{\infty}=\max_{s}|\langle \hat{q}_k(\,\cdot \mid s),U|\leq \lVert \Delta_{k,h+1}\rVert_{\infty}\leq  H$ by Hölder's inequality and because $\sum_{y}\hat{q}_k(y\mid s)\leq 1$.

The first term on the right-hand side of \eqref{eqn:first_bound_of_expansion_estimation_error_nonnegative_vector} still involves the estimate emission kernel $\hat{q}_k$.
We again use \Cref{lem:concentration_bound_nonnegative_vector} to replace it, up to a correction term, by the true emission kernel $q$.
In particular, by first using the nonnegativity of $\Delta_{k,h+1}$, followed by \Cref{lem:concentration_bound_nonnegative_vector}, we find that on $\Omega_k$,
\begin{equation}
	\begin{split}
		 &
		\frac{1}{H}
		\sum_{s}
		p(s\mid f(x),a)
		\langle \hat{q}_k(\,\cdot \mid s),\Delta_{k,h+1}\rangle
		\\
		 &
		=
		\frac{1}{H}
		\sum_{s}
		p(s\mid f(x),a)\bigl(
		\langle q(\,\cdot \mid s),\Delta_{k,h+1}\rangle
		+
		\langle \hat{q}_k(\,\cdot \mid s)-q(\,\cdot \mid s),\Delta_{k,h+1}\rangle
		\bigr)
		\\
		 &
		\leq
		\frac{1}{H}
		\sum_{s}
		p(s\mid f(x),a)\bigl(
		\langle q(\,\cdot \mid s),\Delta_{k,h+1}\rangle
		+
		\langle |\hat{q}_k(\,\cdot \mid s)-q(\,\cdot \mid s)|,\Delta_{k,h+1}\rangle
		\bigr)
		\\
		 &
		\leq
		\frac{1}{H}
		\sum_{s}
		p(s\mid f(x),a)\biggl(
		\biggl(1+\frac{1}{H}\biggr)\langle q(\,\cdot \mid s),\Delta_{k,h+1}\rangle
		+
		\frac{2n H^2L}{1\vee \hat{N}^{\mathrm{to}}_k(f^{-1}(s))}
		\biggr)
		.
	\end{split}
\end{equation}
Together with \eqref{eqn:first_bound_of_expansion_estimation_error_nonnegative_vector}, we conclude that
\begin{align}
	 &
	|\langle \hat{P}_k(\,\cdot \mid x,a)-P(\,\cdot \mid x,a),
	\, \Delta_{k,h+1}\rangle|
	\leq
	\frac{1}{H}
	\biggl(2+\frac{1}{H}\biggr)\sum_{s} p(s\mid f(x),a)\langle q(\,\cdot \mid s),\Delta_{k,h+1}\rangle
	\nonumber
	\\
	 &
	+
	\frac{2H^2SL}{1\vee \hat{N}_k(f^{-1}(f(x)),a)}
	+
	\biggl(1+\frac{1}{H}\biggr)\sum_{s} p(s\mid f(x),a)
	\frac{2n H^2L}{1\vee \hat{N}^{\mathrm{to}}_k(f^{-1}(s))}
	.
\end{align}
The result \eqref{eqn:concentration_bound_nonnegative_vector_large_P} finally follows by recalling that $\sum_{s} p(s\mid f(x),a)\langle q(\,\cdot \mid s),\Delta_{k,h+1}\rangle=\langle P(\,\cdot \mid x,a), \allowbreak \bar{V}^{\pi_k}_{h+1} - V^{\pi^*}_{h+1}\rangle$ and bounding $(1+1/H)\leq 2$ since $H\geq 1$.
\qed

\subsubsection{Proof of \texorpdfstring{\Cref{lem:concentration_bound_nonnegative_vector}}{lemma}}
\label{sec:proof_of_lem_concentration_bound_nonnegative_vector}

We first prove \eqref{eqn:concentration_bound_nonnegative_vector_q}.
Recall \eqref{eqn:highprobeventcomponentdef}, and that $L=1+\ln(nHS^2AT_K^2)$ and $T_k = kH$.
Note that $\ln(2nST_k^2)\leq L$ for $k\in[K]$.
It follows that on $\Omega^f\cap \Omega^3_k$, for all $s$ and $y$,
\begin{equation}
	\label{eqn:element_wise_bound_on_q}
	|\hat{q}_k(y\mid s)-q(y\mid s)|
	\leq
	\sqrt{\frac{q(y\mid s) L}{1\vee\hat{N}^{\mathrm{to}}_k(f^{-1}(s))}}
	+
	\frac{2L}{3(1\vee\hat{N}^{\mathrm{to}}_k(f^{-1}(s)))}
	.
\end{equation}
It follows from \eqref{eqn:element_wise_bound_on_q} that
\begin{align}
	\nonumber
	 &
	\sum_{y\in f^{-1}(s)}
	| \hat{q}_k(y\mid s)-q(y\mid s)| U(y)
	\leq
	\sum_{y\in f^{-1}(s)}
	\biggl(
	\sqrt{\frac{q(y\mid s) L}{1\vee\hat{N}^{\mathrm{to}}_k(f^{-1}(s))}} + \frac{2L}{3(1\vee\hat{N}^{\mathrm{to}}_k(f^{-1}(s)))}
	\biggr) U(y)
	\\
	 &
	\leq
	\Biggl(\sum_{y\in f^{-1}(s)}\sqrt{\frac{q(y\mid s) L}{1\vee\hat{N}^{\mathrm{to}}_k(f^{-1}(s))}}U(y)\Biggr)
	+
	\frac{n\lVert U\rVert_{\infty}HL}{(1\vee\hat{N}^{\mathrm{to}}_k(f^{-1}(s)))}
	,
	\label{eqn:corrtermbdconfbdapplied}
\end{align}
where, in the second inequality, we have included a loose bound $(2/3)U(y)\leq H\lVert U\rVert_{\infty}$ to simplify later steps.
For $k\in[K]$, $s\in[S]$, define
\begin{equation}
	\mcG_{k,s}
	\eqdef
	\bigl\{ y\in f^{-1}(s): q(y\mid s)(1\vee\hat{N}^{\mathrm{to}}_k(f^{-1}(s)))\geq H^2L \bigr\}
\end{equation}
Focusing on the first term on the last line of \eqref{eqn:corrtermbdconfbdapplied}, we split the sum over $y\in f^{-1}(s)$ as follows:
\begin{equation}\label{eqn:sumsplit}
	\sum_{y\in f^{-1}(s)}
	\sqrt{\frac{q(y\mid s) L}{1\vee\hat{N}^{\mathrm{to}}_k(f^{-1}(s))}}U(y)
	=
	\Biggl(\sum_{y\in  \mcG^{\mathrm{c}}_{k,s}}+\sum_{y\in  \mcG_{k,s}} \Biggr) \sqrt{\frac{q(y\mid s) L}{1\vee\hat{N}^{\mathrm{to}}_k(f^{-1}(s))}}U(y).
\end{equation}
We first deal with the terms in \eqref{eqn:sumsplit} that are \textit{not} in $\mcG_{k,s}$.
These satisfy
\begin{equation}
	\sum_{y\in  \mcG^{\mathrm{c}}_{k,s}}\sqrt{\frac{q(y\mid s) L}{1\vee\hat{N}^{\mathrm{to}}_k(f^{-1}(s))}}U(y)
	=
	\sum_{y\in  \mcG^{\mathrm{c}}_{k,s}}\frac{\sqrt{q(y\mid s)(1\vee\hat{N}^{\mathrm{to}}_k(f^{-1}(s))) L}}{1\vee\hat{N}^{\mathrm{to}}_k(f^{-1}(s))}U(y)
	\leq
	\frac{n\lVert U\rVert_{\infty}HL}{1\vee\hat{N}^{\mathrm{to}}_k(f^{-1}(s))}
	.
\end{equation}
The other terms in \eqref{eqn:sumsplit} that \textit{are} in $\mcG_{k,s}$ satisfy
\begin{equation}
	\label{eqn:gksbd}
	\begin{split}
		\sum_{y\in  \mcG_{k,s}}
		\sqrt{\frac{q(y\mid s) L}{1\vee\hat{N}^{\mathrm{to}}_k(f^{-1}(s))}}U(y)
		 &
		=
		\sum_{y\in  \mcG_{k,s}}
		q(y\mid s)\sqrt{\frac{ L}{q(y\mid s)(1\vee\hat{N}^{\mathrm{to}}_k(f^{-1}(s)))}}U(y)
		\\
		 &
		\leq
		\frac{1}{H}
		\sum_{y\in f^{-1}(s)}
		q(y\mid s)U(y).
	\end{split}
\end{equation}
The result \eqref{eqn:concentration_bound_nonnegative_vector_q} now follows by combining \eqref{eqn:corrtermbdconfbdapplied} with \eqref{eqn:sumsplit}-\eqref{eqn:gksbd}.

The result \eqref{eqn:concentration_bound_nonnegative_vector_p} follows from an analogous argument upon replacing \eqref{eqn:element_wise_bound_on_q} by the analogous bound for $|\hat{p}_k(s'\mid s,a)-p(s'\mid s,a)|$ implied by $\Omega^f\cap\Omega_k^4$ and replacing the sets $\mcG_{k,s}$ by $\{s'\in[S]:p(s'\mid s,a)(1\vee \hat{N}_k(f^{-1}(s),a)\geq H^2L)\}$.
\qed

\subsection{Proof of \texorpdfstring{\Cref{prop:estimation_error_bound}}{the recursive bound}}
\label{sec:proof_of_estimation_error_bound}

Let $k\in[K]$ be such that $T_k=kH>\Theta^{\clust}$.
We first use \eqref{eqn:application_simulation_lemma_true_distribution} to rewrite the left-hand side of \eqref{eqn:estimation_error_bound}:
\begin{align}
	 &
	\mbbE[(\bar{V}_h^{\pi_k}(x_{k,h})-V_h^{\pi_k}(x_{k,h}))\mathbbm{1}\{\Omega_k\}]
	\nonumber
	\\
	 &
	=
	\sum_{h=1}^{H-1}
	\mbbE\bigl[
		\mbbE\bigl[
			\hat{b}_{k}(\hat{f}^{\alg}(x_{k,h}),a_{k,h})
			+
			\langle \hat{P}_k(\,\cdot  \mid x_{k,h},a_{k,h})-P(\,\cdot \mid x_{k,h},a_{k,h}), V^{\pi^*}_{h+1}\rangle
			\nonumber
			\\
			&
			\qquad
			+
			\langle \hat{P}_k(\,\cdot  \mid x_{k,h},a_{k,h})-P(\,\cdot  \mid x_{k,h},a_{k,h}), \bar{V}^{\pi_k}_{h+1}-V^{\pi^*}_{h+1}\rangle
			\mid \mcD_k, x_{k,1}
		\bigr]
		\mathbbm{1}\{\Omega_k\}
	\bigr]
	\nonumber
	\\
	&
	=
	\sum_{h=1}^{H-1}
	\mbbE\bigl[
		\hat{b}_{k}(\hat{f}^{\alg}(x_{k,h}),a_{k,h})
		+
		\langle \hat{P}_k(\,\cdot  \mid x_{k,h},a_{k,h})-P(\,\cdot \mid x_{k,h},a_{k,h}), V^{\pi^*}_{h+1}\rangle
		\nonumber
		\\
		&
		\qquad
		+
		\langle \hat{P}_k(\,\cdot  \mid x_{k,h},a_{k,h})-P(\,\cdot  \mid x_{k,h},a_{k,h}), \bar{V}^{\pi_k}_{h+1}-V^{\pi^*}_{h+1}\rangle
		\mathbbm{1}\{\Omega_k\}
	\bigr]
	,
	\label{eqn:expectation_of_simulation_lemma_result}
\end{align}
where in passing to the last line we have used the law of total expectation and the fact that $\Omega_k$ is $\mcD_k$-measurable for any $k\in[K]$ satisfying $T_k=kH>\Theta^{\clust}$.
The constraint on $k$ is required because $\hat{f}^{\alg}$ depends on all observations from \Cref{alg:BUCBVIouter}'s first phase, and $\Omega^f$ depends on $\hat{f}^{\alg}$.

Next, observe that by \Cref{lem:esterrorconfbd}, the following holds on $\Omega_k$:
\begin{equation}
	\langle \hat{P}_k(\,\cdot \mid x_{k,h'},a_{k,h'})-P(\,\cdot \mid x_{k,h'},a_{k,h'}), V^{\pi^*}_{h'+1} \rangle
	\leq
	\hat{b}_k(f(x),a)
	.
\end{equation}
Moreover, since $\hat{f}^{\alg}=f$ on $\Omega^f$, it follows that
\begin{align}
	 &
	\textnormal{RHS of }\eqref{eqn:expectation_of_simulation_lemma_result}
	\label{eqn:bounding_estimation_error_by_bonus}
	\\
	 &
	\leq
	\sum_{h'=h}^{H-1}
	\mbbE\Bigl[
	\bigl(
	\underbrace{2\hat{b}_k(s_{k,h'},a_{k,h'})}_{\defeq T_1}
	+
	\underbrace{\langle \hat{P}_k(\,\cdot \mid x_{k,h'},a_{k,h'})-P(\,\cdot \mid x_{k,h'},a_{k,h'}), \bar{V}^{\pi_k}_{h'+1}- V^{\pi^*}_{h'+1}\rangle}_{\defeq T_2}
	\bigr)
	\mathbbm{1}\{\Omega_k\}
	\Bigr]
	\nonumber
\end{align}
where we have used the notation $s_{k,h'}=f(x_{k,h'})$.
We now bound the terms $T_1$ and $T_2$.

Starting with $T_2$, applying \Cref{prop:concentration_bound_nonnegative_vector} gives that
\begin{align}
	\mbbE\bigl[
		T_2
		\mathbbm{1}\{\Omega_k\}
		\bigr]
	\nonumber
	\leq
	\mbbE\Bigl[
	\biggl(
	 &
	\frac{2SH^2L^2}{1\vee \hat{N}_k(f^{-1}(s_{k,h}),a_{k,h})}
	+
	\sum_{s}
	p(s\mid s_{k,h},a_{k,h})\frac{4n H^2L}{1\vee \hat{N}^{\mathrm{to}}_k(f^{-1}(s))}
	\nonumber
	\\
	 &
	+
	\frac{1}{H}
	\biggl(2+\frac{1}{H}\biggr)\langle P(\,\cdot \mid x_{k,h},a_{k,h}), \bar{V}^{\pi_k}_{h+1} - V^{\pi^*}_{h+1}\rangle
	\biggr)
	\mathbbm{1}\{\Omega_k\}
	\Bigr]
	.
	\label{eqn:bound_on_correction_term}
\end{align}
Here, we have also used that $L=1+\ln(nHS^2AT_k^2)\geq 1$ to include a loose bound $L\leq L^2$.
To simplify the right-hand side, note that for any $\mcD_k$-measurable vector $\hat{U}\in\mbbR^n$, using the law of total expectation and a one-step advancement, we have that
\begin{equation}
	\begin{split}
		\mbbE[\langle P(\,\cdot \mid x_{k,h},a_{k,h}), \hat{U}\rangle ]
		 &
		=
		\sum_{y}
		\mbbE[ \mbbE[P(y\mid x_{k,h},a_{k,h})\hat{U}(y)\mid \mcD_k] ]
		\\
		 &
		=
		\mbbE
		\bigl[
			\mbbE[\hat{U}(x_{k,h+1})\mid \mcD_k]
			\bigr]
		=
		\mbbE[\hat{U}(x_{k,h+1})].
	\end{split}
\end{equation}
A similar identity holds for any $\mcD_k$-measurable $\hat{u}\in\mbbR^S$:
\begin{equation}
	\label{eqn:next_latent_state_expectation}
	\sum_{s}
	\mbbE[ p(s\mid s_{k,h},a_{k,h}) \hat{u}(s) ] = \mbbE[\hat{u}(s_{k,h+1})].
\end{equation}
Then recall that $\Omega_k$ is $\mcD_k$-measurable for any $k\in[K]$ satisfying $T_k=kH>\Theta^{\clust}$.
It follows that for such $k$ also $(1\vee \hat{N}^{\mathrm{to}}_k(f^{-1}(\,\cdot )))^{-1}\mathbbm{1}\{\Omega_k\}$ and $(\bar{V}^{\pi_k}_{h+1}- \bar{V}^{\pi^*}_{h+1})\mathbbm{1}\{\Omega_k\}$ are $\mcD_k$-measurable, so that
\begin{align}
	\textnormal{RHS of }\eqref{eqn:bound_on_correction_term}
	 &
	=
	\mbbE\Bigl[
	\biggl(
	\frac{2SH^2L^2}{1\vee \hat{N}_k(f^{-1}(s_{k,h}),a_{k,h})}
	+
	\frac{4n H^2L}{1\vee \hat{N}^{\mathrm{to}}_k(f^{-1}(s_{k,h+1}))}
	\nonumber
	\\
	 &
	\phantom{=}
	+
	\frac{1}{H}
	\biggl(2+\frac{1}{H}\biggr)(\bar{V}^{\pi_k}_{h+1}(x_{k,h+1}) - V^{\pi^*}_{h+1}(x_{k,h+1}))
	\biggr)
	\mathbbm{1}\{\Omega_k\}
	\Bigr]
	.
	\label{eqn:bound_on_expected_correction_term}
\end{align}

We next derive a bound for $T_1$ in \eqref{eqn:bounding_estimation_error_by_bonus}.
Recall \eqref{eqn:bonus_def}, and use that $\hat{f}^{\alg}=f$ on $\Omega_k$, and $\ln(2HST_k^2)\vee \ln(2SAT_k^2)\allowbreak \leq 1+\ln(nHS^2AT_k)=L$ to write:
\begin{equation}
	\label{eqn:bonus_on_omega}
	\begin{split}
		\hat{b}_k(s,a)
		 &
		\overset{\Omega_k}{=}
		\sum_{s'}
		\hat{p}_k(s'\mid s,a)\sqrt{\frac{H^2\ln(2HST_k^2)}{1\vee \hat{N}^{\mathrm{to}}_k(f^{-1}(s'))}}
		+
		\sqrt{\frac{H^2 \ln(2SAT_k^2)}{1\vee\hat{N}_k(f^{-1}(s), a)}}
		\\
		 &
		\leq
		\sum_{s'}
		\hat{p}_k(s'\mid s,a)\sqrt{\frac{H^2L}{1\vee \hat{N}^{\mathrm{to}}_k(f^{-1}(s'))}}
		+
		\sqrt{\frac{H^2 L}{1\vee\hat{N}_k(f^{-1}(s), a)}}
		.
	\end{split}
\end{equation}
Observe in particular that $\hat{b}_k$ depends on the estimate transition kernel $\hat{p}_k$.
To replace this with the true kernel $p$, we set up to apply \Cref{lem:concentration_bound_nonnegative_vector} with $u(s) = \sqrt{\smash[b]{H^2L/(1\vee \hat{N}^{\mathrm{to}}_k(f^{-1}(s)))}}$.
Notice that $u(s)\geq 0$ and $\lVert u\rVert_{\infty} \leq H\sqrt{L}$.
It then follows, using the triangle inquality and nonnegativity of $u$, followed by \Cref{lem:concentration_bound_nonnegative_vector} that
\begin{align}
	\sum_{s'}
	 &
	\hat{p}_k(s'\mid s,a)\sqrt{\frac{H^2L}{1\vee \hat{N}^{\mathrm{to}}_k(f^{-1}(s'))}}
	\nonumber
	\\
	 &
	\leq
	\sum_{s'}
	\biggl(p(s'\mid s,a)\sqrt{\frac{H^2L}{1\vee \hat{N}^{\mathrm{to}}_k(f^{-1}(s'))}}\biggr)
	+
	\sum_{s'}
	|\hat{p}_k(s'\mid s,a)-p(s'\mid s,a)|\sqrt{\frac{H^2L}{1\vee \hat{N}^{\mathrm{to}}_k(f^{-1}(s'))}}
	\nonumber
	\\
	 &
	\leq
	\biggl(1+\frac{1}{H}\biggr)\sum_{s'}\biggl(p(s'\mid s,a)\sqrt{\frac{H^2L}{1\vee \hat{N}^{\mathrm{to}}_k(f^{-1}(s'))}}\biggr)
	+
	\frac{2H^2SL^{3/2}}{1\vee\hat{N}_k(f^{-1}(s),a)}
	.
	\label{eqn:bound_second_term_bonus}
\end{align}
Together, \eqref{eqn:bonus_on_omega} and \eqref{eqn:bound_second_term_bonus} imply that on $\Omega_k$
\begin{equation}
	\label{eqn:bound_bonus}
	\begin{split}
		\hat{b}_k(s,a)
		\leq
		 &
		\biggl(1+\frac{1}{H}\biggr)\sum_{s'}\biggl(p(s'\mid s,a)\sqrt{\frac{H^2L}{1\vee \hat{N}^{\mathrm{to}}_k(f^{-1}(s'))}}\biggr)
		\\
		 &
		+
		\frac{2H^2SL^{3/2}}{1\vee\hat{N}_k(f^{-1}(s),a)}
		+
		\sqrt{\frac{H^2 L}{1\vee\hat{N}_k(f^{-1}(s), a)}}
		.
	\end{split}
\end{equation}
Finally, noting once more that $\mathbbm{1}\{\Omega_k\}$ and $(1\vee \hat{N}^{\mathrm{to}}_k(f(s')))^{-1}$ are $\mcD_k$-measurable, it follows from \eqref{eqn:next_latent_state_expectation} and \eqref{eqn:bound_bonus} that
\begin{align}
	\label{eqn:final_bonus_bound}
		&
	\mbbE\bigl[T_1\mathbbm{1}\{\Omega_k\}\bigr]
	=
	\mbbE\bigl[2\hat{b}_k(s_{k,h},a_{k,h})\mathbbm{1}\{\Omega_k\}\bigr]
	\\
		&
	\leq
	\mbbE\biggl[
	\biggl(
	\sqrt{\frac{16H^2L}{1\vee \hat{N}^{\mathrm{to}}_k(f^{-1}(s_{k,h+1}))}}
	+
	\frac{4H^2SL^{2}}{1\vee\hat{N}_k(f^{-1}(s_{k,h}),a_{k,h})}
	+
	\sqrt{\frac{4H^2 L}{1\vee\hat{N}_k(f^{-1}(s_{k,h}), a_{k,h})}}
	\biggr)
	\mathbbm{1}\{\Omega_k\}
	\biggr].
	\nonumber
\end{align}
Here, we have also incorporated the loose bounds $1/H\leq 1$ and $L^{3/2}\leq L^2$ (recall that $L\geq 1$) to simplify the resulting expression.

The final result \eqref{eqn:estimation_error_bound} now follows from \eqref{eqn:bounding_estimation_error_by_bonus} by adding the bounds in \eqref{eqn:bound_on_expected_correction_term} and \eqref{eqn:final_bonus_bound}, and identifying the expression for $\hat{B}_k$ in the resulting bound.
\qed

\subsection{Proof of \texorpdfstring{\Cref{lem:recursive_sum}}{the recursive sum lemma}}
\label{sec:proof_of_recursive_sum}

Let $\{\Delta_h\}_{h\in[H]}$ satisfy the assumptions of \Cref{lem:recursive_sum}, and let $\varepsilon>0$.
We first show that for $h\in[H-2]$,
\begin{equation}
	\label{eqn:induction_assumption}
	\sum_{h'=h+1}^{H-1}
	(B_{h'}+\varepsilon \Delta_{h'+1})
	\leq
	\sum_{h'=h+1}^{H-1}
	(1+\varepsilon)^{h'-h-1}B_{h'}
	.
\end{equation}
From the assumption $\Delta_{H} = 0$ it follows that
\begin{equation}
	\sum_{h'=H-1}^{H-1}
	(B_{h'}+\varepsilon\Delta_{h'+1})= \sum_{h'=H-1}^{H-1}B_{h'}
	= \sum_{h'=H-1}^{H-1}(1+\varepsilon)^{h'-(H-2)-1}B_{h'},
\end{equation}
so that \eqref{eqn:induction_assumption} is true for $h = H-2$.
Hence, assume that \eqref{eqn:induction_assumption} is true for some $h\in[H-2]$.
We prove that it is also true for $h-1$.
From the assumption $(i)$ $\Delta_{h+1}\leq \sum_{h'=h+1}^{H-1}(B_{h'}+\varepsilon\Delta_{h'+1})$ of \Cref{lem:recursive_sum}, it follows that
\begin{equation}
	\begin{split}
		\sum_{h'=h}^{H-1}
		(B_{h'}+\varepsilon \Delta_{h'+1})
		 &
		=
		B_h + \varepsilon \Delta_{h+1} + \sum_{h'=h+1}^{H-1}(B_{h'}+\varepsilon \Delta_{h'+1}) \\
		 &
		\overset{(i)}{\leq}
		B_h + (1+\varepsilon) \sum_{h'=h+1}^{H-1}(B_{h'}+\varepsilon\Delta_{h'+1})
		\\
		 &
		\overset{\eqref{eqn:induction_assumption}}{\leq}
		B_h + \sum_{h'=h+1}^{H-1}(1+\varepsilon)^{h'-h}B_{h'}
		=
		\sum_{h'=h}^{H-1}
		(1+\varepsilon)^{h'-h}B_{h'}
		.
	\end{split}
\end{equation}
We conclude by induction that \eqref{eqn:induction_assumption} is true for $h\in[H-2]$.

The claim of \Cref{lem:recursive_sum} now follows from \eqref{eqn:induction_assumption}.
We have for $h\in[H-1]$ that
\begin{equation}
	\label{eqn:claimed_bound_on_Delta_h}
	\Delta_h \overset{(i)}{\leq} \sum_{h'=h}^{H-1}(B_{h'}+\varepsilon \Delta_{h'+1}) \overset{\eqref{eqn:induction_assumption}}{\leq} \sum_{h'=h}^{H-1}(1+\varepsilon)^{h'-h}B_{h'}.
\end{equation}
Set $\varepsilon = (1/H)(2+1/H)$ and note that $h'-h\leq H$ for $h'\in\{h,\ldots,H-1\}$ if $h\in[H-1]$.
It then follows from the inequality $(1-(1/H)(2+1/H))^H\leq e^2$ that
\begin{equation}
	\Delta_h
	\overset{\eqref{eqn:claimed_bound_on_Delta_h}}{\leq}
	\biggl(1+\frac{1}{H}\biggl(2+\frac{1}{H}\biggr)\biggr)^{H}\sum_{h'=h}^{H-1}B_{h'}
	\leq
	e^2\sum_{h'=h}^{H-1}B_{h'},
\end{equation}
for $h\in[H-1]$, which concludes the proof.
\qed

\subsection{Proof of \texorpdfstring{\Cref{prop:final_bound_R3}}{the final regret term bound}}
\label{sec:proof_of_final_bound_R3}

The proof of \Cref{prop:final_bound_R3} follows quickly from the following lemma:
\begin{lemma}
	\label{lem:pigeonhole}
	Let $\hat{N}_k$ and $\hat{N}_k^{\mathrm{to}}$ be as in \eqref{eqn:countsdef} and \eqref{eqn:countsdef_overload}, and let $s_{k,h}=f(x_{k,h})$.
	For any $K\in \mbbN_+$,
	\begin{equation}
		\label{eqn:pigeonhole_state_action}
		\sum_{k=1}^K
		\sum_{h=1}^H
		\frac{1}{\sqrt{\bigl( 1 \vee \hat{N}_k\bigl( f^{-1}(s_{k,h}),a_{k,h}\bigr) \bigr)^m}}
		\leq
		\begin{cases}
			2\sqrt{SAT_K}+SAH
			 &
			\text{ for } m=1, \\
			SA(H + \ln T_K + 1)
			 &
			\text{ for } m=2,
		\end{cases}
	\end{equation}
	and
	\begin{equation}
		\label{eqn:pigeonhole_next_state}
		\sum_{k=1}^K
		\sum_{h=1}^H
		\frac{1}{\sqrt{\bigl(1 \vee \hat{N}^{\mathrm{to}}_k(f^{-1}(s_{k,h+1})) \bigr)^m}}
		\leq
		\begin{cases}
			2\sqrt{ST_K}+SH
			 &
			\text{ for } m=1, \\
			S(H + \ln T_K + 1)
			 &
			\text{ for } m=2,
		\end{cases}
	\end{equation}
	almost surely.
\end{lemma}

Assume that $T_K = \Omega(H^4S^3A)$.
To prove \eqref{eqn:final_bound_R3} we first use the result \eqref{eqn:application_of_recursive_sum_lemma} to bound the left-hand side of \eqref{eqn:final_bound_R3}, before subsequently recalling from \eqref{eqn:bounding_term_definition} that $\hat{B}_k\geq 0$, allowing us to extend the range of summation to all $k\in[K]$ and $h\in[H]$:
\begin{align}
	\sum_{k\in[K]:T_k>\Theta^{\clust}}
	\mbbE\bkt{(\bar{V}^{\pi_k}_{1}(x_{k,1})-V_1^{\pi_k}(x_{k,1}))\mbb{1}\{\Omega_k\}}
	 &
	\leq
	e^2\sum_{k\in[K]:T_k>\Theta^{\clust}}\sum_{h=1}^{H-1}\mbbE\bigl[
	\hat{B}_k(s_{k,h},a_{k,h},s_{k,h+1})
	\bigr]
	\nonumber
	\\
	 &
	\leq
	e^2\sum_{k=1}^K\sum_{h=1}^H\mbbE\bigl[
	\hat{B}_k(s_{k,h},a_{k,h},s_{k,h+1})
	\bigr]
	.
	\label{eqn:estimation_error_bound_1_simplified}
\end{align}
We can now bound the right-hand side of \eqref{eqn:estimation_error_bound_1_simplified} using \Cref{lem:pigeonhole}.
Together with the observation $\ln T_K\leq 1+\ln(nHS^2AT_K^2)=L$ and the assumption $T_K\geq H^4SA$ we conclude that
\begin{align}
	&
	\sum_{k\in[K]:T_k>\Theta^{\clust}}
	\mbbE\bigl[
		(\bar{V}^{\pi_k}_{1}(x_{k,1})-V_1^{\pi_k}(x_{k,1}))\mbb{1}\{\Omega_k\}
	\bigr]
	\nonumber
	\\
	&
	\leq
	e^2
	\sum_{k=1}^K\sum_{h=1}^H
	\mbbE\biggl[
		\sqrt{\frac{4H^2 L}{1\vee \hat{N}_k(f^{-1}(s_{k,h}),a_{k,h})}}
		+
		\frac{6H^2SL^2}{1\vee \hat{N}_k(f^{-1}(s_{k,h}),a_{k,h})}
		\nonumber
		\\
		&
		\phantom{
			\leq e^2\sum_{k=1}^K\sum_{h=1}^H\mbbE\biggl[
			\sqrt{\frac{4H^2 L}{1\vee \hat{N}_k(f^{-1}(s_{k,h}),a_{k,h})}}
		}
		+
		\sqrt{\frac{16H^2 L}{1\vee \hat{N}^{\mathrm{to}}_k(f^{-1}(s_{k,h+1}))}}
		+
		\frac{4nH^2L}{1\vee \hat{N}^{\mathrm{to}}_k(f^{-1}(s_{k,h+1}))}
	\biggr]
	\nonumber
	\\
	&
	\leq
	e^2\bigl(
		\sqrt{4H^2L^2}(2\sqrt{SAT_K}+SAH)
		+
		6H^2S^2AL^2(H+\ln T_K + 1)
		\nonumber
		\\
		&
		\phantom{
			\leq e^2\bigl(
			\sqrt{4H^2L^2}(2\sqrt{SAT_K}+SAH)
		}
		+
		\sqrt{16H^2L}(2\sqrt{ST_K}+SH)
		+
		4nH^2SL(H+\ln T_K+1)
	\bigr)
	\nonumber
	\\
	&
	=
	O(
	H\sqrt{SAT_KL^2}
	+
	H^2SL(H+L)(n+SAL)
	)
	\nonumber
	\\
	&
	=
	\tilde{O}(
	H\sqrt{SAT_K}
	+
	(n+SA)SH^3
	)
	.
\end{align}
Noting finally that $S^2AH^3=O(H\sqrt{SAT_K})$ for $T_K\geq H^4S^3A$ concludes the proof.
\qed

\begin{proof}[Proof of \Cref{lem:pigeonhole}]

	We only state the proof of \eqref{eqn:pigeonhole_state_action}; \eqref{eqn:pigeonhole_next_state} follows analogously upon replacing $\hat{N}_k(f^{-1}(s_{k,h}),a_{k,h})$ with $\hat{N}_k^{\mathrm{to}}(f^{-1}(s_{k,h+1}))$ and summing only over $s$ in \eqref{eqn:pigeonholesum2}.

	Recall \eqref{eqn:countsdef} and \eqref{eqn:countsdef_overload}.
	Then, for $\mcX\subseteq[n]$, $a\in[A]$, $k\in[K]$, and $h\in[H]$, define
	\begin{equation}
		\hat{N}_{k,h}(\mcX,a)\eqdef\hat{N}_k(\mcX,a) + \sum_{h'=1}^h\mbb{1}\{x_{k,h}\in\mcX,a_{k,h}= a\}.
	\end{equation}
	It holds that $\hat{N}_{k}(\mcX,a)\leq \hat{N}_{k,h}(\mcX,a)\leq \hat{N}_{k}(\mcX,a)+H$ for $h\in[H]$ and $\mcX\subset[n]$, from which it follows that for any $m=1,2$,
	\begin{equation}\label{eqn:pigeonholesum}
		\begin{split}
			 &
			\sum_{k=1}^K
			\sum_{h=1}^H
			\frac{1}{\sqrt{\bigl( 1 \vee \hat{N}_k\bigl( f^{-1}(s_{k,h}),a_{k,h}\bigr) \bigr)^m}}
			\leq \sum_{k=1}^K\sum_{h=1}^H\frac{1}{\sqrt{(1\vee (\hat{N}_{k,h}(f^{-1}(s_{k,h}),a_{k,h})-H))^m}}.
		\end{split}
	\end{equation}
	We decompose the sum on the right-hand side of \eqref{eqn:pigeonholesum} into a sum over latent state--action pairs:
	\begin{equation}
		\label{eqn:pigeonholesum2}
		\text{RHS of }\eqref{eqn:pigeonholesum}
		=
		\sum_{s,a}
		\sum_{k=1}^K\sum_{h=1}^H\frac{\mbb{1}\{s_{k,h}=s,a_{k,h}=a\}}{\sqrt{(1\vee(\hat{N}_{k,h}(f^{-1}(s),a)-H))^m}}
		.
	\end{equation}
	Now, consider a fixed pair $(s,a)\in[S]\times[A]$, and assume that $\hat{N}_{K,H}(f^{-1}(s),a)>0$.
	First, observe that $(i)$ if for some $(k,h)\in[K]\times[H]$ it holds that $\mbb{1}\{s_{k,h}=s,a_{k,h}=a\}=1$, then it follows by definition of $\hat{N}_{k,h}$ that $\hat{N}_{k,h}(f^{-1}(s),a)\geq\hat{N}_{k',h'}(f^{-1}(s),a)+1$ for all times $(k',h')\in[K]\times[H]$ ``before'' $(k,h)$.
	Moreover, since $\hat{N}_{K,H}(f^{-1}(s),a)$ counts the number of times that $\mbb{1}\{s_{k,h}=s,a_{k,h}=a\}=1$ is satisfied by definition, there are exactly $\hat{N}_{K,H}(f^{-1}(s),a)$ nonvanishing terms the sum over $k,h$ in \eqref{eqn:pigeonholesum2}.
	It then follows from our observation $(i)$ that $\hat{N}_{k,h}(f^{-1}(s),a)$ takes a different value in $\{1,\ldots,\hat{N}_{K,H}(f^{-1}(s),a)\}$ for each nonvanishing term in \eqref{eqn:pigeonholesum2}.
	Consequently,
	\begin{equation}\label{eqn:pigeonholesum3}
		\begin{split}
			\textnormal{RHS of } \eqref{eqn:pigeonholesum2}
			 &
			= \sum_{s,a:\hat{N}_{K,H}(f^{-1}(s),a)>0}\sum_{N=1}^{\hat{N}_{K,H}(f^{-1}(s),a)}\frac{1}{\sqrt{(1\vee(N-H))^m}}\\
			 &
			\leq
			\sum_{s,a:\hat{N}_{K,H}(f^{-1}(s),a)>0}\lb H + \sum_{M=1}^{1\vee(\hat{N}_{K,H}(f^{-1}(s),a)-H)}\frac{1}{\sqrt{M^m}} \rb\\
			 &
			\leq
			SAH + \sum_{s,a:\hat{N}_{K,H}(f^{-1}(s),a)>0}\sum_{M=1}^{\hat{N}_{K,H}(f^{-1}(s),a)}\frac{1}{\sqrt{M^m}}.
		\end{split}
	\end{equation}

	We will now use that for any nonincreasing $g(k)$, we have the following elementary inequality:
	\begin{equation}
		\label{eqn:elementary_inequality_function_g}
		\sum_{k=1}^n
		g(k)\leq \int_1^ng(k)dk+g(1).
	\end{equation}
	Suppose moreover that $\hat{N}_{K,H}(f^{-1}(s),a)>0$.
	We then apply \eqref{eqn:elementary_inequality_function_g} with $g(k)=1/\sqrt{k^m}$ to find that
	\begin{equation}
		\label{eqn:sum_integral_bound_casedistinction}
		\sum_{M=1}^{\hat{N}_{K,H}(f^{-1}(s),a)}
		\frac{1}{\sqrt{M^m}}
		\leq
		1+
		\int_1^{\hat{N}_{K,H}(f^{-1}(s),a)}
		\frac{dt}{\sqrt{t^m}}
		\leq
		\begin{cases}
			2\sqrt{\hat{N}_{K,H}(f^{-1}(s),a)}
			 &
			\textnormal{for }m=1, \\
			\ln \hat{N}_{K,H}(f^{-1}(s),a)+1
			 &
			\textnormal{for }m=2.
		\end{cases}
	\end{equation}

	We then evaluate the sum over latent state--action pairs for each case.
	Recall that by definition, $\sum_{s,a}\hat{N}_{K,H}(f^{-1}(s),a)=\sum_{a}\hat{N}_{K,H}([n],a)=T_K$.

	For $m=1$, we then have by Cauchy-Schwarz and nonnegativity of $\hat{N}_K(f{-1}(s),a)$ that
	\begin{equation}\label{eqn:casem1}
		\sum_{s,a:\hat{N}_{K,H}(f^{-1}(s),a)>0}\sqrt{\hat{N}_{K}(f^{-1}(s),a)}
		\leq
		\sqrt{\sum_{s,a} 1}\sqrt{\sum_{s,a}\hat{N}_{K}(f^{-1}(s),a)}
		=
		\sqrt{SAT_K},
	\end{equation}

	For $m=2$, it follows by nondecreasingness of the logarithm that
	\begin{equation}\label{eqn:casem2}
		\sum_{s,a:\hat{N}_{K,H}(f^{-1}(s),a)>0}\lb \ln \hat{N}_{K}(f^{-1}(s),a) +1\rb \leq SA\lb \ln T_K +1\rb
		.
	\end{equation}

	Combine \eqref{eqn:sum_integral_bound_casedistinction}, and \eqref{eqn:casem1}--\eqref{eqn:casem2} to bound \eqref{eqn:pigeonholesum3} and conclude the proof of \eqref{eqn:pigeonhole_state_action}.
\end{proof}

\newpage

\end{document}